\documentclass{article}

\PassOptionsToPackage{numbers, compress}{natbib}

\usepackage[main, final]{neurips_2026}

\usepackage[utf8]{inputenc} %
\usepackage[T1]{fontenc}    %
\usepackage{hyperref}       %
\usepackage{url}            %
\usepackage{booktabs}       %
\usepackage{amsfonts}       %
\usepackage{nicefrac}       %
\usepackage{microtype}      %
\usepackage{xcolor}         %

\usepackage{dsfont}
\usepackage{nicefrac}
\usepackage{subcaption}
\usepackage{enumitem}
\usepackage{xparse}
\usepackage{makecell}
\usepackage{afterpage}
\usepackage{placeins}
\usepackage{multicol}
\usepackage{booktabs}
\usepackage{tabularx}

\usepackage{tikz}
\usetikzlibrary{arrows.meta,positioning}

\usepackage{amsmath}
\usepackage{amssymb}
\usepackage{mathtools}
\usepackage{amsthm}
\usepackage{float}
\usepackage{aliascnt}

\usepackage{hyperref}
\usepackage[capitalize,noabbrev]{cleveref}
\theoremstyle{plain}
\newtheorem{theorem}{Theorem}[section]

\newaliascnt{proposition}{theorem}
\newtheorem{proposition}[proposition]{Proposition}
\aliascntresetthe{proposition}

\newaliascnt{lemma}{theorem}
\newtheorem{lemma}[lemma]{Lemma}
\aliascntresetthe{lemma}

\newaliascnt{corollary}{theorem}
\newtheorem{corollary}[corollary]{Corollary}
\aliascntresetthe{corollary}

\theoremstyle{definition}

\newaliascnt{definition}{theorem}
\newtheorem{definition}[definition]{Definition}
\aliascntresetthe{definition}

\newaliascnt{assumption}{theorem}
\newtheorem{assumption}[assumption]{Assumption}
\aliascntresetthe{assumption}

\newaliascnt{claim}{theorem}
\newtheorem{claim}[claim]{Claim}
\aliascntresetthe{claim}

\newaliascnt{example}{theorem}

\aliascntresetthe{example}

\theoremstyle{remark}

\newaliascnt{remark}{theorem}
\newtheorem{remark}[remark]{Remark}
\aliascntresetthe{remark}

\crefname{theorem}{Theorem}{Theorems}
\crefname{proposition}{Proposition}{Propositions}
\crefname{lemma}{Lemma}{Lemmas}
\crefname{corollary}{Corollary}{Corollaries}
\crefname{definition}{Definition}{Definitions}
\crefname{assumption}{Assumption}{Assumptions}
\crefname{claim}{Claim}{Claims}
\crefname{example}{Example}{Examples}
\crefname{remark}{Remark}{Remarks}

\crefname{appendix}{App.}{App.}
\crefname{algorithm}{Scheme}{Schemes}
\crefname{equation}{Eq.}{Eq.}
\crefname{definition}{Definition}{Def.}

\crefname{corollary}{Corollary}{Cor.}

\newcommand{\ie}{\textit{i.e., }}
\newcommand{\eg}{\textit{e.g., }}

\definecolor{codecomment}{RGB}{65, 120, 40}

\newcommand{\vect}[1]{\mathbf{#1}}

\newcommand{\vu}{\vect{u}}
\newcommand{\vv}{\vect{v}}
\newcommand{\vw}{\vect{w}}

\newcommand{\mat}[1]{\mathbf{#1}}
\newcommand{\X}{\mat{X}}

\newcommand{\M}{\mat{M}}

\newcommand{\V}{\mat{V}}
\newcommand{\I}{\mat{I}}
\newcommand{\A}{\mat{A}}
\newcommand{\B}{\mat{B}}

\newcommand{\K}{\mat{K}}
\newcommand{\x}{\vect{x}}
\newcommand{\W}{\mat{W}}
\newcommand{\SM}{\mat{S}}

\newcommand{\vtheta}{\boldsymbol{\theta}}

\newcommand{\w}{\vw}

\floatname{algorithm}{Scheme}

\newcommand{\norm}[1]{\left\Vert{#1}\right\Vert}
\newcommand{\abs}[1]{\left|{#1}\right|}

\newcommand{\expect}[2]{\mathop{\mathbb{E}}_{#1}\sqprn{#2}}
\newcommand{\prn}[1]{\left({#1}\right)}
\newcommand{\bigprn}[1]{\big({#1}\big)}

\newcommand{\sqprn}[1]{\left[{#1}\right]}

\newcommand{\real}{\mathbb{R}}
\newcommand{\mean}[1]{\mathbb{E}\left[{#1}\right]}
\newcommand{\Tr}[1]{\mathrm{Tr}\left({#1}\right)}
\newcommand{\curlprn}[1]{\left\{{#1}\right\}}

\newenvironment{proof-sketch}{\noindent{\bf Proof sketch.}}{}%

\newcommand{\lntk}{\lambda_{\min}(\mathbf{K}_{\textrm{NTK}})}

\def\eqref#1{Eq.~(\ref{#1})}

\def\appref#1{App.~\ref{#1}}

\usepackage{titletoc} %
\newcommand{\appendixtableofcontents}{%
  \section*{Appendix contents}
  \vspace{-0.1em}
  \printcontents[app]{l}{1}{}
}

\titlecontents{section}
  [1.5em]                                       %
  {\vspace{5pt}\bfseries}                       %
  {\contentslabel{1.5em}}                       %
  {}                                            %
  {\titlerule*[0.5pc]{.}\contentspage}          %

\titlecontents{subsection}
  [3.5em]                                       %
  {\vspace{0.2pt}}                                %
  {\contentslabel{2.3em}}                       %
  {}                                            %
  {\titlerule*[0.5pc]{.}\contentspage}          %

\makeatletter
\NewDocumentEnvironment{recall}{ O{} o }{%
  \par
  \normalfont
  \topsep6\p@\@plus6\p@\relax
  \trivlist
  \item\relax
  {%
    \bfseries
    Recall%
    \if\relax\detokenize{#1}\relax\else\space #1\fi
  }%
  \IfValueT{#2}{\space(#2)}%
  .\hspace\labelsep\ignorespaces
}{%
  \endtrivlist
}
\makeatother

\title{Global Convergence in Deep Networks\\
  via the Second Law of Thermodynamics}

\newcommand*\samethanks[1][\value{footnote}]{\footnotemark[#1]}

\author{
Matan Tsipory\thanks{Correspondence to: \texttt{matan.tsipory@gmail.com}.}\\
Technion\thanks{
Technion - Israel Institute of Technology,
The Andrew and Erna Viterbi
Faculty of Electrical \& Computer Engineering, Technion City, Haifa 3200003, Israel.
}
\And
Ofir Gaash\\
Technion\samethanks[2]
\And
Itamar Harel\\
Technion\thanks{
Technion - Israel Institute of Technology, Faculty of Data and Decision Sciences, Technion City, Haifa 3200003, Israel.
}
\AND
Nathan Srebro\\
TTIC\thanks{Toyota Technological Institute at Chicago, Chicago, IL, USA.}
\And
Daniel Soudry\\
Technion\samethanks[2]
}

\begin{document}

\maketitle

\begin{abstract}
Existing convergence analyses of finite-width overparameterized deep neural networks under standard parameterization often rely on parameters remaining close to their initialization and typically neglect the role of noise, despite empirical evidence that stochasticity can improve optimization and generalization.
To bridge this gap, we study $L_2$-regularized continuous-time Langevin dynamics~(CLD) with isotropic noise, for training 
feedforward networks
with smooth activations.
We introduce a change-of-measure inequality, inspired by the second law of thermodynamics, that enables control of key properties of the training dynamics under the time-evolving parameter distribution.
Leveraging a sharpened lower bound on the minimum eigenvalue of the Neural Tangent Kernel~(NTK) at initialization, we show that injected noise helps preserve NTK stability with high probability without confining the weights to a small neighborhood of their initialization. As a result, we derive a non-asymptotic upper bound on the expected training loss.
Our analysis yields explicit conditions on width and noise temperature that guarantee convergence to a finite loss floor.
Finally, we experimentally show the network Jacobian can change from initialization, departing the commonly analyzed lazy regime.
\end{abstract}

\section{Introduction}
\label{sec:introduction}

It has been widely demonstrated that stochasticity in the training of neural networks~(NNs)---mostly introduced by mini-batch sampling in stochastic gradient descent~(SGD)---is often beneficial for optimization and generalization
\citep[e.g.,][]{keskar2017on, goyal2017accurate, golmant2018computational, smith2020generalization, sclocchi2023dissecting}.
While theoretical analyses establish a clear separation in favor of stochasticity in convex optimization \citep{amir2021sgd, amir2021never}, and highlight the benefit of injected noise in escaping saddle points \citep{jin2017escape, jin2021nonconvex}, non-asymptotic convergence guarantees for overparameterized neural networks do not yet fully reflect this empirical picture.

Much prior work focuses on settings where the parameters remain close to their initialization throughout training. In several of these works, the convergence rate is governed by the minimum eigenvalue of the Neural Tangent Kernel (NTK), which is typically controlled by the near-initialization condition \citep{du2019deep, song2021subquadratic, nguyen2020global, bombari2022memorization}. This condition is poorly aligned with noisy dynamics, which tend to move the parameters away from their initialization. Consistent with this, much of the existing theory focuses on deterministic algorithms \citep{oymak2020toward,  cao2020generalization, nguyen2021proof, montanari2022interpolation}. Related analyses that do attempt to account for noise often treat it primarily as a nuisance \citep{chen2021much, lugosi2024convergence, oberweis2026convergence, allen2019convergence, zou2019improved, xu2024overparametrized, chen2020generalized}. As a result, finite-width convergence guarantees for noisy neural network training beyond the near-initialization regime remain limited.

In this work, we 
treat noise not merely as a feature of realistic training, but as a central tool in the analysis.
We study continuous-time Langevin dynamics~(CLD), in which isotropic Gaussian noise is added to gradient flow, and derive a change-of-measure inequality inspired by the second law of thermodynamics.
This inequality allows us to control key quantities under the evolving noisy dynamics, in particular, the minimum eigenvalue of the NTK, without requiring the parameters to remain near initialization.
This leads to a non-asymptotic global convergence guarantee for feedforward neural networks.

Our analysis differs conceptually from many existing works on Langevin dynamics \citep[e.g.,][]{raginsky2017non, xu2018global, kumar2025langevin}:
Rather than proving convergence by analyzing the log-Sobolev inequality~(LSI) constant of the Gibbs distribution, we obtain a direct optimization bound along the dynamics.
Avoiding this route is crucial, since the LSI constant typically deteriorates exponentially as the noise temperature decreases. In turn, this yields linear convergence with a decay exponent that is independent of the noise temperature.

\paragraph{Main contributions.}
\begin{itemize}
\item We prove a general change-of-measure inequality (\Cref{thm:second_law_change_of_measure}) inspired by the second law of thermodynamics \citep{Cover1994,Merhav2011,harel2025temperature}. This bounds expectations under the time‑$t$ distribution of a Markov process using only moment information at initialization, provided the process admits a Gibbs stationary distribution relative to the initialization. 
\item We strengthen the lower bound on the minimum eigenvalue of the NTK at initialization introduced by \citet{banerjee2023neural}, by sharpening their high-probability guarantees (\Cref{thm:lower_bound_ntk_deep}).
\item We apply the ``second-law change-of-measure'' to prove that the NTK stays well-conditioned during training, yielding a global convergence bound on the expected loss under regularized CLD (\Cref{thm:upper_bound_on_expected_loss}) with non-annealing, isotropic noise, and provide explicit dependence on width, temperature, and parameterization. We explicitly show our result applies to standard and NTK parameterizations. For shallow nets, we also prove convergence to zero loss if the noise is eventually removed.
\item 
We empirically demonstrate that the resulting Langevin dynamics evolve the Jacobian away from its initial value under standard parameterization, suggesting that the proof technique does not implicitly enforce lazy training.
\end{itemize}

\paragraph{Notation.}
\label{sec:notations}
We denote $\sqprn{n}=\curlprn{1, \ldots, n}$. We use bold symbols for vectors and matrices, \eg $\vv,\,\X$. 
The gradient of a function $f:\real^d\to\real$ is denoted by $\nabla f\in\real^d$,
the Hessian by $\nabla^2 f\in\real^{d\times d}$, and the Laplacian by
$\Delta f=\Tr{\nabla^2 f}$. For a vector-valued function
$f:\real^d\to\real^n$, we use $\nabla f\in\real^{n\times d}$ to denote its Jacobian,
with the intended meaning made clear by context.

We use $\lesssim$ to denote inequality up to constants, \ie $a \lesssim b$ iff there exists a constant $c>0$ such that $a\leq cb$, where $c$ does not scale with the problem dimension or any other size-related parameter.\footnote{We never ``hide'' input dimension $d$, number of samples $n$, network width $m$ and inverse noise temperature $\beta$ in constants.}
Similarly, we denote $a \asymp b$ iff there exists a constant $c>0$ such that $a = cb$. We interchangeably use these notations and the equivalent $O, \Omega, \Theta$ notations to denote the same properties, and denote $\tilde{O}, \tilde\Omega, \tilde{\Theta}$ when we ignore poly-logarithmic terms.
Throughout, the letter $c$, with any \textit{numerical} subscript or decoration (\eg $c_1,\, \tilde{c},\, C$), is reserved for global constants, 
unless otherwise stated.  

We denote for a random variable (RV) $X$ its $L^p$ norm $\norm{X}_{L^p} \triangleq \prn{\mean{\abs{X}^p}}^{\nicefrac{1}{p}}$. 
The subgaussian and subexponential norms of a random variable $X$ are denoted  $\norm{X}_{\psi_2}$ and $\norm{X}_{\psi_1}$, respectively.\linebreak
For two probability measures $p,\,q$, we denote by $p \ll q$ that $p$ is \emph{absolutely continuous} with respect to $q$, i.e., for any event $\mathcal{E}$, $q \prn{\mathcal{E}} = 0$ implies $p \prn{\mathcal{E}} = 0$. 
When $p \ll q$ and $q \ll p$, the two measures are called \emph{equivalent}, denoted as $p \sim q$.

\section{Key technical result: parameter control via a change of measure inequality}
\label{sec:change_of_measure}
We begin by introducing a change of measure inequality, inspired by the second law of thermodynamics, as a technical tool that allows us to control properties of Markov processes. 
In order to state our result, we require the following definition.
\begin{definition}[Gibbs distribution]
\label{def:gibbs_distribution}
Let $p,\,q$ be probability measures on a measurable space $\mathcal{X}$ such that $p\ll q$.
$p$ is Gibbs with respect to (w.r.t.)\! $q$ with non-negative potential $\Psi:\mathcal{X} \to \real$ iff $Z\triangleq\int e^{-\Psi} \mathrm{d} q < \infty$ and the Radon–Nikodym derivative of $p$ w.r.t. $q$ is
\begin{align*}
    \frac{\mathrm{d}p}{\mathrm{d}q}=Z^{-1} e^{-\Psi}\,.
\end{align*}
\end{definition}
We consider a time-invariant Markov process on a state space $\mathcal{X}$, in either discrete or continuous time, with marginal distribution $p_t$ at time $t$ and Gibbs stationary distribution $p_\infty$ with respect to a reference measure $\nu$. 
 Utilizing the second law of thermodynamics,
 we derive an upper bound on the R\'enyi divergence $D_\alpha$ of $p_t$ from $\nu$ at all times (\cref{lem:divergence_bound}),
 extending \citet[Corollary 2.5]{harel2025temperature} that applies to the KL divergence (the limit of $D_\alpha$ when $\alpha\to 1$) and the R\'enyi-infinity divergence $D_\infty$.
Then, we plug this bound with $p_0=\nu$ into a change-of-measure inequality based on the R\'enyi divergence \citep[Theorem 8]{begin2016pac}, to obtain 
our main technical result.
\begin{lemma}[Second law of thermodynamics change of measure inequality]
\label{thm:second_law_change_of_measure}
Consider a time invariant Markov process over the state space $\mathcal{X}$ whose marginal distribution is $p_t,\, t\geq 0$ such that $p_t \ll p_0$, with a stationary distribution $p_\infty$ which is Gibbs w.r.t.\! $p_0$ with potential $\Psi$,
and $p_\infty \sim p_0$.
Then, for any \emph{non-negative} measurable function $\phi:\mathcal{X}\to\mathbb{R}$ and $\alpha>1$, we have at any time $t \geq 0$,
\begin{align*}
    \mathbb{E}_{p_{t}} \phi\left(x\right) & \leq \prn{\mathbb{E}_{p_{0}}\left[e^{2\left(\alpha-1\right)\Psi}\right]}^{\frac{1}{2\alpha}}\,
    \prn{\mathbb{E}_{ p_{0}}\left[\phi\left(x\right)^{\frac{\alpha}{\alpha-1}}\right]}^{\frac{\alpha-1}{\alpha}}.
\end{align*}
\end{lemma}

We provide the proof in \appref{app:divergence_upper_bound_proof}.
\cref{thm:second_law_change_of_measure} enables us to control specific properties of a Markov process throughout its mixing phase using only properties at initialization.
In the following section, we demonstrate its applicability to regularized Langevin dynamics.

\section{Regularized Langevin dynamics}
\label{sec:train_by_langevin}
We study the evolution of parameters $\vtheta \in \real^{N_p}$ by continuous-time Langevin dynamics (CLD), initialized by a Gaussian distribution $\vtheta_0 \sim \mathcal{N} \prn{\mathbf{0}_{N_p}, \boldsymbol{\Lambda}^{-1}}$, where $\boldsymbol{\Lambda} \in \real^{N_p \times N_p}$ is positive definite.
This scheme enables using our change of measure inequality, as long as its stationary distribution over the trainable parameters, $p_\infty$, is Gibbs w.r.t.\! their initial distribution $p_0$. In order to achieve this, 
we employ $L_2$-regularized CLD.
Formally, the parameters $\boldsymbol{\theta}$ evolve by 
\begin{align}
    \mathrm{d}\boldsymbol{\theta}_{t} &= 
    -\nabla_{\vtheta}\mathcal{L} \left(\boldsymbol{\theta}_{t}\right)\mathrm{d}t-\beta^{-1}\boldsymbol{\Lambda}\boldsymbol{\theta}_{t}\mathrm{d}t+\sqrt{2\beta^{-1}}\mathrm{d}\mathbf{w}_{t},
    \label{eq:regularized_langevin}
\end{align}
where $\mathcal{L} \prn{\boldsymbol{\theta}} \geq 0$ is some empirical loss function that is twice continuously differentiable in the parameters $\vtheta$, $\beta$ is an inverse-temperature (scalar) parameter, and $\mathrm{d}\mathbf{w}_{t}$ is the standard Brownian motion in $\real^{N_p}$.
Since $\mathcal{L}\geq 0$, the $L_2$ regularization ensures that
\[
Z=\int_{\real^{N_p}}e^{-\beta\mathcal{L}\left(\vtheta\right)-\frac{1}{2}\vtheta^\top\boldsymbol{\Lambda}\vtheta}\mathrm{d}\vtheta<\infty \,.
\]
Therefore, denoting by $\phi_{\boldsymbol{\Lambda}}$ the density of the multivariate Gaussian distribution $\mathcal{N}\left(\mathbf{0}_{N_p},\boldsymbol{\Lambda}^{-1}\right)$, the Gibbs distribution
$\propto e^{-\beta\mathcal{L}\left(\boldsymbol{\theta}\right)}\phi_{\boldsymbol{\Lambda}}\left(\boldsymbol{\theta}\right)$ is well defined and is the stationary distribution of the non-explosive process given by \cref{eq:regularized_langevin} (\cref{prop:cld_wellposed}).\footnote{\label{footnote:pavliotis}
Stationarity of the Gibbs distribution for Langevin dynamics is a well known fact, usually established by vanishing probability flux in the Fokker-Planck equation \citep[see, e.g.,][Proposition 4.2]{pavliotis2014stochastic}. However, that fact does not rule out explosion of the process. The later analysis via It\^o's formula requires non-explosion, which we prove in \cref{prop:cld_wellposed}.
Furthermore, the setting can be extended to anisotropic noise with a fixed positive definite covariance $\boldsymbol\Sigma$, provided the drift is preconditioned by the same $\boldsymbol\Sigma$: for $\mathrm{d}\boldsymbol{\theta}_t=-\boldsymbol\Sigma\nabla(\mathcal{L}(\boldsymbol{\theta}_t)+\frac{1}{2\beta}\boldsymbol{\theta}_t^\top\mathbf{\Lambda}\boldsymbol{\theta}_t)\mathrm{d}t+\sqrt{2\beta^{-1}\boldsymbol\Sigma}\mathrm{d}\mathbf{w}_t$, the same Gibbs law remains stationary. We conduct our analysis with $\boldsymbol\Sigma=\mathbf{I}$ for simplicity.
}
It is evident then that $p_\infty\propto e^{-\beta \mathcal{L} \prn{\boldsymbol{\theta}}} \phi_{\boldsymbol{\Lambda}} \prn{\boldsymbol{\theta}}$ is Gibbs w.r.t.\! $\phi_{\boldsymbol{\Lambda}}$ with potential $\Psi=\beta \mathcal{L} \prn{\boldsymbol{\theta}}$. 
Furthermore, the Markov process described in \cref{eq:regularized_langevin}, with \mbox{$\vtheta_0 \sim \mathcal{N} \prn{\mathbf{0}_{N_p},\,\boldsymbol{\Lambda}^{-1} }$}, satisfies the absolute continuity requirements of \cref{thm:second_law_change_of_measure} (\cref{prop:cld_wellposed}).
 Thus, we can apply \cref{thm:second_law_change_of_measure}:
\begin{corollary}
\label{cor:second_law_for_langevin}    
With $\boldsymbol{\theta}$ initialized by \mbox{$p_0=\mathcal{N} \prn{\mathbf{0}_{N_p},\,\boldsymbol{\Lambda}^{-1} }$} and evolving by \cref{eq:regularized_langevin}, we have for any \emph{non-negative} measurable function \mbox{$\varphi: \real^{N_p} \to \mathbb{R}$}, at any time $t\geq 0$ and for any $\alpha>1$,
\begin{align*}
    \mathbb{E}_{p_{t}} \varphi\left(\boldsymbol{\theta}_t\right) & \leq \prn{\mathbb{E}_{p_{0}}\left[e^{2\left(\alpha-1\right) \beta \mathcal{L} \prn{\boldsymbol{\theta}_0}}\right]}^{\frac{1}{2\alpha}}\,
    \norm{\varphi \prn{\boldsymbol{\theta}_0}}_{L^{\frac{\alpha}{\alpha - 1}}} .
\end{align*}
\end{corollary}
\begin{remark}
Under our model, as defined in \cref{sec:setting_deep_nn}, we are mostly focused on regimes where $\prn{\mathbb{E}_{p_{0}}\left[e^{2\left(\alpha-1\right) \beta \mathcal{L} \prn{\boldsymbol{\theta}_0}}\right]}^{\frac{1}{2\alpha}} = O\! \prn{1}$, which simplifies the analysis greatly.
\end{remark}
\paragraph{It\^o's formula.}
To analyze the dynamics of the loss under regularized CLD (\cref{eq:regularized_langevin}), we can employ It\^o's formula \citep{ito1951formula}, \citep[Chapter 1, Theorem 6.4]{mao2007stochastic}:\footnote{See \appref{sec:ito_rigor} for a formal presentation of It\^o's formula and its applicability in our case.}
\begin{align*}
\mathrm{d}\mathcal{L} \left(\boldsymbol{\theta}_{t}\right)  & = \nabla \mathcal{L}  \left(\boldsymbol{\theta}_{t}\right)^{\top}  \left( \left(-\nabla \mathcal{L}  \left(\boldsymbol{\theta}_{t}\right) - \beta^{-1}\boldsymbol{\Lambda}\boldsymbol{\theta}_{t} \right)   \mathrm{d}t
+   \sqrt{2\beta^{-1}}\mathrm{d}\mathbf{w}_{t} \right)   + 
\frac{1}{2}\mathrm{Tr} \left(2\beta^{-1} \nabla^{2}\mathcal{L}  \left(\boldsymbol{\theta}_{t}\right)\right)   \mathrm{d}t \,.
\end{align*}
By taking the expected value \emph{over both the initialization and the Brownian noise}, we arrive at
\begin{align}
\label{eq:after_ito_mean_loss_dynamics}
\frac{\mathrm{d}}{\mathrm{d}t}\mathbb{E}\mathcal{L} \left(\boldsymbol{\theta}_{t}\right) & =\mathbb{E}\left[-\left\Vert \nabla \mathcal{L} \left(\boldsymbol{\theta}_{t}\right)\right\Vert ^{2}+\beta^{-1}\Delta \mathcal{L} \left(\boldsymbol{\theta}_{t}\right) -\beta^{-1}\nabla \mathcal{L} \left(\boldsymbol{\theta}_{t}\right)^{\top}\boldsymbol{\Lambda}\boldsymbol{\theta}_{t}\right]\,,
\end{align}
where $\Delta \mathcal{L} \left(\boldsymbol{\theta}_{t}\right)$ is the Laplacian of the loss.
In the next sections, we use \cref{cor:second_law_for_langevin} to control the terms on the right-hand side of \cref{eq:after_ito_mean_loss_dynamics}, thereby obtaining bounds on the expected loss for feedforward neural networks trained via regularized CLD.

\section{Model: deep neural network}
\label{sec:setting_deep_nn}
We consider a scalar-output feedforward NN trained with the mean squared error~(MSE) loss \citep[e.g.,][]{du2019deep, arora2019fine, montanari2022interpolation}.

Formally, consider a training set $\mathcal{D}=\left\{\left(\mathbf{x}_i,y_i\right)\right\}_{i\in\left[n\right]}$, where $\mathbf{x}_i\in\mathbb{R}^d$ and $y_i\in\mathbb{R}$. We denote by $\mathbf{X}\in\mathbb{R}^{n\times d}$ the matrix whose $i$-th row is $\mathbf{x}_i^\top$, and by $\mathbf{y}\in\mathbb{R}^n$ the vector whose $i$-th component is $y_i$.
In our setting, $f$ is a bias-free network with $L+1$ layers ($L$ hidden layers) and widths $m_l$, $l\in\left[L\right]$, under an arbitrary parameterization. With $\boldsymbol{\alpha}^{\left(0\right)}\left(\mathbf{x}\right)  =\mathbf{x}$, the model is defined recursively $\forall \x\in\real^d$ by
\begin{equation}
\begin{aligned}
\boldsymbol{\alpha}^{\left(l\right)}\left(\mathbf{x}\right) & =\phi\left(s_{l}\mathbf{W}^{\left(l\right)}\boldsymbol{\alpha}^{\left(l-1\right)}\left(\mathbf{x}\right)\right),\quad l\in\left[L\right]\,,\\
f\left(\vtheta;\mathbf{x}\right) & ={\alpha}^{\left(L+1\right)}\left(\mathbf{x}\right)=s_{L+1}\mathbf{v}^{\top}\boldsymbol{\alpha}^{\left(L\right)}\left(\mathbf{x}\right)\,,
\label{eq:network_definition}
\end{aligned}
\end{equation}
where $\mathbf{W}^{\left(l\right)}\in\mathbb{R}^{m_{l}\times m_{l-1}},\,l\in\left[L\right]$
are layer-wise weight matrices, $\mathbf{v}\in\mathbb{R}^{m_{L}}$
is the last layer vector, $\phi\left(\cdot\right)$ is the activation
function,
and $s_{l},\,l\in\left[L+1\right]$ are fixed positive
constants 
(\eg $s_l=1,\,l\in\sqprn{L+1}$ in standard parameterization). 
The total set of trainable parameters is denoted by
\begin{align*}
\vtheta & \triangleq\left(\mathrm{vec}\left(\mathbf{W}^{\left(1\right)}\right)^{\top},\ldots,\mathrm{vec}\left(\mathbf{W}^{\left(L\right)}\right)^{\top},\mathbf{v}^{\top}\right)^{\top}\in\mathbb{R}^{N_{p}}\,,
\end{align*}
with $m_{0}\triangleq d$ and $N_{p}\triangleq\sum_{l=1}^{L}m_{l}m_{l-1}+m_{L}$. 
For each $i\in\sqprn{n}$, let $f_i\left(\boldsymbol{\theta}\right)\triangleq f\left(\boldsymbol{\theta};\mathbf{x}_i\right)\in\mathbb{R}$, and write $f\left(\boldsymbol{\theta}\right)\triangleq f\left(\boldsymbol{\theta};\mathbf{X}\right)=\left(f_1\left(\boldsymbol{\theta}\right),\ldots,f_n\left(\boldsymbol{\theta}\right)\right)^{\top}\in\mathbb{R}^n$. Its Jacobian is denoted by $\nabla f \prn{\boldsymbol{\theta}} \triangleq\frac{\partial f\left(\boldsymbol{\theta};\mathbf{X}\right)}{\partial\boldsymbol{\theta}}\in\mathbb{R}^{n\times N_{p}}$, and the empirical NTK matrix by $\mathbf{K}_{\mathrm{NTK}}\left(\boldsymbol{\theta}\right)\triangleq\nabla f\left(\boldsymbol{\theta}\right)\nabla f\left(\boldsymbol{\theta}\right)^{\top}$. With the residual vector $\mathbf{r}\left(\boldsymbol{\theta}\right)\triangleq f\left(\boldsymbol{\theta}\right)-\mathbf{y}$, the MSE loss is
\begin{align}
\mathcal{L}\left(\boldsymbol{\theta}\right)
&=\frac{1}{2n}\left\Vert f\left(\boldsymbol{\theta}\right)-\mathbf{y}\right\Vert^2
=\frac{1}{2n}\left\Vert\mathbf{r}\left(\boldsymbol{\theta}\right)\right\Vert^2\,.
\label{eq:mse_loss_def}
\end{align}
\paragraph{Weight initialization.}
For some $\sigma_{0}>0$, we denote $c_{\phi,\sigma_{0}}\triangleq\mathbb{E}_{z\sim\mathcal{N}\left(0,\sigma_{0}^{2}\right)}\left[\phi^{2}\left(z\right)\right]=\Theta \prn{1}$, and $\nu_{0}^{2} \triangleq \frac{\sigma_{0}^{2}}{c_{\phi,\sigma_{0}}}\!=\!\Theta  \prn{1}$.
We examine Gaussian initializations,
\begin{align}
\mathrm{vec}\left(\mathbf{W}^{\left(l\right)}\right) \sim \mathcal{N}\left(\mathbf{0}_{m_l m_{l-1}},\,\frac{\nu_0^2}{a_l}\mathbf{I}_{m_l m_{l-1}}\right), \quad l \in \left[ L \right]; \quad \mathbf{v} \sim \mathcal{N}\left(\mathbf{0}_{m_L},\frac{1}{a_{L+1}}\mathbf{I}_{m_L}\right)\,,
\label{eq:Wv_initialization}
\end{align}
where
$a_l,\,l\in\left[L+1\right]$ are predetermined constants such that $\forall l\in\left[L\right]$, $\frac{a_{l}}{s_{l}^{2}}=m_{l-1}$. This demand leads to variance preservation, as in \citet{he2015delving}.\footnote{
The demand $\forall l\in\left[L\right],\,\nicefrac{a_{l}}{s_{l}^{2}}=m_{l-1}$ is required for the proof of the empirical NTK minimum eigenvalue lower bound, see \appref{app:ntk_min_eigenvalue} for more details. Other parts of our derivation do not require this, so we work with $a_l,\,s_l$ throughout.
}
We do not impose a similar demand on $a_{L+1}$, and denote  $\kappa \triangleq s_{L+1}  \sqrt{\frac{m_L}{a_{L+1}}}$ as the net’s overall output-scaling coefficient. The value of $\kappa$ significantly affects the performance of deep nets \citep{chizat2019lazy,woodworth2020kernel}.
Equivalently, \cref{eq:Wv_initialization} may be presented as
\begin{equation}
    \begin{aligned}
        \vtheta_0\sim\mathcal{N} \prn{\mathbf{0}_{N_p},\,\boldsymbol{\Lambda}^{-1} }, \quad
        \boldsymbol{\Lambda} = \mathrm{diag}\prn{
        \frac{a_1}{\nu_0^2}\I_{m_1 m_0},\ldots, \frac{a_L}{\nu_0^2}\I_{m_L m_{L-1}}, a_{L+1} \I_{m_L}}
        \,.
        \label{eq:Lambda_def}
    \end{aligned}
\end{equation}
\paragraph{Activation function.}
We impose the following assumptions on the activation function:
\begin{assumption}[Bounded activation]
\label{ass:activation_bounded}
The activation $\phi$ is bounded, \ie $\abs{\phi} \leq M_\phi$.
\end{assumption}
\begin{assumption}[Activation derivative]
\label{ass:activation_bounded_gradients}
The activation is twice continuously differentiable and its derivatives are bounded ($\abs{\phi'} \leq \alpha_\phi$, $\abs{\phi''} \leq \beta_\phi$).\footnote{\label{footnote:C1_activations}
This assumption can be relaxed to once continuously differentiable activations with bounded and $\beta_\phi$-Lipschitz gradients, such as $\mathrm{softsign}$ \citep{glorot2010understanding} and $\mathrm{ELU}$ \citep{clevert2015fast}. The resulting loss has a locally Lipschitz gradient, so global well-posedness and stationarity of the Gibbs law for the Gaussian initialization both follow from \cref{prop:cld_wellposed}. We rigorously show It\^o's formula remains applicable under such activations in \appref{sec:ito_extension_to_once_diff}.}
\end{assumption}
\begin{assumption}[Activation at origin]
\label{ass:activation_at_origin}
The activation $\phi$ satisfies $\phi\left(0\right)=0$.
\end{assumption}
We note that under our analysis, \cref{ass:activation_bounded} could be relaxed to apply only to the activation function used in the last hidden layer.
Taken together, these assumptions are satisfied in practice by a broad class of activations, including $\mathrm{tanh}$, $\mathrm{erf}$, and $\mathrm{soft\text{-}root\text{-}sign}$ \citep{zhou2020soft}, as well as by $\mathrm{GELU}$ \citep{hendrycks2016gaussian} and $\mathrm{Swish}$ \citep{ramachandran2017searching} in all hidden layers except the last. Activations satisfying similar assumptions have also recently been shown to be beneficial, for example in the context of self-normalization \citep{chen2025stronger}.
\paragraph{Data scaling.}
We make mild data-scaling assumptions that are common in such analyses \citep[e.g.,][]{oymak2020toward, song2021subquadratic, bombari2022memorization} and can be achieved in practice by normalization.
\begin{assumption}[Normalized data]
\label{ass:normalized_data}
Every input sample $\mathbf{x}_{i}\in\mathbb{R}^{d},i\in\left[n\right]$
has norm $\left\Vert \mathbf{x}_{i}\right\Vert _{2}^{2}=c_{\phi,\sigma_{0}}d$.\footnote{
The multiplicative constant $c_{\phi,\sigma_{0}}$ is added to match the setting of \citet{banerjee2023neural}, which we adapt.
}
\end{assumption}
\begin{assumption}[Bounded labels]
\label{ass:bounded_labels}
For all $i\in\sqprn{n},\,\abs{y_i}\leq R_y$, where $R_y=\Theta\prn{1}$.
\end{assumption}
\paragraph{Data separability.}
Similarly to \citet{banerjee2023neural}, our analysis relies on the following assumption:
\begin{assumption}[Data separability]
\label{ass:data_separability}
We assume 
\begin{align*} 
\lambda_{1}\triangleq \lambda_{\min}\left(\mathbb{E}_{\mathbf{g}\sim\mathcal{N}\left(\mathbf{0}_{d},\frac{\nu_{0}^{2}}{a_1}\mathbf{I}_{d}\right)}\left[\phi\left(s_1\mathbf{X}\mathbf{g}\right)\phi\left(s_1\mathbf{X}\mathbf{g}\right)^{\top}\right]\right) > 0\,.
\end{align*}
\end{assumption}
\cref{ass:data_separability} is satisfied under common data assumptions, \eg for i.i.d.\! data uniformly distributed on the sphere with $n$ polynomial in $d$, $\lambda_1$ is lower bounded by a constant independent of $n$ and $d$ with high probability (over the data) \citep[Theorem 3.3]{nguyen2020global}, \ie $\lambda_1=\Omega(1)$; see also the discussion in \citet[Remark 4.2]{banerjee2023neural}.
Conversely, $\lambda_1$ is always upper bounded by $\lambda_1 \leq \alpha_\phi^2  \nu_{0}^{2} c_{\phi,\sigma_0} = O(1) $ (\cref{eq:lambda1_upper_bound}).
Crucially, this condition is imposed only on the inputs; for the labels, we assume only boundedness.

For a summary of the assumptions, their role in the analysis, and possible relaxations, see 
\cref{tab:assumption_table}.

\section{Global convergence of deep networks under regularized CLD}
\label{sec:convergence_by_langevin}
We now present our main theorem.
\begin{theorem}[Global convergence of deep networks via regularized CLD]
\label{thm:upper_bound_on_expected_loss}
Consider the deep network defined in \cref{sec:setting_deep_nn}, under the assumptions therein, with equal widths $m_l = m,\,l\in\sqprn{L}$ and constant depth $L = \Theta(1)$. Suppose it is trained via regularized continuous-time Langevin dynamics as described in \cref{eq:regularized_langevin}, with $\boldsymbol{\Lambda}$ defined as in \cref{eq:Lambda_def}. Denoting $\bar{\kappa}\triangleq \max \prn{1, \kappa}$, suppose \mbox{$m =\Omega \prn{ \bar\kappa^2 \beta \frac{n}{\lambda_1} \ln \frac{n}{\lambda_1}\ln \prn{\bar\kappa^4 \beta}}$}, $\beta = \Omega \prn{\frac{n}{\lambda_1}}$.
Then, for all $t \geq 0$,
\begin{align}
\mathbb{E}\mathcal{L} \left(\boldsymbol{\theta}_{t}\right)\leq\left(\mathbb{E}\mathcal{L} \left(\boldsymbol{\theta}_{0}\right)-\mathcal{L}_\infty\right)e^{-c_1s_{L+1}^{2}\frac{\lambda_{1}m}{n}t}+\mathcal{L}_\infty\,, \label{eq:the_descent}
\end{align}
where $c_1$ is a global constant, and the expectation is taken over both the initialization and the noise, and the loss floor is
\begin{align*}
\mathcal{L}_{\infty}\asymp & \left(1+\bar{\kappa}^{4}\beta\right)\exp\left(-c_2\frac{m}{\bar{\kappa}^{2}\beta\frac{n}{\lambda_{1}}\ln\frac{n}{\lambda_{1}}}\right)
+\beta^{-1}\frac{n}{\lambda_{1}}\left(1+\kappa^{2}\frac{s_{1}^{2} d}{s_{L+1}^{2} m}+\frac{\kappa^{2}}{s_{L+1}^{2}}\sum_{l=2}^{L}s_{l}^{2}\right)\\
 & +\beta^{-1}\frac{n^2}{\lambda_{1}^2}\kappa^{2}\bar{\kappa}^{2}\left(\frac{s_{1}^{4} d^2}{s_{L+1}^{4} m^2}+\frac{1}{s_{L+1}^{4}}\left(\sum_{l=2}^{L}s_{l}^{2}\right)^{2}+\frac{1}{s_{L+1}^{4}}\sum_{l=2}^{L}s_{l}^{4}\right)
 +\beta^{-2}\frac{n^2}{\lambda_{1}^2}\frac{1}{\kappa^{4}}\,.
\end{align*}
For shallow networks, \ie $L=1$, the loss floor can be improved to 
\begin{equation}
\begin{aligned}
\mathcal{L}_\infty^{\prn{L=1}} \asymp& \left(1+\bar\kappa^{4}\beta\right)\exp\left(-{c_{2}}\frac{m}{\bar\kappa^{2}\beta\frac{n}{\lambda_{1}}\ln\frac{n}{\lambda_{1}}}\right)
+\beta^{-1}\frac{n}{\lambda_{1}}\left(1+\kappa^{2}\frac{s_{1}^{2}d}{s_{2}^{2}m}\right) \\
 \\ 
& +\beta^{-1}\frac{n^{2}}{\lambda_{1}^{2}}\bar\kappa^{2} \kappa^2 \frac{s_{1}^{4}d^{2}}{s_{2}^{4}m^{2}}
+\beta^{-2}\frac{n^{2}}{\lambda_{1}^{2}}\left(\frac{1}{\kappa^{4}}+\frac{s_{1}^{4}d^{2}}{s_{2}^{4}m^{2}}\right)
\,.
\label{eq:L_infty_raw}
\end{aligned}
\end{equation}
\end{theorem}
We provide the proofs for the shallow and deep cases in \appref{app:langevin_proof} and \appref{app:langevin_proof_deep}, respectively.
The floor $\mathcal{L}_\infty$ is strictly positive for any finite $\beta$,
since the stationary distribution of \cref{eq:regularized_langevin} is Gibbs
rather than a point mass, and in  \cref{tab:param_comparison} we show it can be
made smaller than any fixed target by the choice of $m$ and $\beta$.
We complement our result with a finite-probability zero-loss convergence guarantee for shallow networks via a scheduling approach, proven in \appref{app:convergence_with_gf_from_stopped_cld}.
\begin{proposition}[Convergence to zero loss for shallow networks]
\label{prop:convergence_to_zero_loss_shallow_stopped_cld}
For shallow networks, \ie $L=1$, suppose that CLD is stopped at a deterministic time
$\tau=\tilde{\Omega}\left(\frac{n}{s_{2}^{2}\lambda_{1}m}\right)$, and that training then proceeds by gradient flow, \ie
\begin{equation}
\frac{\mathrm{d}}{\mathrm{d}t}\vtheta_t=-\nabla\mathcal{L}\left(\vtheta_t\right)\,.
\label{eq:gradient_flow}
\end{equation}
Let $K>2$. Then, the network converges to zero loss with probability at least \mbox{$
1-\eta\exp\left(-c_{4}\frac{m}{c_{\beta}\bar\kappa^2\beta\frac{n}{\lambda_{1}}\ln\left(\frac{n}{\lambda_{1}}\right)}\right)-\frac{2}{K}
$}, provided that parameters are chosen such that
\begin{equation}
\mathcal{L}_{\infty}^{\prn{L=1}}\lesssim\frac{s_{2}^{2}m^{2}}{Kn^{2}d^{2}s_{1}^{4}\left(\frac{1}{ds_{1}^{2}}+\frac{\kappa^{2}}{ms_{2}^{2}}\right)\left(\frac{\kappa^{2}s_{1}^{2}d}{s_{2}^{2}}+\sqrt{nm}\right)}\,.
\label{eq:L_infty_for_song_convergence}
\end{equation}
E.g., under the standard parameterization with $d=o\prn{\sqrt{n}}$ and $K=\Theta\prn{\sqrt{n}/d}$, this holds at $\beta=\tilde\Omega\prn{n},\,m=\tilde\Omega\prn{n^{2}}$, as in \cref{thm:upper_bound_on_expected_loss}, so the loss vanishes with probability $1-o\prn{1}$.
\end{proposition}
In \cref{tab:param_comparison} we provide concrete examples for simplified forms of $\mathcal{L}_{\infty}$ and $\mathcal{L}_\infty^{\prn{L=1}}$ from \cref{thm:upper_bound_on_expected_loss} under the standard and NTK parameterizations, assuming $d = O \prn{n^{3/2}}$ for simplicity, and only give the largest terms. We provide a similar parameterization analysis for \cref{prop:convergence_to_zero_loss_shallow_stopped_cld} in \appref{sec:song_result_under_parameterizations}.
\subsection{Proof outline}
\label{sec:convergence_proof_outline}
By substituting the MSE loss definition, \cref{eq:mse_loss_def}, into \cref{eq:after_ito_mean_loss_dynamics}, we arrive at the following ODE for the expected loss:
\begin{equation}
\begin{aligned}
\frac{\mathrm{d}}{\mathrm{d}t}\mathbb{E}\mathcal{L} \left(\boldsymbol{\theta}_{t}\right) & =\frac{1}{n}\mathbb{E}\biggl[-n^{-1}\mathbf{r}\left(\boldsymbol{\theta}_{t}\right)^{\top}\mathbf{K}_{\mathrm{NTK}}\left(\boldsymbol{\theta}_{t}\right)\mathbf{r}\left(\boldsymbol{\theta}_{t}\right)
  +\beta^{-1}\mathrm{Tr}\left(\nabla f\left(\boldsymbol{\theta}_{t}\right)^{\top}\nabla f\left(\boldsymbol{\theta}_{t}\right)\right)
\\
 & \qquad\ \ \quad +\beta^{-1}\sum_{i=1}^{n}r_{i}\left(\boldsymbol{\theta}_{t}\right)\Delta f_{i}\left(\boldsymbol{\theta}_{t}\right)
 -\beta^{-1}\mathbf{r}\left(\boldsymbol{\theta}_{t}\right)^{\top}\nabla f\left(\boldsymbol{\theta}_{t}\right)\boldsymbol{\Lambda}\boldsymbol{\theta}_{t}\biggr]\,.
 \label{eq:the_four_terms}
\end{aligned}
\end{equation}
The proof then proceeds by upper bounding each of the four terms in \cref{eq:the_four_terms} for all times $t\geq0$, using high-probability properties at initialization and our change-of-measure inequality.

The first term is the descent term, for which we require a lower bound on the minimum eigenvalue of the NTK at initialization, observing that for any $\lambda_\star > 0$,
\begin{align}
&-n^{-1}\mathbb{E}\left[\mathbf{r}\left(\boldsymbol{\theta}_{t}\right)^{\top}\mathbf{K}_{\mathrm{NTK}}\left(\boldsymbol{\theta}_{t}\right)\mathbf{r}\left(\boldsymbol{\theta}_{t}\right)\right] \nonumber\\
& \leq-n^{-1}\mathbb{E}\left[\left\Vert \mathbf{r}\left(\boldsymbol{\theta}_{t}\right)\right\Vert _{2}^{2}\lambda_{\min}\left(\mathbf{K}_{\mathrm{NTK}}\left(\boldsymbol{\theta}_{t}\right)\right)\right] \nonumber\\
 & =-2\mathbb{E}\left[\mathcal{L} \left(\boldsymbol{\theta}_{t}\right)\lambda_{\min}\left(\mathbf{K}_{\mathrm{NTK}}\left(\boldsymbol{\theta}_{t}\right)\right)\right] \nonumber\\
& \leq-2\lambda_{\star}\mathbb{E}\left[\mathcal{L} \left(\boldsymbol{\theta}_{t}\right)\right] 
+2 \mathbb{E}\left[\mathcal{L} \left(\boldsymbol{\theta}_{t}\right)\max\left(0,\,\lambda_{\star}-\lambda_{\min}\left(\mathbf{K}_{\mathrm{NTK}}\left(\boldsymbol{\theta}_{t}\right)\right)\right)\right]\,.
\label{eq:example_of_need_to_use_second_law}
\end{align}
In \appref{sec:ntk_related_work}, we review existing bounds for the NTK minimum eigenvalue, and find that existing bounds are not suitable for our setting, with the closest match being \citet{banerjee2023neural}. However, they analyze only the NTK parameterization, and their bound has failure probability polynomially small in $m$. For our analysis, we require a bound that applies to the more general parameterizations presented in \cref{sec:setting_deep_nn}, and has exponentially low failure probability. To obtain such a bound, we tighten the analysis of \citet{banerjee2023neural}, yielding the following lemma, proven in \appref{app:ntk_min_eigenvalue}.
\begin{lemma}
\label{thm:lower_bound_ntk_deep}
For $m=\Omega \prn{\frac{n}{\lambda_1} \ln \frac{n}{\lambda_1} \ln n}$,
we have with probability at least $1-\exp \prn{-\tilde{c}_1\frac{m}{ \frac{n}{\lambda_1}\ln \frac{n}{\lambda_1}}}$
over the initialization of $\left\{ \mathbf{W}^{\left(l\right)}\right\} _{l\in\left[L\right]}$,
\[
\lambda_{\min}\left(\mathbf{K}_{\mathrm{NTK}}\left(\boldsymbol{\theta}\right)\right)\geq c_{0}s_{L+1}^{2}m\lambda_{1}\,.
\]
\end{lemma}
Then, we choose $\lambda_\star = c_{0}s_{L+1}^{2}m\lambda_{1}$, and together with the subexponentiality of the loss at initialization, we can bound the second term in \cref{eq:example_of_need_to_use_second_law} using our change-of-measure inequality \cref{cor:second_law_for_langevin}.

Similarly, we upper bound the remaining terms in \cref{eq:the_four_terms} using \cref{cor:second_law_for_langevin} and known properties at initialization, relying heavily on properties of subgaussian and subexponential components. Finally, we arrive at a (possibly complicated) constant $a>0$ such that 
\begin{align*}
    \frac{\mathrm{d}}{\mathrm{d}t}\mathbb{E}\mathcal{L} \left(\boldsymbol{\theta}_{t}\right) \leq
    \frac{1}{n} \prn{
    -c\lambda_\star \mathbb{E}\mathcal{L} \left(\boldsymbol{\theta}_{t}\right) + a
    }\,.
\end{align*}
Then, to conclude the proof of \cref{thm:upper_bound_on_expected_loss}, Gr\"onwall's inequality yields, for all $t\geq0$,
\begin{align*}
\mathbb{E}\mathcal{L} \left(\boldsymbol{\theta}_{t}\right)\leq\left(\mathbb{E}\mathcal{L} \left(\boldsymbol{\theta}_{0}\right)-\mathcal{L}_\infty\right)e^{-c \frac{\lambda_\star}{n}t}+\mathcal{L}_\infty\,.
\end{align*}

Convergence to zero loss by gradient flow (\cref{prop:convergence_to_zero_loss_shallow_stopped_cld}) is then established by adapting the proximity-to-initialization claims of \citet{song2021subquadratic}. Using \cref{thm:upper_bound_on_expected_loss} and the change-of-measure inequality, we show that their conditions hold with positive probability at a sufficiently large deterministic stopping time for CLD, without requiring the parameters to stay near $\vtheta_0$.\footnote{
We believe that \cref{prop:convergence_to_zero_loss_shallow_stopped_cld} can be extended to deep networks by combining the proof mechanism of \appref{app:convergence_with_gf_from_stopped_cld} with a sharpened and generalized analysis of the deep-network Hessian at initialization. Such an analysis would strengthen the existing result of \citet{banerjee2023restricted} by yielding an exponentially small failure probability and accommodating arbitrary parameterizations. We leave this extension to future work.
}

\begin{table}
  \caption{Convergence bounds from \cref{thm:upper_bound_on_expected_loss} under standard and NTK parameterizations.}
  \label{tab:param_comparison}
  \centering
  \small
  \begin{tabular}{lllll}
    \toprule
    \begin{tabular}{@{}l@{}}Parameter-\\ization\end{tabular}
    &
    Setting
    &
    Parameters
    &
    $\mathcal{L}_\infty$
    &
    \begin{tabular}{@{}l@{}}Requirements for\\ $\exists\mathcal{L}_g<1 :\mathcal{L}_\infty \leq \mathcal{L}_g$\end{tabular}
    \tabularnewline
    \midrule
    Standard
    &
    Shallow
    &
    $\begin{aligned}
      \left(s_1,s_2\right)&=\left(1,1\right)\\
      \left(a_1,a_2\right)&=\left(d,m\right)\\
      \kappa&=1
    \end{aligned}$
    &
    $\begin{aligned}
      &\asymp\beta\exp\left(
      -c\frac{m}{\beta n\ln n}
      \right)\\
      &\quad+\beta^{-1}n
    \end{aligned}$
    &
    $\begin{aligned}
      m&=\tilde{\Omega}\left(\beta n\right)\\
      \beta&=\Omega\left(\frac{n}{\mathcal{L}_g}\right)
    \end{aligned}$
    \tabularnewline
    \cmidrule(l){2-5}
    &
    Deep
    &
    $\begin{aligned}
      s_l&=1,\,l\in\sqprn{L+1}\\
      a_l&=m_{l-1},\,l\in\sqprn{L+1}\\
      \kappa&=1
    \end{aligned}$
    &
    $\begin{aligned}
      &\asymp\beta\exp\left(-c\frac{m}{\beta n\ln n}\right)\\
      &\quad+\beta^{-1}n^2
    \end{aligned}$
    &
    $\begin{aligned}
      m&=\tilde{\Omega}\left(\beta n\right)\\
      \beta&=\Omega\left(\frac{n^2}{\mathcal{L}_{g}}\right) \\
    \end{aligned}$
    \tabularnewline
    \midrule
    NTK
    &
    Shallow
    &
    $\begin{aligned}
      \left(s_1,s_2\right)&=\left(\frac{1}{\sqrt{d}},\frac{1}{\sqrt{m}}\right)\\
      \left(a_1,a_2\right)&=\left(1,1\right)\\
      \kappa&=1
    \end{aligned}$
    &
    $\begin{aligned}
      &\asymp\beta\exp\left(
      -c\frac{m}{\beta n\ln n}
      \right)\\
      &\quad+\beta^{-1}n^2
    \end{aligned}$
    &
    $\begin{aligned}
      m&=\tilde{\Omega}\left(\beta n\right)\\
      \beta&=\Omega\left(\frac{n^2}{\mathcal{L}_g}\right)
    \end{aligned}$
    \tabularnewline
    \cmidrule(l){2-5}
    &
    Deep
    &
    $\begin{aligned}
      s_l&=\frac{1}{\sqrt{m_{l-1}}},\,l\in\sqprn{L+1}\\
      a_l&=1,\,l\in\sqprn{L+1}\\
      \kappa&=1
    \end{aligned}$
    &
    $\begin{aligned}
      &\asymp\beta\exp\left(-c\frac{m}{\beta n\ln n}\right)\\
      &\quad+\beta^{-1}n^2
    \end{aligned}$
    &
    $\begin{aligned}
      m&=\tilde{\Omega}\left(\beta n\right)\\
      \beta&=\Omega\left(\frac{n^2}{\mathcal{L}_{g}}\right) \\
    \end{aligned}$
    \tabularnewline
    \bottomrule
  \end{tabular}
\end{table}

\subsection{Analysis of the convergence result}
\label{sec:convergence_result_analysis}

Each term of \mbox{$\mathcal{L}_\infty$} in \cref{thm:upper_bound_on_expected_loss} can be traced back to one of the four terms in the It\^o expansion \cref{eq:the_four_terms}. Reading \cref{eq:L_infty_raw} from left to right: The exponent comes from the NTK term; the (unavoidable) \mbox{$\beta^{-1}\frac{n}{\lambda_1}$} factor originates in the diffusion term \mbox{$\beta^{-1}\Tr{\nabla f^\top\nabla f}$}; the \mbox{$\beta^{-1}\frac{n^{2}}{\lambda_{1}^{2}}$} contribution comes from the Laplacian term \mbox{$\beta^{-1}\sum_i r_i \Delta f_i$}; the \mbox{$\beta^{-2}\frac{n^2}{\lambda_1^2}$} factors arise from the regularization term.

In order to interpret our bounds in \cref{thm:upper_bound_on_expected_loss}, we discuss the requirements for achieving $\mathcal{L}_\infty\leq\mathcal{L}_g$, for some $\mathcal{L}_g < 1$. For simplicity, we absorb $\lambda_1=\Theta\prn{1}$ as a constant.
First, consider the first term in $\mathcal{L}_\infty$. To obtain convergence to $\mathcal{L}_g<1$, one must require $m=\tilde{\Omega}\prn{\bar\kappa^2\beta n}$. Since $\bar\kappa\geq1$ for any parameterization, this yields the necessary condition $m=\tilde{\Omega}\prn{\beta n}$, implying an upper bound on the inverse temperature.
Second, consider the term $\beta^{-1}\frac{n}{\lambda_1}$, which imposes the requirement $\beta=\Omega\prn{n/\mathcal{L}_g}$, independently of depth, width, and parameterization.\footnote{As we discuss later, $\beta=\Omega\prn{n}$ is in fact an unavoidable requirement for any optimization bound in our setting, where arbitrary labels are permitted.}

Taken together, these demands point to a tradeoff in the role of noise. 
On one hand, injected noise is essential in our analysis because, through the second-law change-of-measure argument, it allows control of quantities under the evolving training distribution without requiring the parameters to remain in a small neighborhood of initialization.
On the other hand, the same noise leads to an unavoidable loss floor, since it is isotropic and non-diminishing. This is also evident empirically (see \cref{sec:experiments}).
The bound thus makes explicit how width, temperature, and parameterization interact: larger width improves kernel conditioning, 
while larger inverse temperature reduces the optimization floor but weakens the stochastic mechanism of the proof.
Moreover, any scaling $\kappa>1$ requires either larger noise or greater width in our analysis, consistent with its strengthening of the lazy regime \citep{chizat2019lazy}.
\paragraph{Width requirements in context.}
Under standard parameterization, our requirement \mbox{$m=\tilde\Omega\prn{\beta n}$} translates to \mbox{$m=\tilde\Omega\prn{n^2}$} and \mbox{$m=\tilde\Omega\prn{n^3}$} for shallow and deep networks, respectively.
Existing proofs for GD/GF under standard parameterization demand
\mbox{$m=\Omega\prn{n^{5/3}}$} in the shallow case
\citep{song2021subquadratic}\footnote{
\citet{song2021subquadratic} 
state a bound of $md=\Omega(n^{3/2})$ under an added boundedness assumption on $\vv$. In their supplementary material, after removing this assumption, the bound becomes $md=\Omega(n^{5/3})$.
}
and \mbox{$m=\Omega\prn{n^2}$} in the deep case
\citep{nguyen2020global}; 
see \cref{tab:convergence_comparison_gd} for a broader comparison.
Our width requirement may be tightened for deep networks, but the shallow gap is intrinsic to our approach.
\paragraph{From a stopped schedule to noise annealing.}
\cref{prop:convergence_to_zero_loss_shallow_stopped_cld} can be viewed as the simplest non-trivial \emph{noise schedule}: fixed inverse temperature \mbox{$\beta$} until \mbox{$t=\tau$}, and \mbox{$\beta=\infty$} (gradient flow) thereafter. It captures the basic noise-annealing mechanism: use noise to escape rare-event regions while the loss is large, then rely on local geometry once the loss drops below the corresponding stationary floor. This mirrors practical schedules, where \mbox{$\beta\prn{t}$} increases continuously to infinity, as in simulated annealing \citep{kirkpatrick1983optimization} and empirical deep-learning variants \citep{neelakantan2015adding,deng2021shrinking}. 
Under our approach, a rigorous treatment of general annealing schedules would first require extending \cref{thm:second_law_change_of_measure} to time-dependent potentials, with their family of instantaneous Gibbs stationary measures;
we leave this to future work.

\section{Experiments}
\label{sec:experiments}
Our motivation for performing experiments is twofold. First, we wish to investigate a matter that our theory does not fully answer: whether the dynamics fall within the lazy regime under standard parameterization. While our analysis does not rely on proximity to initialization, which is often associated with lazy training, it does not rule out lazy behavior either. Second, we wish to empirically validate aspects of our theory, such as the obtained loss floor and the preservation of a lower bound on $\lntk$. To these ends, we train a shallow neural network for different values of $\beta$ that satisfy our constraints, and also for $\beta=\infty$, which simulates noiseless training. See \appref{app:97_app_experiments} for full details.

\begin{figure*}[ht]
    \centering
    \begin{subfigure}[t]{0.40\textwidth}
        \captionsetup{labelfont=normalfont,textfont=normalfont,justification=centering}
        \centering
        \includegraphics[width=\linewidth]{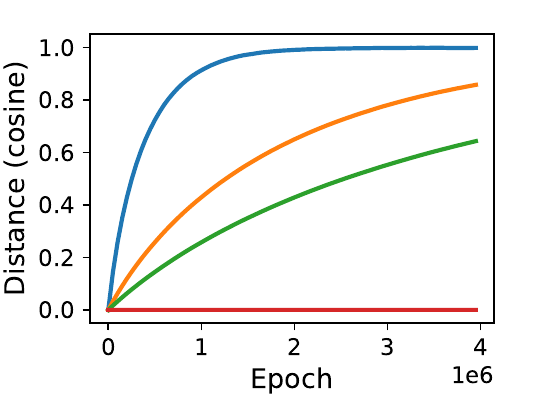}
        \caption{Jacobian drift in cosine distance.}
        \label{fig:expr1_jacobian_dist_co}
    \end{subfigure}
    \begin{subfigure}[t]{0.40\textwidth}
        \captionsetup{labelfont=normalfont,textfont=normalfont,justification=centering}
        \centering
        \includegraphics[width=\linewidth]{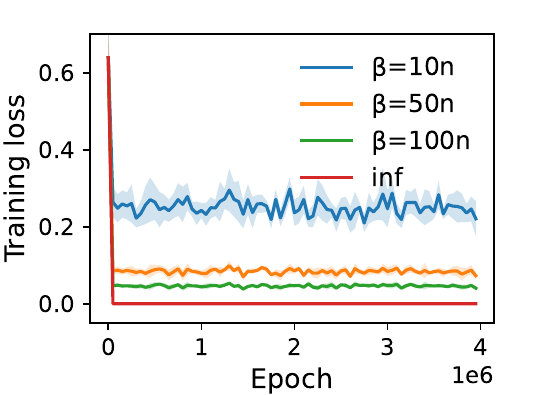}
        \caption{Training loss.}
        \label{fig:expr1_train_loss}
    \end{subfigure}
    \begin{subfigure}[t]{0.60\textwidth}
        \captionsetup{labelfont=normalfont,textfont=normalfont,justification=centering}
        \centering
        \includegraphics[width=\linewidth]{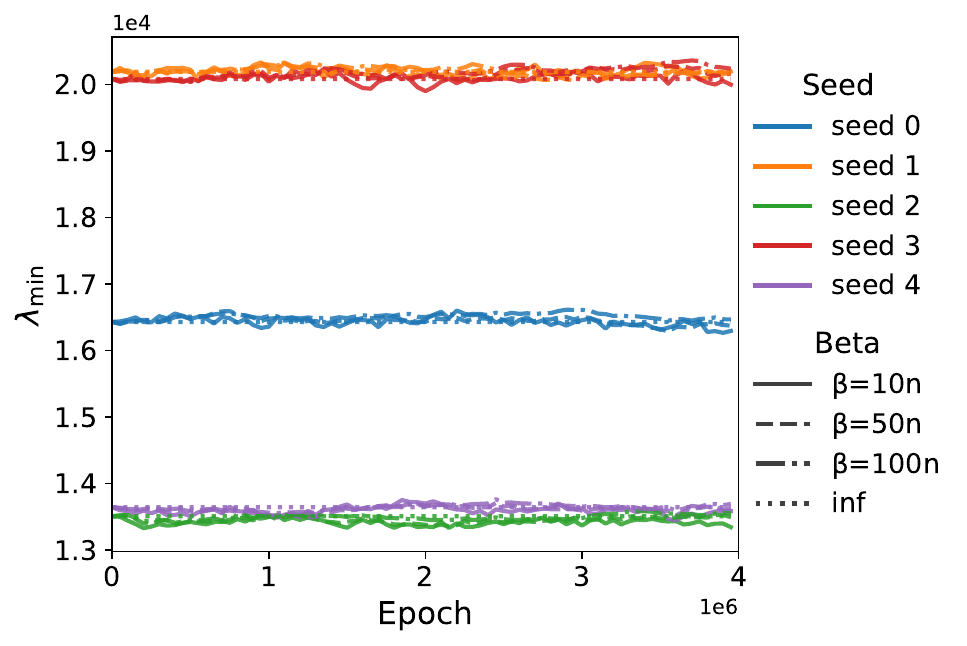}
        \caption{Minimal eigenvalue of feature Gram matrix.}
        \label{fig:feat_gram_lambda_all_betas_seeds}
    \end{subfigure}

    \caption{Training statistics as a function of the number of updates for $\beta \in \{10n, 50n, 100n, \infty\}$. In \cref{fig:expr1_jacobian_dist_co} we plot the distance of the network Jacobian from its value at initialization. In \cref{fig:expr1_train_loss} we plot the training loss. In \cref{fig:feat_gram_lambda_all_betas_seeds} we show the minimal eigenvalue of the feature Gram matrix, plotting each ($\beta$, seed) pair separately. For $\beta=\infty$ the seed controls only the dataset.}
    \label{fig:figure1}
\end{figure*}

\Cref{fig:expr1_jacobian_dist_co} shows \textit{Jacobian drift} during training, defined as the cosine distance between the network Jacobian (evaluated for all datapoints and concatenated) at time $t$ and at initialization. The cosine distance shows substantial deviation from the initial Jacobian direction under noisy training.
This is evidence that the dynamics depart from the commonly analyzed lazy regime. 
In \appref{sec:additional_experiments}, we provide additional evidence that the noisy neural-network dynamics are not fully captured by the frozen-Jacobian linearization around initialization.

\Cref{fig:expr1_train_loss} shows training-loss plateaus consistent with our theory, scaling proportionally to $\beta^{-1}$.
\Cref{fig:feat_gram_lambda_all_betas_seeds} shows the minimum eigenvalue of the feature Gram matrix $\phi(\mathbf{WX})\phi(\mathbf{WX})^\top$. This value can be viewed as a rescaled empirical analogue of $\lambda_1$ from \Cref{ass:data_separability}. 
It remains bounded away from zero throughout training, determined primarily by the sampled dataset. 
For small $\beta$, this agrees with our theory; for large $\beta$, it follows from the near-initialization dynamics of lazy training.

\section{Discussion}
\label{sec:discussion}

\paragraph{Convergence of Langevin dynamics.}
\label{par:related_work_langevin}
A common route to optimization guarantees for Langevin dynamics first proves convergence to the Gibbs law and then converts this distributional convergence into optimization bounds. Non-asymptotic analyses typically rely on functional inequalities for the Gibbs law, such as the log-Sobolev or Poincar\'e inequalities, whose constants control the rate \citep{vempala2019rapid,chewi2025analysis,mousavi2023towards}.
A prominent work of this kind is \citet{raginsky2017non}, where indeed the convergence rate is governed by the log-Sobolev constant associated with the Gibbs measure. As the authors note, this constant is bounded by the spectral gap and, unavoidably, worsens exponentially as the inverse temperature~$\beta$ increases. By proving convergence directly for finite-width neural networks, leveraging a network-specific NTK lower bound that is preserved along Langevin dynamics, we avoid relying on such unspecified constants and instead yield a convergence rate whose exponent does not depend on the noise level.
See \appref{sec:existing_work_langevin_convergence} and \cref{tab:convergence_comparison_noise} for an extended discussion.

\paragraph{Comparison to the generalized NTK analysis of \citet{chen2020generalized}.}
\citet{chen2020generalized} study two-layer networks through an infinite-width mean-field description and show that, under a sufficiently large output scaling, noisy gradient descent with weight decay can still exhibit kernel-like convergence. Their analysis replaces the usual Euclidean closeness-to-initialization argument by a distributional one: the parameter law remains close to the initialization law, in Wasserstein/KL-type distances, which keeps the minimum eigenvalue of the NTK bounded away from zero and yields linear convergence up to a noise floor. Thus, as in our work, convergence is ultimately obtained from a kernel lower bound beyond the conventional lazy-training argument.
When our finite-width bounds are specialized to the infinite-width limit ($m\to\infty$), the resulting convergence picture is consistent with that of \citet{chen2020generalized}. Their analysis is tailored to this limiting setting: the deterministic NTK lower bound is imposed and maintained at the level of the limiting parameter law. In contrast, our result applies directly at finite width, where the empirical NTK is random and its initialization-level nondegeneracy must be transferred to the trained parameters with high probability.\footnote{
Extending the results of \citet{chen2020generalized} to finite width using current methods, such as propagation-of-chaos~(PoC) bounds, would require quantitative control of the particle approximation for the scaled objective used in \citet{chen2020generalized}. 
Existing uniform-in-time PoC results for mean-field Langevin dynamics involve problem-dependent constants \citep[e.g.,][]{chen2025uniform,nitanda2025propagation}, which may depend on the output-scaling parameter in this setting; hence, meaningful finite-width guarantees do not seem to follow directly from existing PoC bounds.
}
A deeper difference lies in the way the two arguments treat noise. In \citet{chen2020generalized}, the regularization/noise parameter is part of the dynamics, but not a constructive ingredient in preserving the kernel lower bound; rather, increasing it worsens both their convergence guarantee and their KL-divergence control.
In our analysis, increasing the Langevin noise strengthens the divergence control, as in \citet{harel2025temperature}, which is precisely what propagates the initialization-level NTK lower bound along the finite-width trajectory.
Together, these perspectives suggest a more complete picture of how favorable initialization properties may be maintained during training in practice.

\paragraph{Is a bound with $\beta=o\prn{n}$ possible?}
A natural question is whether the analysis can be tightened.
Our result requires $\beta=\Omega\left(n\right)$ to obtain a nontrivial upper bound on the expected limiting loss.
While our width requirement may be loose, we argue that an optimization bound with $\beta=o\prn{n}$ cannot hold in this generality without some restriction on the labels.
Our setting permits random labels, which neural networks can fit \citep{zhang2017understanding}.
However, consider the generalization result of \citet{harel2025temperature}: under similar assumptions, with probability at least $1-\delta$ over $S\sim D^n$,
\begin{align*}
\mathbb{E}_{p_t}\sqprn{\mathcal{L}_D\prn{\boldsymbol{\theta}_t}-\mathcal{L}_S\prn{\boldsymbol{\theta}_t}\mid S}
&\leq
\sqrt{\frac{\beta\mathbb{E}_{p_0}\sqprn{\mathcal{L}_S\prn{\boldsymbol{\theta}_0}\mid S}+\ln\prn{n/\delta}}{2n}}\,.
\end{align*}
Thus, for any $\beta=o\prn{n}$, \eg $\beta=\Theta\left(n^{0.99}\right)$, their bound vanishes as $n\to\infty$ while still allowing random labels. Combining such a generalization bound with an optimization bound on $\mathbb{E}_{p_t}\sqprn{\mathcal{L}_S\prn{\boldsymbol{\theta}_t}}$ would then imply generalization under random labels, a contradiction. Consequently, any optimization result in this regime must avoid this implication, for instance by adding label structure, as in \citet{Ji2020Polylogarithmic}, or by including a corrective term, as in \citet{arora2019fine}.

\paragraph{Neural network parameterizations.}
NN parameterizations, \ie the scaling of weights, outputs, and learning rate, largely determine the limiting training dynamics. Under NTK and standard parameterizations, infinite-width training with GF is well approximated by kernel gradient descent with a fixed kernel, suppressing feature learning. Mean-field~(MF) parameterization instead captures genuine feature evolution through the evolution of the empirical distribution of neurons \citep{mei2018mean, chizat2018global, sirignano2020mean, rotskoff2022trainability}. However, finite-width MF guarantees remain limited and typically rely on restrictive assumptions \citep{javanmard2020analysis, mahankali2023beyond, dana2025convergence}; moreover, MF dynamics are largely limited to two or three layers, as they may freeze at greater depth \citep{nguyen2023rigorous}.
Thus, despite its ability to model feature learning, it remains unclear whether MF analysis explains the practical success of deep networks, whose default implementation often follows the standard parameterization, as in \texttt{PyTorch} \citep{paszke2019pytorch}. We were also unable to extend our proof mechanism to the MF parameterization; see \appref{sec:neural_network_parameterizations_introduction} for further discussion.

\paragraph{Why study CLD?}
We study CLD as a tractable model for the role of stochasticity in neural-network training. 
Although CLD differs from common training methods such as SGD, it is the continuous-time limit of other discrete-time algorithms, such as stochastic gradient Langevin dynamics~(SGLD) that is motivated both empirically and theoretically
\citep{kirkpatrick1983optimization,neelakantan2015adding,deng2021shrinking,jin2017escape,jin2021nonconvex,welling2011bayesian,zhang2022low}. Technically, CLD admits a Gibbs stationary distribution, enabling our change-of-measure argument. Related Gibbs-like stationary behavior and SDE approximations have been extensively studied for SGD \citep{mandt2017stochastic,yaida2018fluctuationdissipation,azizian2024what,shi2023learning,li2017stochastic,li2019stochastic,hu2019diffusion,cheng2020stochastic,lanconelli2025note,li2021validity}. Thus, CLD is not meant to reproduce SGD exactly, but to isolate the effect of noise in a mathematically tractable and practically motivated setting, with extensions to anisotropic state-dependent noise left for future work. We elaborate further in \appref{sec:why_study_cld}.

\paragraph{Future work.}
Our work could be extended in several natural directions to better capture practical training, for instance by incorporating noise annealing or state-dependent noise, such as noise induced by the SGD covariance, as discussed above, and by deriving analogous bounds for discretized Langevin dynamics.
Another direction is to investigate how noise in the training process may promote feature learning, using our method.
More broadly, our change-of-measure lemma may apply to other settings requiring rare-event control or finite-time analysis when convergence times are unavailable.

\begin{ack}
The research of DS was Funded by the European Union (ERC, A-B-C-Deep, 101039436). Views and opinions expressed are however those of the author only and do not necessarily reflect those of the European Union or the European Research Council Executive Agency (ERCEA). Neither the European Union nor the granting authority can be held responsible for them. The research of DS was also supported by the Council for Higher Education in Israel under the Moonshot Project. 
NS was partially supported by the NSF TRIPOD Institute on Data Economics Algorithms
and Learning (IDEAL) and an NSF-IIS award.

\end{ack}

\bibliography{99_bibliography}
\bibliographystyle{plainnat}

\newpage
\appendix
\startcontents[app]
\appendixtableofcontents
\newpage
\section{Additional related work}
\label{app:additional_related_work}

\subsection{Existing work on the convergence of Langevin dynamics}
\label{sec:existing_work_langevin_convergence}
Langevin diffusion has been used for global optimization for decades \citep{chiang1987diffusion}, alongside its classical role in sampling from the Gibbs distribution \citep{besag1994comments, Roberts1996Exponential}.

Works that study the convergence rate of Langevin dynamics from the sampling perspective typically focus on establishing distributional convergence to the Gibbs posterior, \ie convergence of a divergence measure such as the KL divergence or R\'enyi divergence to the Gibbs law. Representative examples include \citet{vempala2019rapid} and subsequent works \citep{chewi2025analysis, mousavi2023towards}; these works assume that the Gibbs distribution satisfies a suitable functional inequality, such as a log-Sobolev or Poincar\'e inequality. These inequalities rely on unspecified constants, which can be arbitrarily bad.

Many works on convergence rates of Langevin dynamics for optimization proceed through the same route, namely by converting distributional convergence to the Gibbs distribution into optimization guarantees, driven by the fact that the Gibbs distribution concentrates around minimizers.
A prominent example is \citet{raginsky2017non}, where the authors decompose the generalization error under SGLD and bound its constituent terms: the error controlled by convergence of SGLD toward the Gibbs distribution, the generalization gap of the Gibbs distribution itself, and the empirical loss under the Gibbs distribution.
The resulting convergence rate depends on the log-Sobolev constant of the Gibbs distribution, which, as the authors discuss, is controlled by the spectral gap and necessarily deteriorates exponentially with the inverse temperature~$\beta$; see their discussion section.\footnote{The LSI constant also deteriorates exponentially in the number of parameters $N_p$ in the worst case. However, as the authors note, this  $N_p$ dependence can be removed in some specific problems.}
Furthermore, their bound on the empirical loss under the Gibbs distribution scales as $\frac{N_p \log(\beta + 1)}{\beta}$, where $N_p$ is the number of trainable parameters. Hence, the resulting optimization guarantees become vacuous when $N_p > \beta$, limiting the applicability of the result to the low-noise regime, where the convergence rate is exponentially slow.
Follow-up works, such as \citet{xu2018global, kumar2025langevin}, employ similar analyses and suffer from the same limitations, applying only in the low-noise regime with a convergence rate that is exponentially slow.
A conceptually different recent contribution is \citet{chen2024langevin}, who replace the standard sampling-to-optimization argument with a direct Lyapunov/hitting-time analysis of Langevin dynamics; nevertheless, their main quantitative guarantees still rely on a Poincar\'e inequality with an unspecified constant for the Gibbs measure.

\paragraph{Neural-network-specific works.}
More recent works on the convergence of Langevin dynamics that focus specifically on neural network training still mostly rely on the Gibbs functional-inequality route, or are confined to special regimes such as lazy training.
\citet{lugosi2024convergence} study a Langevin-type continuous-time approximation of SGD with anisotropic noise covariance matching that of the stochastic gradient. However, their analysis requires the parameters to remain close to initialization, preventing its extension to settings in which parameters drift away from initialization.
Their work extends the noiseless analysis of \citet{chatterjee2022convergence}, assumes a non-standard initialization scheme, and importantly yields its strongest bounds in the zero-loss regime, thereby treating the noise essentially as a nuisance.
\citet{oberweis2026convergence} analyze similar noise dynamics to those of \citet{lugosi2024convergence}, but their analysis is explicitly confined to the lazy-training regime induced by output scaling, similarly to \citet{chizat2019lazy}.

Another major line of research is mean-field Langevin dynamics~(MFLD), where convergence is studied in an infinite-width lifted measure-space regime. Quantitative exponential-rate results in this line commonly rely on uniform log-Sobolev-type assumptions \citep[e.g.,][]{chizat2022meanfield, nitanda2022convex, takakura2024mean, chen2025uniform, nitanda2025propagation};
see \cref{sec:existing_mean_field_convergence_results} for further details.

The above works help distinguish our contribution: we prove empirical-loss convergence directly for finite-width neural networks.
Our proof uses noise as a crucial mechanism, together with a network-specific NTK lower bound that is preserved along Langevin dynamics, thereby avoiding the derivation of optimization guarantees from global functional inequalities of the Gibbs measure, and their associated unspecified constants.
\Cref{tab:convergence_comparison_noise} places the analyses discussed here alongside ours, recording for each the width required, whether the parameters may leave initialization, the objective that is contracted, the rate, and the loss floor.

\subsubsection{$L_2$-regularization in Langevin analyses}
Weight decay, \ie $L_2$-regularization, is standard in convergence analyses of Langevin dynamics. It ensures that $\mathcal{L}\to\infty$ as $\norm{\boldsymbol{\theta}}\to\infty$, a condition needed for the integrability of $e^{-\beta\mathcal{L}}$ and thus for the existence of the Gibbs distribution, where $\mathcal{L}$ here denotes the regularized loss.\footnote{
Without such a confining term the picture can differ markedly: for example, under cross-entropy the loss is minimized only as $\norm{\boldsymbol{\theta}}\to\infty$, so the empirical NTK of unregularized gradient flow drifts from initialization \citep{yu2025divergence}; $L_2$ regularization precludes this.
}
In \citet{raginsky2017non}, $L_2$-regularization is noted as a mechanism for enforcing \emph{dissipativity} of the loss, which is leveraged to derive a log-Sobolev inequality for the Gibbs distribution. Dissipativity is likewise assumed in later analyses of Langevin dynamics \citep[e.g.,][]{xu2018global, chen2024langevin}, and $L_2$-regularization is used explicitly in works such as \citet{kumar2025langevin}, as well as in the MFLD literature \citep{nitanda2022convex, chizat2022meanfield, takakura2024mean, nitanda2025propagation}.

In our setting, we introduce $L_2$-regularization so that the Gibbs reference measure matches the initialization, choosing precisely the strength for which $\nu=p_0$; see \cref{thm:second_law_change_of_measure}. Other regularization strengths are also possible, at the price of an additional divergence term in the change-of-measure inequality; see \cref{rem:avoid_d_infinity}.

\subsection{Neural network parameterizations}
\label{sec:neural_network_parameterizations_introduction}
\subsubsection{Introduction}
The problem of neural network optimization has been analyzed under several parameterizations, that is, under varying choices of neuron scaling and initialization. Importantly, these choices have a profound impact on the optimization and generalization behavior of the network, as observed in prior works such as \citet{chizat2019lazy, woodworth2020kernel, yang2021tensor}. 

The parameterization most commonly used in practice---also the default in \texttt{PyTorch}---is the so-called \emph{standard parameterization}, in which no explicit scaling is applied within the network, and the parameters are initialized by He initialization \citep{he2015delving, paszke2019pytorch}. However, exploring alternative parameterizations, and particularly different scalings with respect to network width, has led to fundamentally distinct perspectives on overparameterized neural network training dynamics. In particular, two main theoretical frameworks have emerged, each corresponding to a different scaling regime: the \emph{neural tangent kernel}~(NTK) approach \citep{jacot2018neural, du2018gradient, chizat2019lazy}, and the \emph{mean-field}~(MF) approach \citep{mei2018mean, chizat2018global, sirignano2020mean, rotskoff2022trainability}.

\paragraph{NTK.} In the infinite-width limit, the NTK analysis shows that the network training dynamics are equivalent to gradient descent in a linearized model with a fixed kernel \citep{jacot2018neural, lee2019wide, yu2026branch}. In this regime, the parameters remain close to their initialization throughout training, a phenomenon known as \emph{lazy learning} \citep{chizat2019lazy}. These properties extend to standard parameterization in the limit of infinite width \citep{yang2021tensor}. Under these assumptions, one can prove \emph{linear} convergence of deep networks to zero \emph{empirical} loss with width scaling polynomially in the number of samples \citep{du2019deep, allen2019convergence, zou2019improved, song2021subquadratic}. However, since the parameters move only slightly from their initial values, the network fails to learn new features. This significantly limits the expressive power observed in practice and may lead to weaker generalization performance \citep{wei2019regularization, woodworth2020kernel}.

\paragraph{Mean-field.} In contrast, under mean-field analysis, the network is viewed as a collection of infinite particles in the limit of infinite width. Then, the empirical distribution of neurons (particles) under gradient flow evolves in the space of measures according to gradient flow in the 2-Wasserstein metric.
This PDE perspective captures feature learning dynamics, as demonstrated, for example, by \citet{suzuki2023feature} in the context of the $k$-sparse parity problem, and by \citet{mahankali2023beyond, glasgow2025propagation} for learning single-index models.
However, as we elaborate in the following section, non-asymptotic global convergence analyses of mean-field parametrized networks \emph{under realistic assumptions} are still an open problem.

\subsubsection{Existing convergence results for mean-field parameterized networks}
\label{sec:existing_mean_field_convergence_results}
Preliminary mean-field analyses of noiseless training dynamics established asymptotic convergence to zero population loss in the infinite-width, infinite-time regime \citep{mei2018mean, chizat2018global, sirignano2020mean, rotskoff2022trainability}. In contrast, existing non-asymptotic convergence results for shallow networks under the mean-field parameterization typically rely on restrictive assumptions on the data distribution, such as strong separability conditions \citep[e.g.,][]{chen2022on, boursier2022gradient, dana2025convergence}, or on highly simplified teacher models (e.g., concave \citep{javanmard2020analysis} or quartic \citep{mahankali2023beyond} objectives).

More broadly, convergence results under milder assumptions, which would accommodate more realistic learning problems, remain largely open. A major obstacle is controlling the finite-width deviation from the infinite-particle PDE limit, termed the \emph{propagation of chaos}~(PoC) error, where ``chaos'' refers to the asymptotic independence of neurons \citep{sznitman1991topics}. Available bounds on this error, typically obtained via Gr\"onwall-type arguments, grow exponentially with training time \citep{mei2018mean, mei2019mean}, and thus do not suffice for problems which require longer time horizons for feature learning, such as learning multi-index models \citep{arous2021online, abbe2023sgd, bietti2025learning}.
Recently, \citet{glasgow2025propagation} derived a bound that grows polynomially with training time and thus addresses this category of problems, but requires additional stability and local strong convexity assumptions, which do not necessarily hold in neural networks. 
More recently, \citet{glasgow2026uniformintimeweakpropagationofchaosshallow} established uniform-in-time weak PoC bound, relying on a convergence assumption for the mean-field loss.

Meanwhile, analyses of \emph{mean-field Langevin dynamics}~(MFLD), in which Gaussian noise is added to the gradient updates (as in our work), establish linear convergence of both the infinite-width PDE limit \citep{nitanda2022convex, chizat2022meanfield} and the finite-width network \citep{nitanda2022particle, takakura2024mean}, together with uniform-in-time propagation-of-chaos bounds \citep{chen2025uniform, nitanda2025propagation}. Notably, these results rely on the associated Gibbs distribution satisfying a \emph{uniform log-Sobolev inequality}~(LSI). However, existing bounds on the LSI constant deteriorate exponentially with the inverse temperature~($\beta$) \citep{raginsky2017non, chizat2022meanfield, chewi2024uniform}, leading to vacuous guarantees in the low-noise regime for the convergence rates derived in the aforementioned works. More recently, \citet{wang2026local} strengthened this line of research by showing that \textit{local} convergence (i.e., convergence when initialized near the limiting distribution) can instead be established under the weaker Poincar\'e inequality, yielding faster rates. 
Importantly, both the LSI and Poincar\'e inequalities are abstract functional inequalities imposed on the Gibbs measure, and they are not verified to hold with usable constants in neural network models.
It is also worth mentioning \citet{hu2021mean}, who establish asymptotic convergence of MFLD without relying on log-Sobolev inequalities, and briefly discuss non-asymptotic convergence under stronger assumptions that are not shown to hold for neural networks in general. Moreover, the corresponding convergence rate deteriorates as the noise level decreases, with no explicit characterization of this dependence.
The MFLD rows of \cref{tab:convergence_comparison_noise} summarize how these guarantees compare with ours.

\subsubsection{Discussion: connection to practical neural network parameterizations}
\label{sec:connection_to_practical_parameterizations}
A natural question is to what extent these approaches explain training dynamics under the parameterizations used in practice, since neither the NTK parameterization (as introduced by \citet{jacot2018neural}) nor the mean-field parameterization (averaging over neurons) is commonly implemented.

As noted above, the standard parameterization remains the dominant choice in practice. As \citet{yang2021tensor} point out, in the infinite-width limit it leads to the kernel (lazy) regime when trained with sufficiently small step sizes, and therefore does not exhibit feature learning. Small step sizes are also assumed in several convergence analyses of standardly parameterized neural networks, which show---and rely on---the parameters remaining close to their initialization throughout training \citep[e.g.,][]{oymak2020toward, song2021subquadratic}.

However, in practice, neural networks trained under the standard parameterization do learn features and achieve strong generalization performance \citep{krizhevsky2012imagenet, he2015delving, zhang2017understanding}. This discrepancy is plausibly due to the fact that practical networks have finite width and are typically trained with step sizes larger than those allowed by the conservative stability conditions in the infinite-width analysis of \citet{yang2021tensor}. Indeed, larger step sizes have been shown to improve generalization capabilities \citep{damian2022neural, qiao2024stable}.

To reconcile stable infinite-width training with feature learning for deep networks, \citet{yang2021tensor} propose the \emph{maximal-update parameterization}~($\mu$P), which has gained traction in the recent literature \citep[see, e.g.,][]{bordelon2023depthwise, chen2025global, dey2025dont}.
For shallow networks, $\mu$P reduces to the mean-field parameterization (up to symmetries), and can be viewed as a generalization of mean-field scaling to deep architectures.

\subsubsection{Failure to extend our result to mean-field parameterization}
Mean-field parameterization, as well as $\mu$P parameterization, induces the output scaling $\kappa=\frac{1}{\sqrt{m}}$. This scaling conflicts directly with the NTK-minimum-eigenvalue term in $\mathcal{L}_\infty$ from \cref{thm:upper_bound_on_expected_loss}, through the regularization-induced factor
\[
\beta^{-2}\frac{n^{2}}{\lambda_{1}^{2}}\frac{1}{\kappa^{4}}\,.
\]
When $\kappa=\frac{1}{\sqrt{m}}$, controlling this term requires $\beta=\Omega\left(nm\right)$, in contrast to the requirement \mbox{$m=\tilde{\Omega}\left(\beta n\right)$}.

It is not surprising that this obstruction arises from regularization: we use regularization to preserve favorable initialization properties, and this may mark a boundary of our method.
The second-law change-of-measure argument controls training by propagating initialization-time properties of the full parameter process. This is sufficient when convergence follows from preserving a favorable initialization, but mean-field training is expected to rely on the creation of new, label-dependent features through the particle dynamics. Since our assumptions allow arbitrary labels and our proof does not track label-dependent feature evolution, the framework may be poorly suited to proving a teacher-free mean-field convergence theorem.

\subsection{Why study continuous-time Langevin dynamics for training neural networks?}
\label{sec:why_study_cld}
Our work employs CLD as a tractable framework for studying the role of stochasticity in neural network training. However, CLD differs from practical training algorithms, especially mini-batch SGD, in two central respects: (i) it is formulated in continuous rather than discrete time, and (ii) it uses additive isotropic noise on top of the full-batch gradient, rather than the anisotropic, state-dependent noise induced by mini-batch sampling. Nevertheless, there are several theoretical and practical reasons to study CLD in the context of neural network training.

First, although CLD is not the continuous-time limit of SGD, it \emph{is} the continuous-time limit of algorithms such as Langevin Monte Carlo (LMC)\footnote{Also known as the Unadjusted Langevin Algorithm (ULA).} and stochastic gradient Langevin dynamics (SGLD), \ie full-batch or stochastic-gradient methods with injected Gaussian noise, respectively \citep{karimi2023dynamical}. In particular, SGLD is well motivated both theoretically and empirically. Injected noise has been shown empirically to help optimization escape poor local minima in highly non-convex landscapes, as in simulated annealing \citep{kirkpatrick1983optimization}, and to improve both optimization and generalization in neural network training \citep{neelakantan2015adding, deng2021shrinking}. On the theoretical side, injected noise is also known to facilitate escape from saddle points \citep{jin2017escape, jin2021nonconvex}. Beyond its theoretical appeal, SGLD is used in practice in settings where explicit stochasticity is beneficial, such as Bayesian posterior sampling \citep{welling2011bayesian} and low-precision optimization \citep{zhang2022low}.

Second, the main technical reason we focus on CLD is that it admits the Gibbs stationary distribution, which enables us to apply our change-of-measure inequality. Importantly, \citet{mandt2017stochastic, yaida2018fluctuationdissipation, azizian2024what} show that, under certain conditions, SGD admits a stationary distribution that resembles the Gibbs distribution, while works such as \citet{vempala2019rapid, chewi2025analysis, mousavi2023towards} bound the R\'enyi divergence of LMC from the Gibbs distribution under functional inequalities such as the Poincar\'e inequality and the log-Sobolev inequality.
These results suggest that tools similar to our change-of-measure inequality may apply beyond the exact CLD setting.

Finally, analyzing SGD via stochastic differential equations (SDEs) in continuous time is a standard approach. \citet{shi2023learning} study the effects of learning rate under an isotropic Gaussian noise model, as in our setting, while \citet{li2017stochastic, li2019stochastic, hu2019diffusion, cheng2020stochastic, lanconelli2025note} rigorously justify approximating SGD in the small-step regime by an SDE with anisotropic, state-dependent Gaussian noise. \citet{li2021validity} further support this perspective, both theoretically and empirically, by showing that such SDEs can accurately capture the behavior of practical SGD even at finite learning rates, provided that the equilibrium gradient noise is sufficiently large relative to the norm of the gradient. 
Extending our analysis to more realistic anisotropic noise models is a natural direction for future work. 

Overall, the point of studying CLD is not that it reproduces SGD in every detail, but that it isolates the role of stochasticity in a mathematically tractable and still practically relevant setting. Moreover, the role of noise in our proof highlights its theoretical---and potentially practical---utility for global convergence.
\subsection{Existing lower bounds on the NTK minimum eigenvalue at initialization} 
\label{sec:ntk_related_work}
The eigenvalues of the empirical NTK matrix at initialization are well-studied \citep{oymak2020toward, nguyen2021tight, song2021subquadratic, montanari2022interpolation, bombari2022memorization, banerjee2023neural,karhadkar2024bounds}. Prior work varies substantially in the scope of their guarantees: they consider different assumptions on the initialization, output scaling, activations and data. Additionally, some are restricted to shallow networks or to polynomially-decaying failure probability. We first provide a detailed survey of prior work, and then discuss their relevance to our setting.

\citet{oymak2020toward} consider shallow networks with $\Theta(1)$ initialization variance and a fixed last layer with coordinates of order $\Theta \prn{\sqrt{m}}$. They show that $\lntk = \Omega(m)$ whenever $m = \Omega \prn{n \log^2 n}$. This result holds with failure probability of order $1/n$. 
\citet{montanari2022interpolation} consider a linearized model with output scaling of $1/\sqrt{m}$ and show a bound of $\Omega(1)$ whenever $md \log(md) > n$ and $n = \mathrm{poly}(d)$. 
\citet{nguyen2021tight} consider deep networks with ReLU activations under any scaling and initialization variance. For standard parameterization, they show that $\lntk = \Theta(m + d)$ as long as there exists a layer with width $m = \Omega \prn{n \log n}$. 
\citet{song2021subquadratic} consider shallow networks with standard parameterization and bounded activations. They show a bound of $\Omega(1)$ for $m = \Omega \prn{d^{-1} n^{5/3}}$. 
\citet{bombari2022memorization} consider deep networks where the last layer has initialization variance of $\Theta(1)$ and use data assumptions identical to \citet{nguyen2020global}. Denoting the widths of the two last hidden layers as $m_{L-1}$ and $m_{L-2}$, they show that $\lntk = \Omega \prn{m_{L-1} m_{L-2}}$ whenever $m_{L-1} m_{L-2} = \omega \prn{n \log^8 n}$ and $m_{L-1} = \Omega \prn{n^\gamma}$ for some $\gamma > 0$. Their results hold with failure probability of order $\exp \prn{-\log^2 m}$, which is asymptotically between polynomial and exponential. 
\citet{banerjee2023neural} consider deep networks with NTK parameterization, and show that $\lntk = \Omega(1)$ for $m = \Omega(n \log^2 n)$ with probability that decays like $\frac{1}{m}$. 
\citet{karhadkar2024bounds} consider both shallow and deep networks with NTK parameterization and ReLU activations. They assume the data exhibits $\delta$-separation. For shallow networks, they show that $\lntk = \Omega \prn{d^{-3} \delta^2}$ for $m = \Omega (\norm{\bf{X}}^2 d^3 \delta^{-2} \log n)$. 

None of these results directly apply to our setting. The works of \citet{oymak2020toward,montanari2022interpolation,bombari2022memorization,banerjee2023neural,karhadkar2024bounds} do not match our initialization and output scaling; \citet{nguyen2021tight,karhadkar2024bounds} focus on ReLU activations; \citet{oymak2020toward,bombari2022memorization,banerjee2023neural} do not provide sufficiently strong failure probabilities for our analysis; \citet{song2021subquadratic} use different data scaling and assumptions on the output layer. For this reason we extend the work of \citet{banerjee2023neural} to arbitrary parameterizations and strengthen its probability guarantee.

\subsection{Tabular comparison with prior convergence analyses}
\label{sec:prior_work_comparison}

\Cref{tab:convergence_comparison_gd,tab:convergence_comparison_noise} place our result
among existing convergence analyses for overparameterized networks. The first collects
analyses of gradient descent and its stochastic variant; the second collects analyses that
model training as a diffusion, which is where our result sits.
Several of the works listed are discussed in prose in
\cref{sec:existing_work_langevin_convergence,sec:existing_mean_field_convergence_results};
here we place them, together with further analyses, side by side.
Widths and floors are stated
up to constants and logarithmic factors, under each work's own conventions and omitting
dependence on depth and data separability.
The output scaling $\kappa$ is what places a model in the lazy regime; our analysis
requires no such scaling, holding at $\kappa=1$ under both parameterizations.

Every analysis in \cref{tab:convergence_comparison_gd} contracts the empirical risk to zero
at a rate set by width, output scaling and data separability, with the parameters confined
to a neighborhood of initialization; the rows differ only in the architecture they cover and
the width they demand. The analyses in \cref{tab:convergence_comparison_noise} differ from
one another in more ways, beginning with what their theorems contract: the empirical risk
itself, the population risk, the law of the iterates, or an entropy-regularized free energy.
Every floor is nonetheless reported on the empirical risk, so that the rows remain
comparable; where the theorem contracts something else, the qualification below gives the
conversion. They differ also in what sets the rate. For \citet{lugosi2024convergence},
\citet{oberweis2026convergence} and \citet{chen2020generalized} it is the familiar
combination of width, scaling and separability, and the noise plays no part; for the
remaining rows it is a Poincar\'e or log-Sobolev constant, whose only general estimates
degrade like $e^{-\Theta(\beta)}$, and which \citet{chizat2022meanfield} assumes to do so
outright. Every floor in the table grows as the noise grows, and only
\citet{chen2020generalized} shares the property our argument turns on, a noise level that
can be reduced to shrink the floor while the rate is left untouched. They have it only at
infinite width, and only once the output scaling is large enough to hold the model in its
linear regime. Among the finite-width analyses in either table, ours is therefore the only
one that lets the parameters leave a neighborhood of initialization, treats the injected
noise as the mechanism of the proof rather than as an obstacle, and still contracts the
empirical risk at a rate independent of the noise.

\begin{table}[H]
  \caption[Convergence analyses of gradient descent]{\textbf{Convergence analyses of
  gradient descent and stochastic gradient descent.} In every row the empirical risk
  contracts to zero at a linear rate, the parameters must remain within a neighborhood of
  initialization, and the randomness is either absent or an obstacle the analysis must
  absorb; the rows differ in the architecture they cover and in the width they require.
  \emph{Regime}: the parameterization the analysis is stated in or, where the guarantee is
  driven by an explicit output scaling rather than by the scale of the initialization, the
  scaling regime it requires.
  Here $n$ is the number of samples, $d$ the input dimension, $m$ the hidden-layer width and
  $\kappa$ the output scaling.
  Rows marked $\dagger$ carry qualifications set out after
  \cref{tab:convergence_comparison_noise}.
  }
  \label{tab:convergence_comparison_gd}
  \centering
  \scriptsize
  \setlength{\tabcolsep}{2.6pt}
  \renewcommand{\arraystretch}{1.35}
  \begin{tabular}{@{}
    >{\raggedright\arraybackslash}p{3.10cm}
    >{\raggedright\arraybackslash}p{2.10cm}
    >{\raggedright\arraybackslash}p{3.50cm}
    >{\raggedright\arraybackslash}p{1.00cm}
    >{\raggedright\arraybackslash}p{3.53cm}@{}}
    \toprule
    Work & Depth & Regime & Algo. & Width \\
    \midrule
    \citet{du2019deep}
      & Deep & NTK & GD
      & $m=\tilde\Omega(n^{4})$ \\
    \citet{oymak2020toward}
      & Shallow & $\Theta(1)$ init., fixed last layer & GD
      & $m=\tilde\Omega(n^{2}/d)$ \\
    \citet{song2021subquadratic}
      & Shallow & Standard & GD
      & $m=\tilde\Omega(n^{5/3}/d)$ \\
    \citet{nguyen2020global}
      & Deep, pyramidal & Standard & GD
      & $m_{1}=\tilde\Omega(n^{2})$ in the wide layer \\
    \citet{montanari2022interpolation}\textsuperscript{$\dagger$}
      & Shallow & Large $\kappa$, $\kappa\gtrsim n/\sqrt{md}$ & GF
      & $m=\tilde\Omega(n/d)$ \\
    \addlinespace[0.3em]
    \citet{allen2019convergence}
      & Deep & Standard & SGD
      & $m=\tilde\Omega(n^{24})$ \\
    \citet{zou2019improved}
      & Deep & Standard & SGD
      & $m=\tilde\Omega(n^{17})$
      \\
    \bottomrule
  \end{tabular}
\end{table}

\begin{table}[H]
  \caption[Convergence analyses of noisy training dynamics]{\textbf{Convergence analyses of
  noisy training dynamics.} Rows are grouped by the role the randomness plays in the proof.
  \emph{Regime}: as in \cref{tab:convergence_comparison_gd}.
  \emph{Leaves init.}: whether the parameters may travel far from their initialization
  ($\checkmark$) or must remain close ($\times$).
  \emph{Objective}: the quantity the theorem itself contracts, namely the empirical risk,
  the population risk, the law of the iterates towards a Gibbs measure, or the
  entropy-regularized free energy.
  \emph{Floor}: the non-vanishing residual left in the \emph{empirical} risk; where the
  objective is something else, the qualifications give the conversion.
  \emph{Rate}: \emph{linear} when the objective contracts at an exponential rate set by the
  data and the architecture; \emph{exponentially slow in $\beta$} when that rate is instead
  governed by a Poincar\'e or log-Sobolev constant, for which the only general estimates
  degrade like $e^{-\Theta(\beta)}$.
  Here $N_{p}$ is the number of parameters and $\beta$ the inverse temperature; the
  remaining symbols are as in \cref{tab:convergence_comparison_gd}.
  Rows marked $\dagger$ carry qualifications set out below.
  }
  \label{tab:convergence_comparison_noise}
  \centering
  \scriptsize
  \setlength{\tabcolsep}{2.6pt}
  \renewcommand{\arraystretch}{1.35}
  \begin{tabular}{@{}
    >{\raggedright\arraybackslash}p{2.25cm}
    >{\raggedright\arraybackslash}p{1.05cm}
    >{\raggedright\arraybackslash}p{1.50cm}
    >{\raggedright\arraybackslash}p{0.95cm}
    >{\raggedright\arraybackslash}p{1.55cm}
    >{\centering\arraybackslash}p{0.90cm}
    >{\raggedright\arraybackslash}p{1.35cm}
    >{\raggedright\arraybackslash}p{1.53cm}
    >{\raggedright\arraybackslash}p{1.42cm}@{}}
    \toprule
    Work & Depth & Regime & Algo. & Width & Leaves init. & Objective & Rate & Floor \\
    \midrule
    \multicolumn{9}{@{}l@{}}{\small\itshape Randomness as an obstacle the proof must absorb}\\
    \midrule
    \citet{lugosi2024convergence}\textsuperscript{$\dagger$}
      & Deep & Non-standard init. & Cont.\ SGD
      & Not characterized
      & $\times$ & Population & Linear & None \\
    \citet{oberweis2026convergence}\textsuperscript{$\dagger$}
      & Deep & Large $\kappa$ & Cont.\ SGD
      & Finite; not characterized
      & $\times$ & Empirical & Linear & $\propto\kappa^{-1/q}$ \\
    \citet{chen2020generalized}\textsuperscript{$\dagger$}
      & Shallow & Generalized NTK & CLD
      & $m=\infty$
      & $\checkmark$ & Empirical & Linear & $\propto\kappa^{-2}\beta^{-2}$ \\
    \midrule
    \multicolumn{9}{@{}l@{}}{\small\itshape Randomness as the mechanism of the proof}\\
    \midrule
    \citet{raginsky2017non}\textsuperscript{$\dagger$}
      & Any & Any & SGLD
      & Any
      & $\checkmark$ & Empirical & Exponen\-tially slow in $\beta$ & $\propto N_{p}\beta^{-1}$ \\
    \citet{kumar2025langevin}\textsuperscript{$\dagger$}
      & Shallow & Any & LMC
      & Any
      & $\checkmark$ & Gibbs law & Exponen\-tially slow in $\beta$ & $\propto N_{p}\beta^{-1}$ \\
    \citet{nitanda2022convex}; \citet{chizat2022meanfield}\textsuperscript{$\dagger$}
      & Shallow & Mean-field & MFLD
      & $m=\infty$
      & $\checkmark$ & Free energy & Exponen\-tially slow in $\beta$ & $\propto\beta^{-1}$ \\
    \citet{nitanda2025propagation}\textsuperscript{$\dagger$}
      & Shallow & Mean-field & MFLD
      & Finite $m$
      & $\checkmark$ & Free energy & Exponen\-tially slow in $\beta$ & $\propto\beta^{-1}+m^{-1}$ \\
    \addlinespace[0.45em]
    \textbf{Ours} (\Cref{thm:upper_bound_on_expected_loss})\textsuperscript{$\dagger$}
      & Shallow & Standard & CLD
      & $m=\tilde\Omega(n^{2})$
      & $\checkmark$ & Empirical & Linear & $\propto n\beta^{-1}$ \\
    \textbf{Ours} (\Cref{thm:upper_bound_on_expected_loss})\textsuperscript{$\dagger$}
      & Deep & Standard / NTK & CLD
      & $m=\tilde\Omega(n^{3})$
      & $\checkmark$ & Empirical & Linear & $\propto n^2\beta^{-1}$ \\
    \bottomrule
  \end{tabular}
\end{table}
\pagebreak
\paragraph{Qualifications.}
The rows marked $\dagger$ carry conditions that a table cell cannot hold.
\begin{itemize}
\item \citet{montanari2022interpolation} require the output scaling
  $\kappa\gtrsim n/\sqrt{md}$, which keeps the model in its linear regime, and their inputs
  must be i.i.d.\ uniform on the sphere, with $d\leq n\leq\mathrm{poly}(d)$.

\item \citet{lugosi2024convergence} study the \emph{population} risk, under the diffusion
  driven by the covariance of the stochastic gradients rather than by injected isotropic
  noise. Their general theorem assumes a local gradient-dominance condition on
  $\|\nabla f\|^{2}/f$ together with a noise-curvature condition, and the exponential
  convergence it yields is conditional on a high-probability event that the trajectory never
  leaves a prescribed ball around initialization. For networks these conditions are verified
  only under strong structural assumptions: twice-differentiable activations with everywhere
  positive derivative, bounded inputs with positive-definite population covariance,
  realizability, and a highly constrained initialization in which one weight matrix and all
  biases vanish, the hidden weights are non-negative, and the remaining layer has
  sufficiently large entries. Realizability is what carries their conclusion to the
  empirical risk: a population risk of zero forces the network to match the labels almost
  surely, hence on the sample. No width requirement in terms of $n$ is given.

\item \citet{oberweis2026convergence} keep the kernel non-degenerate throughout training at
  finite width, but only inside the lazy regime induced by the output scaling. Theirs is the
  one floor in the table not governed by the noise: their bound holds for every $q>1$, a
  H\"older conjugate exponent arising in the proof, so the floor is at best of order
  $\kappa^{-1}$ and vanishes only as $\kappa\to\infty$. Their width is never quantified in
  terms of $n$.

\item \citet{chen2020generalized} work at infinite width, and assume that the Gram matrix of
  the initial law satisfies $\lambda_{\min}(\mathbf{H}(p_{0}))=\Lambda>0$. Their parameters
  are free to leave a neighborhood of initialization; what their proof controls is the
  parameter \emph{law}, which is kept close to its initial value in Wasserstein and KL
  divergence. Their output scaling interpolates between the mean-field setting, $\kappa=1$,
  and NTK parameterization, $\kappa=\sqrt m$, and their theorem admits only $\kappa$ above
  an explicit threshold. Identifying their joint weight-decay and noise coefficient with
  $\beta^{-1}$, their Theorem~4.4 reads
  \[
    L(p_{t})\;\leq\;2e^{-2\kappa^{2}\Lambda t/n}
    \;+\;O\!\left(\frac{d^{2}n^{2}}{\beta^{2}\kappa^{2}\Lambda^{2}}\right).
  \]
  At any fixed admissible $\kappa$, lowering the noise shrinks this floor and leaves the
  rate unchanged. Choosing instead the smallest $\kappa$ their proof allows ties the two
  together, and in the large-noise regime the floor that results is of order $\beta^{-1}$,
  matching the other Langevin analyses.

\item \citet{raginsky2017non} reach the empirical risk in two steps: their Proposition~3.3
  bounds the Wasserstein distance between the law of the iterates and the Gibbs measure
  $\pi_{\mathbf{z}}\propto e^{-\beta F_{\mathbf{z}}}$ of the empirical risk, decaying like
  $e^{-k\eta/(\beta c_{\mathrm{LS}})}$, and their Proposition~3.4 bounds that measure's own
  suboptimality,
  \[
    \int F_{\mathbf{z}}\,\mathrm{d}\pi_{\mathbf{z}}-\min_{w}F_{\mathbf{z}}(w)
    \;\leq\;\frac{N_{p}}{2\beta}
      \log\!\left(\frac{eM}{m}\left(\frac{b\beta}{N_{p}}+1\right)\right),
  \]
  which is the floor listed. Driving it to $o(1)$ forces $\beta$ to scale at least with the
  parameter count, so a larger model tightens rather than relaxes the temperature
  requirement, while the log-Sobolev constant $c_{\mathrm{LS}}$ that sets the rate is left
  unbounded and is generally exponential in $\beta$; the two requirements therefore pull
  against each other. Adding a stability term turns the same chain into their headline
  population-risk bound.

\item \citet{kumar2025langevin} prove convergence of the law of the iterates to the Gibbs
  measure $\mu_{s}\propto e^{-2\tilde L/s}$ of a Frobenius-norm-regularized \emph{empirical}
  loss, in total variation and in $q$-R\'enyi divergence, for depth-two networks of
  arbitrary width and arbitrary data. Their Section~4.1 then reaches the risk along exactly
  the route above, invoking the proposition of \citet{raginsky2017non} quoted there for the
  step that makes the Gibbs measure an approximate empirical risk minimizer; the floor
  listed is that step's, and it is sketched rather than stated as a theorem. Their rate is
  governed by the Poincar\'e constant of $\mu_{s}$, which they leave as an open problem to
  bound. The threshold they place on the regularization strength is independent of the
  width, but it leaves a residual bias of order $\lambda\|\mathbf{W}^{\star}\|_{F}^{2}$ in
  the empirical loss, and the noise level is a separate parameter that this
  width-independence does not cover.

\item \citet{nitanda2022convex} and \citet{chizat2022meanfield} contract the free energy
  $F_{\beta}=L+\beta^{-1}\mathrm{Ent}$ to its own minimizer with no residual gap; the latter
  obtains $F_{\beta}(\mu_{t})-\min F_{\beta}\leq e^{-2\rho_{\beta}t/\beta}
  \left(F_{\beta}(\mu_{0})-\min F_{\beta}\right)$, with $\rho_{\beta}$ the log-Sobolev
  constant of the proximal Gibbs measure, which their annealing result assumes to satisfy
  $\rho_{\beta}\geq C_{0}e^{-\alpha^{*}\beta}$. The floor listed is therefore not an
  optimization floor but the bias their minimizer carries in the empirical risk, which the
  entropy sandwich bounds by $\beta^{-1}\mathrm{KL}(\mu^{\circ}\|\mu_{0})$ for any
  near-minimizer $\mu^{\circ}$ of $L$ of finite relative entropy. Removing it requires
  annealing, established only on the schedule $\beta_{t}^{-1}=\alpha/\log t$, along which
  the unregularized objective decays like $\log\log t/\log t$. At finite width
  \citet{nitanda2025propagation} add a uniform-in-time particle term of order $1/m$ which,
  unlike the rate, is independent of the log-Sobolev constant and so does not itself
  deteriorate as the noise is lowered; their discrete-time bound carries a further step-size
  term.

\item Our own floor, $\mathcal{L}_{\infty}$, grows with the noise, but our rate is unaffected
  by it: lowering the floor costs width, through $m=\tilde\Omega(\beta n/\lambda_{1})$ with
  $\lambda_{1}$ the data-separability constant, and not rate. The widths listed above follow
  by taking $\beta$ just large enough to make $\mathcal{L}_{\infty}$ small. Under the
  standard parameterization this needs $\beta=\tilde\Omega(n)$ when shallow and
  $\beta=\tilde\Omega(n^{2})$ when deep, giving $\tilde\Omega(n^{2})$ and
  $\tilde\Omega(n^{3})$; under the NTK parameterization the shallow floor carries the same
  extra factor of $n$ as the deep case, so $\beta=\tilde\Omega(n^{2})$ is required in both
  and both width requirements become $\tilde\Omega(n^{3})$.
\end{itemize}

\newpage

\section{Change-of-measure inequality derivation}
\label{app:divergence_upper_bound_proof}
Throughout this section, we work in a density setting. Namely, all probability measures on the relevant spaces are assumed to admit densities with respect to fixed underlying reference measures, and all Markov transition kernels are assumed to admit corresponding transition densities. We use the same notation for a probability measure and its density. On product spaces, joint densities and marginal densities are taken with respect to the corresponding product and marginal reference measures. All ratios of densities are understood on the sets where the denominator is positive; their values elsewhere are irrelevant for the divergence quantities considered below.

\subsection{Divergence upper bound}
\subsubsection{Second law of thermodynamics}
\begin{definition}[R\'enyi divergence \citep{renyi1961measures, van2014renyi}]
\label{def:renyi_divergence}
Let $p,\,q$ be probability measures on a measurable space $\mathcal{X}$. For $\alpha\in\prn{0,1}\cup\left(1,\infty\right]$, the \emph{R\'enyi divergence} of order $\alpha$ of $p$ from $q$ is defined, when $p\ll q$, by
\[
D_\alpha\prn{p\Vert q}
\triangleq\frac{1}{\alpha-1}\ln\int\prn{\frac{\mathrm{d}p}{\mathrm{d}q}}^{\alpha-1}\mathrm{d}p,\quad
\alpha\in\prn{0,1}\cup\prn{1,\infty};
\quad
D_\infty\prn{p\Vert q}
\triangleq\ln\operatorname*{ess\,sup}_{x\sim p}\frac{\mathrm{d}p}{\mathrm{d}q}\prn{x}.
\]
If $p\not\ll q$, then $D_\alpha\prn{p\Vert q}\triangleq\infty$ for every $\alpha\in\prn{0,1}\cup\left(1,\infty\right]$.
\end{definition}

We begin by showing a data processing inequality for the R\'enyi divergence, similar to Theorem 9 in \citet{van2014renyi}, and provide a proof for completeness and to fit our notation.
\begin{proposition}[Data processing inequality]
\label{prop:data_processing_inequality}
Let $p \prn{x,y}$ and $q \prn{x,y}$ denote the joint density functions of two
probability distributions defined on the product space
$\mathcal{X} \times \mathcal{Y}$. Let $p_X \prn{x}$ and $q_X\prn{x}$ be the
corresponding marginal densities; for example,
\[
p_X \prn{x} = \int_{\mathcal{Y}} p \prn{x,y} \mathrm{d}y .
\]
We further denote by $p \prn{y \mid x}$ and $q\prn{y \mid x}$ the conditional
densities, so that the joint densities factorize as
$p \prn{x,y} = p\prn{y \mid x} p_X\prn{x}$, and similarly for $q$.
Then for all $\alpha > 1$,
\[
D_{\alpha}\left(p_{X}\Vert q_{X}\right)\leq D_{\alpha}\left(p\Vert q\right)\,.
\]
\end{proposition}
\begin{proof}
    Directly from \cref{def:renyi_divergence}, we have
\allowdisplaybreaks{
\begin{align*}
D_{\alpha}\left(p\Vert q\right) & =\frac{1}{\alpha-1}\ln\left(\int_{\mathcal{X}\times\mathcal{Y}}\left(\frac{p\left(x,y\right)}{q\left(x,y\right)}\right)^{\alpha-1}\mathrm{d}p\left(x,y\right)\right)\\
 & =\frac{1}{\alpha-1}\ln\left(\int_{\mathcal{X}\times\mathcal{Y}}\left(\frac{p\left(x,y\right)}{q\left(x,y\right)}\right)^{\alpha-1}p\left(x,y\right)\mathrm{d}x\mathrm{d}y\right)\\
 & =\frac{1}{\alpha-1}\ln\left(\int_{\mathcal{X}\times\mathcal{Y}}\left(\frac{p\left(y\mid x\right)p_{X}\left(x\right)}{q\left(y\mid x\right)q_{X}\left(x\right)}\right)^{\alpha-1}p\left(y\mid x\right)p_{X}\left(x\right)\mathrm{d}x\mathrm{d}y\right)\\
 & =\frac{1}{\alpha-1}\ln\left(\int_{\mathcal{X}\times\mathcal{Y}}\left(\frac{p_{X}\left(x\right)}{q_{X}\left(x\right)}\right)^{\alpha-1}p_{X}\left(x\right)\left(\frac{p\left(y\mid x\right)}{q\left(y\mid x\right)}\right)^{\alpha-1}p\left(y\mid x\right)\mathrm{d}x\mathrm{d}y\right)\\
\left[\text{Fubini-Tonelli}\right] & =\frac{1}{\alpha-1}\ln\left(\int_{\mathcal{X}}\left(\frac{p_{X}\left(x\right)}{q_{X}\left(x\right)}\right)^{\alpha-1}p_{X}\left(x\right)\left(\int_{\mathcal{Y}}\left(\frac{p\left(y\mid x\right)}{q\left(y\mid x\right)}\right)^{\alpha-1}p\left(y\mid x\right)\mathrm{d}y\right)\mathrm{d}x\right)\\
 & =\frac{1}{\alpha-1}\ln\,\Biggl(\int_{\mathcal{X}}\left(\frac{p_{X}\left(x\right)}{q_{X}\left(x\right)}\right)^{\alpha-1}p_{X}\left(x\right) \cdot \\
 & \qquad \qquad \qquad  \cdot \exp\left(\prn{\alpha - 1} \frac{1}{\alpha-1}\ln\left(\int_{\mathcal{Y}}\left(\frac{p\left(y\mid x\right)}{q\left(y\mid x\right)}\right)^{\alpha-1}\mathrm{d}p\left(y\mid x\right)\right)\right)\mathrm{d}x\Biggr)\\
  & =\frac{1}{\alpha-1}\ln\!\left(\int_{\mathcal{X}}\left(\frac{p_{X}\left(x\right)}{q_{X}\left(x\right)}\right)^{\alpha-1}\! p_{X}\left(x\right)\exp\Bigl(\prn{\alpha - 1} 
  D_\alpha \bigprn{p\prn{\cdot\mid x}\Vert q\prn{\cdot\mid x}}
  \Bigr)\mathrm{d}x\right)\\
\left[D_{\alpha}\geq0,\,*\right] & \geq\frac{1}{\alpha-1}\ln\left(\int_{\mathcal{X}}\left(\frac{p_{X}\left(x\right)}{q_{X}\left(x\right)}\right)^{\alpha-1}p_{X}\left(x\right)\exp\left(0\right)\mathrm{d}x\right)\\
 & =\frac{1}{\alpha-1}\ln\left(\int_{\mathcal{X}}\left(\frac{p_{X}\left(x\right)}{q_{X}\left(x\right)}\right)^{\alpha-1}\mathrm{d}p_{X}\left(x\right)\right)=D_{\alpha}\left(p_{X}\Vert q_{X}\right)\,,
\end{align*}
where the non-negativity of the R\'enyi divergence (*) was shown in Theorem 8 of \citet{van2014renyi}, and the sign is maintained since $\alpha > 1$.
}
\end{proof}

From here, the proof of a form of the second law of thermodynamics for the R\'enyi divergence in time-invariant Markov processes follows along the lines of the proof of Lemma A.2 in \citet{harel2025temperature} for the second law of thermodynamics for the KL divergence of discrete-time Markov processes, and then their Corollary A.4, which extends the result to continuous-time Markov processes.
\begin{proposition}[Second law of thermodynamics for discrete-time Markov processes]
    \label{prop:second_law_discrete_time}
    Given a time-invariant Markov process on state space $\mathcal{X}$ with marginal distributions $p_n,\, n\in \mathbb{N}_0$, with a stationary distribution $p_\infty$, we have for all $n\in \mathbb{N}_0$ and $\alpha>1$,
    \begin{align*}
        D_\alpha \prn{p_{n+1} \Vert p_\infty} \leq D_\alpha \prn{p_{n} \Vert p_\infty} \,.
    \end{align*}
\end{proposition}

\begin{proof}
Fix any $n\in\mathbb{N}_0$. If \(p_n\not\ll p_\infty\), then
\[
D_\alpha\left(p_n\Vert p_\infty\right)=\infty,
\]
and the claim is trivial. Hence assume \(p_n\ll p_\infty\).

    Let $K\prn{\cdot \mid x}$ denote the transition kernel of the Markov process, and define joint densities on $\mathcal{X} \times \mathcal{X}$,
    \begin{align*}
        p \prn{x,y} = p_n \prn{x} K\prn{y \mid x},\qquad q \prn{x,y} = p_\infty \prn{x} K\prn{y \mid x}\,.
    \end{align*}
    The $Y$-marginal distributions are
        \begin{align*}
        p_Y \prn{y} &= \int p \prn{x,y} \mathrm{d} x =  \int p_n \prn{x} K\prn{y \mid x} \mathrm{d} x = p_{n+1} \prn{y}\,,\\
        q_Y \prn{y} &=\int p_\infty \prn{x} K\prn{y \mid x} \mathrm{d} x =  p_\infty \prn{y}\,.
    \end{align*}
    The identity \(p_Y=p_{n+1}\) follows from the Markov evolution, while \(q_Y=p_\infty\) follows from the stationarity of \(p_\infty\).
    Applying the data processing inequality, we get
    \begin{align*}
        D_\alpha \prn{p_{n+1} \Vert p_\infty} \leq D_\alpha \prn{p \Vert q} \,.
    \end{align*}
    From \cref{def:renyi_divergence} we have,
    \begin{align*}
        D_\alpha \prn{p \Vert q} 
        &=\frac{1}{\alpha-1}\ln\int\prn{\frac{p \prn{x,y}}{q \prn{x,y}}}^{\prn{\alpha-1}} \mathrm{d}p\prn{x,y} \\
        &=\frac{1}{\alpha-1}\ln\iint\prn{\frac{p_n \prn{x} K\prn{y \mid x}}{p_\infty \prn{x} K\prn{y \mid x}}}^{\prn{\alpha-1}} p\prn{x,y} \mathrm{d}x\mathrm{d}y \\
        &= \frac{1}{\alpha-1}\ln\iint\prn{\frac{p_n \prn{x}}{p_\infty \prn{x} }}^{\prn{\alpha-1}}  p_n \prn{x} K\prn{y \mid x} \mathrm{d}x\mathrm{d}y\\
        \sqprn{\text{Fubini-Tonelli}} &= \frac{1}{\alpha-1}\ln \prn{\int\prn{\frac{p_n \prn{x}}{p_\infty \prn{x} }}^{\prn{\alpha-1}}  p_n \prn{x} \prn{\int K\prn{y \mid x} \mathrm{d}y} \mathrm{d}x} \\
        &=\frac{1}{\alpha-1}\ln\int\prn{\frac{p_n \prn{x}}{p_\infty \prn{x} }}^{\prn{\alpha-1}}   \mathrm{d}p_n \prn{x}\\
        &= D_\alpha \prn{p_n \Vert p_\infty }\,.
    \end{align*}
    Note that the value of the ratio on the set where \(q\prn{x,y}=0\) is irrelevant, since the R\'enyi divergence only depends on the ratio \(q\)-a.e.
    Substituting this into the previous inequality, we get the desired result.
\end{proof}

\begin{corollary}[Second law of thermodynamics for continuous-time Markov processes]
    \label{cor:second_law_continuous_time}
Given a time-invariant Markov process on state space $\mathcal{X}$ with marginal distributions $p_t,\, t\geq 0$, with a stationary distribution $p_\infty$, we have for all $t \geq 0$ and $\alpha >1$,
    \begin{align*}
        D_\alpha \prn{p_{t} \Vert p_\infty} \leq D_\alpha \prn{p_{0} \Vert p_\infty} \,.
    \end{align*}
\end{corollary}
\begin{proof}
The case $t=0$ is immediate.
    By noting that for any $t> 0$, the process which has the marginal distributions $\tilde{p}_n = p_{nt},\,\forall n \in \mathbb{N}_0$ is a Markov process with the stationary distribution $\tilde{p}_\infty = p_\infty$, the result follows for $n=0$ in \cref{prop:second_law_discrete_time}.
\end{proof}

\subsubsection{Divergence bound derivation}
We now define a ``relative'' version of the R\'enyi divergence, similarly to \citet{harel2025temperature}.
\begin{definition}[Relative R\'enyi divergence]
\label{def:relative_renyi_divergence}
Let $p,\,q, \mu$ be probability measures on a measurable space $\mathcal{X}$ such that $\mu \ll q,\, \mu \ll p$. For $\alpha \in \prn{0,1} \cup \prn{1,\infty}$, the \textit{relative R\'enyi divergence} of order $\alpha$ of $p$ from $q$ relative to $\mu$ is defined as
\begin{align*}
    D_\alpha^\mu \prn{p \Vert q} \triangleq \frac{1}{\alpha - 1} \ln \int \prn{\frac{\mathrm{d}p}{\mathrm{d}q}} ^{\alpha - 1} \mathrm{d}\mu
\end{align*}
when $p \ll q$, and $D_\alpha^\mu \prn{p \Vert q} =\infty$ otherwise.\footnote{
The condition $p\ll q$ is imposed so that the Radon-Nikodym derivative $\frac{\mathrm{d}p}{\mathrm{d}q}$ exists. Under the standing assumption $\mu\ll q$, the integral depends only on the $q$-a.e. behavior of the integrand, so that when $p\not\ll q$ one could still naturally interpret the expression through the $q$-absolutely continuous part of $p$. The assignment $D_\alpha^\mu\prn{p\Vert q}=\infty$ in this case is therefore a definitional choice made to retain a globally stated formula in terms of $\frac{\mathrm{d}p}{\mathrm{d}q}$. The additional assumption $\mu\ll p$ is likewise not needed for this interpretation, but is convenient for $\alpha\in\prn{0,1}$, since it ensures that vanishing of the density does not produce blow-up on a set of positive $\mu$-measure.
}
The \textit{R\'enyi divergence} of $p$ from $q$ (\cref{def:renyi_divergence}) is then $D_\alpha \prn{p \Vert q} =D_\alpha^p \prn{p \Vert q} $.
\end{definition}
We proceed by showing a R\'enyi-divergence analogue of Claim 2.3 of \citet{harel2025temperature}.
\begin{claim}
\label{claim:renyi_to_potential}
Let $p,q,\mu_1,\mu_2$ be probability measures such that $p$ is Gibbs w.r.t. $q$ with potential $\Psi$ (\cref{def:gibbs_distribution}),
$p \sim  q$, and $\mu_1,\mu_2 \ll p,q$.
Let $\alpha \in \prn{0,1} \cup \prn{1,\infty}$, and assuming integrability, \ie
\[
\mathbb{E}_{\mu_1}\left[e^{-\left(\alpha-1\right)\Psi}\right]<\infty\,,
\quad
\mathbb{E}_{\mu_2}\left[e^{\left(\alpha-1\right)\Psi}\right]<\infty \,.
\]
Then,
\[
D_{\alpha}^{\mu_1}\left(p\Vert q\right)+D_{\alpha}^{\mu_2}\left(q\Vert p\right)=\frac{1}{\alpha-1}\left[\ln\mathbb{E}_{\mu_1}\left[e^{-\left(\alpha-1\right)\Psi}\right]+\ln\mathbb{E}_{\mu_2}\left[e^{\left(\alpha-1\right)\Psi}\right]\right]\,.
\]
\end{claim}
\begin{proof}
     By \cref{def:gibbs_distribution}, $\frac{\mathrm{d}p}{\mathrm{d}q}=Z^{-1}e^{-\Psi}$,
where $Z<\infty$ is the partition function (a deterministic constant).
Then
\begin{align*}
& D_{\alpha}^{\mu_1}\left(p\Vert q\right)+D_{\alpha}^{\mu_2}\left(q\Vert p\right) \\
& =\frac{1}{\alpha-1}\ln\left(\int\left(\frac{\mathrm{d}p}{\mathrm{d}q}\right)^{\alpha-1}\mathrm{d}\mu_1\right)+\frac{1}{\alpha-1}\ln\left(\int\left(\frac{\mathrm{d}q}{\mathrm{d}p}\right)^{\alpha-1}\mathrm{d}\mu_2\right)\\
 & =\frac{1}{\alpha-1}\left[\ln\left(\int Z^{-\left(\alpha-1\right)}e^{-\left(\alpha-1\right)\Psi}\mathrm{d}\mu_1\right)+\ln\left(\int Z^{\left(\alpha-1\right)}e^{\left(\alpha-1\right)\Psi}\mathrm{d}\mu_2\right)\right]\\
 & =\frac{1}{\alpha-1}\left[\ln\left(Z^{-\left(\alpha-1\right)}\right)+\ln\left(\int e^{-\left(\alpha-1\right)\Psi}\mathrm{d}\mu_1\right)+\ln\left(Z^{\left(\alpha-1\right)}\right)+\ln\left(\int e^{\left(\alpha-1\right)\Psi}\mathrm{d}\mu_2\right)\right]\\
 & =\frac{1}{\alpha-1}\left[\ln\left(\int e^{-\left(\alpha-1\right)\Psi}\mathrm{d}\mu_1\right)+\ln\left(\int e^{\left(\alpha-1\right)\Psi}\mathrm{d}\mu_2\right)\right]\\
 & =\frac{1}{\alpha-1}\left[\ln\mathbb{E}_{\mu_1}\left[e^{-\left(\alpha-1\right)\Psi}\right]+\ln\mathbb{E}_{\mu_2}\left[e^{\left(\alpha-1\right)\Psi}\right]\right]\,.
\end{align*}
\end{proof}
We now obtain an upper bound on the R\'enyi divergence of $p_t$ to $\nu$ at all times via the second law of thermodynamics,
extending Corollary 2.5 of \citet{harel2025temperature}.
\begin{lemma}[Divergence Bound]
\label{lem:divergence_bound}
Consider a time-invariant Markov process with marginal distributions $p_t,\, t\geq 0$, with a stationary distribution $p_\infty$ which is Gibbs w.r.t. a distribution $\nu$ with potential $\Psi$, such that $p_\infty \sim \nu$, and $\forall t \geq 0$, $p_t \ll \nu, p_\infty$.
Then, we have at any time $t \geq 0$ and for any $\alpha>1$,
\begin{align*}
     D_{\alpha}\left(p_{t}\Vert\nu\right) & \leq D_{\infty}\left(p_{0}\Vert\nu\right)
    +\frac{1}{2 \prn{\alpha-1}}\ln\mathbb{E}_{p_{0}}\left[e^{2\left(\alpha-1\right)\Psi}\right]
    +\frac{1}{2 \prn{\alpha-1}}\ln\mathbb{E}_{p_{t}}\left[e^{-2\left(\alpha-1\right)\Psi}\right]\,.
\end{align*}
\end{lemma}
\begin{proof}
Applying the chain rule,
\begin{align*}
D_{\alpha}\left(p_{t}\Vert\nu\right) & =\frac{1}{\alpha-1}\ln\left(\int\left(\frac{\mathrm{d}p_{t}}{\mathrm{d}\nu}\right)^{\alpha-1}\mathrm{d}p_{t}\right) \\
&=\frac{1}{\alpha-1}\ln\left(\int\left(\frac{\mathrm{d}p_{t}}{\mathrm{d}p_{\infty}}\frac{\mathrm{d}p_{\infty}}{\mathrm{d}\nu}\right)^{\alpha-1}\mathrm{d}p_{t}\right)\\
\left[\text{H\"older: }r=s=2\right] & \leq\frac{1}{\alpha-1}\ln\Biggl(\left(\int\left(\frac{\mathrm{d}p_{t}}{\mathrm{d}p_{\infty}}\right)^{2\left(\alpha-1\right)}\mathrm{d}p_{t}\right)^{1/2} \cdot\\
& \qquad \qquad \qquad \qquad \cdot \left(\int\left(\frac{\mathrm{d}p_{\infty}}{\mathrm{d}\nu}\right)^{2\left(\alpha-1\right)}\mathrm{d}p_{t}\right)^{1/2}\Biggr)\\
 & =\frac{1}{2\left(\alpha-1\right)}\ln\left(\int\left(\frac{\mathrm{d}p_{t}}{\mathrm{d}p_{\infty}}\right)^{2\left(\alpha-1\right)}\mathrm{d}p_{t}\right) \\
 & \qquad \qquad + \frac{1}{2\left(\alpha-1\right)}\ln\left(\int\left(\frac{\mathrm{d}p_{\infty}}{\mathrm{d}\nu}\right)^{2\left(\alpha-1\right)}\mathrm{d}p_{t}\right)\\
 & =D_{2\alpha-1}\left(p_{t}\Vert p_{\infty}\right)+D_{2\alpha-1}^{p_{t}}\left(p_{\infty}\Vert\nu\right)\\
\left[\text{\cref{cor:second_law_continuous_time}}\right] & \leq D_{2\alpha-1}\left(p_{0}\Vert p_{\infty}\right)+D_{2\alpha-1}^{p_{t}}\left(p_{\infty}\Vert\nu\right)\\
 \left[\text{Chain rule}\right] & =\frac{1}{2\alpha-2}\ln\left(\int\left(\frac{\mathrm{d}p_{0}}{\mathrm{d}\nu}\frac{\mathrm{d}\nu}{\mathrm{d}p_{\infty}}\right)^{2\alpha-2}\mathrm{d}p_{0}\right)+D_{2\alpha-1}^{p_{t}}\left(p_{\infty}\Vert\nu\right)\\
\left[\text{H\"older: }r=\infty,\,s=1\right] & \leq\frac{1}{2\alpha-2}\ln\left(
\operatorname*{ess\,sup}_{p_0}
\left(\frac{\mathrm{d}p_{0}}{\mathrm{d}\nu}\right)^{2\alpha-2}
\int\left(\frac{\mathrm{d}\nu}{\mathrm{d}p_{\infty}}\right)^{2\alpha-2}\mathrm{d}p_{0}\right) \\
& \qquad \qquad + D_{2\alpha-1}^{p_{t}}\left(p_{\infty}\Vert\nu\right)\\
 & =D_{\infty}\left(p_{0}\Vert\nu\right)+D_{2\alpha-1}^{p_{0}}\left(\nu\Vert p_{\infty}\right)+D_{2\alpha-1}^{p_{t}}\left(p_{\infty}\Vert\nu\right)\\
\left[\text{\cref{claim:renyi_to_potential}}\right] & = D_{\infty}\left(p_{0}\Vert\nu\right)+\frac{1}{2\alpha-2}\left(\ln\mathbb{E}_{p_{t}}\left[e^{-\left(2\alpha-2\right)\Psi}\right]+\ln\mathbb{E}_{p_{0}}\left[e^{\left(2\alpha-2\right)\Psi}\right]\right)\,,
\end{align*}
where the integrability requirements of \cref{claim:renyi_to_potential} are omitted since, when they do not hold, the RHS is infinite and the inequality is trivial.
\end{proof}
\begin{remark}[Avoiding R\'enyi-infinity divergence]
\label{rem:avoid_d_infinity}
 Using $r=\frac{3}{2},\,s=3$ in the first H\"older's inequality application and $r=s=2$ in the second application, we avoid $D_{\infty}\left(p_{0}\Vert\nu\right)$ and achieve 
\begin{align*}
D_{\alpha}\left(p_{t}\Vert\nu\right)  \leq D_{3 \alpha - 2}\left(p_{0}\Vert\nu\right)
    +\frac{1}{3 \prn{\alpha-1}}\prn{\ln\mathbb{E}_{p_{0}}\left[e^{3\left(\alpha-1\right)\Psi}\right] 
    +\ln\mathbb{E}_{p_{t}}\left[e^{-3\left(\alpha-1\right)\Psi}\right]}\,,
\end{align*}
which may be more useful when one does not assume $\nu=p_0$.
\end{remark}

\subsection{R\'enyi change-of-measure}
\label{app:change_of_measure}
We now cite Theorem 8 from \citet{begin2016pac}, stating a change of measure inequality based on the R\'enyi divergence, and provide the proof for completeness.
\begin{lemma}[R\'enyi change-of-measure \citep{begin2016pac}]
\label{lem:begin16}
Let $p,\,q$ be probability measures on a measurable
space $\mathcal{X}$ such that $p\ll q$. For any \emph{non-negative} measurable
function $\phi:\mathcal{X}\to\mathbb{R}$ and for $\alpha>1$, we have
\[
\frac{\alpha}{\alpha-1}\ln\mathbb{E}_{p}\left[\phi\left(x\right)\right]\leq D_{\alpha}\left(p\Vert q\right)+\ln\mathbb{E}_{q}\left[\phi\left(x\right)^{\frac{\alpha}{\alpha-1}}\right]\,.
\]
\end{lemma}

\begin{proof}
We rewrite the expectation over \(p\) as an expectation over \(q\), which is valid under \(p\ll q\), and apply H\"older's inequality with \(r=\alpha\) and \(s=\frac{\alpha}{\alpha-1}\), so that \(\frac{1}{r}+\frac{1}{s}=1\):
\begin{align*}
\frac{\alpha}{\alpha-1}\ln\mathbb{E}_{p}\left[\phi\left(x\right)\right] 
&=\frac{\alpha}{\alpha-1}\ln\mathbb{E}_{q}\left[\frac{\mathrm{d}p}{\mathrm{d}q}\phi\left(x\right)\right]\\
&\leq\frac{\alpha}{\alpha-1}\ln\left(\left(\mathbb{E}_{q}\left[\left(\frac{\mathrm{d}p}{\mathrm{d}q}\right)^{\alpha}\right]\right)^{\frac{1}{\alpha}}\left(\mathbb{E}_{q}\left[\phi\left(x\right)^{\frac{\alpha}{\alpha-1}}\right]\right)^{\frac{\alpha-1}{\alpha}}\right)\\
&=\frac{1}{\alpha-1}\ln\left(\mathbb{E}_{q}\left[\left(\frac{\mathrm{d}p}{\mathrm{d}q}\right)^{\alpha}\right]\right)+\ln\left(\mathbb{E}_{q}\left[\phi\left(x\right)^{\frac{\alpha}{\alpha-1}}\right]\right)\\
&=D_{\alpha}\left(p\Vert q\right)+\ln\left(\mathbb{E}_{q}\left[\phi\left(x\right)^{\frac{\alpha}{\alpha-1}}\right]\right)\,,
\end{align*}
where in the last equality we used the definition of the R\'enyi divergence and rewrote the integral with respect to \(q\). Indeed, since \(p\ll q\), the Radon--Nikodym derivative \(\frac{\mathrm{d}p}{\mathrm{d}q}\) exists, and
\begin{align*}
D_{\alpha}\left(p\Vert q\right)
&\triangleq
\frac{1}{\alpha-1}\ln\left(\int\left(\frac{\mathrm{d}p}{\mathrm{d}q}\right)^{\alpha-1}\mathrm{d}p\right)\\
&=
\frac{1}{\alpha-1}\ln\left(\int\left(\frac{\mathrm{d}p}{\mathrm{d}q}\right)^{\alpha-1}\frac{\mathrm{d}p}{\mathrm{d}q}\mathrm{d}q\right)\\
&=
\frac{1}{\alpha-1}\ln\left(\mathbb{E}_{q}\left[\left(\frac{\mathrm{d}p}{\mathrm{d}q}\right)^{\alpha}\right]\right)\,.
\end{align*}
\end{proof}

\subsection{Conclusion}
\begin{recall}[\cref{thm:second_law_change_of_measure}][Second law of thermodynamics change of measure inequality]
Consider a time invariant Markov process over the state space $\mathcal{X}$ whose marginal distribution is $p_t,\, t\geq 0$ such that $p_t \ll p_0$, with a stationary distribution $p_\infty$ which is Gibbs w.r.t.\! $p_0$ with potential $\Psi$,
and $p_\infty \sim p_0$.
Then, for any \emph{non-negative} measurable function $\phi:\mathcal{X}\to\mathbb{R}$ and $\alpha>1$, we have at any time $t \geq 0$,
\begin{align*}
    \mathbb{E}_{p_{t}} \phi\left(x\right) & \leq \prn{\mathbb{E}_{p_{0}}\left[e^{2\left(\alpha-1\right)\Psi}\right]}^{\frac{1}{2\alpha}}\,
    \prn{\mathbb{E}_{ p_{0}}\left[\phi\left(x\right)^{\frac{\alpha}{\alpha-1}}\right]}^{\frac{\alpha-1}{\alpha}}.
\end{align*}
\end{recall}
\begin{proof}
We simply plug \cref{lem:divergence_bound} with $\nu=p_0$ into \cref{lem:begin16} with $p=p_t,\,q=p_0$ to achieve the desired result, omitting the final non-positive term in \cref{lem:divergence_bound}.%
\end{proof}

\begin{remark}[Absolute continuity requirements]
\label{rem:absolute_continuity_under_standard_transitions}
When imposing $\nu=p_0=\mathcal{N}\prn{\mathbf{0}_{N_p},\,\boldsymbol{\Lambda}^{-1}}$, as in \cref{cor:second_law_for_langevin}, we have that $\nu$ has strictly positive Lebesgue density on $\real^{N_p}$. In addition, for the MSE loss (\cref{eq:mse_loss_def}) we have $\Psi=\beta\mathcal{L}\prn{\boldsymbol{\theta}}<\infty$ Lebesgue-a.e. on $\real^{N_p}$. Hence,
the Gibbs distribution
\[
p_\infty \propto e^{-\Psi}\nu
\]
also has strictly positive Lebesgue density on $\real^{N_p}$, and therefore $p_\infty\sim \nu$.

By \cref{prop:cld_wellposed}, the process is global, $p_\infty$ is stationary, and $p_t\ll p_\infty\sim\nu$ for every $t\geq0$. Hence the absolute-continuity requirements of \cref{thm:second_law_change_of_measure} hold directly.
\end{remark}

\newpage

\section{Assumption table}
\label{app:assumption_table}

\begin{table}[H]
\centering
\small
\setlength{\tabcolsep}{4pt}
\renewcommand{\arraystretch}{1.2}
\caption{Assumptions grouped by their role in the analysis, with where each enters
and whether it can be relaxed.}
\label{tab:assumption_table}
\begin{tabularx}{\linewidth}{@{}>{\raggedright\arraybackslash}p{0.25\linewidth} >{\raggedright\arraybackslash}X >{\raggedright\arraybackslash}p{0.26\linewidth}@{}}
\toprule
\textbf{Assumption} & \textbf{Role in the analysis} & \textbf{Relaxation}\\
\midrule
\multicolumn{3}{@{}l}{\emph{Essential to the mechanism}}\\
\midrule
\addlinespace[2pt]
Gibbs stationary law: the dynamics admit a stationary $p_\infty$ that is Gibbs w.r.t.\ a reference $\nu$ &
Structural requirement enabling the change-of-measure inequality (\cref{thm:second_law_change_of_measure}); holds for regularized CLD with any $\mathbf{\Lambda}\succ\mathbf{0}$ &
None; intrinsic to the mechanism\\
\midrule
\multicolumn{3}{@{}l}{\emph{Control the terms of the change-of-measure inequality}}\\
\midrule
\addlinespace[2pt]
Reference matched to initialization, $\nu=p_0$, by tying $\mathbf{\Lambda}$ to the initialization (\cref{eq:Lambda_def}) &
Eliminates the $D_\infty(p_0\Vert\nu)$ term from the right-hand side (see \cref{lem:divergence_bound}, \cref{rem:avoid_d_infinity}) &
Yes, at the cost of a multiplicative factor in \cref{thm:second_law_change_of_measure} and \cref{cor:second_law_for_langevin}\\
\addlinespace[2pt]
Bounded activation in the last hidden layer (\cref{ass:activation_bounded}) &
Yields a subexponential loss at initialization, bounding the exponential-moment prefactor in \cref{cor:second_law_for_langevin} &
Yes, via output normalization or soft clipping\\
\midrule
\multicolumn{3}{@{}l}{\emph{Artifacts of the analysis}}\\
\midrule
\addlinespace[2pt]
$C^2$ activation with bounded derivatives (\cref{ass:activation_bounded_gradients}) &
$C^2$ simplifies the It\^o analysis; the bounds control the worst-case terms in \cref{eq:the_four_terms} &
Yes, to $C^1$ with bounded, Lipschitz derivative (\cref{footnote:C1_activations})\\
\addlinespace[2pt]
Activation vanishing at origin (\cref{ass:activation_at_origin}) &
Normalizes the network at initialization &
Technical\\
\addlinespace[2pt]
Constant depth $L$ &
Suppresses $e^{O(L)}$ constants in the NTK bound &
Technical\\
\addlinespace[2pt]
Scalar output, no biases &
Standard simplification &
Technical\\
\addlinespace[2pt]
Normalized data, bounded labels (\cref{ass:normalized_data,ass:bounded_labels}) &
Simplifies constants &
By preprocessing\\
\midrule
\multicolumn{3}{@{}l}{\emph{Necessary for any NTK-based argument}}\\
\midrule
\addlinespace[2pt]
Data separability, $\lambda_1>0$ (\cref{ass:data_separability}) &
Keeps the descent term in \cref{eq:example_of_need_to_use_second_law} non-vacuous; a standard non-degeneracy condition. Constrains only the inputs (labels arbitrary); holds w.h.p.\ for i.i.d.\ data on the sphere and for separable data &
None; inherent to NTK analyses\\
\bottomrule
\end{tabularx}
\end{table}

\newpage

\section{Lower bound on the empirical NTK at initialization}
\label{app:ntk_min_eigenvalue}
\subsection{Setting: deep fully connected network}
We follow most of the notation from \cref{sec:setting_deep_nn},
with a feed-forward $L+1$ layer network ($L$ hidden layers), given
by
\begin{align*}
\boldsymbol{\alpha}^{\left(0\right)}\left(\mathbf{x}\right) & =\mathbf{x}\,,\\
\boldsymbol{\alpha}^{\left(l\right)}\left(\mathbf{x}\right) & =\phi\left(s_{l}\mathbf{W}^{\left(l\right)}\boldsymbol{\alpha}^{\left(l-1\right)}\left(\mathbf{x}\right)\right),\quad l\in\left[L\right]\,,\\
f\left(\theta;\mathbf{x}\right) & ={\alpha}^{\left(L+1\right)}\left(\mathbf{x}\right)=s_{L+1}\mathbf{v}^{\top}\boldsymbol{\alpha}^{\left(L\right)}\left(\mathbf{x}\right)\,,
\end{align*}
where $\mathbf{W}^{\left(l\right)}\in\mathbb{R}^{m_{l}\times m_{l-1}},\,l\in\left[L\right]$
are layer-wise weight matrices, $\mathbf{v}\in\mathbb{R}^{m_{L}}$
is the last layer vector, $\phi\left(\cdot\right)$ is the activation
function, $s_{l},\,l\in\left[L+1\right]$ are predetermined positive
constants, and the total set of parameters is denoted by the weight
vector 
\begin{align*}
\theta & \triangleq\left(\mathrm{vec}\left(\mathbf{W}^{\left(1\right)}\right)^{\top},\ldots,\mathrm{vec}\left(\mathbf{W}^{\left(L\right)}\right)^{\top},\mathbf{v}^{\top}\right)^{\top}\in\mathbb{R}^{N_{p}}\,,
\end{align*}
with $m_{0}\triangleq d$ and $N_{p}\triangleq\sum_{l=1}^{L}m_{l}m_{l-1}+m_{L}$.
We assume $L=O\left(1\right)$ to simplify the analysis.

Recall that $c_{\phi,\sigma_{0}}\triangleq\mathbb{E}_{z\sim\mathcal{N}\left(0,\sigma_{0}^{2}\right)}\left[\phi^{2}\left(z\right)\right],\,\nu_{0}^{2}=\frac{\sigma_{0}^{2}}{c_{\phi,\sigma_{0}}},\,\left\Vert \mathbf{x}_{i}\right\Vert _{2}^{2}=c_{\phi,\sigma_{0}}d,\,i\in\left[n\right]$
and assume that each $\mathbf{W}^{\left(l\right)},\,l\in\left[L\right]$
is initialized as 
\[
\mathrm{vec}\left(\mathbf{W}^{\left(l\right)}\right)\sim\mathcal{N}\left(\mathbf{0}_{m_{l}m_{l-1}},\,\frac{\nu_{0}^{2}}{a_{l}}\mathbf{I}_{m_{l}m_{l-1}}\right)\,,
\]
such that $\forall l\in\left[L\right],\,\frac{a_{l}}{s_{l}^{2}}=m_{l-1}$.
The results in this section require no assumption on the initialization of $\mathbf{v}$,
only that it is trainable.
Throughout this section, we assume \(\phi\prn{0}=0\) (\cref{ass:activation_at_origin}) and that \(\abs{\phi'} \leq \alpha_\phi\). In addition, we assume the data separability condition \(\lambda_1>0\) (\cref{ass:data_separability}).\footnote{
The condition \(\lambda_1>0\) hides a subtle property: it excludes cases where the activation and data induce a degenerate random-feature Gram matrix. For example, with a linear activation, \(\lambda_1\) reduces to a data-rank condition and vanishes whenever the normalized data Gram is singular. More generally, finite-degree polynomial activations can only generate finite-dimensional polynomial feature spaces, so a positive lower bound for arbitrary datasets requires a non-polynomial activation.
}
In this section, we do not require the activations to be bounded or their derivative to be Lipschitz continuous.

Let $\mathbf{A}^{\left(l\right)}\in\mathbb{R}^{n\times m_{l}}$ be
the matrix whose rows are $\mathbf{A}_{i,:}^{\left(l\right)}\triangleq\boldsymbol{\alpha}^{\left(l\right)}\left(\mathbf{x}_{i}\right)^{\top},\,l\in\left[L\right],\,i\in\left[n\right]$.
We further denote $\mathbf{w}_{j}^{\left(l\right)}\in\mathbb{R}^{m_{l-1}}$
as the transpose of the $j$-th row of $\mathbf{W}^{\left(l\right)},\,l\in\left[L\right]$, and $m\triangleq\min_{l\in\left[L\right]}m_{l}$.

The analysis in this section closely follows that of \citet{banerjee2023neural}, with adjustments used to tighten the probability of their bound, apply it to our generalized parameterization, and fix a minor error we found in their proof.

\subsection{Analysis}
We begin with a supporting proposition:
\begin{proposition}
\label{prop:sub-gaussian-outputs-per-layer}
For all $l\in\left[L\right],\,j\in\left[m_{l}\right],\,i\in\left[n\right]$
we have that $\left|\alpha_{j}^{\left(l\right)}\left(\mathbf{x}_{i}\right)\right|$
conditioned on $\mathbf{A}^{\left(l-1\right)}$ is subgaussian, with
$\left\Vert \left|\alpha_{j}^{\left(l\right)}\left(\mathbf{x}_{i}\right)\right|\right\Vert _{\psi_{2}}\leq c\frac{\sqrt{2}\alpha_{\phi}\nu_{0}s_{l}}{\sqrt{a_{l}}}\left\Vert \boldsymbol{\alpha}^{\left(l-1\right)}\left(\mathbf{x}_{i}\right)\right\Vert_{2}$
for some absolute constant $c>0$.
\end{proposition}
\begin{proof}
We follow a similar proof to their proof of Lemma A.1.

Observe that for all $j\in\left[m_{l}\right]$, $\alpha_{j}^{\left(l\right)}\left(\mathbf{x}_{i}\right)=\phi\left(s_{l}\left(\mathbf{w}_{j}^{\left(l\right)}\right)^{\top}\boldsymbol{\alpha}^{\left(l-1\right)}\left(\mathbf{x}_{i}\right)\right)$,
where $\mathbf{w}_{j}^{\left(l\right)}\sim\mathcal{N}\left(\mathbf{0}_{m_{l-1}},\frac{\nu_{0}^{2}}{a_{l}}\mathbf{I}_{m_{l-1}}\right)$.

Note that if \(\left\Vert \boldsymbol{\alpha}^{\left(l-1\right)}\left(\mathbf{x}_{i}\right)\right\Vert_2=0\), then
\[
\alpha_j^{\left(l\right)}\left(\mathbf{x}_i\right)=\phi\left(0\right)=0,
\]
and the claim is trivial. Hence assume this norm is positive.

\begin{align*}
&\mathbb{P}\left(\left|s_{l}\left(\mathbf{w}_{j}^{\left(l\right)}\right)^{\top}\boldsymbol{\alpha}^{\left(l-1\right)}\left(\mathbf{x}_{i}\right)\right|\geq\epsilon s_{l}\left\Vert \boldsymbol{\alpha}^{\left(l-1\right)}\left(\mathbf{x}_{i}\right)\right\Vert _{2}\right) \\& =\mathbb{P}\left(\left|\left(\mathbf{w}_{j}^{\left(l\right)}\right)^{\top}\boldsymbol{\alpha}^{\left(l-1\right)}\left(\mathbf{x}_{i}\right)\right|\geq\epsilon\left\Vert \boldsymbol{\alpha}^{\left(l-1\right)}\left(\mathbf{x}_{i}\right)\right\Vert _{2}\right)\\
 & =\mathbb{P}\left(\left|\left(\mathbf{w}_{j}^{\left(l\right)}\right)^{\top}\frac{\boldsymbol{\alpha}^{\left(l-1\right)}\left(\mathbf{x}_{i}\right)}{\left\Vert \boldsymbol{\alpha}^{\left(l-1\right)}\left(\mathbf{x}_{i}\right)\right\Vert _{2}}\right|\geq\epsilon\right)\\
 & \leq2\exp\left(\frac{-\epsilon^{2}a_{l}}{2\nu_{0}^{2}}\right)\,,
\end{align*}
since $\left(\mathbf{w}_{j}^{\left(l\right)}\right)^{\top}\frac{\boldsymbol{\alpha}^{\left(l-1\right)}\left(\mathbf{x}_{i}\right)}{\left\Vert \boldsymbol{\alpha}^{\left(l-1\right)}\left(\mathbf{x}_{i}\right)\right\Vert _{2}}\sim\mathcal{N}\left(0,\frac{\nu_{0}^{2}}{a_{l}}\right)$
when conditioned on $\mathbf{A}^{\left(l-1\right)}$, we applied the
Gaussian tail property \citep[Exercise 2.6]{vershynin2018high}.

Denoting $\tilde{\epsilon}=\epsilon s_{l}\left\Vert \boldsymbol{\alpha}^{\left(l-1\right)}\left(\mathbf{x}_{i}\right)\right\Vert _{2}$,
we get that 
\[
\mathbb{P}\left(\left|s_{l}\left(\mathbf{w}_{j}^{\left(l\right)}\right)^{\top}\boldsymbol{\alpha}^{\left(l-1\right)}\left(\mathbf{x}_{i}\right)\right|\geq\tilde{\epsilon}\right)\leq2\exp\left(-\frac{\tilde{\epsilon}^{2}a_{l}}{2\nu_{0}^{2}s_{l}^{2}\left\Vert \boldsymbol{\alpha}^{\left(l-1\right)}\left(\mathbf{x}_{i}\right)\right\Vert _{2}^{2}}\right)\,.
\]
Since $\phi$ is $\alpha_{\phi}$-Lipschitz and $\phi\left(0\right)=0$,
we have $\left|\alpha_{j}^{\left(l\right)}\left(\mathbf{x}_{i}\right)\right|=\left|\phi\left(s_{l}\left(\mathbf{w}_{j}^{\left(l\right)}\right)^{\top}\boldsymbol{\alpha}^{\left(l-1\right)}\left(\mathbf{x}_{i}\right)\right)\right|\ensuremath{\leq}\alpha_{\phi}\left|s_{l}\left(\mathbf{w}_{j}^{\left(l\right)}\right)^{\top}\boldsymbol{\alpha}^{\left(l-1\right)}\left(\mathbf{x}_{i}\right)\right|$,
thus 
\begin{align*}
\mathbb{P}\left(\left|\phi\left(s_{l}\left(\mathbf{w}_{j}^{\left(l\right)}\right)^{\top}\boldsymbol{\alpha}^{\left(l-1\right)}\left(\mathbf{x}_{i}\right)\right)\right|\geq\tilde{\epsilon}\right) & \leq\mathbb{P}\left(\left|s_{l}\left(\mathbf{w}_{j}^{\left(l\right)}\right)^{\top}\boldsymbol{\alpha}^{\left(l-1\right)}\left(\mathbf{x}_{i}\right)\right|\geq\frac{\tilde{\epsilon}}{\alpha_{\phi}}\right)\\
 & \leq2\exp\left(-\frac{\tilde{\epsilon}^{2}a_{l}}{2\nu_{0}^{2}\alpha_{\phi}^2 s_{l}^{2}\left\Vert \boldsymbol{\alpha}^{\left(l-1\right)}\left(\mathbf{x}_{i}\right)\right\Vert _{2}^{2}}\right)\,,
\end{align*}
proving $\left|\alpha_{j}^{\left(l\right)}\left(\mathbf{x}_{i}\right)\right|$
is subgaussian with $\left\Vert \left|\alpha_{j}^{\left(l\right)}\left(\mathbf{x}_{i}\right)\right|\right\Vert _{\psi_{2}}\leq c\frac{\sqrt{2}\alpha_{\phi}\nu_{0}s_{l}}{\sqrt{a_{l}}}\left\Vert \boldsymbol{\alpha}^{\left(l-1\right)}\left(\mathbf{x}_{i}\right)\right\Vert _{2}$
for some absolute constant $c>0$.
\end{proof}
We now produce a tighter alternative to Lemma A.1 of \citet{banerjee2023neural},
to be used later.
\begin{proposition}
\label{prop:tighter-lemma-a1}For all $l\in\left[L\right],\,j\in\left[m_{l}\right]$
we have that $\left\Vert \mathbf{A}_{:,j}^{\left(l\right)}\right\Vert _{2}$
conditioned on $\mathbf{A}^{\left(l-1\right)}$ is subgaussian, with
$\left\Vert \left\Vert \mathbf{A}_{:,j}^{\left(l\right)}\right\Vert _{2}\right\Vert _{\psi_{2}}\leq c\frac{\sqrt{2}\alpha_{\phi}\nu_{0}s_{l}}{\sqrt{a_{l}}}\left\Vert \mathbf{A}^{\left(l-1\right)}\right\Vert _{\mathrm{F}}$
for some absolute constant $c>0$.
\end{proposition}
\begin{proof}
Observe that $\left\Vert \mathbf{A}_{:,j}^{\left(l\right)}\right\Vert _{2}^{2}=\sum_{i=1}^{n}\left(\alpha_{j}^{\left(l\right)}\left(\mathbf{x}_{i}\right)\right)^{2}=\sum_{i=1}^{n}\left|\alpha_{j}^{\left(l\right)}\left(\mathbf{x}_{i}\right)\right|^{2}$.
In \cref{prop:sub-gaussian-outputs-per-layer} we showed that $\forall i\in\left[n\right],\,\left\Vert \left|\alpha_{j}^{\left(l\right)}\left(\mathbf{x}_{i}\right)\right|\right\Vert _{\psi_{2}}\leq c\frac{\sqrt{2}\alpha_{\phi}\nu_{0}s_{l}}{\sqrt{a_{l}}}\left\Vert \boldsymbol{\alpha}^{\left(l-1\right)}\left(\mathbf{x}_{i}\right)\right\Vert _{2}$,
hence $\left\Vert \left|\alpha_{j}^{\left(l\right)}\left(\mathbf{x}_{i}\right)\right|^{2}\right\Vert _{\psi_{1}}\leq\left(c\frac{\sqrt{2}\alpha_{\phi}\nu_{0}s_{l}}{\sqrt{a_{l}}}\right)^{2}\left\Vert \boldsymbol{\alpha}^{\left(l-1\right)}\left(\mathbf{x}_{i}\right)\right\Vert _{2}^{2}$
\citep[Lemma 2.8.5]{vershynin2018high}. From the triangle
inequality for the subexponential norm we have 
\begin{align*}
\left\Vert \left\Vert \mathbf{A}_{:,j}^{\left(l\right)}\right\Vert _{2}^{2}\right\Vert _{\psi_{1}} & =\left\Vert \sum_{i=1}^{n}\left|\alpha_{j}^{\left(l\right)}\left(\mathbf{x}_{i}\right)\right|^{2}\right\Vert _{\psi_{1}}\leq\sum_{i=1}^{n}\left\Vert \left|\alpha_{j}^{\left(l\right)}\left(\mathbf{x}_{i}\right)\right|^{2}\right\Vert _{\psi_{1}}\\
 & \leq\left(c\frac{\sqrt{2}\alpha_{\phi}\nu_{0}s_{l}}{\sqrt{a_{l}}}\right)^{2}\sum_{i=1}^{n}\left\Vert \boldsymbol{\alpha}^{\left(l-1\right)}\left(\mathbf{x}_{i}\right)\right\Vert _{2}^{2}=\left(c\frac{\sqrt{2}\alpha_{\phi}\nu_{0}s_{l}}{\sqrt{a_{l}}}\right)^{2}\left\Vert \mathbf{A}^{\left(l-1\right)}\right\Vert _{\mathrm{F}}^{2}\\
\Rightarrow\left\Vert \left\Vert \mathbf{A}_{:,j}^{\left(l\right)}\right\Vert _{2}\right\Vert _{\psi_{2}} & \leq c\frac{\sqrt{2}\alpha_{\phi}\nu_{0}s_{l}}{\sqrt{a_{l}}}\left\Vert \mathbf{A}^{\left(l-1\right)}\right\Vert _{\mathrm{F}}\,.
\end{align*}
\end{proof}
We proceed to prove a version of Lemma 4.1 in \citet{banerjee2023neural}, adjusted
to general parameterizations and fixing a minor issue that we found in the
original proof. 
\begin{proposition}
\label{lem:lemma_4.1_equivalent}Let $\lambda_{l}\triangleq\lambda_{\min}\left(\mathbb{E}_{\mathbf{g}\sim\mathcal{N}\left(\mathbf{0}_{m_{l-1}},\frac{\nu_{0}^{2}}{a_{l}}\mathbf{I}_{m_{l-1}}\right)}\left[\phi\left(s_{l}\mathbf{A}^{\left(l-1\right)}\mathbf{g}\right)\phi\left(s_{l}\mathbf{A}^{\left(l-1\right)}\mathbf{g}\right)^{\top}\right]\right)$
for some $l\in\left[L\right]$. With $m_{l}=\Omega\left(\max\left[n,\,\ln\left(\frac{nL}{\delta}\right)\frac{1}{\lambda_{l}}\frac{s_{l}^{2}}{a_{l}}\left\Vert \mathbf{A}^{\left(l-1\right)}\right\Vert _{\mathrm{F}}^{2}\max\left[1,\,\ln\left(\frac{1}{\lambda_{l}}\frac{s_{l}^{2}}{a_{l}}\left\Vert \mathbf{A}^{\left(l-1\right)}\right\Vert _{\mathrm{F}}^{2}\right)\right]\right]\right)$,
we have with probability at least $1-\frac{\delta}{L}$ over the initialization
of $\mathbf{W}^{\left(l\right)}$,
\[
\lambda_{\min}\left(\mathbf{A}^{\left(l\right)}\left(\mathbf{A}^{\left(l\right)}\right)^{\top}\right)\geq\frac{m_{l}\lambda_{l}}{4}\,.
\]
\end{proposition}
\begin{proof}
The proof follows theirs very closely for the most part.
A necessary condition for
$\mathbf{A}^{\left(l\right)}\left(\mathbf{A}^{\left(l\right)}\right)^{\top}$
 to be full rank is $m_{l}\geq n$, hence we impose $m_{l}\geq n$.

We define the matrix $\hat{\mathbf{A}}^{\left(l\right)}\in\mathbb{R}^{n\times m_{l}}$
such that for $j\in\left[m_{l}\right]$, the $j$-th column of $\hat{\mathbf{A}}^{\left(l\right)}$
is
\[
\hat{\mathbf{A}}_{:,j}^{\left(l\right)}=\phi\left(s_{l}\mathbf{A}^{\left(l-1\right)}\mathbf{w}_{j}^{\left(l\right)}\right)\mathds{1}_{\left[\left\Vert \phi\left(s_{l}\mathbf{A}^{\left(l-1\right)}\mathbf{w}_{j}^{\left(l\right)}\right)\right\Vert _{2}\leq t\right]}\,,
\]
where $\left(\mathbf{w}_{j}^{\left(l\right)}\right)^{\top}$ is the
$j$-th row of $\mathbf{W}^{\left(l\right)}$, $\mathds{1}_{\left[B\right]}$
is the indicator function of the event $B$, and $t>0$. From \cref{claim:zero_columns_smaller}, we have that
\begin{equation}
\lambda_{\min}\left(\hat{\mathbf{A}}^{\left(l\right)}\left(\hat{\mathbf{A}}^{\left(l\right)}\right)^{\top}\right)\leq\lambda_{\min}\left(\mathbf{A}^{\left(l\right)}\left(\mathbf{A}^{\left(l\right)}\right)^{\top}\right)\,.\label{eq:lambda_hat_smaller}
\end{equation}
Since $\hat{\mathbf{A}}_{:,j}^{\left(l\right)}\in\mathbb{R}^{n}$,
we have $\lambda_{\max}\left(\hat{\mathbf{A}}_{:,j}^{\left(l\right)}\left(\hat{\mathbf{A}}_{:,j}^{\left(l\right)}\right)^{\top}\right)=\left\Vert \hat{\mathbf{A}}_{:,j}^{\left(l\right)}\right\Vert ^{2}\leq t^{2}$,
by definition of $\hat{\mathbf{A}}_{:,j}^{\left(l\right)}$. Since
$\hat{\mathbf{A}}^{\left(l\right)}\left(\hat{\mathbf{A}}^{\left(l\right)}\right)^{\top}=\sum_{j\in\left[m_{l}\right]}\hat{\mathbf{A}}_{:,j}^{\left(l\right)}\left(\hat{\mathbf{A}}_{:,j}^{\left(l\right)}\right)^{\top}$,
we can use the matrix Chernoff bound, and observe that for any $\epsilon\in\left[0,1\right)$,
\begin{align*}
& \mathbb{P}\left(\lambda_{\min}\left(\hat{\mathbf{A}}^{\left(l\right)}\left(\hat{\mathbf{A}}^{\left(l\right)}\right)^{\top}\right)
 \leq\left(1-\epsilon\right)\lambda_{\min}\left(\mathbb{E}_{\mathbf{W}^{\left(l\right)}}\left[\hat{\mathbf{A}}^{\left(l\right)}\left(\hat{\mathbf{A}}^{\left(l\right)}\right)^{\top}\right]\right)\right) \\
& \leq n\left(\frac{e^{-\epsilon}}{\left(1-\epsilon\right)^{1-\epsilon}}\right)^{\nicefrac{\lambda_{\min}\left(\mathbb{E}_{\mathbf{W}^{\left(l\right)}}\left[\hat{\mathbf{A}}^{\left(l\right)}\left(\hat{\mathbf{A}}^{\left(l\right)}\right)^{\top}\right]\right)}{t^{2}}}\,.
\end{align*}
Recall that $\forall j\in\left[m_{l}\right],\,\mathbf{w}_{j}^{\left(l\right)}\sim\mathcal{N}\left(\mathbf{0}_{m_{l-1}},\frac{\nu_{0}^{2}}{a_{l}}\mathbf{I}_{m_{l-1}}\right)$,
and then
\begin{align*}
&\mathbb{E}_{\mathbf{W}^{\left(l\right)}}\left[\hat{\mathbf{A}}^{\left(l\right)}\left(\hat{\mathbf{A}}^{\left(l\right)}\right)^{\top}\right] \\
& = \mathbb{E}_{\mathbf{W}^{\left(l\right)}}\left[\sum_{j\in\left[m_{l}\right]}\hat{\mathbf{A}}_{:,j}^{\left(l\right)}\left(\hat{\mathbf{A}}_{:,j}^{\left(l\right)}\right)^{\top}\right]\\
 & =\sum_{j\in\left[m_{l}\right]}\mathbb{E}_{\mathbf{w}_{j}^{\left(l\right)}}\left[\hat{\mathbf{A}}_{:,j}^{\left(l\right)}\left(\hat{\mathbf{A}}_{:,j}^{\left(l\right)}\right)^{\top}\right]\\
 & =m_{l}\mathbb{E}_{\mathbf{g}\sim\mathcal{N}\left(\mathbf{0}_{m_{l-1}},\frac{\nu_{0}^{2}}{a_{l}}\mathbf{I}_{m_{l-1}}\right)}\left[\phi\left(s_{l}\mathbf{A}^{\left(l-1\right)}\mathbf{g}\right)\phi\left(s_{l}\mathbf{A}^{\left(l-1\right)}\mathbf{g}\right)^{\top}\mathds{1}_{\left[\left\Vert \phi\left(s_{l}\mathbf{A}^{\left(l-1\right)}\mathbf{g}\right)\right\Vert _{2}\leq t\right]}\right]\,.
\end{align*}
Denoting 
\begin{align*}
    \mathbf{G}&=\mathbb{E}_{\mathbf{g}\sim\mathcal{N}\left(\mathbf{0}_{m_{l-1}},\frac{\nu_{0}^{2}}{a_{l}}\mathbf{I}_{m_{l-1}}\right)}\left[\phi\left(s_{l}\mathbf{A}^{\left(l-1\right)}\mathbf{g}\right)\phi\left(s_{l}\mathbf{A}^{\left(l-1\right)}\mathbf{g}\right)^{\top}\right]\,, \\
\hat{\mathbf{G}}&=\mathbb{E}_{\mathbf{g}\sim\mathcal{N}\left(\mathbf{0}_{m_{l-1}},\frac{\nu_{0}^{2}}{a_{l}}\mathbf{I}_{m_{l-1}}\right)}\left[\phi\left(s_{l}\mathbf{A}^{\left(l-1\right)}\mathbf{g}\right)\phi\left(s_{l}\mathbf{A}^{\left(l-1\right)}\mathbf{g}\right)^{\top}\mathds{1}_{\left[\left\Vert \phi\left(s_{l}\mathbf{A}^{\left(l-1\right)}\mathbf{g}\right)\right\Vert _{2}\leq t\right]}\right]\,,
\end{align*}
we have for $\epsilon=\frac{1}{2}$,
\begin{align*}
\mathbb{P}\left(\lambda_{\min}\left(\hat{\mathbf{A}}^{\left(l\right)}\left(\hat{\mathbf{A}}^{\left(l\right)}\right)^{\top}\right)\leq\frac{m_{l}}{2}\lambda_{\min}\left(\hat{\mathbf{G}}\right)\right) & \leq n\left(\frac{e^{-1/2}}{\left(1/2\right)^{1/2}}\right)^{\nicefrac{m_{l}\lambda_{\min}\left(\hat{\mathbf{G}}\right)}{t^{2}}}\\
 & =n\exp\left(\ln\left(\sqrt{\frac{2}{e}}\right)\frac{m_{l}\lambda_{\min}\left(\hat{\mathbf{G}}\right)}{t^{2}}\right)\\
 & =\exp\left(-\frac{\left(1-\ln2\right)m_{l}\lambda_{\min}\left(\hat{\mathbf{G}}\right)}{2t^{2}}+\ln n\right)\,.
\end{align*}
We denote $\tilde{\delta}\triangleq\exp\left(-\frac{\left(1-\ln2\right)m_{l}\lambda_{\min}\left(\hat{\mathbf{G}}\right)}{2t^{2}}+\ln n\right)$,
and then we have with probability at least $1-\tilde{\delta}$,
\begin{equation}
\lambda_{\min}\left(\hat{\mathbf{A}}^{\left(l\right)}\left(\hat{\mathbf{A}}^{\left(l\right)}\right)^{\top}\right)\geq\frac{m_{l}}{2}\lambda_{\min}\left(\hat{\mathbf{G}}\right)\,.\label{eq:lambda_min_A_hat}
\end{equation}
Now we observe that
\begin{align*}
\left\Vert \hat{\mathbf{G}}-\mathbf{G}\right\Vert _{2} 
& =\Biggl\Vert \mathbb{E}_{\mathbf{g}\sim\mathcal{N}\left(\mathbf{0}_{m_{l-1}},\frac{\nu_{0}^{2}}{a_{l}}\mathbf{I}_{m_{l-1}}\right)}\biggl[\phi\left(s_{l}\mathbf{A}^{\left(l-1\right)}\mathbf{g}\right)\phi\left(s_{l}\mathbf{A}^{\left(l-1\right)}\mathbf{g}\right)^{\top} \cdot\\
& \qquad \qquad \qquad \qquad \qquad \qquad \cdot\left(1-\mathds{1}_{\left[\left\Vert \phi\left(s_{l}\mathbf{A}^{\left(l-1\right)}\mathbf{g}\right)\right\Vert _{2}\leq t\right]}\right)\biggr]\Biggr\Vert _{2}\\
 & =\left\Vert \mathbb{E}_{\mathbf{g}\sim\mathcal{N}\left(\mathbf{0}_{m_{l-1}},\frac{\nu_{0}^{2}}{a_{l}}\mathbf{I}_{m_{l-1}}\right)}\left[\phi\left(s_{l}\mathbf{A}^{\left(l-1\right)}\mathbf{g}\right)\phi\left(s_{l}\mathbf{A}^{\left(l-1\right)}\mathbf{g}\right)^{\top} \! \mathds{1}_{\left[\left\Vert \phi\left(s_{l}\mathbf{A}^{\left(l-1\right)}\mathbf{g}\right)\right\Vert _{2}>t\right]}\right]\right\Vert _{2}\\
\left[\text{Jensen}\right] & \leq\mathbb{E}_{\mathbf{g}\sim\mathcal{N}\left(\mathbf{0}_{m_{l-1}},\frac{\nu_{0}^{2}}{a_{l}}\mathbf{I}_{m_{l-1}}\right)}\! \left[\left\Vert \phi\left(s_{l}\mathbf{A}^{\left(l-1\right)}\mathbf{g}\right)\phi\left(s_{l}\mathbf{A}^{\left(l-1\right)}\mathbf{g}\right)^{\top} \! \mathds{1}_{\left[\left\Vert \phi\left(s_{l}\mathbf{A}^{\left(l-1\right)}\mathbf{g}\right)\right\Vert _{2}>t\right]}\right\Vert _{2}\right]\\
 & =\mathbb{E}_{\mathbf{g}\sim\mathcal{N}\left(\mathbf{0}_{m_{l-1}},\frac{\nu_{0}^{2}}{a_{l}}\mathbf{I}_{m_{l-1}}\right)} \! \left[\left\Vert \phi\left(s_{l}\mathbf{A}^{\left(l-1\right)}\mathbf{g}\right)\right\Vert _{2}^{2}\mathds{1}_{\left[\left\Vert \phi\left(s_{l}\mathbf{A}^{\left(l-1\right)}\mathbf{g}\right)\right\Vert _{2}>t\right]}\right]\\
 & =\int_{s=0}^{\infty}\mathbb{P}\left(\left\Vert \phi\left(s_{l}\mathbf{A}^{\left(l-1\right)}\mathbf{g}\right)\right\Vert _{2}^{2}\mathds{1}_{\left[\left\Vert \phi\left(s_{l}\mathbf{A}^{\left(l-1\right)}\mathbf{g}\right)\right\Vert _{2}>t\right]}>s\right)\mathrm{d}s\\
 & =\int_{s=0}^{\infty}\mathbb{P}\left(\left\Vert \phi\left(s_{l}\mathbf{A}^{\left(l-1\right)}\mathbf{g}\right)\right\Vert _{2}\mathds{1}_{\left[\left\Vert \phi\left(s_{l}\mathbf{A}^{\left(l-1\right)}\mathbf{g}\right)\right\Vert _{2}>t\right]}>\sqrt{s}\right)\mathrm{d}s\,,
\end{align*}
where Jensen could be used since the operator norm is convex. This
is the first place our proof deviates from that of \citet{banerjee2023neural},
by correcting the first transition in the following:
\begin{align*}
\left\Vert \hat{\mathbf{G}}-\mathbf{G}\right\Vert _{2} & \leq\int_{s=0}^{\infty}\mathbb{P}\left(\left\Vert \phi\left(s_{l}\mathbf{A}^{\left(l-1\right)}\mathbf{g}\right)\right\Vert _{2}\mathds{1}_{\left[\left\Vert \phi\left(s_{l}\mathbf{A}^{\left(l-1\right)}\mathbf{g}\right)\right\Vert _{2}>t\right]}>\sqrt{s}\right)\mathrm{d}s\\
 & =\int_{0}^{t^{2}}\mathbb{P}\left(\left\Vert \phi\left(s_{l}\mathbf{A}^{\left(l-1\right)}\mathbf{g}\right)\right\Vert _{2}>t\right)\mathrm{d}s+\int_{t^{2}}^{\infty}\mathbb{P}\left(\left\Vert \phi\left(s_{l}\mathbf{A}^{\left(l-1\right)}\mathbf{g}\right)\right\Vert _{2}>\sqrt{s}\right)\mathrm{d}s\,,
\end{align*}
since when $0<\sqrt{s}<t$ the event $x\mathds{1}_{\left[x>t\right]}>\sqrt{s}$
occurs iff $x>t$, and when $\sqrt{s}\geq t$ the event $x\mathds{1}_{\left[x>t\right]}>\sqrt{s}$
occurs iff $x>\sqrt{s}$.

From \cref{prop:tighter-lemma-a1} we have that $\left\Vert \mathbf{A}_{:,j}^{\left(l\right)}\right\Vert _{2}=\left\Vert \phi\left(s_{l}\mathbf{A}^{\left(l-1\right)}\mathbf{w}_{j}^{\left(l\right)}\right)\right\Vert _{2}$
is subgaussian, with $\left\Vert \left\Vert \mathbf{A}_{:,j}^{\left(l\right)}\right\Vert _{2}\right\Vert _{\psi_{2}}\leq c\frac{\sqrt{2}\alpha_{\phi}\nu_{0}s_{l}}{\sqrt{a_{l}}}\left\Vert \mathbf{A}^{\left(l-1\right)}\right\Vert _{\mathrm{F}}$.
Since $\mathbf{g}$ and $\mathbf{w}_{j}^{\left(l\right)}$ have the
same distribution, we observe that $\left\Vert \left\Vert \phi\left(s_{l}\mathbf{A}^{\left(l-1\right)}\mathbf{g}\right)\right\Vert _{2}\right\Vert _{\psi_{2}}\leq c\frac{\sqrt{2}\alpha_{\phi}\nu_{0}s_{l}}{\sqrt{a_{l}}}\left\Vert \mathbf{A}^{\left(l-1\right)}\right\Vert _{\mathrm{F}}$.
We denote $b\triangleq c\frac{\sqrt{2}\alpha_{\phi}\nu_{0}s_{l}}{\sqrt{a_{l}}}\left\Vert \mathbf{A}^{\left(l-1\right)}\right\Vert _{\mathrm{F}}$.
Then we have
\begin{align*}
\left\Vert \hat{\mathbf{G}}-\mathbf{G}\right\Vert _{2} & \leq\int_{0}^{t^{2}}\mathbb{P}\left(\left\Vert \phi\left(s_{l}\mathbf{A}^{\left(l-1\right)}\mathbf{g}\right)\right\Vert _{2}>t\right)\mathrm{d}s+\int_{t^{2}}^{\infty}\mathbb{P}\left(\left\Vert \phi\left(s_{l}\mathbf{A}^{\left(l-1\right)}\mathbf{g}\right)\right\Vert _{2}>\sqrt{s}\right)\mathrm{d}s\\
 & \leq2t^{2}\exp\left(-\frac{c_{1}}{b^{2}}t^{2}\right)+\int_{t^{2}}^{\infty}2\exp\left(-\frac{c_{1}}{b^{2}}s\right)\mathrm{d}s\\
 & =2t^{2}\exp\left(-\frac{c_{1}}{b^{2}}t^{2}\right)+2\frac{b^{2}}{c_1}\exp\left(-\frac{c_{1}}{b^{2}}t^{2}\right)\\
 & =2\left(t^{2}+\frac{b^{2}}{c_{1}}\right)\exp\left(-\frac{c_{1}}{b^{2}}t^{2}\right)\,.
\end{align*}
Let
\[
R_l\triangleq\max\left\{e,\frac{4b^2}{\lambda_lc_1}\right\},
\qquad
t\triangleq\frac{b}{\sqrt{c_1}}\sqrt{2\ln\left(8R_l\right)}\,.
\]
Then all logarithms and roots are well-defined. Using the bound above, and writing
\[
u\triangleq\frac{c_1t^2}{b^2}=2\ln\left(8R_l\right),
\]
we get
\begin{align*}
\left\Vert \hat{\mathbf{G}}-\mathbf{G}\right\Vert_2
&\leq
2\left(t^2+\frac{b^2}{c_1}\right)
\exp\left(-\frac{c_1}{b^2}t^2\right)\\
&=
2\frac{b^2}{c_1}\left(u+1\right)e^{-u}\\
&=
2\frac{b^2}{c_1}\left(2\ln\left(8R_l\right)+1\right)\frac{1}{\left(8R_l\right)^2}\,.
\end{align*}
Since \(R_l\geq e\), we have
\[
2\ln\left(8R_l\right)+1\leq 8R_l\,.
\]
Hence,
\[
\left\Vert \hat{\mathbf{G}}-\mathbf{G}\right\Vert_2
\leq
2\frac{b^2}{c_1}\frac{8R_l}{64R_l^2}
=
\frac{1}{4}\frac{b^2}{c_1R_l}\,.
\]
Moreover, by definition of \(R_l\),
\[
R_l\geq\frac{4b^2}{\lambda_lc_1}\,,
\quad\Longrightarrow
\frac{b^2}{c_1R_l}\leq\frac{\lambda_l}{4}\,.
\]
Therefore,
\[
\left\Vert \hat{\mathbf{G}}-\mathbf{G}\right\Vert_2
\leq
\frac{\lambda_l}{16}
\leq
\frac{\lambda_l}{2}\,.
\]
From the definition of \(\mathbf{G}\), we have
\[
\lambda_l=\lambda_{\min}\left(\mathbf{G}\right)\,,
\]
so Weyl's inequality gives
\[
\lambda_{\min}\left(\hat{\mathbf{G}}\right)
\geq
\lambda_{\min}\left(\mathbf{G}\right)
-
\left\Vert \hat{\mathbf{G}}-\mathbf{G}\right\Vert_2
\geq
\frac{\lambda_l}{2}\,.
\]
From \cref{eq:lambda_hat_smaller}, \cref{eq:lambda_min_A_hat}, we then have with probability at least \(1-\tilde{\delta}\),
\[
\lambda_{\min}\left(\mathbf{A}^{\left(l\right)}\left(\mathbf{A}^{\left(l\right)}\right)^{\top}\right)
\geq
\lambda_{\min}\left(\hat{\mathbf{A}}^{\left(l\right)}\left(\hat{\mathbf{A}}^{\left(l\right)}\right)^{\top}\right)
\geq
\frac{m_l}{2}\lambda_{\min}\left(\hat{\mathbf{G}}\right)
\geq
\frac{m_l\lambda_l}{4}\,,
\]
where
\begin{align*}
\tilde{\delta}
&=
\exp\left(
-\frac{\left(1-\ln2\right)m_l\lambda_{\min}\left(\hat{\mathbf{G}}\right)}{2t^2}
+\ln n
\right)\\
&\leq
\exp\left(
-\frac{\left(1-\ln2\right)m_l\lambda_l}{4t^2}
+\ln n
\right)\\
&=
\exp\left(
-\frac{\left(1-\ln2\right)c_1}{8}
\frac{m_l\lambda_l}{b^2\ln\left(8R_l\right)}
+\ln n
\right)\\
&\leq
\exp\left(
-c_2
\frac{m_l\lambda_l}{b^2\max\left[1,\ln\left(\frac{4b^2}{\lambda_lc_1}\right)\right]}
+\ln n
\right),
\end{align*}
for an absolute constant \(c_2>0\), where in the last step we used
\[
\ln\left(8R_l\right)
=
O\left(\max\left[1,\ln\left(\frac{4b^2}{\lambda_lc_1}\right)\right]\right).
\]
We now demand \(\tilde{\delta}\leq\frac{\delta}{L}\), which is ensured by
\[
m_l
\geq
\frac{1}{c_2}
\ln\left(\frac{nL}{\delta}\right)
\frac{b^2}{\lambda_l}
\max\left[1,\ln\left(\frac{4b^2}{\lambda_lc_1}\right)\right]\,.
\]
Substituting
\[
b
=
c\frac{\sqrt{2}\alpha_{\phi}\nu_0s_l}{\sqrt{a_l}}
\left\Vert \mathbf{A}^{\left(l-1\right)}\right\Vert_{\mathrm{F}}\,,
\]
we obtain the sufficient condition
\begin{align*}
m_l
&\gtrsim
\ln\left(\frac{nL}{\delta}\right)
\frac{1}{\lambda_l}
\frac{s_l^2}{a_l}
\left\Vert \mathbf{A}^{\left(l-1\right)}\right\Vert_{\mathrm{F}}^2
\max\left[
1,
\ln\left(
\frac{1}{\lambda_l}
\frac{s_l^2}{a_l}
\left\Vert \mathbf{A}^{\left(l-1\right)}\right\Vert_{\mathrm{F}}^2
\right)
\right]\,,
\end{align*}
after absorbing constants depending only on \(\alpha_\phi,\nu_0,c,c_1\) into the implicit constant.
Together with the imposed condition \(m_l\geq n\), this is precisely the stated lower bound on \(m_l\), up to absolute constants. Hence \(\tilde{\delta}\leq \delta/L\), and the claimed bound holds with probability at least \(1-\delta/L\).
\end{proof}
We now proceed to prove an analogue of their \citep{banerjee2023neural} Lemma
4.2, again adjusting for general parameterization and, importantly,
improving the bound's probability.

First, we need to restate their Proposition A.2, and adjust the proof
to our modified settings. Specifically, in this proof we differ from
theirs by the generalized gradient bound $\alpha_{\phi}$ of \cref{ass:activation_bounded_gradients}.
\begin{proposition}
\label{prop:phi_squared_distance_bound}With $c_{\phi,\sigma_{0}}\triangleq\mathbb{E}_{z\sim\mathcal{N}\left(0,\sigma_{0}^{2}\right)}\left[\phi^{2}\left(z\right)\right]$
and assuming $\phi\left(0\right)=0,\,\left|\dot{\phi}\right|\leq\alpha_{\phi}$,
we have for all $\zeta>0$,
\[
c_{\phi,\sigma_{0}}-\alpha_{\phi}^{2}\sigma_{0}^{2}\left|\zeta^{2}-1\right|\leq\mathbb{E}_{z\sim\mathcal{N}\left(0,\sigma_{0}^{2}\right)}\left[\phi^{2}\left(\zeta z\right)\right]\leq c_{\phi,\sigma_{0}}+\alpha_{\phi}^{2}\sigma_{0}^{2}\left|\zeta^{2}-1\right|\,.
\]
\end{proposition}
\begin{proof}
Observe that the bounded gradient implies $\phi$ is $\alpha_{\phi}$-Lipschitz,
and combined with $\phi\left(0\right)=0$ we get that $\left|\phi\left(\cdot\right)\right|\leq\alpha_{\phi}\left|\cdot\right|$.
Then, we have
\begin{align*}
\left|\mathbb{E}_{z\sim\mathcal{N}\left(0,\sigma_{0}^{2}\right)}\left[\phi^{2}\left(\zeta z\right)\right]-\mathbb{E}_{z\sim\mathcal{N}\left(0,\sigma_{0}^{2}\right)}\left[\phi^{2}\left(z\right)\right]\right| & =\left|\mathbb{E}_{z\sim\mathcal{N}\left(0,\sigma_{0}^{2}\right)}\left[\phi^{2}\left(\zeta z\right)-\phi^{2}\left(z\right)\right]\right|\\
\left[\text{Jensen}\right] & \leq\mathbb{E}_{z\sim\mathcal{N}\left(0,\sigma_{0}^{2}\right)}\left[\left|\phi^{2}\left(\zeta z\right)-\phi^{2}\left(z\right)\right|\right]\\
 & =\mathbb{E}_{z\sim\mathcal{N}\left(0,\sigma_{0}^{2}\right)}\left[\left|\phi\left(\zeta z\right)-\phi\left(z\right)\right|\left|\phi\left(\zeta z\right)+\phi\left(z\right)\right|\right]\\
\left[\phi\text{ is }\alpha_{\phi}\text{-Lipschitz}\right] & \leq\mathbb{E}_{z\sim\mathcal{N}\left(0,\sigma_{0}^{2}\right)}\left[\alpha_{\phi}\left|z\left(\zeta-1\right)\right|\left|\phi\left(\zeta z\right)+\phi\left(z\right)\right|\right]\\
\left[\text{Triangle Inequality}\right] & \leq\alpha_{\phi}\left|\zeta-1\right|\mathbb{E}_{z\sim\mathcal{N}\left(0,\sigma_{0}^{2}\right)}\left[\left|z\right|\left(\left|\phi\left(\zeta z\right)\right|+\left|\phi\left(z\right)\right|\right)\right]\\
\left[\left|\phi\left(\cdot\right)\right|\leq\alpha_{\phi}\left|\cdot\right|\right] & \leq\alpha_{\phi}\left|\zeta-1\right|\mathbb{E}_{z\sim\mathcal{N}\left(0,\sigma_{0}^{2}\right)}\left[\left|z\right|\alpha_{\phi}\left(\left|\zeta z\right|+\left|z\right|\right)\right]\\
 & =\alpha_{\phi}^{2}\left|\zeta-1\right|\left(\left|\zeta\right|+1\right)\mathbb{E}_{z\sim\mathcal{N}\left(0,\sigma_{0}^{2}\right)}\left[\left|z\right|^{2}\right]\\
\left[\zeta>0\right] & =\alpha_{\phi}^{2}\sigma_{0}^{2}\left|\zeta^{2}-1\right|\,.
\end{align*}
Since $\mathbb{E}_{z\sim\mathcal{N}\left(0,\sigma_{0}^{2}\right)}\left[\phi^{2}\left(z\right)\right]=c_{\phi,\sigma_{0}}$,
the desired result is immediate.
\end{proof}
\begin{proposition}
\label{lem:lemma_4.2_equivalent}
Let $h_{C}\left(l\right)\triangleq\sum_{i=0}^{l}\left(\alpha_{\phi}\nu_{0}\right)^{2i}$.
Assume $\forall l\in\left[L\right],\,\frac{a_{l}}{s_{l}^{2}}=m_{l-1}$.\footnote{
This lemma is the source of this demand, since it simplifies the proof greatly. This can, however, be relaxed, leading to a more complicated bound which is only useful if $\forall l\in\left[L\right],\,\frac{a_{l}}{s_{l}^{2}} =\Theta \prn{m_{l-1}}$.
} With $m\triangleq\min_{l\in\left[L\right]}m_{l}=\Omega\left(h_{C}\left(L\right)\max\left[h_{C}\left(L\right),1\right]\ln\left(nL\right)\right)$,
then with probability at least $1-2nL\exp\left(-c_{3}\frac{m}{h_{C}\left(L\right)\max\left[h_{C}\left(L\right),1\right]}\right)$
over the initialization of $\left\{ \mathbf{W}^{\left(l\right)}\right\} _{l\in\left[L\right]}$,
we have for all $i\in\left[n\right],\,l\in\left[L\right]$:
\[
\frac{c_{\phi,\sigma_{0}}}{2}m_{l}\leq\left\Vert \boldsymbol{\alpha}^{\left(l\right)}\left(\mathbf{x}_{i}\right)\right\Vert _{2}^{2}\leq\frac{3c_{\phi,\sigma_{0}}}{2}m_{l}\,.
\]
\end{proposition}
\begin{proof}
We start by following their proof exactly, using induction over the
layers to prove the following holds for all $i\in\left[n\right],\,l\in\left[L\right]$:
\[
c_{\phi,\sigma_{0}}m_{l}\left(1-\frac{h_{C}\left(l\right)}{2h_{C}\left(L\right)}\right)\leq\left\Vert \boldsymbol{\alpha}^{\left(l\right)}\left(\mathbf{x}_{i}\right)\right\Vert _{2}^{2}\leq c_{\phi,\sigma_{0}}m_{l}\left(1+\frac{h_{C}\left(l\right)}{2h_{C}\left(L\right)}\right)\,.
\]
 For $l=0$, we have $\left\Vert \boldsymbol{\alpha}^{\left(0\right)}\left(\mathbf{x}_{i}\right)\right\Vert _{2}^{2}=\left\Vert \mathbf{x}_{i}\right\Vert _{2}^{2}=c_{\phi,\sigma_{0}}d=c_{\phi,\sigma_{0}}m_{0}$,
so the desired holds with probability 1. Now, for some $l\in\left[L\right]$,
we assume that the result holds for $l-1$, so that for all $i\in\left[n\right]$,
\begin{equation}
\begin{aligned}
c_{\phi,\sigma_{0}}m_{l-1}\left(1-\frac{h_{C}\left(l-1\right)}{2h_{C}\left(L\right)}\right) & \leq\left\Vert \boldsymbol{\alpha}^{\left(l-1\right)}\left(\mathbf{x}_{i}\right)\right\Vert _{2}^{2}\leq c_{\phi,\sigma_{0}}m_{l-1}\left(1+\frac{h_{C}\left(l-1\right)}{2h_{C}\left(L\right)}\right) \\
\left|\frac{\left\Vert \boldsymbol{\alpha}^{\left(l-1\right)}\left(\mathbf{x}_{i}\right)\right\Vert _{2}^{2}}{c_{\phi,\sigma_{0}}m_{l-1}}-1\right| & \leq\frac{h_{C}\left(l-1\right)}{2h_{C}\left(L\right)}\,.\label{eq:induction_dist_bound}
\end{aligned}
\end{equation}
Now we observe for all $i\in\left[n\right]$,
\begin{align*}
& \mathbb{E}_{\mathbf{W}^{\left(l\right)}}\left[\left\Vert \boldsymbol{\alpha}^{\left(l\right)}\left(\mathbf{x}_{i}\right)\right\Vert _{2}^{2}\right] \\
& =\mathbb{E}_{\mathbf{W}^{\left(l\right)}}\left[\sum_{j\in\left[m_{l}\right]}\left(\alpha_{j}^{\left(l\right)}\left(\mathbf{x}_{i}\right)\right)^{2}\right]=\sum_{j\in\left[m_{l}\right]}\mathbb{E}_{\mathbf{w}_{j}^{\left(l\right)}}\left[\left(\alpha_{j}^{\left(l\right)}\left(\mathbf{x}_{i}\right)\right)^{2}\right]\\
 & =m_{l}\mathbb{E}_{\mathbf{g}\sim\mathcal{N}\left(\mathbf{0}_{m_{l-1}},\frac{\nu_{0}^{2}}{a_{l}}\mathbf{I}_{m_{l-1}}\right)}\left[\left(\phi\left(s_{l}\mathbf{g}^{\top}\boldsymbol{\alpha}^{\left(l-1\right)}\left(\mathbf{x}_{i}\right)\right)\right)^{2}\right]\\
 & =m_{l}\mathbb{E}_{\mathbf{g}\sim\mathcal{N}\left(\mathbf{0}_{m_{l-1}},\frac{\sigma_{0}^{2}}{c_{\phi,\sigma_{0}}a_{l}}\mathbf{I}_{m_{l-1}}\right)}\left[\left(\phi\left(s_{l}\left\Vert \boldsymbol{\alpha}^{\left(l-1\right)}\left(\mathbf{x}_{i}\right)\right\Vert _{2}\cdot\mathbf{g}^{\top}\frac{\boldsymbol{\alpha}^{\left(l-1\right)}\left(\mathbf{x}_{i}\right)}{\left\Vert \boldsymbol{\alpha}^{\left(l-1\right)}\left(\mathbf{x}_{i}\right)\right\Vert _{2}}\right)\right)^{2}\right]\\
 & =m_{l}\mathbb{E}_{z\sim\mathcal{N}\left(0,\frac{\sigma_{0}^{2}}{c_{\phi,\sigma_{0}}a_{l}}\right)}\left[\left(\phi\left(s_{l}\left\Vert \boldsymbol{\alpha}^{\left(l-1\right)}\left(\mathbf{x}_{i}\right)\right\Vert _{2}z\right)\right)^{2}\right]\\
 & =m_{l}\mathbb{E}_{z\sim\mathcal{N}\left(0,\sigma_{0}^{2}\right)}\left[\phi^{2}\left(\frac{s_{l}\left\Vert \boldsymbol{\alpha}^{\left(l-1\right)}\left(\mathbf{x}_{i}\right)\right\Vert _{2}}{\sqrt{c_{\phi,\sigma_{0}}a_{l}}}z\right)\right]\,.
\end{align*}
Then, from \cref{prop:phi_squared_distance_bound}, we have
\begin{align*}
    \left| \mathbb{E}_{z\sim\mathcal{N}\left(0,\sigma_{0}^{2}\right)}\left[\phi^{2}\left(\frac{s_{l}\left\Vert \boldsymbol{\alpha}^{\left(l-1\right)}\left(\mathbf{x}_{i}\right)\right\Vert _{2}}{\sqrt{c_{\phi,\sigma_{0}}a_{l}}}z\right)\right]  - c_{\phi,\sigma_{0}} \right| & \leq \alpha_{\phi}^{2}\sigma_{0}^{2}\left|\left(\frac{s_{l}\left\Vert \boldsymbol{\alpha}^{\left(l-1\right)}\left(\mathbf{x}_{i}\right)\right\Vert _{2}}{\sqrt{c_{\phi,\sigma_{0}}a_{l}}}\right)^{2}-1\right| \\
    \left| \mathbb{E}_{\mathbf{W}^{\left(l\right)}} 
\left[\left\Vert \boldsymbol{\alpha}^{\left(l\right)}\left(\mathbf{x}_{i}\right)\right\Vert _{2}^{2}\right] - m_{l} c_{\phi,\sigma_{0}} \right| & \leq m_l \alpha_{\phi}^{2}\sigma_{0}^{2}\left|\frac{s_{l}^{2}\left\Vert \boldsymbol{\alpha}^{\left(l-1\right)}\left(\mathbf{x}_{i}\right)\right\Vert _{2}^{2}}{c_{\phi,\sigma_{0}}a_{l}}-1\right|
\end{align*}
From $\frac{a_{l}}{s_{l}^{2}}=m_{l-1},$ we can apply \cref{eq:induction_dist_bound}, and get:
\begin{equation}
m_{l}\left(c_{\phi,\sigma_{0}}-\alpha_{\phi}^{2}\sigma_{0}^{2}\frac{h_{C}\left(l-1\right)}{2h_{C}\left(L\right)}\right)\leq\mathbb{E}_{\mathbf{W}^{\left(l\right)}}\left[\left\Vert \boldsymbol{\alpha}^{\left(l\right)}\left(\mathbf{x}_{i}\right)\right\Vert _{2}^{2}\right]\leq m_{l}\left(c_{\phi,\sigma_{0}}+\alpha_{\phi}^{2}\sigma_{0}^{2}\frac{h_{C}\left(l-1\right)}{2h_{C}\left(L\right)}\right)\,.\label{eq:mean_between_two}
\end{equation}
 By \cref{prop:sub-gaussian-outputs-per-layer}, for all $l\in\left[L\right],\,j\in\left[m_{l}\right],\,i\in\left[n\right]$
we have that $\left|\alpha_{j}^{\left(l\right)}\left(\mathbf{x}_{i}\right)\right|$
(conditioned on $\mathbf{A}^{\left(l-1\right)}$) is subgaussian.
As such, $\left|\alpha_{j}^{\left(l\right)}\left(\mathbf{x}_{i}\right)\right|^{2}$
is subexponential with $K\triangleq\left\Vert \left|\alpha_{j}^{\left(l\right)}\left(\mathbf{x}_{i}\right)\right|^{2}\right\Vert _{\psi_{1}}=\left\Vert \left|\alpha_{j}^{\left(l\right)}\left(\mathbf{x}_{i}\right)\right|\right\Vert _{\psi_{2}}^{2}$,
due to \citep[Lemma 2.8.5]{vershynin2018high}. By applying the induction hypothesis and using $\frac{a_l}{s_l^2} = m_{l-1}$, we get
\begin{align*}
K & =\left\Vert \left|\alpha_{j}^{\left(l\right)}\left(\mathbf{x}_{i}\right)\right|\right\Vert _{\psi_{2}}^{2}
\leq\frac{2c^{2}\alpha_{\phi}^2\nu_{0}^{2}s_{l}^{2}}{a_{l}}\left\Vert \boldsymbol{\alpha}^{\left(l-1\right)}\left(\mathbf{x}_{i}\right)\right\Vert _{2}^{2} \\
& \leq\frac{2c^{2}\alpha_{\phi}^2s_{l}^{2}}{a_{l}}\frac{\sigma_{0}^{2}}{c_{\phi,\sigma_{0}}}\cdot\frac{3}{2}c_{\phi,\sigma_{0}}\frac{a_{l}}{s_{l}^{2}}
=3c^{2}\alpha_{\phi}^2\sigma_{0}^{2}\,.
\end{align*}

Conditioned on $\mathbf{A}^{\left(l-1\right)}$, since $\left|\alpha_{j}^{\left(l\right)}\left(\mathbf{x}_{i}\right)\right|^{2},\,j\in\left[m_{l}\right]$
are independent, subexponential random variables for all $i\in\left[n\right]$,
we can apply Bernstein's inequality on them centered \citep[Theorem 2.9.1]{vershynin2018high}:
\begin{align*}
\mathbb{P}\left(\left|\sum_{j=1}^{m_{l}}\left(\left|\alpha_{j}^{\left(l\right)}\left(\mathbf{x}_{i}\right)\right|^{2}-\mathbb{E}\left[\left|\alpha_{j}^{\left(l\right)}\left(\mathbf{x}_{i}\right)\right|^{2}\right]\right)\right|\geq t\right) & \leq2\exp\left(-c_{2}\min\left(\frac{t^{2}}{m_{l}K^{2}},\frac{t}{K}\right)\right)\\
\mathbb{P}\left(\left|\sum_{j=1}^{m_{l}}\left(\left|\alpha_{j}^{\left(l\right)}\left(\mathbf{x}_{i}\right)\right|^{2}\right)-\mathbb{E}\left[\sum_{j=1}^{m_{l}}\left|\alpha_{j}^{\left(l\right)}\left(\mathbf{x}_{i}\right)\right|^{2}\right]\right|\geq t\right) & \leq2\exp\left(-c_{2}\min\left(\frac{t^{2}}{m_{l}K^{2}},\frac{t}{K}\right)\right)\\
\mathbb{P}\left(\left|\left\Vert \boldsymbol{\alpha}^{\left(l\right)}\left(\mathbf{x}_{i}\right)\right\Vert _{2}^{2}-\mathbb{E}\left\Vert \boldsymbol{\alpha}^{\left(l\right)}\left(\mathbf{x}_{i}\right)\right\Vert _{2}^{2}\right|\geq t\right) & \leq2\exp\left(-c_{2}\min\left(\frac{t^{2}}{m_{l}K^{2}},\frac{t}{K}\right)\right)\,,
\end{align*}
Combining this with \cref{eq:mean_between_two}, we get for all $i\in\left[n\right]$
with probability at least $1-2\exp\left(-c_{2}\min\left(\frac{t^{2}}{m_{l}K^{2}},\frac{t}{K}\right)\right)$,
\[
\left| \left\Vert \boldsymbol{\alpha}^{\left(l\right)}\left(\mathbf{x}_{i}\right)\right\Vert _{2}^{2} - m_l c_{\phi,\sigma_{0}} \right| \leq m_l \alpha_{\phi}^{2}\sigma_{0}^{2}\frac{h_{C}\left(l-1\right)}{2h_{C}\left(L\right)} + t \,.
\]
This is where our proof diverges from the original proof. By selecting
$t=\frac{c_{\phi,\sigma_{0}}m_{l}}{2h_{C}\left(L\right)}$, and remembering that 
$\nu_{0}^{2} = \frac{\sigma_{0}^{2}}{c_{\phi,\sigma_{0}}}$
we get
\[
\left| \left\Vert \boldsymbol{\alpha}^{\left(l\right)}\left(\mathbf{x}_{i}\right)\right\Vert _{2}^{2} - m_l c_{\phi,\sigma_{0}} \right| \leq m_l  c_{\phi,\sigma_{0}}\left( \alpha_{\phi}^{2}\nu_{0}^{2}\frac{h_{C}\left(l-1\right)}{2h_{C}\left(L\right)} + \frac{1}{2h_{C}\left(L\right)} \right) \,.
\]
Now applying $h_{C}\left(l\right)=1+\alpha_{\phi}^{2}\nu_{0}^{2}h_{C}\left(l-1\right)$,
we get
\[
m_{l}c_{\phi,\sigma_{0}}\left(1-\frac{h_{C}\left(l\right)}{2h_{C}\left(L\right)}\right)\leq\left\Vert \boldsymbol{\alpha}^{\left(l\right)}\left(\mathbf{x}_{i}\right)\right\Vert _{2}^{2}\leq m_{l}c_{\phi,\sigma_{0}}\left(1+\frac{h_{C}\left(l\right)}{2h_{C}\left(L\right)}\right)\,.
\]
The probability of this event is at least 
\begin{align*}
& 1-2\exp\left(-c_{2}\min\left(\frac{t^{2}}{m_{l}K^{2}},\frac{t}{K}\right)\right) \\
& =1-2\exp\left(-c_{2}\min\left(\frac{c_{\phi,\sigma_{0}}^{2}m_{l}^{2}}{4h_{C}\left(L\right)^{2}}\frac{1}{m_{l}\cdot9c^{4}\alpha_{\phi}^{4}\sigma_{0}^{4}},\frac{c_{\phi,\sigma_{0}}m_{l}}{2h_{C}\left(L\right)}\frac{1}{3c^2\alpha_{\phi}^2\sigma_{0}^{2}}\right)\right)\\
 & =1-2\exp\left(-c_{2}\frac{m_{l}}{6c^2\alpha_{\phi}^2h_{C}\left(L\right)\nu_{0}^{2}}\min\left(\frac{1}{6c^2\alpha_{\phi}^2h_{C}\left(L\right)\nu_{0}^{2}},1\right)\right)\\
 & \geq1-2\exp\left(-c_{3}\frac{m_{l}}{h_{C}\left(L\right)\max\left[h_{C}\left(L\right),1\right]}\right)\,,
\end{align*}
for some constant $c_{3}>0$. Applying union bound over all layers
(reminding that $\forall l\in\left[L\right],\,m_{l}\geq m$) and all
$i\in\left[n\right]$, we get that the probability of the desired
event is at least $1-2nL\exp\left(-c_{3}\frac{m}{h_{C}\left(L\right)\max\left[h_{C}\left(L\right),1\right]}\right).$
\end{proof}
Similarly to \citet{banerjee2023neural}, this result leads to an upper bound on
$\left\Vert \mathbf{A}^{\left(l\right)}\right\Vert _{\mathrm{F}}^{2}$
for all $l\in\left[L\right]$:
\[
\left\Vert \mathbf{A}^{\left(l\right)}\right\Vert _{\mathrm{F}}^{2}=\sum_{i\in\left[n\right]}\left\Vert \boldsymbol{\alpha}^{\left(l\right)}\left(\mathbf{x}_{i}\right)\right\Vert _{2}^{2}\leq\frac{3}{2}c_{\phi,\sigma_{0}}nm_{l}\,,
\]
with probability at least $1-2nL\exp\left(-c_{3}\frac{m}{h_{C}\left(L\right)\max\left[h_{C}\left(L\right),1\right]}\right)$.
Conditioned on this occurring, a sufficient condition on $m_l$ for \cref{lem:lemma_4.1_equivalent} to hold is that for all $l\in\left[L\right]$,
\begin{align*}
m_{l} & =\Omega\left(\max\left[n,\,\ln\left(\frac{nL}{\delta}\right)\frac{1}{\lambda_{l}}\frac{s_{l}^{2}}{a_{l}}nm_{l-1}\max\left[1,\,\ln\left(\frac{1}{\lambda_{l}}\frac{s_{l}^{2}}{a_{l}}nm_{l-1}\right)\right]\right]\right)\\
 & =\Omega\left(\max\left[n,\,\ln\left(\frac{nL}{\delta}\right)\frac{n}{\lambda_{l}}\max\left[1,\,\ln\left(\frac{n}{\lambda_{l}}\right)\right]\right]\right)\,.
\end{align*}

Since we assumed $L=O\left(1\right)$, we have $h_{C}\left(L\right)=O\left(1\right)$,
so this demand on $m\triangleq\min_{l\in\left[L\right]}m_{l}$ is
enough for \cref{lem:lemma_4.2_equivalent} to hold with probability
at least $1-\exp\left(-c_{4}m\right)$.

Their Lemma 4.3 can be used without proving again, since by assuming
$\forall l\in\left[L\right],\,\frac{a_{l}}{s_{l}^{2}}=m_{l-1}$, the
definition of $\lambda_{l},\,l\in\left[L\right]$ can be reparameterized
to the definition in their proof:
\begin{align*}
\lambda_{l} & \triangleq\lambda_{\min}\left(\mathbb{E}_{\mathbf{g}\sim\mathcal{N}\left(\mathbf{0}_{m_{l-1}},\frac{\nu_{0}^{2}}{a_{l}}\mathbf{I}_{m_{l-1}}\right)}\left[\phi\left(s_{l}\mathbf{A}^{\left(l-1\right)}\mathbf{g}\right)\phi\left(s_{l}\mathbf{A}^{\left(l-1\right)}\mathbf{g}\right)^{\top}\right]\right)\\
 & =\lambda_{\min}\left(\mathbb{E}_{\mathbf{g}\sim\mathcal{N}\left(\mathbf{0}_{m_{l-1}},\nu_{0}^{2}\mathbf{I}_{m_{l-1}}\right)}\left[\phi\left(\frac{s_{l}}{\sqrt{a_{l}}}\mathbf{A}^{\left(l-1\right)}\mathbf{g}\right)\phi\left(\frac{s_{l}}{\sqrt{a_{l}}}\mathbf{A}^{\left(l-1\right)}\mathbf{g}\right)^{\top}\right]\right)\\
 & =\lambda_{\min}\left(\mathbb{E}_{\mathbf{g}\sim\mathcal{N}\left(\mathbf{0}_{m_{l-1}},\nu_{0}^{2}\mathbf{I}_{m_{l-1}}\right)}\left[\phi\left(\frac{1}{\sqrt{m_{l-1}}}\mathbf{A}^{\left(l-1\right)}\mathbf{g}\right)\phi\left(\frac{1}{\sqrt{m_{l-1}}}\mathbf{A}^{\left(l-1\right)}\mathbf{g}\right)^{\top}\right]\right)\,.
\end{align*}
We also note that their proof of their Lemma 4.3 relies on having the events
of \cref{lem:lemma_4.1_equivalent}, \cref{lem:lemma_4.2_equivalent}
occur for every layer $l\in\left[L\right]$, so it inherits their
probabilities through a union bound. To summarize, we can cite an
adjusted version of their Lemma 4.3:
\begin{proposition}
\label{lem:lemma_4.3_equivalent}
Assume $m_{l}=\Omega\left(\max\left[n,\,\ln\left(\frac{nL}{\delta}\right)\frac{n}{\lambda_{1}}\max\left[1,\,\ln\left(\frac{n}{\lambda_{1}}\right)\right]\right]\right),\,\forall l\in\left[L\right]$.
We then have for all $l\in\left[L\right]$,
\[
\lambda_{l}\triangleq\lambda_{\min}\left(\mathbb{E}_{\mathbf{g}\sim\mathcal{N}\left(\mathbf{0}_{m_{l-1}},\frac{\nu_{0}^{2}}{a_{l}}\mathbf{I}_{m_{l-1}}\right)}\left[\phi\left(s_{l}\mathbf{A}^{\left(l-1\right)}\mathbf{g}\right)\phi\left(s_{l}\mathbf{A}^{\left(l-1\right)}\mathbf{g}\right)^{\top}\right]\right)\geq\tilde{c}_{l}\lambda_{1}\,,
\]
with probability at least $1-\delta-\exp\left(-c_{4}m\right)$, where
$\tilde{c}_{l}$ is a constant such that $\tilde{c}_{l}=\Theta\left(1\right)$
and $\lambda_{1}\triangleq\lambda_{\min}\left(\mathbb{E}_{\mathbf{g}\sim\mathcal{N}\left(\mathbf{0}_{d},\frac{\nu_{0}^{2}}{a_{1}}\mathbf{I}_{d}\right)}\left[\phi\left(s_{1}\mathbf{X}\mathbf{g}\right)\phi\left(s_{1}\mathbf{X}\mathbf{g}\right)^{\top}\right]\right).$
\end{proposition}
\subsection{Conclusion}
We are now ready to prove the main result of this section:
\begin{lemma}
\label{lem:lower_bound_ntk_deep_in_appendix}
For $m_{l}=\Omega\left(\max\left[n,\,\ln\left(\frac{nL}{\delta}\right)\frac{n}{\lambda_{1}} \ln\left(\frac{n}{\lambda_{1}}\right)\right]\right),\,\forall l\in\left[L\right]$,
we have with probability at least $1-\delta-\exp\left(-c_{1}m\right)$
over the initialization of $\left\{ \mathbf{W}^{\left(l\right)}\right\} _{l\in\left[L\right]}$,
\[
\lambda_{\min}\left(\mathbf{K}_{\mathrm{NTK}}\left(\boldsymbol{\theta}\right)\right)\geq c_{0}s_{L+1}^{2}m\lambda_{1}\,.
\]
\end{lemma}
\begin{proof}
First, we note that $\lambda_1$ is upper bounded (similarly to \citet{nguyen2020global}):
\begin{equation}
\begin{aligned}
\lambda_{1}&=\lambda_{\min}\left(\mathbb{E}_{\mathbf{g}\sim\mathcal{N}\left(\mathbf{0}_{d},\frac{\nu_{0}^{2}}{a_1}\mathbf{I}_{d}\right)}\left[\phi\left(s_1\mathbf{X}\mathbf{g}\right)\phi\left(s_1\mathbf{X}\mathbf{g}\right)^{\top}\right]\right) \\
& \leq \frac{1}{n} \mathrm{Tr} \left(\mathbb{E}_{\mathbf{g}\sim\mathcal{N}\left(\mathbf{0}_{d},\frac{\nu_{0}^{2}}{a_1}\mathbf{I}_{d}\right)}\left[\phi\left(s_1\mathbf{X}\mathbf{g}\right)\phi\left(s_1\mathbf{X}\mathbf{g}\right)^{\top}\right]\right) \\
& = \frac{1}{n} \mathbb{E}_{\mathbf{g}\sim\mathcal{N}\left(\mathbf{0}_{d},\frac{\nu_{0}^{2}}{a_1}\mathbf{I}_{d}\right)} \left\Vert \phi\left(s_1\mathbf{X}\mathbf{g}\right) \right\Vert^2  \\
& \leq \frac{1}{n} s_1^2 \alpha_\phi^2 \mathbb{E}_{\mathbf{g}\sim\mathcal{N}\left(\mathbf{0}_{d},\frac{\nu_{0}^{2}}{a_1}\mathbf{I}_{d}\right)}\left\Vert \mathbf{X}\mathbf{g} \right\Vert^2
=\frac{1}{n} s_1^2 \alpha_\phi^2 n c_{\phi,\sigma_0}d\,\mathbb{E}_{z\sim\mathcal{N}\left(0,\frac{\nu_{0}^{2}}{a_1}\right)}\left| z \right|^2 \\
& =   \alpha_\phi^2  \nu_{0}^{2} c_{\phi,\sigma_0} \frac{s_1^2 d}{a_1}=\alpha_\phi^2  \nu_{0}^{2} c_{\phi,\sigma_0} = O \prn{1}\,. \label{eq:lambda1_upper_bound}
\end{aligned}
\end{equation}
Then, $\ln \prn {\frac{n}{\lambda_1}} \geq 1$ for modestly large $n$, simplifying the width demands of the above propositions to the form presented here.

By plugging \cref{lem:lemma_4.3_equivalent} into \cref{lem:lemma_4.1_equivalent}, we have for all $l\in\left[L\right]$,
\[
\lambda_{\min}\left(\mathbf{A}^{\left(l\right)}\left(\mathbf{A}^{\left(l\right)}\right)^{\top}\right)\geq\frac{m_{l}\lambda_{l}}{4}\geq c_{0}m\lambda_{1}\,,
\]
with probability at least $1-\delta-\exp\left(-c_{1}m\right)$. Then,
similarly to \citet{banerjee2023neural}, we observe that 
\begin{align*}
\mathbf{K}_{\mathrm{NTK}}\left(\boldsymbol{\theta}\right) & =\nabla f \prn{\boldsymbol{\theta}} \nabla f \prn{\boldsymbol{\theta}}^{\top}=\sum_{l\in\left[L\right]}\left[\frac{\partial f\left(\boldsymbol{\theta}\right)}{\partial\mathrm{vec}\left(\mathbf{W}^{\left(l\right)}\right)}\right]\left[\frac{\partial f\left(\boldsymbol{\theta}\right)}{\partial\mathrm{vec}\left(\mathbf{W}^{\left(l\right)}\right)}\right]^{\top} \! 
+ \left[\frac{\partial f\left(\boldsymbol{\theta}\right)}{\partial\mathbf{v}}\right]\left[\frac{\partial f\left(\boldsymbol{\theta}\right)}{\partial\mathbf{v}}\right]^{\top}\\
 & \succeq\left[\frac{\partial f\left(\boldsymbol{\theta}\right)}{\partial\mathbf{v}}\right]\left[\frac{\partial f\left(\boldsymbol{\theta}\right)}{\partial\mathbf{v}}\right]^{\top}=s_{L+1}^{2}\mathbf{A}^{\left(L\right)}\left(\mathbf{A}^{\left(L\right)}\right)^{\top}\,,
\end{align*}
so, by the Rayleigh-Ritz characterization of the minimum eigenvalue
of a symmetric matrix, 
\begin{align*}
\lambda_{\min}\left(\mathbf{K}_{\mathrm{NTK}}\left(\boldsymbol{\theta}\right)\right) & =\min_{\|\mathbf{u}\|=1}\mathbf{u}^{\top}\mathbf{K}_{\mathrm{NTK}}\left(\boldsymbol{\theta}\right)\mathbf{u}=\mathbf{u}_{\star}^{\top}\mathbf{K}_{\mathrm{NTK}}\left(\boldsymbol{\theta}\right)\mathbf{u}_{\star}\\
 & \geq \mathbf{u}_{\star}^{\top}s_{L+1}^{2}\mathbf{A}^{\left(L\right)}\left(\mathbf{A}^{\left(L\right)}\right)^{\top}\mathbf{u}_{\star}\geq s_{L+1}^{2}\lambda_{\min}\left(\mathbf{A}^{\left(L\right)}\left(\mathbf{A}^{\left(L\right)}\right)^{\top}\right)\\
 & \geq c_{0}s_{L+1}^{2}m\lambda_{1}\,.
\end{align*}
\end{proof}
By choosing $\delta=\exp \prn{-\tilde{c}\frac{m}{ \frac{n}{\lambda_1}\ln \frac{n}{\lambda_1}}}$ for some $\tilde{c}\in\prn{0,1}$, and recalling that $L=O\prn{1}$, the bound applies for $m=\Omega \prn{\frac{n}{\lambda_1} \ln \frac{n}{\lambda_1} \ln n}$ with probability at least $1-\exp \prn{-\tilde{c}\frac{m}{ \frac{n}{\lambda_1}\ln \frac{n}{\lambda_1}}}-\exp(-c_1 m) \geq 1-\exp \prn{-\tilde{c}_1\frac{m}{ \frac{n}{\lambda_1}\ln \frac{n}{\lambda_1}}}$ for a suitable constant $\tilde{c}_1$, yielding \cref{thm:lower_bound_ntk_deep}.
\newpage

\section{Auxiliary claims}
\label{app:supporting_claims}
\begin{claim}
\label{claim:zero_columns_smaller}
Let $\A \in \mathbb{R}^{n \times m}$ and let $\B \in \mathbb{R}^{n \times m}$ be a matrix such that each column of $\B$ is either the zero vector or equal to the corresponding column of $\A$. Then
\[
\lambda_{\min}(\B\B^\top) \le \lambda_{\min}(\A\A^\top)\,.
\]
\end{claim}
\begin{proof}
Since each column of $\B$ is either zero or equal to the corresponding column of $\A$, there exists a diagonal matrix
\[
\SM = \mathrm{diag}(s_1,\dots,s_n)\,, \quad s_i \in \{0,1\}\,,
\]
such that
\[
\B = \A\SM\,.
\]
Then
\[
\B\B^\top = (\A\SM)(\A\SM)^\top = \A\SM\SM^\top \A^\top = \A\SM^2\A^\top = \A\SM\A^\top,
\]
where we used the facts that $\SM^\top = \SM$ and $\SM^2 = \SM$.
Consider the difference
\[
\A\A^\top - \B\B^\top = \A\A^\top - \A\SM\A^\top = \A(\I - \SM)\A^\top.
\]
The matrix $\I - \SM$ is diagonal with entries $1 - s_i \in \{0,1\}$, and is therefore positive semidefinite. Hence, for any $\x \in \mathbb{R}^m$,
\[
\x^\top \A(\I - \SM)\A^\top \x = (\A^\top \x)^\top (\I - \SM)(\A^\top \x) \ge 0\,,
\]
leading to
\[
\A\A^\top \succeq \B\B^\top\,.
\]
Finally, by the Rayleigh-Ritz characterization of the minimum eigenvalue of a symmetric matrix,
\[
\lambda_{\min}(\M) = \min_{\|\x\|=1} \x^\top \M \x\,.
\]
Since $\B\B^\top \preceq \A\A^\top$ implies $\x^\top \B\B^\top \x \le \x^\top \A\A^\top \x$ for all $\x$, taking the minimum over all unit vectors yields
\[
\lambda_{\min}(\B\B^\top)
= \min_{\|\x\|=1} \x^\top \B\B^\top \x
\le
\min_{\|\x\|=1} \x^\top \A\A^\top \x
= \lambda_{\min}(\A\A^\top)\,.
\]
\end{proof}
\begin{claim}
\label{prop:bounded moments of gaussian vector}
Let $\mathbf{z}\sim\mathcal{N}\left(\mathbf{0}_{m},\mathbf{I}_{m}\right)$.
Then, for any $p\geq1$,
\[
\left\Vert \left\Vert \mathbf{z}\right\Vert _{2}\right\Vert _{L^{p}}\triangleq\left(\mathbb{E}\left[\left\Vert \mathbf{z}\right\Vert _{2}^{p}\right]\right)^{1/p}\leq C\sqrt{p}+\sqrt{m}\,,
\]
where $C>0$ is a global constant.
\end{claim}
\begin{proof}
From \citet[Theorem 3.1.1 and Exercise 2.24(d)]{vershynin2018high} we have $\left\Vert \left\Vert \mathbf{z}\right\Vert _{2}-\sqrt{m}\right\Vert _{\psi_{2}}\leq\frac{8}{3}C_{1}$.
Then:
\begin{align*}
\left\Vert \left\Vert \mathbf{z}\right\Vert _{2}\right\Vert _{L^{p}} & \leq\left\Vert \left|\left\Vert \mathbf{z}\right\Vert _{2}-\sqrt{m}\right|+\sqrt{m}\right\Vert _{L^{p}}\\
\left[\text{Minkowski's inequality}\right] & \leq\left\Vert \left|\left\Vert \mathbf{z}\right\Vert _{2}-\sqrt{m}\right|\right\Vert _{L^{p}}+\left\Vert \sqrt{m}\right\Vert _{L^{p}}\\
\text{\citep[Proposition 2.6.6]{vershynin2018high}} & \leq C_{2}\frac{8}{3}C_{1}\sqrt{p}+\sqrt{m}=C\sqrt{p}+\sqrt{m}\,.
\end{align*}
\end{proof}
\begin{claim}
\label{claim:bounded moments of gaussian matrix} Let $\mathbf{G}\in\mathbb{R}^{m\times n}$
be a matrix with i.i.d. standard Gaussian entries. Then there exists
a global constant $c>0$ such that for any $p\geq1$,
\begin{align*}
\left\Vert \left\Vert \mathbf{G}\right\Vert \right\Vert _{L^{p}} & \leq c\left(\sqrt{m}+\sqrt{n}+\sqrt{p}\right).
\end{align*}
As a result, there exists a constant $c'>0$ such that 
\begin{align*}
\left\Vert \left\Vert \mathbf{G}\right\Vert ^{2}\right\Vert _{L^{p}} & \leq c'\left(m+n+p\right).
\end{align*}
\end{claim}
\begin{proof}
From \citet[Theorem 4.4.3]{vershynin2018high}, there exists a constant
$c_{1}$ such that for any $t>0$,
\begin{align*}
\mathbb{P}\left(\left\Vert \mathbf{G}\right\Vert \geq c_{1}\left(\sqrt{m}+\sqrt{n}+t\right)\right)\leq & 2\exp\left(-t^{2}\right).
\end{align*}
Let $Y\triangleq\max\left\{ 0,\left\Vert \mathbf{G}\right\Vert -c_{1}\left(\sqrt{m}+\sqrt{n}\right)\right\} $.
Then 
\begin{align*}
\mathbb{P}\left(\left|Y\right|\geq t\right)\leq & 2\exp\left(-t^{2}/c_{1}^{2}\right).
\end{align*}
Therefore, by \citet[Proposition 2.6.1]{vershynin2018high},
$\left|Y\right|$ is subgaussian and $\left\Vert \left|Y\right|\right\Vert _{L^{p}}\leq c_{2}\sqrt{p}$
for some $c_{2}>0$ and any $p\geq1$. Using Minkowski's inequality,
we get 
\begin{align*}
\left\Vert \left\Vert \mathbf{G}\right\Vert \right\Vert _{L^{p}} & \leq\left\Vert Y+c_{1}\left(\sqrt{m}+\sqrt{n}\right)\right\Vert _{L^{p}}\\
 & \leq c_{1}\left(\sqrt{m}+\sqrt{n}\right)+\left\Vert \left|Y\right|\right\Vert _{L^{p}}\leq\max\left\{ c_{1},c_{2}\right\} \left(\sqrt{m}+\sqrt{n}+\sqrt{p}\right).
\end{align*}
Using \cref{claim:moment of power} and applying the inequality
$\left(a+b\right)^{2}\leq2a^{2}+2b^{2}$ twice, we get
\begin{align*}
\left\Vert \left\Vert \mathbf{G}\right\Vert ^{2}\right\Vert _{L^{p}}\leq\left\Vert \left\Vert \mathbf{G}\right\Vert \right\Vert _{L^{2p}}^{2} & \leq\left(\max\left\{ c_{1},c_{2}\right\} \right)^{2}\left(\sqrt{m}+\sqrt{n}+\sqrt{2p}\right)^{2}\\
 & \leq\left(\max\left\{ c_{1},c_{2}\right\} \right)^{2}\left(2m+4n+8p\right).
\end{align*}
\end{proof}
\begin{claim}[Tower property for $L^p$-norm]
\label{claim:moments total expectation}
Given some RV $X$ and some sub-$\sigma$-algebra $\mathcal{F}$, we have for any $p>0$:
\[
\left\Vert X\right\Vert _{L^{p}}=\left\Vert \left\Vert X\mid\mathcal{F}\right\Vert _{L^{p}}\right\Vert _{L^{p}}\,.
\]
\end{claim}
\begin{proof}
Observe
\begin{align*}
\left\Vert \left\Vert X\mid\mathcal{F}\right\Vert _{L^{p}}\right\Vert _{L^{p}} & =\left(\mathbb{E}\left[\left(\left\Vert X\mid\mathcal{F}\right\Vert _{L^{p}}\right)^{p}\right]\right)^{1/p}=\left(\mathbb{E}\left[\left(\left(\mathbb{E}\left[\left|X\right|^{p}\mid\mathcal{F}\right]\right)^{1/p}\right)^{p}\right]\right)^{1/p}\\
 & =\left(\mathbb{E}\left[\mathbb{E}\left[\left|X\right|^{p}\mid\mathcal{F}\right]\right]\right)^{1/p}=\left(\mathbb{E}\left[\left|X\right|^{p}\right]\right)^{1/p}=\left\Vert X\right\Vert _{L^{p}}\,.
\end{align*}
\end{proof}
\begin{claim}
\label{claim:total_expectation_subgaussian}
Suppose that, for some deterministic $K>0$, some RV $X$ and some
$\sigma$-algebra $\mathcal{F}$, we have that
\[
\left\Vert X\mid\mathcal{F}\right\Vert _{\psi_{2}}\leq K
\]
almost surely. Then,
\[
\left\Vert X\right\Vert _{\psi_{2}}\leq c K\,.
\]
\end{claim}

\begin{proof}
By \citet[Proposition 2.6.6]{vershynin2018high}, for all $p\geq1$,
\[
\left\Vert X\mid\mathcal{F}\right\Vert _{L^{p}}\leq CK\sqrt{p}\,.
\]
Then,
\[
\left\Vert \left\Vert X\mid\mathcal{F}\right\Vert _{L^{p}}\right\Vert _{L^{p}}\leq\left\Vert CK\sqrt{p}\right\Vert _{L^{p}}=CK\sqrt{p}\,.
\]
By \cref{claim:moments total expectation},
\[
\left\Vert X\right\Vert _{L^{p}}=\left\Vert \left\Vert X\mid\mathcal{F}\right\Vert _{L^{p}}\right\Vert _{L^{p}}\leq CK\sqrt{p}\,,
\]
for all $p\geq1$. Hence, by \citet[Proposition 2.6.6]{vershynin2018high},
\[
\left\Vert X\right\Vert _{\psi_{2}}\leq c K\,.
\]
\end{proof}
\begin{claim}
\label{claim:moment of power}
Given some RV $X$, we have for any
$p>0$ and $k>0$:
\[
\left\Vert X^{k}\right\Vert _{L^{p}}=\left\Vert X\right\Vert _{L^{kp}}^{k}\,.
\]
\end{claim}
\begin{proof}
Observe
\[
\left\Vert X^{k}\right\Vert _{L^{p}}=\left(\mathbb{E}\left[\left(X^{k}\right)^{p}\right]\right)^{1/p}=\left(\left(\mathbb{E}\left[X^{kp}\right]\right)^{1/\left(kp\right)}\right)^{k}=\left\Vert X\right\Vert _{L^{kp}}^{k}\,.
\]
\end{proof}
\begin{claim}[H\"older's inequality for $L^{p}$-norm]
\label{claim:holder_Lp_norm}
For random variables $X,\,Y$
and $p>0$, for any $r,s\geq1$ such that $\frac{1}{r}+\frac{1}{s}=1$,
\[
\left\Vert XY\right\Vert_{L^{p}}\leq\left\Vert X\right\Vert_{L^{pr}}\left\Vert Y\right\Vert_{L^{ps}}\,.
\]
\end{claim}
\begin{proof}
Observe
\begin{align*}
\left\Vert XY\right\Vert_{L^{p}} & =\left(\mathbb{E}\left[\left(\left|XY\right|\right)^{p}\right]\right)^{1/p}=\left(\mathbb{E}\left[\left|X\right|^{p}\left|Y\right|^{p}\right]\right)^{1/p}\\
\left[\text{H\"older}\right] & \leq\left(\left(\mathbb{E}\left[\left|X\right|^{pr}\right]\right)^{1/r}\left(\mathbb{E}\left[\left|Y\right|^{ps}\right]\right)^{1/s}\right)^{1/p}=\left\Vert X\right\Vert_{L^{pr}}\left\Vert Y\right\Vert_{L^{ps}}\,.
\end{align*}
\end{proof}
\begin{proposition}
\label{prop:mgf_of_init_loss}
For the neural network defined in \cref{sec:setting_deep_nn}, under the weight initialization and assumptions therein, we have that the loss at initialization satisfies 
\begin{align*}
\mathcal{L} \left(\boldsymbol{\mathbf{\theta}}_{0}\right)  \leq R_{y}^{2}+\frac{1}{n}\sum_{i\in\left[n\right]}f_{i}^{2}\left(\boldsymbol{\mathbf{\theta}}_{0}\right)\,,
\end{align*}
and
\[
\left\Vert \frac{1}{n}\sum_{i\in\left[n\right]}f_{i}^{2}\left(\boldsymbol{\mathbf{\theta}}_{0}\right)\right\Vert _{\psi_{1}}\leq c M_{\phi}^{2}\frac{s_{L+1}^{2}{m_L}}{a_{L+1}}\,,
\]
where $c>0$ is a global constant.
Then, denoting $L_\psi \triangleq c M_{\phi}^{2}\frac{s_{L+1}^{2}{m_L}}{a_{L+1}}$, for $\xi\leq L_\psi^{-1}$,
\[
\mathbb{E}\left[e^{\xi \mathcal{L} \left(\boldsymbol{\mathbf{\theta}}_{0}\right)}\right]\leq2e^{\xi R_{y}^{2}}\,.
\]
\end{proposition}
\begin{proof}
\begin{align*}
\mathcal{L} \left(\boldsymbol{\mathbf{\theta}}_{0}\right) & =\frac{1}{2n}\sum_{i\in\left[n\right]}\left(f_{i}\left(\boldsymbol{\mathbf{\theta}}_{0}\right)-y_{i}\right)^{2}\leq\frac{1}{2n}\sum_{i\in\left[n\right]}\left(2f_{i}^{2}\left(\boldsymbol{\mathbf{\theta}}_{0}\right)+2y_{i}^{2}\right)\\
 & \leq R_{y}^{2}+\frac{1}{n}\sum_{i\in\left[n\right]}f_{i}^{2}\left(\boldsymbol{\mathbf{\theta}}_{0}\right)\,,
\end{align*}
where for all $i \in\left[n\right]$,
\begin{align*}
f_{i}^{2}\left(\boldsymbol{\theta}_{0}\right) & =\left(s_{L+1}\mathbf{v}^{\top}\phi\left(s_{L}\mathbf{W}^{\left(L\right)}\boldsymbol{\alpha}^{\left(L-1\right)}\left(\mathbf{x}_{i}\right)\right)\right)^{2}\\
 & =s_{L+1}^{2}\left(\sum_{j\in\left[m_{L}\right]}v_{j}\phi\left(s_{L}\left(\mathbf{w}_{j}^{\left(L\right)}\right)^{\top}\boldsymbol{\alpha}^{\left(L-1\right)}\left(\mathbf{x}_{i}\right)\right)\right)^{2}\,.
\end{align*}
Conditioning on all weights except $\mathbf{v}$, we have that 
\[
\phi\left(s_{L}\left(\mathbf{w}_{j}^{\left(L\right)}\right)^{\top}\boldsymbol{\alpha}^{\left(L-1\right)}\left(\mathbf{x}_{i}\right)\right),\,j\in\left[m_{L}\right]
\]
are fixed. Moreover, $v_{j},\,j\in\left[m_{L}\right]$ are independent,
zero-mean, and subgaussian:
\[
\left\Vert v_{j}\right\Vert _{\psi_{2}}\leq C\frac{1}{\sqrt{a_{L+1}}}\,.
\]
Then, conditioned on all weights except $\mathbf{v}$, we apply \citet[Proposition 2.7.1]{vershynin2018high}
\begin{align*}
& \left\Vert s_{L+1}\left(\sum_{j\in\left[m_{L}\right]}v_{j}\phi\left(s_{L}\left(\mathbf{w}_{j}^{\left(L\right)}\right)^{\top}\boldsymbol{\alpha}^{\left(L-1\right)}\left(\mathbf{x}_{i}\right)\right)\right)\mid\left\{ \mathbf{W}^{\left(l\right)}\right\} _{l\in\left[L\right]}\right\Vert _{\psi_{2}} \\
& \leq s_{L+1}\sqrt{C\sum_{j\in\left[m_{L}\right]}\left\Vert v_{j}\phi\left(s_{L}\left(\mathbf{w}_{j}^{\left(L\right)}\right)^{\top}\boldsymbol{\alpha}^{\left(L-1\right)}\left(\mathbf{x}_{i}\right)\right)\mid\left\{ \mathbf{W}^{\left(l\right)}\right\} _{l\in\left[L\right]}\right\Vert _{\psi_{2}}^{2}}\\
 & \leq C_{1}s_{L+1}\sqrt{\sum_{j\in\left[m_{L}\right]}\phi^{2}\left(s_{L}\left(\mathbf{w}_{j}^{\left(L\right)}\right)^{\top}\boldsymbol{\alpha}^{\left(L-1\right)}\left(\mathbf{x}_{i}\right)\right)\left\Vert v_{j}\right\Vert _{\psi_{2}}^{2}}\\
 & \leq C_{1}s_{L+1}\sqrt{\sum_{j\in\left[m_{L}\right]}M_{\phi}^{2}\frac{C_{2}}{a_{L+1}}}=C_{3}M_{\phi}s_{L+1}\sqrt{\frac{m_{L}}{a_{L+1}}}\,.
\end{align*}
Applying \cref{claim:total_expectation_subgaussian}, we can remove the conditioning
and get
\[
\left\Vert s_{L+1}\left(\sum_{j\in\left[m_{L}\right]}v_{j}\phi\left(s_{L}\left(\mathbf{w}_{j}^{\left(L\right)}\right)^{\top}\boldsymbol{\alpha}^{\left(L-1\right)}\left(\mathbf{x}_{i}\right)\right)\right)\right\Vert _{\psi_{2}}\leq C_{4}M_{\phi}s_{L+1}\sqrt{\frac{m_{L}}{a_{L+1}}}\,.
\]
Now, applying \citet[Proposition 2.8.1]{vershynin2018high}
\begin{align*}
&\left\Vert \left(s_{L+1}\mathbf{v}^{\top}\phi\left(s_{L}\mathbf{W}^{\left(L\right)}\boldsymbol{\alpha}^{\left(L-1\right)}\left(\mathbf{x}_{i}\right)\right)\right)^{2}\right\Vert _{\psi_{1}} \\
& =\left\Vert s_{L+1}\left(\sum_{j\in\left[m_{L}\right]}v_{j}\phi\left(s_{L}\left(\mathbf{w}_{j}^{\left(L\right)}\right)^{\top}\boldsymbol{\alpha}^{\left(L-1\right)}\left(\mathbf{x}_{i}\right)\right)\right)\right\Vert _{\psi_{2}}^{2}
\leq C_{5}M_{\phi}^{2}\frac{s_{L+1}^{2}m_{L}}{a_{L+1}}\,.
\end{align*}
We conclude by triangle inequality that 
\[
\left\Vert \frac{1}{n}\sum_{i\in\left[n\right]}f_{i}^{2}\left(\boldsymbol{\mathbf{\theta}}_{0}\right)\right\Vert _{\psi_{1}}\leq C_{5}M_{\phi}^{2}\frac{s_{L+1}^{2}m_L}{a_{L+1}}\,.
\]
Now note that for all $\xi>0$,
\begin{align*}
\mathbb{E}\left[e^{\xi \mathcal{L} \left(\boldsymbol{\mathbf{\theta}}_{0}\right)}\right] & \leq e^{\xi R_{y}^{2}}\mathbb{E}\left[\exp\left(\xi\frac{1}{n}\sum_{i\in\left[n\right]}f_{i}^{2}\left(\boldsymbol{\mathbf{\theta}}_{0}\right)\right)\right]\,.
\end{align*}
By choosing $\xi\leq\frac{1}{\left\Vert \frac{1}{n}\sum_{i\in\left[n\right]}f_{i}^{2}\left(\boldsymbol{\mathbf{\theta}}_{0}\right)\right\Vert _{\psi_{1}}},$we
have from \citet[Proposition 2.8.1]{vershynin2018high} that
\[
\mathbb{E}\left[\exp\left(\xi\frac{1}{n}\sum_{i\in\left[n\right]}f_{i}^{2}\left(\boldsymbol{\mathbf{\theta}}_{0}\right)\right)\right]\leq2\,.
\]
By choosing $\xi
\leq \frac{1}{C_{5}M_{\phi}^{2}\frac{s_{L+1}^{2}m_L}{a_{L+1}}}
\leq \frac{1}{\left\Vert \frac{1}{n}\sum_{i\in\left[n\right]}f_{i}^{2}\left(\boldsymbol{\mathbf{\theta}}_{0}\right)\right\Vert _{\psi_{1}}}$, we arrive at the desired result.
\end{proof}
\begin{claim}
\label{claim:ode_inequality}
Given the ordinary differential inequality for $y\left(t\right)$,
\[
\dot{y}\leq-by+a\,,
\]
for some $a,b>0$, with a known $y\left(0\right)$, we have for all
$t\geq0$:
\[
y\left(t\right)\leq\left(y\left(0\right)-y_{\infty}\right)e^{-bt}+y_{\infty}\,,
\]
where $y_{\infty}\triangleq\frac{a}{b}$.
\end{claim}
\begin{proof}
Define $z\left(t\right)=y\left(t\right)-y_{\infty}$, where $y_{\infty}=\frac{a}{b}$.
Then,
\[
\dot{z}=\dot{y}\leq-by+a=-b\left(z+\frac{a}{b}\right)+a=-bz\,.
\]
The claim follows from Gr\"onwall's inequality; we include the elementary derivation for completeness.
Observe that
\[
\frac{\mathrm{d}}{\mathrm{d}t}\left(e^{bt}z\left(t\right)\right)=be^{bt}z\left(t\right)+e^{bt}\dot{z}\left(t\right)\leq be^{bt}z\left(t\right)-be^{bt}z\left(t\right)=0,
\]
hence $e^{bt}z\left(t\right)$ is non increasing, meaning that for
all $t\geq0$,
\begin{align*}
e^{bt}z\left(t\right) & \leq z\left(0\right)\\
\Longrightarrow e^{bt}\left(y\left(t\right)-y_{\infty}\right) & \leq y\left(0\right)-y_{\infty}\\
y\left(t\right) & \leq\left(y\left(0\right)-y_{\infty}\right)e^{-bt}+y_{\infty}\,.
\end{align*}
\end{proof}
\newpage

\section{Well-posedness of the dynamics and It\^o's formula}
\label{sec:ito_rigor}

\subsection{Global well-posedness and stationarity of regularized CLD}
\label{sec:cld_wellposed}

Two places in the paper use the fact that \cref{eq:regularized_langevin}
has a solution defined for all times whose law is absolutely continuous with
respect to the Gibbs measure, and that this measure is stationary for the
process: \cref{cor:second_law_for_langevin}, which applies
\cref{thm:second_law_change_of_measure}; 
and the
verification of the It\^o hypotheses below, which needs the
path to be bounded on compact time intervals. The vanishing of the stationary
probability flux does not by itself supply any of this, as
\cref{rem:flux_not_enough} explains. We therefore record the required
statement and derive it from the analytic theory of It\^o SDEs with
non-smooth coefficients developed by \citet{leestannattrutnau2022}. In this subsection, we refer to \citet{leestannattrutnau2022} whenever we use "they" or "their". 
We note that this proof was LLM-assisted (Claude), and was verified by the authors, both mathematically and for proper citations.

\begin{proposition}[Non-explosion, stationarity, and absolute continuity]
\label{prop:cld_wellposed}
Suppose $N_p\geq2$, $\mathcal{L}\in C^1\prn{\real^{N_p}}$ with
$\mathcal{L}\geq0$ and $\nabla\mathcal{L}$ locally Lipschitz, and
$\boldsymbol{\Lambda}\succ\mathbf{0}$. Let
$\vtheta_0\sim\nu\triangleq\mathcal{N}\prn{\mathbf{0}_{N_p},
\boldsymbol{\Lambda}^{-1}}$, independently of the driving Brownian motion.
Then \cref{eq:regularized_langevin} has a pathwise-unique global strong
solution, which is a time-homogeneous Markov process. The Gibbs measure
\begin{equation}
p_\infty\prn{\mathrm{d}\vtheta}
\propto e^{-\beta\mathcal{L}\prn{\vtheta}}
\phi_{\boldsymbol{\Lambda}}\prn{\vtheta}\,\mathrm{d}\vtheta
\label{eq:cld_stationary_law}
\end{equation}
is stationary for it, and its law $p_t$ satisfies $p_t\ll p_\infty\sim\nu$ for
every $t\geq0$.
\end{proposition}

The hypotheses are weaker than those imposed in the body: they ask only that
$\mathcal{L}$ be once continuously differentiable with a locally Lipschitz
gradient, so the proposition covers both the $C^2$ activations of
\cref{ass:activation_bounded_gradients} and their $C^1$ relaxation in
\cref{footnote:C1_activations}. No growth condition on $\nabla\mathcal{L}$ and
no condition on second derivatives of $\mathcal{L}$ appears anywhere in the
proof; the only global input is $\mathcal{L}\geq0$, through which the $L_2$
regularization makes $Z$ finite.

\begin{proof}
\textbf{The cited framework.}
\citet{leestannattrutnau2022} study, on $\real^{d}$ with $d\geq2$, the
It\^o equation
\begin{equation}
\mathbf{X}_t=\mathbf{x}_0+\int_0^t\mathbf{S}\prn{\mathbf{X}_s}
\,\mathrm{d}\mathbf{w}_s
+\int_0^t\mathbf{b}\prn{\mathbf{X}_s}\,\mathrm{d}s\,,
\qquad0\leq t<\zeta\,,
\label{eq:lst_sde}
\end{equation}
with $\zeta$ the lifetime of $\mathbf{X}$, whose generator on
$C^2\prn{\real^d}$ is
\begin{equation}
\mathcal{G}f=\tfrac12\mathrm{Tr}\prn{\mathbf{A}\nabla^2f}
+\mathbf{b}^\top\nabla f\,,
\qquad
\mathbf{A}\triangleq\mathbf{S}\mathbf{S}^\top\,,
\label{eq:lst_generator}
\end{equation}
together with a measure $\mu$ having Lebesgue density $\rho$ that is
\emph{infinitesimally invariant}, meaning $\int\mathcal{G}f\,\mathrm{d}\mu=0$
for every $f\in C_c^\infty\prn{\real^d}$, where $C_c^\infty\prn{\real^d}$ are the smooth functions on $\real^d$
that vanish outside a compact set. For such a density and diffusion
matrix they define the logarithmic derivative
\begin{equation}
\boldsymbol{\beta}^{\rho,\mathbf{A}}\triangleq\tfrac12\nabla\mathbf{A}
+\frac{\mathbf{A}\nabla\rho}{2\rho}\,,
\qquad
\prn{\nabla\mathbf{A}}_i\triangleq\sum_{j=1}^d\partial_jA_{ij}\,,
\label{eq:lst_logderiv}
\end{equation}
where $\nabla\mathbf{A}$ is the row-wise divergence of $\mathbf{A}$, and call
$\mathbf{b}=\boldsymbol{\beta}^{\rho,\mathbf{A}}$ the
\emph{symmetric} case. Their condition~\textbf{(c)}
\citep[Section~3.3]{leestannattrutnau2022} asks, for some $p\in\prn{d,\infty}$,
that $S_{ij}\in H^{1,p}_{\mathrm{loc}}\prn{\real^d}\cap C\prn{\real^d}$ for all
$i,j$, that $\mathbf{A}$ be locally strictly elliptic, and that
$\mathbf{b}\in L^p_{\mathrm{loc}}\prn{\real^d,\real^d}$; it implies their
conditions~\textbf{(a)} and~\textbf{(b)}. Here $L^p_{\mathrm{loc}}$ is the
space of functions that are $p$-integrable on every compact subset of
$\real^d$, $H^{1,p}_{\mathrm{loc}}$ the space of functions lying in
$L^p_{\mathrm{loc}}$ together with their first-order weak derivatives,
$C\prn{\real^d}$ the continuous functions, and \emph{locally strictly
elliptic} means that on every relatively compact open $V\subset\real^d$ the
quadratic form 
$\boldsymbol{\xi}^\top \mathbf{A} \boldsymbol{\xi}$
lies between two positive multiples of $\norm{\boldsymbol{\xi}}^2$, uniformly on
$V$. Write $\prn{T_t}_{t>0}$ for the semigroup of $\mathcal{G}$ on
$L^s\prn{\real^d,\mu}$, the functions $s$-integrable against $\mu$,
$s\in\sqprn{1,\infty}$
\citep[Definition~2.7]{leestannattrutnau2022}; it is sub-Markovian, so
$T_t1\leq1$, with equality not automatic. Under~\textbf{(a)} and~\textbf{(b)}
they construct a Hunt process $\mathbb{M}$, \ie a strong Markov process, here
with continuous sample paths, with lifetime $\zeta$ and laws
$\prn{\mathbb{P}_{\mathbf{x}_0}}_{\mathbf{x}_0\in\real^d}$, where
$\mathbb{P}_{\mathbf{x}_0}$ is the law of $\mathbb{M}$ started at
$\mathbf{x}_0$. Its transition semigroup $\prn{P_t}_{t\geq0}$ is a pointwise
version of $\prn{T_t}_{t>0}$, and
$P_t\prn{\mathbf{x}_0,E}=\mathbb{P}_{\mathbf{x}_0}\prn{\mathbf{X}_t\in E}$
for every $\mathbf{x}_0\in\real^d$, $t\geq0$ and Borel $E\subseteq\real^d$
\citep[Theorems~2.31 and~3.11]{leestannattrutnau2022}.

Their Sections~2.3 and~3.1--3.3 assume in addition that the density of $\mu$ is
either produced from the coefficients as in their
Theorem~2.24 or prescribed in advance as in their
Remark~2.28. Ours is prescribed, so the second
applies: from any $\rho\in H^{1,p}_{\mathrm{loc}}\prn{\real^d}\cap
C\prn{\real^d}$ that is strictly positive, that remark builds admissible
operators indexed by a diffusion matrix, an antisymmetric matrix, and a
divergence-free field, and the member indexed by
$\prn{\mathbf{A},\mathbf{0},\mathbf{0}}$ has drift
$\boldsymbol{\beta}^{\rho,\mathbf{A}}$ and satisfies their
standing assumptions \citep[eqs.~(2.10)--(2.14)]{leestannattrutnau2022}.
Steps~1 and~2 below exhibit our equation as exactly this member, and verify
(2.10)--(2.14) directly, so that nothing here rests on an unchecked hypothesis.

We import four results, each stated under the assumptions just described:
\begin{enumerate}[topsep=2pt,itemsep=1pt]
\item[(R1)] \emph{Conservativeness}
\citep[Corollary~2.16(i)]{leestannattrutnau2022}: if
$A_{ij}\in L^1\prn{\mu}$ and $b_i-\beta_i^{\rho,\mathbf{A}}\in L^1\prn{\mu}$
for all $i,j$, then $T_t1=1$.
\item[(R2)] \emph{Non-explosion}
\citep[Corollary~3.23]{leestannattrutnau2022}: $\prn{T_t}_{t>0}$ is
conservative if and only if
$\mathbb{P}_{\mathbf{x}_0}\prn{\zeta=\infty}=1$ for every
$\mathbf{x}_0\in\real^d$.
\item[(R3)] \emph{Strong well-posedness}
\citep[Theorem~3.52]{leestannattrutnau2022}: under~\textbf{(c)} and
non-explosion, \cref{eq:lst_sde} has, for every $\mathbf{x}_0\in\real^d$, a
pathwise-unique strong solution defined for all times; uniqueness in law
holds, and that solution has the law of $\mathbb{M}$ started at
$\mathbf{x}_0$.
\item[(R4)] \emph{Absolute continuity and invariance}
\citep[Proposition~3.1(i)]{leestannattrutnau2022}:
$P_t\prn{\mathbf{x}_0,\cdot}$
is equivalent to $\mu$ for every $\mathbf{x}_0\in\real^d$ and $t>0$. Moreover, in the
symmetric case conservativeness makes $\mu$ invariant for the semigroup
\citep[Remark~2.13(i)]{leestannattrutnau2022} and hence for the process
\citep[Remark~3.44]{leestannattrutnau2022}.
\end{enumerate}

\medskip
\noindent\textbf{Step 1: the dictionary, and condition~\textbf{(c)}.}
We read \cref{eq:regularized_langevin} as \cref{eq:lst_sde} with $d=N_p\geq2$
and
\begin{equation}
\mathbf{S}=\sqrt{2\beta^{-1}}\,\I_{N_p}\,,
\qquad
\mathbf{A}=\mathbf{S}\mathbf{S}^\top=2\beta^{-1}\I_{N_p}\,,
\qquad
\mathbf{b}=-\nabla\mathcal{L}-\beta^{-1}\boldsymbol{\Lambda}\vtheta\,,
\label{eq:cld_dictionary}
\end{equation}
and with $\rho$ the density of $p_\infty$ in
\cref{eq:cld_stationary_law}, so that
$\mu=\rho\,\mathrm{d}\vtheta=p_\infty$. Because $\mathcal{L}\geq0$,
\begin{equation}
Z=\int_{\real^{N_p}}
e^{-\beta\mathcal{L}\prn{\vtheta}
-\frac12\vtheta^\top\boldsymbol{\Lambda}\vtheta}
\,\mathrm{d}\vtheta
\leq\int_{\real^{N_p}}
e^{-\frac12\vtheta^\top\boldsymbol{\Lambda}\vtheta}\,\mathrm{d}\vtheta
=\frac{\prn{2\pi}^{N_p/2}}{\sqrt{\det\boldsymbol{\Lambda}}}<\infty\,,
\label{eq:cld_partition_finite}
\end{equation}
so the normalizing constant of \cref{eq:cld_stationary_law} is finite and
strictly positive and $p_\infty$ is a well-defined probability measure; it is
equivalent to $\nu$ because $e^{-\beta\mathcal{L}}$ is strictly positive and
finite everywhere.

Condition~\textbf{(c)} is immediate here, because both $\mathbf{S}$ and
$\mathbf{A}$ are \emph{constant}: their entries lie in
$H^{1,p}_{\mathrm{loc}}\cap C$ for every $p$, and
$
\boldsymbol{\xi}^\top \mathbf{A} \boldsymbol{\xi}
=2\beta^{-1}\norm{\boldsymbol{\xi}}^2$, so $\mathbf{A}$ is strictly elliptic
globally, not merely locally. Finally $\mathbf{b}$ is continuous, since
$\mathcal{L}\in C^1$, hence locally bounded and so in $L^p_{\mathrm{loc}}$ for
every $p$. Thus \cref{eq:lst_sde} satisfies~\textbf{(c)} for every
$p\in\prn{N_p,\infty}$, and therefore also~\textbf{(a)} and~\textbf{(b)}.

\medskip
\noindent\textbf{Step 2: the standing assumptions, and symmetry.}
The density $\rho\propto e^{-\beta\mathcal{L}-\frac12\vtheta^\top
\boldsymbol{\Lambda}\vtheta}$ is strictly positive and $C^1$ because
$\mathcal{L}$ is, which places it in $H^{1,p}_{\mathrm{loc}}\cap C$ for every
$p$, as \citet[Remark~2.28]{leestannattrutnau2022} requires; consequently
$\sqrt{\rho}\in H^{1,2}_{\mathrm{loc}}$ and $p_\infty$ has full support, which
is their eq.~(2.10). The constant matrix $\mathbf{A}$ is symmetric and, being
bounded with vanishing weak derivatives, lies in every local Sobolev space,
weighted by $p_\infty$ or not; with the ellipticity of Step~1 this gives their
eqs.~(2.11) and (2.12). Finally $\mathbf{b}$ is
locally bounded, hence locally square-integrable against $p_\infty$, which is their eq.~(2.13).

For the symmetry, $\mathbf{A}$ being constant kills the first term of
\cref{eq:lst_logderiv}, leaving
\begin{equation}
\boldsymbol{\beta}^{\rho,\mathbf{A}}
=\frac{\mathbf{A}\nabla\rho}{2\rho}
=\beta^{-1}\nabla\ln\rho
=\beta^{-1}\nabla\prn{-\beta\mathcal{L}
-\tfrac12\vtheta^\top\boldsymbol{\Lambda}\vtheta}
=-\nabla\mathcal{L}-\beta^{-1}\boldsymbol{\Lambda}\vtheta
=\mathbf{b}\,,
\label{eq:cld_symmetric}
\end{equation}
by \cref{eq:cld_dictionary}. This is the symmetric case. It remains to check
their eq.~(2.14). Let $f\in C_c^\infty\prn{\real^{N_p}}$; its compact support is what
makes the integration by parts below free of boundary terms. Since
$\mathbf{A}=2\beta^{-1}\I_{N_p}$, the second-order part of
\cref{eq:lst_generator} is $\beta^{-1}\Delta f$, and
\begin{align}
\beta^{-1}\int_{\real^{N_p}}\Delta f\,\rho\,\mathrm{d}\vtheta
&=-\beta^{-1}\int_{\real^{N_p}}\nabla f^\top\nabla\rho\,\mathrm{d}\vtheta
\nonumber\\
&=-\int_{\real^{N_p}}
\prn{\beta^{-1}\frac{\nabla\rho}{\rho}}^{\!\top}\nabla f
\,\mathrm{d}p_\infty
=-\int_{\real^{N_p}}
\prn{\boldsymbol{\beta}^{\rho,\mathbf{A}}}^{\!\top}
\nabla f\,\mathrm{d}p_\infty\,,
\label{eq:cld_by_parts}
\end{align}
the last equality by \cref{eq:cld_symmetric}. Restoring the first-order part
$\mathbf{b}^\top\nabla f$ of \cref{eq:lst_generator},
\begin{equation}
\int_{\real^{N_p}}\mathcal{G}f\,\mathrm{d}p_\infty
=\int_{\real^{N_p}}
\prn{\mathbf{b}-\boldsymbol{\beta}^{\rho,\mathbf{A}}}^{\!\top}
\nabla f\,\mathrm{d}p_\infty=0\,,
\end{equation}
again by \cref{eq:cld_symmetric}. So $p_\infty$ is infinitesimally invariant,
completing (2.10)--(2.14) and the match with
\citet[Remark~2.28]{leestannattrutnau2022}. This is the weak form of the
vanishing-flux verification cited in \cref{footnote:pavliotis}, stated against
test functions rather than pointwise, which is the form the theory of
\citet{leestannattrutnau2022} consumes.

\medskip
\noindent\textbf{Step 3: conservativeness and non-explosion.}
Both integrability conditions of (R1) are immediate. The drift discrepancy
$\mathbf{b}-\boldsymbol{\beta}^{\rho,\mathbf{A}}$ vanishes identically by
\cref{eq:cld_symmetric}, and $\mathbf{A}$ is constant while $p_\infty$ is a
probability measure by \cref{eq:cld_partition_finite}, so
\begin{equation}
\int_{\real^{N_p}}\abs{A_{ij}}\,\mathrm{d}p_\infty
=2\beta^{-1}\delta_{ij}<\infty\,,
\qquad\forall i,j\in\sqprn{N_p}\,.
\end{equation}
Hence (R1) gives $T_t1=1$, and (R2) turns this into
$\mathbb{P}_{\vtheta}\prn{\zeta=\infty}=1$ for every
$\vtheta\in\real^{N_p}$.

\medskip
\noindent\textbf{Step 4: the conclusions.}
Step~1 supplies~\textbf{(c)} and Step~3 non-explosion, so (R3) gives, for
every deterministic starting point, a pathwise-unique strong solution of
\cref{eq:regularized_langevin} defined for all $t\geq0$, with the law of
$\mathbb{M}$; in particular it is a time-homogeneous Markov process with
transition kernels $\prn{P_t}_{t\geq0}$. For the Gaussian initialization,
$\mathbf{S}$ is constant and
$\mathbf{b}=-\nabla\mathcal{L}-\beta^{-1}\boldsymbol{\Lambda}\vtheta$ is
locally Lipschitz, since $\nabla\mathcal{L}$ is so by the hypothesis of
\cref{prop:cld_wellposed} and the weight-decay term is linear; standard
local SDE theory therefore supplies a pathwise-unique strong solution up to its
explosion time for an initial variable independent of the Brownian motion
\citep[Chapter~5, Theorems~2.5 and~2.9, applied after
localization]{karatzasshreve1991brownian}. By that independence,
conditioning on $\vtheta_0=\vtheta$ leaves the Brownian motion unchanged, so
the conditioned solution is the solution started at $\vtheta$, whose law is
$\mathbb{P}_{\vtheta}$ by (R3) and whose lifetime Step~3 showed to be
infinite.

For the stationarity, \cref{eq:cld_symmetric} puts us in the symmetric case
and Step~3 gives conservativeness, so $p_\infty$ is invariant for the process
by (R4). For the absolute continuity, let $E$ be Borel with
$p_\infty\prn{E}=0$. Conditionally on $\vtheta_0=\vtheta$ the solution has law
$\mathbb{P}_{\vtheta}$, so it lies in $E$ at time $t$ with probability
$P_t\prn{\vtheta,E}$; averaging over $\vtheta_0\sim\nu$,
\begin{equation}
p_t\prn{E}=\int_{\real^{N_p}}P_t\prn{\vtheta,E}\,
\nu\prn{\mathrm{d}\vtheta}\,,
\end{equation}
and by (R4) the integrand vanishes for every $\vtheta$ and every $t>0$, so
$p_t\prn{E}=0$. Finally $p_0=\nu\sim p_\infty$ covers $t=0$. Hence
$p_t\ll p_\infty\sim\nu$ for every $t\geq0$.
\end{proof}

\begin{remark}[Why the flux computation is not enough]
\label{rem:flux_not_enough}
The verification cited in \cref{footnote:pavliotis} shows that the Gibbs
density solves the \emph{stationary} Fokker--Planck equation, its probability
flux vanishing identically. That is strictly weaker than invariance under the
transition semigroup, which is what \cref{prop:cld_wellposed} asserts and what
\cref{thm:second_law_change_of_measure} consumes. If the equation explodes
with positive probability, the minimal semigroup is sub-Markov and, for a
symmetric diffusion, the reference measure is only \emph{sub}-invariant, with
equality if and only if the process is conservative. The flux computation is
local and cannot see this: it is insensitive to whether mass escapes to
infinity in finite time. This is why the stationarity claim waits until
non-explosion has been established, rather than being read off the flux
alone.
\end{remark}

\subsection{Introduction and smooth activations}
We first recall the standard multidimensional It\^o formula, as stated in
\citet[Chapter 1, Theorem 6.4]{mao2007stochastic}, adapted to our notation.
\begin{theorem}[Multidimensional It\^o formula]
    \label{thm:ito_full_presentation}
    Let $\left(\w_t\right)_{t\geq0}$ be a standard Brownian motion in
    $\real^{N_p}$. Let $\vu_t\in\real^{N_p}$ and
    $\V_t\in\real^{N_p\times N_p}$ be progressively measurable processes
    satisfying
    \[
    \int_0^T\norm{\vu_t}\mathrm{d}t<\infty,\qquad
    \int_0^T\norm{\V_t}_{\mathrm{F}}^2\mathrm{d}t<\infty\,,
    \]
    a.s. for every $T>0$. Let $\vtheta_t\in\real^{N_p}$ be the continuous
    It\^o process
    \[
    \mathrm{d}\vtheta_t=\vu_t\mathrm{d}t+\V_t\mathrm{d}\w_t\,.
    \]
    If $\mathcal{L}\in C^{2,1}\left(\real^{N_p}\times\real_+;\real\right)$,\footnote{That is, twice continuously differentiable in the spatial variable and once continuously differentiable in the time variable.}
    then $\mathcal{L}\left(\vtheta_t,t\right)$ is again an It\^o process and
    \begin{align*}
        \mathrm{d}\mathcal{L}\left(\vtheta_t,t\right)
        &=
        \frac{\partial\mathcal{L}}{\partial t}\left(\vtheta_t,t\right)\mathrm{d}t
        +\nabla_{\vtheta}\mathcal{L}\left(\vtheta_t,t\right)^{\top}\vu_t\mathrm{d}t\\
        &\quad+
        \frac{1}{2}\mathrm{Tr}\left(
        \V_t^{\top}\nabla_{\vtheta}^2\mathcal{L}\left(\vtheta_t,t\right)\V_t
        \right)\mathrm{d}t
        +\nabla_{\vtheta}\mathcal{L}\left(\vtheta_t,t\right)^{\top}\V_t\mathrm{d}\w_t\,.
    \end{align*}
\end{theorem}

In our setting, as presented in \cref{sec:setting_deep_nn},
\begin{align*}
    \mathcal{L}\left(\vtheta\right)
    &=
    \frac{1}{2n}\left\Vert f\left(\vtheta\right)-\mathbf{y}\right\Vert^2\,,\\
    \vu_t
    &=
    -\nabla_{\vtheta}\mathcal{L}\left(\vtheta_t\right)
    -\beta^{-1}\boldsymbol{\Lambda}\vtheta_t\,,\\
    \V_t
    &=
    \sqrt{2\beta^{-1}}\I_{N_p}\,.
\end{align*}
The diffusion condition is immediate, since
\[
\int_0^T\norm{\V_t}_{\mathrm{F}}^2\mathrm{d}t
=
2\beta^{-1}N_pT<\infty\,.
\]
For the drift condition, \cref{prop:cld_wellposed} gives that the
solution $\vtheta_t$ is global and continuous for the Gaussian initialization
used here. Then, on every finite interval $\left[0,T\right]$, the path
$\left(\vtheta_t\right)_{0\leq t\leq T}$ is a.s. bounded.
If the network map
$f$ has continuously differentiable activations, then
$\nabla_{\vtheta}\mathcal{L}$ is continuous and hence locally bounded.
Therefore $\vu_t$ is a.s. bounded on $\left[0,T\right]$, and in particular
\[
\int_0^T\norm{\vu_t}\mathrm{d}t<\infty
\]
a.s. for every $T>0$.

When the activation function is twice continuously differentiable, 
the loss $\mathcal{L}$ is twice continuously
differentiable as a function of $\vtheta$. Hence all requirements of
\cref{thm:ito_full_presentation} hold.

\subsection{Extension to once continuously differentiable activations}
\label{sec:ito_extension_to_once_diff}
We next describe how to extend the argument to activations that are only once
continuously differentiable, such as $\mathrm{softsign}$ and $\mathrm{ELU}$
\citep{clevert2015fast}. For this purpose, we use the following one-dimensional
generalized It\^o formula from \citet[Proposition 1.1]{ondrejat2025generalizedito},
adapted to our notation.

\begin{theorem}[Generalized one-dimensional It\^o formula]
    \label{thm:ito_generalized}
    Let $X_t$ be a real-valued continuous semimartingale. Let
    $F\in C^1\left(\real\right)$,\footnote{That is, once continuously differentiable.} and assume that $F'$ is absolutely continuous
    on every compact interval. Let $g:\real\to\real$ be a Borel function
    satisfying
    \(
    g=F''
    \)
    Lebesgue-a.e. on $\real$. Then, for every $t\geq0$,
    \[
    F\left(X_t\right)-F\left(X_0\right)
    =
    \int_0^tF'\left(X_s\right)\mathrm{d}X_s
    +\frac{1}{2}\int_0^tg\left(X_s\right)\mathrm{d}\left\langle X\right\rangle_s\,,
    \]
    a.s.
\end{theorem}

\begin{remark}[A practical sufficient condition]
    \label{rem:lipschitz_derivative_implies_ac}
    The absolute-continuity assumption in \cref{thm:ito_generalized} is implied
    by a simpler and more familiar condition: it is enough that $F'$ is locally
    Lipschitz. That is, for every compact interval $\left[a,b\right]$, there
    exists a finite constant $L_{a,b}$ such that
    \[
    \left|F'\left(x\right)-F'\left(y\right)\right|
    \leq L_{a,b}\left|x-y\right|
    \qquad\forall x,y\in\left[a,b\right]\,.
    \]
    Under this condition, $F'$ is absolutely continuous on every compact
    interval. Moreover, $F''$ exists Lebesgue-a.e., and on every compact
    interval it is bounded by a finite constant. Therefore
    \cref{thm:ito_generalized} applies with any Borel representative
    $g:\real\to\real$ satisfying $g=F''$ Lebesgue-a.e.
\end{remark}

We will use the following specialization to scalar It\^o processes driven by
the same Brownian motion $\w_t$.

\begin{corollary}
    \label{cor:ito_generalized_for_use}
    Let $u_t\in\real$ and $\vv_t\in\real^{N_p}$ be progressively measurable
    processes satisfying
    \[
    \int_0^T\abs{u_t}\mathrm{d}t<\infty,\qquad
    \int_0^T\norm{\vv_t}^2\mathrm{d}t<\infty
    \]
    a.s. for every $T>0$, and let $X_t\in\real$ be the It\^o process
    \[
    \mathrm{d}X_t=u_t\mathrm{d}t+\vv_t^{\top}\mathrm{d}\w_t\,.
    \]
    Let $F\in C^1\left(\real\right)$, and assume that $F'$ is absolutely
    continuous on every compact interval. Let $g:\real\to\real$ be a Borel
    function satisfying $g=F''$ Lebesgue-a.e. on $\real$. Then
    \[
    F\left(X_t\right)-F\left(X_0\right)
    =
    \int_0^t
    \left[
    F'\left(X_s\right)u_s
    +\frac{1}{2}g\left(X_s\right)\norm{\vv_s}^2
    \right]\mathrm{d}s
    +
    \int_0^tF'\left(X_s\right)\vv_s^{\top}\mathrm{d}\w_s
    \]
    a.s. for every $t\geq0$. Equivalently, in differential notation,
    \[
    \mathrm{d}F\left(X_t\right)
    =
    F'\left(X_t\right)u_t\mathrm{d}t
    +\frac{1}{2}g\left(X_t\right)\norm{\vv_t}^2\mathrm{d}t
    +F'\left(X_t\right)\vv_t^{\top}\mathrm{d}\w_t\,.
    \]
\end{corollary}

\begin{proof}
    The martingale part of $X_t$ is
    \[
    \int_0^t\vv_s^{\top}\mathrm{d}\w_s\,.
    \]
    Hence its quadratic variation is
    \[
    \left\langle X\right\rangle_t
    =
    \int_0^t\norm{\vv_s}^2\mathrm{d}s\,.
    \]
    Substituting this identity into \cref{thm:ito_generalized} gives the
    result.
\end{proof}

We now explain why this scalar generalized formula is sufficient for neural
networks with pointwise activations. Although $\vtheta_t\in\real^{N_p}$ is not
scalar, the nonsmoothness enters only through scalar applications of the
activation function. Therefore, we apply the standard multidimensional
It\^o formula to the smooth algebraic operations between layers, and apply the
generalized one-dimensional formula only to the scalar activation maps.

Assume that the activation function satisfies
\[
\phi\in C^1\left(\real\right)
\]
and that $\phi'$ is locally Lipschitz. By
\cref{rem:lipschitz_derivative_implies_ac}, $\phi'$ is absolutely continuous on
every compact interval, $\phi''$ exists Lebesgue-a.e., and
\cref{thm:ito_generalized} applies to $\phi\left(X_t\right)$ whenever $X_t$ is a
scalar It\^o process.

Write the layer pre-activations and post-activations as
\[
\widetilde{\A}_{\ell,t}
=
s_{\ell}\A_{\ell-1,t}\W_{\ell,t}^{\top},
\qquad
\A_{\ell,t}
=
\phi\left(\widetilde{\A}_{\ell,t}\right)\,,
\]
where $\phi$ is applied componentwise and $\A_{0,t}=\X$ is deterministic. We
argue by induction over the layers. The coordinates of each weight matrix
$\W_{\ell,t}$ are coordinates of $\vtheta_t$, hence are It\^o processes. Suppose
that all entries of $\A_{\ell-1,t}$ are It\^o processes. Then the finite vector
formed by the entries of $\W_{\ell,t}$ and $\A_{\ell-1,t}$ is itself a
finite-dimensional It\^o process. Each coordinate of
$\widetilde{\A}_{\ell,t}$ is a smooth polynomial function of this finite vector.
Therefore, by \cref{thm:ito_full_presentation}, every pre-activation coordinate
$\widetilde{A}_{\ell,t}^{i,j}$ is an It\^o process.

Applying \cref{cor:ito_generalized_for_use} to the scalar process
$\widetilde{A}_{\ell,t}^{i,j}$ gives that
\[
A_{\ell,t}^{i,j}
=
\phi\left(\widetilde{A}_{\ell,t}^{i,j}\right)
\]
is again an It\^o process. This completes the induction over layers.
Consequently, all network output coordinates are It\^o processes. Since the
squared loss
\[
\mathcal{L}\left(\mathbf{z}\right)
=
\frac{1}{2n}\norm{\mathbf{z}-\mathbf{y}}^2
\]
is a smooth function of the output vector $\mathbf{z}$, a final application of
the standard multidimensional It\^o formula yields the It\^o decomposition of
$\mathcal{L}\left(\vtheta_t\right)$.
\newpage

\section{Proof of global convergence of regularized CLD for shallow networks}
\label{app:langevin_proof}
\subsection{Introduction and goal}
\begin{recall}[\cref{thm:upper_bound_on_expected_loss}][Shallow networks]
Consider the network defined in \cref{sec:setting_deep_nn}, under the assumptions therein, with $L=1$. Suppose it is trained via regularized continuous-time Langevin dynamics as described in \cref{eq:regularized_langevin}, with 
$\boldsymbol{\Lambda}$ defined as in \cref{eq:Lambda_def}.
Denoting $\bar{\kappa}\triangleq \max \prn{1, \kappa}$, suppose \mbox{$m =\Omega \prn{ \bar\kappa^2 \beta \frac{n}{\lambda_1} \ln \frac{n}{\lambda_1}\ln \prn{\bar\kappa^4 \beta}}$}, $\beta = \Omega \prn{\frac{n}{\lambda_1}}$.
Then, for all $t \geq 0$,
\begin{align}
\mathbb{E}\mathcal{L} \left(\boldsymbol{\theta}_{t}\right)\leq\left(\mathbb{E}\mathcal{L} \left(\boldsymbol{\theta}_{0}\right)-\mathcal{L}_\infty\right)e^{-c_1s_{L+1}^{2}\frac{\lambda_{1}m}{n}t}+\mathcal{L}_\infty\,, 
\end{align}
where $c_1$ is a global constant, and the expectation is taken over both the initialization and the noise, with
\begin{equation}
\begin{aligned}
\mathcal{L}_\infty& \asymp\left(1+\bar\kappa^{4}\beta\right)\exp\left(-{c_{2}}\frac{m}{\bar\kappa^{2}\beta\frac{n}{\lambda_{1}}\ln\frac{n}{\lambda_{1}}}\right)+\beta^{-2}\frac{n^{2}}{\lambda_{1}^{2}}\left(\frac{1}{\kappa^{4}}+\frac{s_{1}^{4}d^{2}}{s_{2}^{4}m^{2}}\right) \\ 
& \qquad +\beta^{-1}\frac{n^{2}}{\lambda_{1}^{2}}\bar\kappa^{2} \kappa^2 \frac{s_{1}^{4}d^{2}}{s_{2}^{4}m^{2}}+\beta^{-1}\frac{n}{\lambda_{1}}\left(1+\kappa^{2}\frac{s_{1}^{2}d}{s_{2}^{2}m}\right)\,. 
\end{aligned}
\end{equation}
\end{recall}
\begin{proof}
We begin by noting that for the MSE loss, as defined in \cref{eq:mse_loss_def}, we have
\begin{align*}
\nabla \mathcal{L} \left(\boldsymbol{\theta}_{t}\right) & =\frac{1}{n}\nabla f\left(\boldsymbol{\theta}_{t}\right)^{\top}\mathbf{r}\left(\boldsymbol{\theta}_{t}\right)\,, \quad
\nabla^{2}\mathcal{L} \left( \boldsymbol{\theta}_{t}\right)  =\frac{1}{n}
\left( \nabla f\left(\boldsymbol{\theta}_{t}\right)^{\top}\nabla f\left(\boldsymbol{\theta}_{t}\right)+\sum_{i=1}^{n}r_{i}\left(\boldsymbol{\theta}_{t}\right)\nabla^{2}f_{i}\left(\boldsymbol{\theta}_{t}\right)
\right)\,.
\end{align*}
Then, substituting these into \cref{eq:regularized_langevin}, we get
\begin{equation}
\begin{aligned}
\frac{\mathrm{d}}{\mathrm{d}t}\mathbb{E}\mathcal{L} \left(\boldsymbol{\theta}_{t}\right) &  =\frac{1}{n}\biggl(-n^{-1}\mathbb{E}\left[\mathbf{r}\left(\boldsymbol{\theta}_{t}\right)^{\top}\mathbf{K}_{\mathrm{NTK}}\left(\boldsymbol{\theta}_{t}\right)\mathbf{r}\left(\boldsymbol{\theta}_{t}\right)\right]\\
 & \qquad\quad+\beta^{-1}\mathbb{E}\left[\mathrm{Tr}\left(\nabla f\left(\boldsymbol{\theta}_{t}\right)^{\top}\nabla f\left(\boldsymbol{\theta}_{t}\right)\right)\right]\\
 & \qquad\quad+\beta^{-1}\mathbb{E}\left[\sum_{i=1}^{n}r_{i}\left(\boldsymbol{\theta}_{t}\right)\Delta f_{i}\left(\boldsymbol{\theta}_{t}\right)\right]\\
 & \qquad\quad-\beta^{-1}\mathbb{E}\left[\mathbf{r}\left(\boldsymbol{\theta}_{t}\right)^{\top}\nabla f\left(\boldsymbol{\theta}_{t}\right)\boldsymbol{\Lambda}\boldsymbol{\theta}_{t}\right]\biggr)\,.
 \label{eq:the_four_terms_in_appendix}
\end{aligned}
\end{equation}
We wish to bound every term in the right-hand side of \cref{eq:the_four_terms_in_appendix}, using \cref{cor:second_law_for_langevin}.
In order to use \cref{cor:second_law_for_langevin}, we must bound $\prn{\mathbb{E}_{p_{0}}\left[e^{2\left(\alpha-1\right) \beta \mathcal{L} \prn{\boldsymbol{\theta}_0}}\right]}^{\frac{1}{2\alpha}}$.
Under the bounded activation assumption (\cref{ass:activation_bounded}), we show in \cref{prop:mgf_of_init_loss} that the subgaussianity of the network outputs implies the following bound on the exponential moment: 
\[
\mathbb{E}_{p_0}\left[e^{2 \prn{\alpha-1} \beta \mathcal{L} \left(\boldsymbol{\theta}_{0}\right)}\right]
\leq 2e^{2 \prn{\alpha-1} \beta R_{y}^{2}}\,,
\]
provided that $2 \prn{\alpha-1} \beta \leq L_\psi^{-1}$, where $L_\psi \triangleq cM_{\phi}^{2}\frac{s_{L+1}^{2}m_L}{a_{L+1}}=cM_{\phi}^{2} \kappa^2$.
That is, the moment is only guaranteed to exist for $\alpha \leq 1+\frac{1}{2cM_{\phi}^{2} \kappa^2 \beta}$. However, we also require that $2e^{2 \prn{\alpha-1} \beta R_{y}^{2}}=O \prn{1}$, thus limiting ourselves only to $\alpha \leq 1+\frac{1}{2cM_{\phi}^{2} \beta}$. Taken together, by using $\alpha=1+\frac{1}{2cM_{\phi}^{2} \max \prn{\kappa^2, 1} \beta}$, we ensure $\prn{\mathbb{E}_{p_{0}}\left[e^{2\left(\alpha-1\right) \beta \mathcal{L} \prn{\boldsymbol{\theta}_0}}\right]}^{\frac{1}{2\alpha}}$ exists and is $O\prn{1}$. Denoting $\bar{\kappa}\triangleq \max \prn{\kappa, 1}$, we get that
\[
\alpha = 1+ \frac{1}{2cM_{\phi}^{2} {\bar{\kappa}}^2 \beta}\,,
\]
and
\begin{align*}
    \prn{\mathbb{E}_{p_{0}}\left[e^{2\left(\alpha-1\right) \beta \mathcal{L} \prn{\boldsymbol{\theta}_0}}\right]}^{\frac{1}{2\alpha}} 
    &\leq \prn{2e^{2 \prn{\alpha-1} \beta R_{y}^{2}}}^{\frac{1}{2\alpha}} \leq \sqrt{2}e^{\frac{\beta R_y^2}{2cM_\phi^2{\bar{\kappa}}^2 \beta}}=\sqrt{2}e^{\frac{R_y^2}{2cM_\phi^2{\bar{\kappa}}^2 }}=O \prn{1}\,,   
\end{align*}
since $\bar{\kappa} \geq 1$. We denote from here on $\eta \triangleq \sqrt{2}e^{\frac{R_y^2}{2cM_\phi^2{\bar{\kappa}}^2 }}=O \prn{1}$.

Given this choice of $\alpha$, we have 
\begin{align*}
    \frac{\alpha}{\alpha - 1} = 1+2cM_{\phi}^{2} {\bar{\kappa}}^2 \beta = c_\beta {\bar{\kappa}}^2 \beta\,,
\end{align*}
where we defined $c_\beta \triangleq 2cM_\phi^2 + \frac{1}{{\bar{\kappa}}^2 \beta}$. 
This parametrization is useful in the regime \(\bar\kappa^2\beta\gg1\), which holds under our assumptions. Then, $c_\beta = O \prn{1}$.
With this choice of \(\alpha\), \cref{cor:second_law_for_langevin} gives for any \emph{non-negative} measurable function $\varphi:\real^{N_p}\to\mathbb{R}$,
\begin{align}
\mathbb{E}_{p_{t}} \varphi\left(\boldsymbol{\theta}_t\right)  \leq \eta
\norm{\varphi \prn{\boldsymbol{\theta}_0}}_{L^{c_\beta {\bar{\kappa}}^{2} \beta}} .
\label{eq:second_law_with_eta}
\end{align}
We proceed to bound each of the four terms in \cref{eq:the_four_terms_in_appendix} using the change-of-measure inequality \cref{eq:second_law_with_eta}.

\begin{remark}
Throughout the proof, inequalities involving both descent and error terms are interpreted in the following one-sided asymptotic sense: an estimate of the form
\[
X\lesssim -A+B
\]
means that there exist universal constants \(c,C>0\), independent of the problem parameters, such that
\[
X\leq -cA+CB.
\]
Thus, positive error terms may be enlarged and descent coefficients may be reduced by universal constant factors. These constants are absorbed into the final constants \(c_1,c_2\) and into the implicit constants in \(\mathcal{L}_\infty\).

Similarly, when writing \(\mathcal L_\infty\asymp F\), we mean that the loss floor in the upper bound may be chosen as \(\mathcal L_\infty=CF\) for a sufficiently large universal constant \(C>0\). Thus, \(\asymp\) records the parameter dependence of the upper-bound floor and is not meant as a matching lower bound on the optimal possible floor.
\end{remark}
\subsection{First term}
We have
\begin{align*}
-n^{-1}\mathbb{E}\left[\mathbf{r}\left(\boldsymbol{\theta}_{t}\right)^{\top}\mathbf{K}_{\mathrm{NTK}}\left(\boldsymbol{\theta}_{t}\right)\mathbf{r}\left(\boldsymbol{\theta}_{t}\right)\right] & \leq-n^{-1}\mathbb{E}\left[\left\Vert \mathbf{r}\left(\boldsymbol{\theta}_{t}\right)\right\Vert _{2}^{2}\lambda_{\min}\left(\mathbf{K}_{\mathrm{NTK}}\left(\boldsymbol{\theta}_{t}\right)\right)\right]\\
 & =-n^{-1}\mathbb{E}\left[2n\mathcal{L} \left(\boldsymbol{\theta}_{t}\right)\lambda_{\min}\left(\mathbf{K}_{\mathrm{NTK}}\left(\boldsymbol{\theta}_{t}\right)\right)\right]\\
 & =-2\mathbb{E}\left[\mathcal{L} \left(\boldsymbol{\theta}_{t}\right)\lambda_{\min}\left(\mathbf{K}_{\mathrm{NTK}}\left(\boldsymbol{\theta}_{t}\right)\right)\right]\,.
\end{align*}
By \cref{thm:lower_bound_ntk_deep}, if $m=\Omega\left(\frac{n}{\lambda_{1}}\ln\frac{n}{\lambda_{1}}\ln n\right)$ then $\lambda_{\min}\left(\mathbf{K}_{\mathrm{NTK}}\left(\boldsymbol{\theta}_{0}\right)\right)\geq c_{0}\lambda_{1}s_{2}^{2}m$
with probability at least $1-\exp\left(-\tilde{c}_{1}\frac{m}{\frac{n}{\lambda_{1}}\ln\frac{n}{\lambda_{1}}}\right)$. This width condition is implied by our assumptions.
For the remainder of this section, denote $\lambda_{\star}\triangleq c_{0}\lambda_{1}s_{2}^{2}m$ and $\delta\triangleq\exp\left(-\tilde{c}_{1}\frac{m}{\frac{n}{\lambda_{1}}\ln\frac{n}{\lambda_{1}}}\right)$. Now,
\begin{align*}
& -n^{-1}\mathbb{E}\left[\mathbf{r}\left(\boldsymbol{\theta}_{t}\right)^{\top}\mathbf{K}_{\mathrm{NTK}}\left(\boldsymbol{\theta}_{t}\right)\mathbf{r}\left(\boldsymbol{\theta}_{t}\right)\right] \\
& \leq-2\mathbb{E}\left[\mathcal{L} \left(\boldsymbol{\theta}_{t}\right)\lambda_{\min}\left(\mathbf{K}_{\mathrm{NTK}}\left(\boldsymbol{\theta}_{t}\right)\right)\right]\\
 & =-2\lambda_{\star}\mathbb{E}\left[\mathcal{L} \left(\boldsymbol{\theta}_{t}\right)\right]+2\mathbb{E}\left[\mathcal{L} \left(\boldsymbol{\theta}_{t}\right)\left(\lambda_{\star}-\lambda_{\min}\left(\mathbf{K}_{\mathrm{NTK}}\left(\boldsymbol{\theta}_{t}\right)\right)\right)\right]\\
 & \leq-2\lambda_{\star}\mathbb{E}\left[\mathcal{L} \left(\boldsymbol{\theta}_{t}\right)\right]+2\mathbb{E}\left[\mathcal{L} \left(\boldsymbol{\theta}_{t}\right)\max\left(0,\,\lambda_{\star}-\lambda_{\min}\left(\mathbf{K}_{\mathrm{NTK}}\left(\boldsymbol{\theta}_{t}\right)\right)\right)\right]\,.
\end{align*}
We upper bound $\mathbb{E}\left[\mathcal{L} \left(\boldsymbol{\theta}_{t}\right)\max\left(0,\,\lambda_{\star}-\lambda_{\min}\left(\mathbf{K}_{\mathrm{NTK}}\left(\boldsymbol{\theta}_{t}\right)\right)\right)\right]$
using our second law change of measure theorem (\cref{eq:second_law_with_eta}):
\begin{align*}
& \mathbb{E}\left[\mathcal{L} \left(\boldsymbol{\theta}_{t}\right)\max\left(0,\,\lambda_{\star}-\lambda_{\min}\left(\mathbf{K}_{\mathrm{NTK}}\left(\boldsymbol{\theta}_{t}\right)\right)\right)\right] \\
& \leq\eta\left\Vert \mathcal{L} \left(\boldsymbol{\theta}_{0}\right)\max\left(0,\,\lambda_{\star}-\lambda_{\min}\left(\mathbf{K}_{\mathrm{NTK}}\left(\boldsymbol{\theta}_{0}\right)\right)\right)\right\Vert_{L^{c_{\beta}{\bar{\kappa}}^{2}\beta}}\\
\left[\text{H\"older, \cref{claim:holder_Lp_norm}}\right] & \leq\eta\left\Vert \mathcal{L} \left(\boldsymbol{\theta}_{0}\right)\right\Vert_{L^{2c_{\beta}{\bar{\kappa}}^{2}\beta}}\left\Vert \max\left(0,\,\lambda_{\star}-\lambda_{\min}\left(\mathbf{K}_{\mathrm{NTK}}\left(\boldsymbol{\theta}_{0}\right)\right)\right)\right\Vert_{L^{2c_{\beta}{\bar{\kappa}}^{2}\beta}}\,.
\end{align*}
From \cref{prop:mgf_of_init_loss},
\begin{align*}
\left\Vert \mathcal{L} \left(\boldsymbol{\theta}_{0}\right)\right\Vert _{L^{2c_{\beta}{\bar{\kappa}}^{2}\beta}} & \leq\left\Vert R_{y}^{2}+\frac{1}{n}\sum_{i\in\left[n\right]}f_{i}^{2}\left(\boldsymbol{\theta}_{0}\right)\right\Vert _{L^{2c_{\beta}{\bar{\kappa}}^{2}\beta}}\\
\left[\text{Minkowski's inequality}\right] & \leq\left\Vert R_{y}^{2}\right\Vert_{L^{2c_{\beta}{\bar{\kappa}}^{2}\beta}}+\left\Vert \frac{1}{n}\sum_{i\in\left[n\right]}f_{i}^{2}\left(\boldsymbol{\theta}_{0}\right)\right\Vert _{L^{2c_{\beta}{\bar{\kappa}}^{2}\beta}}\\
\text{\citep[Proposition 2.8.1]{vershynin2018high}} & \leq R_{y}^{2}+C\cdot2c_{\beta}{\bar{\kappa}}^{2}\beta\left\Vert \frac{1}{n}\sum_{i\in\left[n\right]}f_{i}^{2}\left(\boldsymbol{\theta}_{0}\right)\right\Vert _{\psi_{1}}\\
 & \leq R_{y}^{2}+C_{2}c_{\beta}{\bar{\kappa}}^{4}\beta\,.
\end{align*}
Splitting according to the high-probability event
\[
\mathcal E=\left\{\lambda_\star\leq\lambda_{\min}\left(\mathbf K_{\mathrm{NTK}}(\boldsymbol\theta_0)\right)\right\}\,,
\]
we have for the second term,
\begin{align*}
&\left\Vert \max\left(0,\,\lambda_{\star}-\lambda_{\min}\left(\mathbf{K}_{\mathrm{NTK}}\left(\boldsymbol{\theta}_{0}\right)\right)\right)\right\Vert_{L^{2c_{\beta}{\bar{\kappa}}^{2}\beta}} \\
& =\left(\mathbb{E}\left[\max\left(0,\,\lambda_{\star}-\lambda_{\min}\left(\mathbf{K}_{\mathrm{NTK}}\left(\boldsymbol{\theta}_{0}\right)\right)\right)^{2c_{\beta}{\bar{\kappa}}^{2}\beta}\right]\right)^{\nicefrac{1}{2c_{\beta}{\bar{\kappa}}^{2}\beta}}\\
 & \leq\left(0\cdot\mathbb{P}\left(\lambda_{\star}\leq\lambda_{\min}\left(\mathbf{K}_{\mathrm{NTK}}\left(\boldsymbol{\theta}_{0}\right)\right)\right)+\lambda_{\star}^{2c_{\beta}{\bar{\kappa}}^{2}\beta}\mathbb{P}\left(\lambda_{\star}>\lambda_{\min}\left(\mathbf{K}_{\mathrm{NTK}}\left(\boldsymbol{\theta}_{0}\right)\right)\right)\right)^{\nicefrac{1}{2c_{\beta}{\bar{\kappa}}^{2}\beta}}\\
 & \leq\lambda_{\star}\delta^{\nicefrac{1}{2c_{\beta}{\bar{\kappa}}^{2}\beta}}=c_{0}\lambda_{1}s_{2}^{2}m\exp\left(-\tilde{c}_{2}\frac{m}{c_{\beta}{\bar{\kappa}}^{2}\beta\frac{n}{\lambda_{1}}\ln\frac{n}{\lambda_{1}}}\right)\,,
\end{align*}
where in the first transition we used the fact that the empirical NTK matrix is PSD, hence $\lambda_{\min}\prn{\K_\mathrm{NTK}} \geq 0$. Therefore:
\begin{align*}
&-n^{-1}\mathbb{E}\left[\mathbf{r}\left(\boldsymbol{\theta}_{t}\right)^{\top}\mathbf{K}_{\mathrm{NTK}}\left(\boldsymbol{\theta}_{t}\right)\mathbf{r}\left(\boldsymbol{\theta}_{t}\right)\right] \\
& \leq-2c_{0}\lambda_{1}s_{2}^{2}m\mathbb{E}\left[\mathcal{L} \left(\boldsymbol{\theta}_{t}\right)\right]+2 \eta c_{0}\lambda_{1}s_{2}^{2}m\left(R_{y}^{2}+C_{2}c_{\beta}{\bar{\kappa}}^{4}\beta\right)\exp\left(-\tilde{c}_{2}\frac{m}{c_{\beta}{\bar{\kappa}}^{2}\beta\frac{n}{\lambda_{1}}\ln\frac{n}{\lambda_{1}}}\right)\\
 & \lesssim-\lambda_{1}s_{2}^{2}m\mathbb{E}\left[\mathcal{L} \left(\boldsymbol{\theta}_{t}\right)\right]+\lambda_{1}s_{2}^{2}m\left(1+{\bar{\kappa}}^{4}\beta\right)\exp\left(-\tilde{c}_{2}\frac{m}{c_{\beta}{\bar{\kappa}}^{2}\beta\frac{n}{\lambda_{1}}\ln\frac{n}{\lambda_{1}}}\right)\,.
\end{align*}
\subsection{Second term}
Observe that,
\begin{align*}
\beta^{-1}\mathbb{E}\mathrm{Tr}\left(\nabla f\left(\boldsymbol{\theta}_{t}\right)^{\top}\nabla f\left(\boldsymbol{\theta}_{t}\right)\right) & =\beta^{-1}\mathbb{E}\left[\sum_{j=1}^{N_{p}}\sum_{i=1}^{n}\left(\left(\left(\nabla f_{i}\left(\boldsymbol{\theta}_{t}\right)\right)_{j}\right)^{2}\right)\right]\\
 & =\beta^{-1}\sum_{i=1}^{n}\mathbb{E}\left[\left\Vert \nabla f_{i}\left(\boldsymbol{\theta}_{t}\right)\right\Vert ^{2}\right]\,.
\end{align*}
Now, for all $i\in\left[n\right]$, we denote $\boldsymbol{\alpha}\left(\mathbf{x}_{i}\right)\triangleq\phi\left(s_{1}\mathbf{W}\mathbf{x}_{i}\right)$,
such that
\begin{align*}
f_{i}\left(\boldsymbol{\theta}\right) & =s_{2}\mathbf{v}^{\top}\boldsymbol{\alpha}\left(\mathbf{x}_{i}\right)\\
\Longrightarrow\frac{\partial f_{i}\left(\boldsymbol{\theta}\right)}{\partial\mathbf{v}} & =s_{2}\boldsymbol{\alpha}\left(\mathbf{x}_{i}\right)=s_{2}\phi\left(s_{1}\mathbf{W}\mathbf{x}_{i}\right)\,,\\
\forall j\in\left[m\right],\quad\frac{\partial f_{i}\left(\boldsymbol{\theta}\right)}{\partial\mathbf{w}_{j}} & =s_{2}v_{j}\frac{\partial \left(\boldsymbol{\alpha}\left(\mathbf{x}_{i}\right)\right)_{j}}{\partial\mathbf{w}_{j}}\,,\\
\frac{\partial \left(\boldsymbol{\alpha}\left(\mathbf{x}_{i}\right)\right)_{j}}{\partial\mathbf{w}_{j}} & =\phi'\left(s_{1}\mathbf{w}_{j}^{\top}\mathbf{x}_{i}\right)s_{1}\mathbf{x}_{i}\\
\Longrightarrow\forall j\in\left[m\right],\quad\frac{\partial f_{i}\left(\boldsymbol{\theta}\right)}{\partial\mathbf{w}_{j}} & =s_{2}v_{j}\phi'\left(s_{1}\mathbf{w}_{j}^{\top}\mathbf{x}_{i}\right)s_{1}\mathbf{x}_{i}\,.
\end{align*}
Then we have:
\begin{align*}
\left\Vert \nabla f_{i}\left(\boldsymbol{\theta}\right)\right\Vert ^{2} & =\left\Vert s_{2}\phi\left(s_{1}\mathbf{W}\mathbf{x}_{i}\right)\right\Vert ^{2}+\sum_{j=1}^{m}\left\Vert s_{2}v_{j}\phi'\left(s_{1}\mathbf{w}_{j}^{\top}\mathbf{x}_{i}\right)s_{1}\mathbf{x}_{i}\right\Vert ^{2}\\
\left[\abs{\phi\left(\cdot\right)}\leq\alpha_{\phi}\left|\cdot\right|,\,\left|\phi'\right|\le\alpha_{\phi}\right] & \leq\alpha_{\phi}^{2}s_{2}^{2}s_{1}^{2}\left(\left\Vert \mathbf{W}\mathbf{x}_{i}\right\Vert ^{2}+\sum_{j=1}^{m}v_{j}^{2}\left\Vert \mathbf{x}_{i}\right\Vert ^{2}\right)\\
 & =\alpha_{\phi}^{2}s_{2}^{2}s_{1}^{2}c_{\phi,\sigma_{0}}d\left(\frac{\nu_{0}^{2}}{a_{1}}\left\Vert \mathbf{z}^{\left(1\right)}\right\Vert ^{2}+\frac{1}{a_{2}}\left\Vert \mathbf{z}^{\left(2\right)}\right\Vert ^{2}\right)\,,
\end{align*}
where $\mathbf{z}^{\left(1\right)},\mathbf{z}^{\left(2\right)}\sim\mathcal{N}\left(\mathbf{0}_{m},\mathbf{I}_{m}\right)$
(under $p_{0}$).

Observe that for all $p\geq1$, from \cref{claim:moment of power} and \cref{prop:bounded moments of gaussian vector},
\begin{align*}
\left\Vert \left\Vert \mathbf{z}\right\Vert ^{2}\right\Vert _{L^{p}} & =\left\Vert \left\Vert \mathbf{z}\right\Vert \right\Vert _{L^{2p}}^{2}\leq\left(C\sqrt{2p}+\sqrt{m}\right)^{2}\leq4C^{2}p+2m\,.
\end{align*}
So, at initialization, using Minkowski's inequality,
\begin{align*}
\left\Vert \left\Vert \nabla f_{i}\left(\boldsymbol{\theta}\right)\right\Vert ^{2}\right\Vert _{L^{p}} & \leq\alpha_{\phi}^{2}s_{2}^{2}s_{1}^{2}c_{\phi,\sigma_{0}}d\left(\frac{\nu_{0}^{2}}{a_{1}}\left\Vert \left\Vert \mathbf{z}^{\left(1\right)}\right\Vert ^{2}\right\Vert _{L^{p}}+\frac{1}{a_{2}}\left\Vert \left\Vert \mathbf{z}^{\left(2\right)}\right\Vert ^{2}\right\Vert _{L^{p}}\right)\\
 & \leq\alpha_{\phi}^{2}c_{\phi,\sigma_{0}}s_{2}^{2}s_{1}^{2}d\left(\frac{\nu_{0}^{2}}{a_{1}}+\frac{1}{a_{2}}\right)\left(4C^{2}p+2m\right)\,.
\end{align*}
Using the second law (\cref{eq:second_law_with_eta}), we get that at time $t\geq0$,
for all $i\in\left[n\right]$, 
\begin{align*}
\mathbb{E}\left[\left\Vert \nabla f_{i}\left(\boldsymbol{\theta}_{t}\right)\right\Vert ^{2}\right] & \leq\eta\left\Vert \left\Vert \nabla f_{i}\left(\boldsymbol{\theta}_{0}\right)\right\Vert ^{2}\right\Vert _{L^{c_{\beta}{\bar{\kappa}}^{2}\beta}}\leq\eta\alpha_{\phi}^{2}c_{\phi,\sigma_{0}}s_{2}^{2}s_{1}^{2}d\left(\frac{\nu_{0}^{2}}{a_{1}}+\frac{1}{a_{2}}\right)\left(4C^{2}c_{\beta}{\bar{\kappa}}^{2}\beta+2m\right)\,.
\end{align*}
Summing over \(i\),
\begin{align*}
\beta^{-1}\mathbb{E}\mathrm{Tr}\left(\nabla f\left(\boldsymbol{\theta}_{t}\right)^{\top}\nabla f\left(\boldsymbol{\theta}_{t}\right)\right) & =\beta^{-1}\sum_{i=1}^{n}\mathbb{E}\left[\left\Vert \nabla f_{i}\left(\boldsymbol{\theta}_{t}\right)\right\Vert ^{2}\right]\\
 & \leq\eta\alpha_{\phi}^{2}c_{\phi,\sigma_{0}}\beta^{-1}ns_{2}^{2}s_{1}^{2}d\left(\frac{\nu_{0}^{2}}{a_{1}}+\frac{1}{a_{2}}\right)\left(C_{3}c_{\beta}{\bar{\kappa}}^{2}\beta+2m\right)\,.\\
 & \lesssim\beta^{-1}ns_{2}^{2}s_{1}^{2}d\left(\frac{1}{a_{1}}+\frac{1}{a_{2}}\right)\left({\bar{\kappa}}^{2}\beta+m\right)\,.
\end{align*}
\subsection{Third term}
We have
\begin{align*}
\beta^{-1}\mathbb{E}\left[\sum_{i=1}^{n}r_{i}\left(\boldsymbol{\theta}_{t}\right)\Delta f_{i}\left(\boldsymbol{\theta}_{t}\right)\right] & =\beta^{-1}\sum_{i=1}^{n}\mathbb{E}\left[r_{i}\left(\boldsymbol{\theta}_{t}\right)\Delta f_{i}\left(\boldsymbol{\theta}_{t}\right)\right]\\
 & \leq\beta^{-1}\sum_{i=1}^{n}\mathbb{E}\left[\left|r_{i}\left(\boldsymbol{\theta}_{t}\right)\Delta f_{i}\left(\boldsymbol{\theta}_{t}\right)\right|\right]\\
\left[\text{Cauchy-Schwarz}\right] & \leq\beta^{-1}\sum_{i=1}^{n}\sqrt{\mathbb{E}\left[\left|r_{i}\left(\boldsymbol{\theta}_{t}\right)\right|^{2}\right]\mathbb{E}\left[\left|\Delta f_{i}\left(\boldsymbol{\theta}_{t}\right)\right|^{2}\right]}\\
 & =\beta^{-1}\sum_{i=1}^{n}\sqrt{2\mathbb{E}\left[\ell_{i}\left(\boldsymbol{\theta}_{t}\right)\right]\mathbb{E}\left[\left|\Delta f_{i}\left(\boldsymbol{\theta}_{t}\right)\right|^{2}\right]}\,,
\end{align*}
where $\ell_{i}\left(\boldsymbol{\theta}\right)=\frac{1}{2}\left|r_{i}\left(\boldsymbol{\theta}\right)\right|^{2}$.

Note that, for all $i\in\left[n\right]$,
\[
\Delta f_{i}\left(\boldsymbol{\theta}\right)=\sum_{j\in\left[m\right]}\frac{\partial^{2}f_{i}\left(\boldsymbol{\theta}\right)}{\partial v_{j}^{2}}+\sum_{j\in\left[m\right]}\sum_{k\in\left[d\right]}\frac{\partial^{2}f_{i}\left(\boldsymbol{\theta}\right)}{\partial W_{jk}^{2}}\,.
\]
Observe that for all $i\in\left[n\right]$, $f_{i}\left(\boldsymbol{\theta}\right)=s_{2}\mathbf{v}^{\top}\boldsymbol{\alpha}\left(\mathbf{x}_{i}\right)$,
where $\boldsymbol{\alpha}\left(\mathbf{x}_{i}\right)$ does not depend
on $\mathbf{v}$, so for all $j\in\left[m\right]:$
\begin{align*}
\frac{\partial f_{i}\left(\boldsymbol{\theta}\right)}{\partial v_{j}} & =s_{2}\alpha_{j}\left(\mathbf{x}_{i}\right)\,,\\
\frac{\partial^{2}f_{i}\left(\boldsymbol{\theta}\right)}{\partial v_{j}^{2}} & =0\,.
\end{align*}
 In terms of the derivatives w.r.t. $W_{jk}$, $\forall j\in\left[m\right],\,k\in\left[d\right]$,
\begin{align*}
\frac{\partial f_{i}\left(\boldsymbol{\theta}\right)}{\partial W_{jk}} & =s_{2}\mathbf{v}^{\top}\frac{\partial\boldsymbol{\alpha}\left(\mathbf{x}_{i}\right)}{\partial W_{jk}}=s_{2}\sum_{j'\in\left[m\right]}v_{j'}\frac{\partial\alpha_{j'}\left(\mathbf{x}_{i}\right)}{\partial W_{jk}}\,,
\end{align*}
where
\[
\alpha_{j'}\left(\mathbf{x}_{i}\right)=\phi\left(s_{1}\mathbf{w}_{j'}^{\top}\mathbf{x}_{i}\right)\,,
\]
denoting $\mathbf{w}_{j'}\in\mathbb{R}^{d}$ as the $j'$-th row of
$\mathbf{W}$ (transposed). Thus,
\begin{align*}
\forall j'\neq j,\quad & \frac{\partial\alpha_{j'}\left(\mathbf{x}_{i}\right)}{\partial W_{jk}}=0\,,\\
 & \frac{\partial\alpha_{j}\left(\mathbf{x}_{i}\right)}{\partial W_{jk}}=\phi'\left(s_{1}\mathbf{w}_{j}^{\top}\mathbf{x}_{i}\right)s_{1}x_{i,k}\,.
\end{align*}
We then have
\begin{align}
\frac{\partial f_{i}\left(\boldsymbol{\theta}\right)}{\partial W_{jk}} & =s_{2}s_{1}v_{j}x_{i,k}\phi'\left(s_{1}\mathbf{w}_{j}^{\top}\mathbf{x}_{i}\right)\label{eq:df_dW_jk}\\
\Longrightarrow\frac{\partial^{2}f_{i}\left(\boldsymbol{\theta}\right)}{\partial W_{jk}^{2}} & =s_{2}s_{1}v_{j}x_{i,k}\phi''\left(s_{1}\mathbf{w}_{j}^{\top}\mathbf{x}_{i}\right)s_{1}x_{i,k}\nonumber \\
 & =s_{2}s_{1}^{2}v_{j}x_{i,k}^{2}\phi''\left(s_{1}\mathbf{w}_{j}^{\top}\mathbf{x}_{i}\right)\,.\nonumber 
\end{align}
Calculating the Laplacian, we get
\begin{align*}
\Delta f_{i}\left(\boldsymbol{\theta}\right) & =\sum_{j\in\left[m\right]}\frac{\partial^{2}f_{i}\left(\boldsymbol{\theta}\right)}{\partial v_{j}^{2}}+\sum_{j\in\left[m\right]}\sum_{k\in\left[d\right]}\frac{\partial^{2}f_{i}\left(\boldsymbol{\theta}\right)}{\partial W_{jk}^{2}}\\
 & =\sum_{j\in\left[m\right]}\sum_{k\in\left[d\right]}s_{2}s_{1}^{2}v_{j}x_{i,k}^{2}\phi''\left(s_{1}\mathbf{w}_{j}^{\top}\mathbf{x}_{i}\right)\\
 & =s_{2}s_{1}^{2}\left(\sum_{k\in\left[d\right]}x_{i,k}^{2}\right)\left(\sum_{j\in\left[m\right]}v_{j}\phi''\left(s_{1}\mathbf{w}_{j}^{\top}\mathbf{x}_{i}\right)\right)\\
\left[\left\Vert \mathbf{x}_{i}\right\Vert ^{2}=c_{\phi,\sigma_{0}}d\right] & =s_{2}s_{1}^{2}c_{\phi,\sigma_{0}}d\mathbf{v}^{\top}\phi''\left(s_{1}\mathbf{W}\mathbf{x}_{i}\right)\,.
\end{align*}
In order to bound $\mathbb{E}\left[\left|\Delta f_{i}\left(\boldsymbol{\theta}_{t}\right)\right|^{2}\right]$,
we will use the change-of-measure inequality (\cref{eq:second_law_with_eta}), employing the notation $\norm{X}_{L^p,\,q} \triangleq \prn{\expect{q}{\abs{X}^p}}^{\nicefrac{1}{p}}$ for clarity:
\begin{align*}
\mathbb{E}_{p_{t}}\left[\left|\Delta f_{i}\left(\boldsymbol{\theta}_{t}\right)\right|^{2}\right] & =\mathbb{E}_{p_{t}}\left[\left|s_{2}s_{1}^{2}c_{\phi,\sigma_{0}}d\mathbf{v}^{\top}\phi''\left(s_{1}\mathbf{W}\mathbf{x}_{i}\right)\right|^{2}\right] \\
& \leq\eta\left(s_{2}s_{1}^{2}c_{\phi,\sigma_{0}}d\right)^{2}\left\Vert \left|\mathbf{v}^{\top}\phi''\left(s_{1}\mathbf{W}\mathbf{x}_{i}\right)\right|^{2}\right\Vert _{L^{c_{\beta}{\bar{\kappa}}^{2}\beta},p_{0}}\,.
\end{align*}
We have
\[
\mathbf{v}^{\top}\phi''\left(s_{1}\mathbf{W}\mathbf{x}_{i}\right)=\sum_{j\in\left[m\right]}v_{j}\phi''\left(s_{1}\mathbf{w}_{j}^{\top}\mathbf{x}_{i}\right)\,.
\]
This is a sum of zero-mean, independent and subgaussian random variables, each of them satisfying
\begin{align*}
\left|v_{j}\phi''\left(s_{1}\mathbf{w}_{j}^{\top}\mathbf{x}_{i}\right)\right| & \leq\beta_{\phi}\left|v_{j}\right|\\
\left\Vert \left|v_{j}\phi''\left(s_{1}\mathbf{w}_{j}^{\top}\mathbf{x}_{i}\right)\right|\right\Vert _{\psi_{2}} & \leq\beta_{\phi}\text{\ensuremath{\left\Vert \left|v_{j}\right|\right\Vert }}_{\psi_{2}}\leq C\beta_{\phi}\frac{1}{\sqrt{a_{2}}}\,.
\end{align*}
Then, by \citet[Proposition 2.7.1]{vershynin2018high},
\begin{align*}
\left\Vert \left|\mathbf{v}^{\top}\phi''\left(s_{1}\mathbf{W}\mathbf{x}_{i}\right)\right|^{2}\right\Vert _{\psi_{1}} & =\left\Vert \left(\mathbf{v}^{\top}\phi''\left(s_{1}\mathbf{W}\mathbf{x}_{i}\right)\right)\right\Vert _{\psi_{2}}^{2}\\
 & \leq C_{1}m\left(C\beta_{\phi}\frac{1}{\sqrt{a_{2}}}\right)^{2}=C_{2}\beta_{\phi}^{2}\frac{m}{a_{2}}\,.
\end{align*}
Thus, by \citet[Proposition 2.8.1]{vershynin2018high},
\[
\left\Vert \left|\mathbf{v}^{\top}\phi''\left(s_{1}\mathbf{W}\mathbf{x}_{i}\right)\right|^{2}\right\Vert _{L^{c_{\beta}{\bar{\kappa}}^{2}\beta}}\leq C_{3}C_{2}\beta_{\phi}^{2}\frac{m}{a_{2}}c_{\beta}{\bar{\kappa}}^{2}\beta\,.
\]
Consequently,
\begin{align*}
\beta^{-1}\mathbb{E}\left[\sum_{i=1}^{n}r_{i}\left(\boldsymbol{\theta}_{t}\right)\Delta f_{i}\left(\boldsymbol{\theta}_{t}\right)\right] & \leq\beta^{-1}\sum_{i=1}^{n}\sqrt{2\mathbb{E}\left[\ell_{i}\left(\boldsymbol{\theta}_{t}\right)\right]\mathbb{E}\left[\left|\Delta f_{i}\left(\boldsymbol{\theta}_{t}\right)\right|^{2}\right]}\\
 & \leq\beta^{-1}s_{2}s_{1}^{2}c_{\phi,\sigma_{0}}d\sqrt{2\eta C_{3}C_{2}\beta_{\phi}^{2}\frac{m}{a_{2}}c_{\beta}{\bar{\kappa}}^{2}\beta}\sum_{i=1}^{n}\sqrt{\mathbb{E}\left[\ell_{i}\left(\boldsymbol{\theta}_{t}\right)\right]}\\
\left[\text{Jensen}\right] & \leq\beta^{-1/2}s_{2}s_{1}^{2}c_{\phi,\sigma_{0}}d\sqrt{2\eta C_{3}C_{2}\beta_{\phi}^{2}\frac{m}{a_{2}}c_{\beta}{\bar{\kappa}}^{2}}n\sqrt{\frac{1}{n}\sum_{i=1}^{n}\mathbb{E}\left[\ell_{i}\left(\boldsymbol{\theta}_{t}\right)\right]}\\
 & =\beta^{-1/2}ns_{2}s_{1}^{2}c_{\phi,\sigma_{0}}d\sqrt{2\eta C_{3}C_{2}\beta_{\phi}^{2}\frac{m}{a_{2}}c_{\beta}{\bar{\kappa}}^{2}}\sqrt{\mathbb{E}\mathcal{L} \left(\boldsymbol{\theta}_{t}\right)}\\
 & \lesssim\beta^{-1/2}ns_{2}s_{1}^{2}d\sqrt{\frac{m}{a_{2}}{\bar{\kappa}}^{2}}\sqrt{\mathbb{E}\mathcal{L} \left(\boldsymbol{\theta}_{t}\right)}\,.
\end{align*}
\subsection{Fourth term}
Note that under $\phi\left(0\right)=0$ and $\left|\phi'\right|\le\alpha_{\phi}$,
we get $\left|\phi\left(\cdot\right)\right|\leq\alpha_{\phi}\left|\cdot\right|$,
and also $\forall x\in\mathbb{R}$:
\[
\left|\phi\left(x\right)-x\phi'\left(x\right)\right|\leq\left|\phi\left(x\right)\right|+\left|x\phi'\left(x\right)\right|\leq2\alpha_{\phi}\left|x\right|\,,
\]
Thus, setting \(C_\phi=2\alpha_\phi\), the following bound follows from the standing assumptions.
\begin{claim}
There exists a constant $C_{\phi}>0$ such that for all $x\in\mathbb{R}$,
$\left|\phi\left(x\right)-x\phi'\left(x\right)\right|\leq C_{\phi}\left|x\right|$.\footnote{In practice, $C_{\phi}$ is often much smaller than
$2\alpha_{\phi}$.}
\end{claim}
We now have:
\begin{align*}
-\beta^{-1}\mathbb{E}\left[\mathbf{r}\left(\boldsymbol{\theta}_{t}\right)^{\top}\nabla f\left(\boldsymbol{\theta}_{t}\right)\boldsymbol{\Lambda}\boldsymbol{\theta}_{t}\right] & =-\beta^{-1}\sum_{i=1}^{n}\mathbb{E}\left[r_{i}\left(\boldsymbol{\theta}_{t}\right)\nabla f_{i}\left(\boldsymbol{\theta}_{t}\right)^{\top}\boldsymbol{\Lambda}\boldsymbol{\theta}_{t}\right]\,.
\end{align*}
Observe for each $i\in\left[n\right]$:
\begin{align*}
f_{i}\left(\boldsymbol{\theta}\right) & =s_{2}\mathbf{v}^{\top}\boldsymbol{\alpha}\left(\mathbf{x}_{i}\right)\\
\boldsymbol{\alpha}\left(\mathbf{x}_{i}\right) & =\phi\left(s_{1}\mathbf{W}\mathbf{x}_{i}\right)\\
\frac{\partial f_{i}\left(\boldsymbol{\theta}\right)}{\partial\mathbf{v}} & =s_{2}\boldsymbol{\alpha}\left(\mathbf{x}_{i}\right)\,,
\end{align*}
Recall from \cref{eq:df_dW_jk} that
\begin{align*}
\frac{\partial f_{i}\left(\boldsymbol{\theta}\right)}{\partial W_{jk}} & =s_{2}s_{1}v_{j}x_{i,k}\phi'\left(s_{1}\mathbf{w}_{j}^{\top}\mathbf{x}_{i}\right)\\
\Longrightarrow\frac{\partial f_{i}\left(\boldsymbol{\theta}\right)}{\partial\mathbf{w}_{j}} & =s_{2}s_{1}v_{j}\phi'\left(s_{1}\mathbf{w}_{j}^{\top}\mathbf{x}_{i}\right)\mathbf{x}_{i}\,.
\end{align*}
We then have,
{
\allowdisplaybreaks
\begin{align*}
\nabla f_{i}\left(\boldsymbol{\theta}\right)^{\top}\boldsymbol{\Lambda}\boldsymbol{\theta} & =s_{2}\phi\left(s_{1}\mathbf{W}\mathbf{x}_{i}\right)^{\top}a_{2}\mathbf{I}_{m}\mathbf{v}+\sum_{j=1}^{m}s_{2}v_{j}\phi'\left(s_{1}\mathbf{w}_{j}^{\top}\mathbf{x}_{i}\right)s_{1}\mathbf{x}_{i}^{\top}\frac{a_{1}}{\nu_{0}^{2}}\mathbf{I}_{d}\mathbf{w}_{j}\\
\left[\text{define }g\left(x\right)=\phi'\left(x\right)x\right] & =s_{2}\left(a_{2}\mathbf{v}^{\top}\phi\left(s_{1}\mathbf{W}\mathbf{x}_{i}\right)+\frac{a_{1}}{\nu_{0}^{2}}\sum_{j=1}^{m}v_{j}g\left(s_{1}\mathbf{w}_{j}^{\top}\mathbf{x}_{i}\right)\right)\\
 & =s_{2}\left(a_{2}\mathbf{v}^{\top}\phi\left(s_{1}\mathbf{W}\mathbf{x}_{i}\right)+\frac{a_{1}}{\nu_{0}^{2}}\mathbf{v}^{\top}g\left(s_{1}\mathbf{W}\mathbf{x}_{i}\right)\right)\\
 & =s_{2}\mathbf{v}^{\top}\left(a_{2}\phi\left(s_{1}\mathbf{W}\mathbf{x}_{i}\right)+\frac{a_{1}}{\nu_{0}^{2}}g\left(s_{1}\mathbf{W}\mathbf{x}_{i}\right)\right)\,.
\end{align*}
Defining $\gamma\left(x\right)\triangleq \phi\left(x\right)-g\left(x\right)$, we continue:
\begin{align*}
 r_{i}\left(\boldsymbol{\theta}\right)\nabla f_{i}\left(\boldsymbol{\theta}\right)^{\top}\boldsymbol{\Lambda}\boldsymbol{\theta} & =s_{2}\left(s_{2}\mathbf{v}^{\top}\phi\left(s_{1}\mathbf{W}\mathbf{x}_{i}\right)-y_{i}\right)\mathbf{v}^{\top}\left(a_{2}\phi\left(s_{1}\mathbf{W}\mathbf{x}_{i}\right)+\frac{a_{1}}{\nu_{0}^{2}}g\left(s_{1}\mathbf{W}\mathbf{x}_{i}\right)\right)\\
 & =s_{2}\left(s_{2}\mathbf{v}^{\top}\phi\left(s_{1}\mathbf{W}\mathbf{x}_{i}\right)-y_{i}\right)
\\&\qquad\cdot\mathbf{v}^{\top}\left(a_{2}\phi\left(s_{1}\mathbf{W}\mathbf{x}_{i}\right)+\frac{a_{1}}{\nu_{0}^{2}}\left(g\left(s_{1}\mathbf{W}\mathbf{x}_{i}\right)-\phi\left(s_{1}\mathbf{W}\mathbf{x}_{i}\right)+\phi\left(s_{1}\mathbf{W}\mathbf{x}_{i}\right)\right)\right)\\
 & =s_{2}\left(s_{2}\mathbf{v}^{\top}\phi\left(s_{1}\mathbf{W}\mathbf{x}_{i}\right)-y_{i}\right)
 \\&\qquad\cdot \mathbf{v}^{\top}\left(\left(a_{2}+\frac{a_{1}}{\nu_{0}^{2}}\right)\phi\left(s_{1}\mathbf{W}\mathbf{x}_{i}\right)-\frac{a_{1}}{\nu_{0}^{2}}\gamma\left(s_{1}\mathbf{W}\mathbf{x}_{i}\right)\right)\\
 & =\left(s_{2}\mathbf{v}^{\top}\phi\left(s_{1}\mathbf{W}\mathbf{x}_{i}\right)-y_{i}\right)
 \\&\qquad\cdot\left[\left(a_{2}+\frac{a_{1}}{\nu_{0}^{2}}\right)\left(s_{2}\mathbf{v}^{\top}\phi\left(s_{1}\mathbf{W}\mathbf{x}_{i}\right)-y_{i}+y_{i}\right)-s_{2}\frac{a_{1}}{\nu_{0}^{2}}\mathbf{v}^{\top}\gamma\left(s_{1}\mathbf{W}\mathbf{x}_{i}\right)\right]\\
 & =\Biggl(\left(a_{2}+\frac{a_{1}}{\nu_{0}^{2}}\right)\left(s_{2}\mathbf{v}^{\top}\phi\left(s_{1}\mathbf{W}\mathbf{x}_{i}\right)-y_{i}\right)^{2}
 \\ & \qquad +\left(s_{2}\mathbf{v}^{\top}\phi\left(s_{1}\mathbf{W}\mathbf{x}_{i}\right)-y_{i}\right)\left(\left(a_{2}+\frac{a_{1}}{\nu_{0}^{2}}\right)y_{i}-s_{2}\frac{a_{1}}{\nu_{0}^{2}}\mathbf{v}^{\top}\gamma\left(s_{1}\mathbf{W}\mathbf{x}_{i}\right)\right)\Biggr)\\
 & =\left(a_{2}+\frac{a_{1}}{\nu_{0}^{2}}\right)2\ell_{i}\left(\boldsymbol{\theta}\right)+r_{i}\left(\boldsymbol{\theta}\right)\left(\left(a_{2}+\frac{a_{1}}{\nu_{0}^{2}}\right)y_{i}-s_{2}\frac{a_{1}}{\nu_{0}^{2}}\mathbf{v}^{\top}\gamma\left(s_{1}\mathbf{W}\mathbf{x}_{i}\right)\right)\,,
\end{align*}
where $\ell_{i}\left(\boldsymbol{\theta}\right)=\frac{1}{2}\left|r_{i}\left(\boldsymbol{\theta}\right)\right|^{2}$.
\begin{align*}
\Longrightarrow&-\beta^{-1}\mathbb{E}\left[\mathbf{r}\left(\boldsymbol{\theta}_{t}\right)^{\top}\nabla f\left(\boldsymbol{\theta}_{t}\right)\boldsymbol{\Lambda}\boldsymbol{\theta}_{t}\right] \\
& =-\beta^{-1}\sum_{i=1}^{n}\mathbb{E}\left[r_{i}\left(\boldsymbol{\theta}_{t}\right)\nabla f_{i}\left(\boldsymbol{\theta}_{t}\right)^{\top}\boldsymbol{\Lambda}\boldsymbol{\theta}_{t}\right]\\
 & =-\beta^{-1}\Biggl[2\left(a_{2}+\frac{a_{1}}{\nu_{0}^{2}}\right)n\mathbb{E}\mathcal{L} \left(\boldsymbol{\theta}_{t}\right) + \\
 & \qquad \qquad \qquad +\sum_{i=1}^{n}\mathbb{E}\left[r_{i}\left(\boldsymbol{\theta}_t\right)\left(\left(a_{2}+\frac{a_{1}}{\nu_{0}^{2}}\right)y_{i}-s_{2}\frac{a_{1}}{\nu_{0}^{2}}\mathbf{v}^{\top}\gamma\left(s_{1}\mathbf{W}\mathbf{x}_{i}\right)\right)\right]\Biggr]\\
 & \leq-2\beta^{-1}n\left(a_{2}+\frac{a_{1}}{\nu_{0}^{2}}\right)\mathbb{E}\mathcal{L} \left(\boldsymbol{\theta}_{t}\right)
\\&\qquad+\beta^{-1}\sum_{i=1}^{n}\mathbb{E}\left[\left|r_{i}\left(\boldsymbol{\theta}_t\right)\left(\left(a_{2}+\frac{a_{1}}{\nu_{0}^{2}}\right)y_{i}-s_{2}\frac{a_{1}}{\nu_{0}^{2}}\mathbf{v}^{\top}\gamma\left(s_{1}\mathbf{W}\mathbf{x}_{i}\right)\right)\right|\right]\\
\left[(1)\right] & \leq-2\beta^{-1}n\left(a_{2}+\frac{a_{1}}{\nu_{0}^{2}}\right)\mathbb{E}\mathcal{L} \left(\boldsymbol{\theta}_{t}\right)
\\&\qquad+\beta^{-1}\sum_{i=1}^{n}\sqrt{\mathbb{E}\left[\left|r_{i}\left(\boldsymbol{\theta}_{t}\right)\right|^{2}\right]\mathbb{E}\left[\left|\left(a_{2}+\frac{a_{1}}{\nu_{0}^{2}}\right)y_{i}-s_{2}\frac{a_{1}}{\nu_{0}^{2}}\mathbf{v}^{\top}\gamma\left(s_{1}\mathbf{W}\mathbf{x}_{i}\right)\right|^{2}\right]}\\
\left[(2)\right] & \leq-2\beta^{-1}n\left(a_{2}+\frac{a_{1}}{\nu_{0}^{2}}\right)\mathbb{E}\mathcal{L} \left(\boldsymbol{\theta}_{t}\right)
\\&\qquad+\beta^{-1}\sum_{i=1}^{n}\sqrt{2\mathbb{E}\left[\ell_{i}\left(\boldsymbol{\theta}_{t}\right)\right]\mathbb{E}\left[2\left(a_{2}+\frac{a_{1}}{\nu_{0}^{2}}\right)^{2}y_{i}^{2}+2s_{2}^{2}\frac{a_{1}^{2}}{\nu_{0}^{4}}\left|\mathbf{v}^{\top}\gamma\left(s_{1}\mathbf{W}\mathbf{x}_{i}\right)\right|^{2}\right]}\\
\left[(3)\right] & \leq -2 \beta^{-1}n\left(a_{2}+\frac{a_{1}}{\nu_{0}^{2}}\right)\mathbb{E}\mathcal{L} \left(\boldsymbol{\theta}_{t}\right)
\\&\qquad +2\beta^{-1}\sum_{i=1}^{n}\sqrt{\mathbb{E}\left[\ell_{i}\left(\boldsymbol{\theta}_{t}\right)\right]}\left(R_{y}\left(a_{2}+\frac{a_{1}}{\nu_{0}^{2}}\right)+s_{2}\frac{a_{1}}{\nu_{0}^{2}}\sqrt{\mathbb{E}_{p_{t}}\left[\left|\mathbf{v}^{\top}\gamma\left(s_{1}\mathbf{W}\mathbf{x}_{i}\right)\right|^{2}\right]}\right)\,,
\end{align*}
}
where in transition (1) we used Cauchy-Schwarz, in transition (2) we used $\left(a-b\right)^{2}\leq2a^{2}+2b^{2}$, and in transition (3) we used $\sqrt{a+b}\leq\sqrt{a}+\sqrt{b}$.
\subsubsection{Bounding $\mathbb{E}_{p_{t}}\left[\left|\mathbf{v}^{\top}\gamma\left(s_{1}\mathbf{W}\mathbf{x}_{i}\right)\right|^{2}\right]$}
Conditioned on $\mathbf{W}$, $\gamma\left(s_{1}\mathbf{W}\mathbf{x}_{i}\right)$
is fixed and thus $\mathbf{v}^{\top}\gamma\left(s_{1}\mathbf{W}\mathbf{x}_{i}\right)$
is Gaussian at initialization (recall: $\mathbf{v}_{p_{0}}\sim\mathcal{N}\left(\mathbf{0}_{m},\frac{1}{a_{2}}\mathbf{I}_{m}\right)$),
and 
\[
\mathbf{v}^{\top}\gamma\left(s_{1}\mathbf{W}\mathbf{x}_{i}\right)\overset{d}{=}\frac{1}{\sqrt{a_{2}}}\left\Vert \gamma\left(s_{1}\mathbf{W}\mathbf{x}_{i}\right)\right\Vert _{2}G\,,
\]
where $G\sim\mathcal{N}\left(0,1\right)$, independent of 
$\mathbf{W}$, and $\overset{d}{=}$ denotes equality in distribution. Thus, for any $p\geq1$:
\begin{align*}
\left\Vert \mathbf{v}^{\top}\gamma\left(s_{1}\mathbf{W}\mathbf{x}_{i}\right)\right\Vert _{L^{p}\mid\mathbf{W}} & =\frac{1}{\sqrt{a_{2}}}\left\Vert \gamma\left(s_{1}\mathbf{W}\mathbf{x}_{i}\right)\right\Vert _{2}\left\Vert G\right\Vert _{L^{p}}\,.\\
\left[\text{\citep[Proposition 2.6.6]{vershynin2018high}}\right] & \leq\frac{1}{\sqrt{a_{2}}}\left\Vert \gamma\left(s_{1}\mathbf{W}\mathbf{x}_{i}\right)\right\Vert _{2}C\sqrt{p}\,.
\end{align*}
By using the tower property for \(p\)-th moments (\cref{claim:moments total expectation}), we get:
\[
\left\Vert \mathbf{v}^{\top}\gamma\left(s_{1}\mathbf{W}\mathbf{x}_{i}\right)\right\Vert _{L^{p}}
=\left\Vert \left\Vert \mathbf{v}^{\top}\gamma\left(s_{1}\mathbf{W}\mathbf{x}_{i}\right) \mid\mathbf{W} \right\Vert _{L^{p}}\right\Vert _{L^{p}}
\leq\frac{1}{\sqrt{a_{2}}}C\sqrt{p}\left\Vert \left\Vert \gamma\left(s_{1}\mathbf{W}\mathbf{x}_{i}\right)\right\Vert _{2}\right\Vert _{L^{p}}\,.
\]
Note that
\begin{align*}
\left\Vert \gamma\left(s_{1}\mathbf{W}\mathbf{x}_{i}\right)\right\Vert _{2} & \leq C_{\phi}\left\Vert s_{1}\mathbf{W}\mathbf{x}_{i}\right\Vert _{2}=C_{\phi}s_{1}\sqrt{\sum_{j=1}^{m}\left(\mathbf{w}_{j}^{\top}\mathbf{x}_{i}\right)^{2}}=C_{\phi}s_{1}\sqrt{c_{\phi,\sigma_{0}}d\frac{\nu_{0}^{2}}{a_{1}}}\left\Vert \mathbf{z}\right\Vert _{2}\,,
\end{align*}
where
\[
\mathbf z\triangleq
\frac{1}{\sqrt{c_{\phi,\sigma_0}d\,\nu_0^2/a_1}}
\mathbf W\mathbf x_i,
\]
so that \(\mathbf z\sim\mathcal N(\mathbf 0_m,\mathbf I_m)\) at initialization.
As above, \cref{prop:bounded moments of gaussian vector} gives:
\[
\left\Vert \left\Vert \gamma\left(s_{1}\mathbf{W}\mathbf{x}_{i}\right)\right\Vert _{2}\right\Vert _{L^{p}}\leq C_{\phi}s_{1}\sqrt{c_{\phi,\sigma_{0}}d\frac{\nu_{0}^{2}}{a_{1}}}\left\Vert \left\Vert \mathbf{z}\right\Vert _{2}\right\Vert _{L^{p}}\leq C_{\phi}s_{1}\sqrt{c_{\phi,\sigma_{0}}d\frac{\nu_{0}^{2}}{a_{1}}}\left(C\sqrt{p}+\sqrt{m}\right).
\]
Applying \cref{eq:second_law_with_eta} to this nonnegative quantity gives, for all \(t\geq0\),
\begin{align*}
\mathbb{E}_{p_{t}}\left[\left|\mathbf{v}^{\top}\gamma\left(s_{1}\mathbf{W}\mathbf{x}_{i}\right)\right|^{2}\right] & \leq\eta\left\Vert \left|\mathbf{v}^{\top}\gamma\left(s_{1}\mathbf{W}\mathbf{x}_{i}\right)\right|^{2}\right\Vert _{L^{c_{\beta}{\bar{\kappa}}^{2}\beta},p_{0}}\\
\left[\text{\cref{claim:moment of power}}\right] & =\eta\left\Vert \left|\mathbf{v}^{\top}\gamma\left(s_{1}\mathbf{W}\mathbf{x}_{i}\right)\right|\right\Vert _{L^{2c_{\beta}{\bar{\kappa}}^{2}\beta},p_{0}}^{2}\\
 & \leq\eta\left(\frac{1}{\sqrt{a_{2}}}C_{1}\sqrt{2c_{\beta}{\bar{\kappa}}^{2}\beta}\left\Vert \left\Vert \gamma\left(s_{1}\mathbf{W}\mathbf{x}_{i}\right)\right\Vert _{2}\right\Vert _{L^{2c_{\beta}{\bar{\kappa}}^{2}\beta}}\right)^{2}\\
 & \leq\eta\left(\frac{1}{\sqrt{a_{2}}}C_{1}\sqrt{2c_{\beta}{\bar{\kappa}}^{2}\beta}C_{\phi}s_{1}\sqrt{c_{\phi,\sigma_{0}}d\frac{\nu_{0}^{2}}{a_{1}}}\left(C_{2}\sqrt{2c_{\beta}{\bar{\kappa}}^{2}\beta}+\sqrt{m}\right)\right)^{2}\\
 & \leq\eta C_{3}c_{\phi,\sigma_{0}}\nu_{0}^{2}s_{1}^{2}\frac{d}{a_{2}a_{1}}c_{\beta}{\bar{\kappa}}^{2}\beta\left(C_{4}c_{\beta}{\bar{\kappa}}^{2}\beta+m\right)\\
 & \lesssim s_{1}^{2}\frac{d}{a_{2}a_{1}}{\bar{\kappa}}^{2}\beta\left({\bar{\kappa}}^{2}\beta+m\right)\,.
\end{align*}
\subsubsection{Summary of fourth term}
\begin{align*}
&-\beta^{-1}\mathbb{E}\left[\mathbf{r}\left(\boldsymbol{\theta}_{t}\right)^{\top}\nabla f\left(\boldsymbol{\theta}_{t}\right)\boldsymbol{\Lambda}\boldsymbol{\theta}_{t}\right] \\
& \leq-2\beta^{-1}n\left(a_{2}+\frac{a_{1}}{\nu_{0}^{2}}\right)\mathbb{E}\mathcal{L} \left(\boldsymbol{\theta}_{t}\right) \\
& \qquad +2\beta^{-1}\sum_{i=1}^{n}\sqrt{\mathbb{E}\left[\ell_{i}\left(\boldsymbol{\theta}_{t}\right)\right]}\left(R_{y}\left(a_{2}+\frac{a_{1}}{\nu_{0}^{2}}\right)+s_{2}\frac{a_{1}}{\nu_{0}^{2}}\sqrt{\mathbb{E}_{p_{t}}\left[\left|\mathbf{v}^{\top}\gamma\left(s_{1}\mathbf{W}\mathbf{x}_{i}\right)\right|^{2}\right]}\right)\\
 & \lesssim-\beta^{-1}n\left(a_{2}+a_{1}\right)\mathbb{E}\mathcal{L} \left(\boldsymbol{\theta}_{t}\right)\\
 & \qquad +\beta^{-1}\sum_{i=1}^{n}\sqrt{\mathbb{E}\left[\ell_{i}\left(\boldsymbol{\theta}_{t}\right)\right]}\left(a_{2}+a_{1}+s_{2}a_{1}\sqrt{s_{1}^{2}\frac{d}{a_{2}a_{1}}{\bar{\kappa}}^{2}\beta\left({\bar{\kappa}}^{2}\beta+m\right)}\right)\\
\left[\text{Jensen}\right] & \lesssim-\beta^{-1}n\left(a_{2}+a_{1}\right)\mathbb{E}\mathcal{L} \left(\boldsymbol{\theta}_{t}\right) \\
& \qquad+\beta^{-1}n\left(a_{2}+a_{1}+s_{2}a_{1}\sqrt{s_{1}^{2}\frac{d}{a_{2}a_{1}}{\bar{\kappa}}^{2}\beta\left({\bar{\kappa}}^{2}\beta+m\right)}\right)\sqrt{\frac{1}{n}\sum_{i=1}^{n}\mathbb{E}\left[\ell_{i}\left(\boldsymbol{\theta}_{t}\right)\right]}\\
 & \lesssim-\beta^{-1}n\left(a_{2}+a_{1}\right)\mathbb{E}\mathcal{L} \left(\boldsymbol{\theta}_{t}\right) \\
 & \qquad +\beta^{-1}n\left(a_{2}+a_{1}+s_{2}s_{1}\sqrt{\frac{da_{1}}{a_{2}}}\sqrt{{\bar{\kappa}}^{2}\beta\left({\bar{\kappa}}^{2}\beta+m\right)}\right)\sqrt{\mathbb{E}\mathcal{L} \left(\boldsymbol{\theta}_{t}\right)}\,.
\end{align*}
Importantly, observe that for $\mathbb{E}\mathcal{L} \left(\boldsymbol{\theta}_{t}\right)\leq1$,
\begin{align*}
& -\beta^{-1}n\left(a_{2}+a_{1}\right)\mathbb{E}\mathcal{L} \left(\boldsymbol{\theta}_{t}\right)+\beta^{-1}n\left(a_{2}+a_{1}\right)\sqrt{\mathbb{E}\mathcal{L} \left(\boldsymbol{\theta}_{t}\right)} \\
& =\beta^{-1}n\left(a_{2}+a_{1}\right)\sqrt{\mathbb{E}\mathcal{L} \left(\boldsymbol{\theta}_{t}\right)}\left(1-\sqrt{\mathbb{E}\mathcal{L} \left(\boldsymbol{\theta}_{t}\right)}\right)\geq0\,,
\end{align*}
thus this upper bound does not yield a descent contribution from the regularization term in the small-loss regime.
\subsection{Bounding the entire right-hand side}
Recall:
\begin{align*}
\frac{\mathrm{d}}{\mathrm{d}t}\mathbb{E}\mathcal{L} \left(\boldsymbol{\theta}_{t}\right) & =\frac{1}{n}\biggl(-n^{-1}\mathbb{E}\left[\mathbf{r}\left(\boldsymbol{\theta}_{t}\right)^{\top}\mathbf{K}_{\mathrm{NTK}}\left(\boldsymbol{\theta}_{t}\right)\mathbf{r}\left(\boldsymbol{\theta}_{t}\right)\right]\\
 & \qquad\quad+\beta^{-1}\mathbb{E}\left[\mathrm{Tr}\left(\nabla f\left(\boldsymbol{\theta}_{t}\right)^{\top}\nabla f\left(\boldsymbol{\theta}_{t}\right)\right)\right]\\
 & \qquad\quad+\beta^{-1}\mathbb{E}\left[\sum_{i=1}^{n}r_{i}\left(\boldsymbol{\theta}_{t}\right)\Delta f_{i}\left(\boldsymbol{\theta}_{t}\right)\right]\\
 & \qquad\quad-\beta^{-1}\mathbb{E}\left[\mathbf{r}\left(\boldsymbol{\theta}_{t}\right)^{\top}\nabla f\left(\boldsymbol{\theta}_{t}\right)\boldsymbol{\Lambda}\boldsymbol{\theta}_{t}\right]\biggr)\,.
\end{align*}
Bounding each term, we obtained:
\begin{align*}
-n^{-1}\mathbb{E}\left[\mathbf{r}\left(\boldsymbol{\theta}_{t}\right)^{\top}\mathbf{K}_{\mathrm{NTK}}\left(\boldsymbol{\theta}_{t}\right)\mathbf{r}\left(\boldsymbol{\theta}_{t}\right)\right] & \lesssim-\lambda_{1}s_{2}^{2}m\mathbb{E}\left[\mathcal{L} \left(\boldsymbol{\theta}_{t}\right)\right] \\
& \quad +\lambda_{1}s_{2}^{2}m\left(1+{\bar{\kappa}}^{4}\beta\right)\exp\left(-\tilde{c}_{2}\frac{m}{c_{\beta}{\bar{\kappa}}^{2}\beta\frac{n}{\lambda_{1}}\ln\frac{n}{\lambda_{1}}}\right)\,,\\
+\beta^{-1}\mathbb{E}\left[\mathrm{Tr}\left(\nabla f\left(\boldsymbol{\theta}_{t}\right)^{\top}\nabla f\left(\boldsymbol{\theta}_{t}\right)\right)\right] & \lesssim\beta^{-1}ns_{2}^{2}s_{1}^{2}d\left(\frac{1}{a_{1}}+\frac{1}{a_{2}}\right)\left({\bar{\kappa}}^{2}\beta+m\right)\,,\\
+\beta^{-1}\mathbb{E}\left[\sum_{i=1}^{n}r_{i}\left(\boldsymbol{\theta}_{t}\right)\Delta f_{i}\left(\boldsymbol{\theta}_{t}\right)\right] & \lesssim\beta^{-1/2}ns_{2}s_{1}^{2}d\sqrt{\frac{m}{a_{2}}{\bar{\kappa}}^{2}}\sqrt{\mathbb{E}\mathcal{L} \left(\boldsymbol{\theta}_{t}\right)}\,,\\
-\beta^{-1}\mathbb{E}\left[\mathbf{r}\left(\boldsymbol{\theta}_{t}\right)^{\top}\nabla f\left(\boldsymbol{\theta}_{t}\right)\boldsymbol{\Lambda}\boldsymbol{\theta}_{t}\right] & \lesssim-\beta^{-1}n\left(a_{2}+a_{1}\right)\mathbb{E}\mathcal{L} \left(\boldsymbol{\theta}_{t}\right)
 +\beta^{-1}n \cdot  \\&\quad \cdot \left(a_{2}+a_{1}+s_{2}s_{1}\sqrt{\frac{da_{1}}{a_{2}}}\sqrt{{\bar{\kappa}}^{2}\beta\left({\bar{\kappa}}^{2}\beta+m\right)}\right)\sqrt{\mathbb{E}\mathcal{L} \left(\boldsymbol{\theta}_{t}\right)}\,.
\end{align*}
Together, we can bound $\frac{\mathrm{d}}{\mathrm{d}t}\mathbb{E}\mathcal{L} \left(\boldsymbol{\theta}_{t}\right)$
at all times $t\geq0$ by
\begin{align*}
& \frac{\mathrm{d}}{\mathrm{d}t}\mathbb{E}\mathcal{L} \left(\boldsymbol{\theta}_{t}\right) \\
& \lesssim\frac{1}{n}\Biggl(-\left(\lambda_{1}s_{2}^{2}m+\beta^{-1}n\!\left(a_{2}+a_{1}\right)\right)\mathbb{E}\mathcal{L} \left(\boldsymbol{\theta}_{t}\right)\\
 & \qquad\ +\left[\beta^{-1}n\left(a_{2}+a_{1}+s_{2}s_{1}\sqrt{\frac{da_{1}}{a_{2}}}\sqrt{{\bar{\kappa}}^{2}\beta\left({\bar{\kappa}}^{2}\beta+m\right)}\right)\!+\!\beta^{-1/2}ns_{2}s_{1}^{2}d\sqrt{\frac{m}{a_{2}}{\bar{\kappa}}^{2}}\right] \! \sqrt{\mathbb{E}\mathcal{L} \left(\boldsymbol{\theta}_{t}\right)}\\
 & \qquad\ +\lambda_{1}s_{2}^{2}m\left(1+{\bar{\kappa}}^{4}\beta\right)\exp\left(-\tilde{c}_{2}\frac{m}{c_{\beta}{\bar{\kappa}}^{2}\beta\frac{n}{\lambda_{1}}\ln\frac{n}{\lambda_{1}}}\right)+\beta^{-1}ns_{2}^{2}s_{1}^{2}d\left(\frac{1}{a_{1}}+\frac{1}{a_{2}}\right)\left({\bar{\kappa}}^{2}\beta+m\right)\Biggr).
\end{align*}
The regularization contribution does not provide a useful descent term in the small-loss regime: for \(u=\mathbb{E}\mathcal{L}(\boldsymbol{\theta}_t)\leq1\),
\[
-\beta^{-1}n(a_2+a_1)u+\beta^{-1}n(a_2+a_1)\sqrt u
=
\beta^{-1}n(a_2+a_1)\sqrt u(1-\sqrt u)\geq0\,.
\]
Therefore, in closing the differential inequality, we rely only on the NTK term for the linear descent and treat all \(\sqrt u\)-terms as error terms, to be absorbed by Young's inequality.
Therefore, discarding the additional negative regularization contribution, we obtain
\begin{align*}
& \frac{\mathrm{d}}{\mathrm{d}t}\mathbb{E}\mathcal{L} \left(\boldsymbol{\theta}_{t}\right) \\
& \lesssim\frac{1}{n}\Biggl(-\lambda_{1}s_{2}^{2}m\mathbb{E}\mathcal{L} \left(\boldsymbol{\theta}_{t}\right)\\
 & \qquad\ +\left[\beta^{-1}n\!\left(a_{2}+a_{1}+s_{2}s_{1}\sqrt{\frac{da_{1}}{a_{2}}}\sqrt{{\bar{\kappa}}^{2}\beta\left({\bar{\kappa}}^{2}\beta+m\right)}\right) \! + \! \beta^{-1/2}ns_{2}s_{1}^{2}d\sqrt{\frac{m}{a_{2}}{\bar{\kappa}}^{2}}\right] \! \sqrt{\mathbb{E}\mathcal{L} \left(\boldsymbol{\theta}_{t}\right)}\\
 & \qquad\ +\lambda_{1}s_{2}^{2}m\left(1+{\bar{\kappa}}^{4}\beta\right)\exp\left(-\tilde{c}_{2}\frac{m}{c_{\beta}{\bar{\kappa}}^{2}\beta\frac{n}{\lambda_{1}}\ln\frac{n}{\lambda_{1}}}\right)+\beta^{-1}ns_{2}^{2}s_{1}^{2}d\left(\frac{1}{a_{1}}+\frac{1}{a_{2}}\right)\left({\bar{\kappa}}^{2}\beta+m\right)\Biggr).
\end{align*}
Next, in order to have
\[
\left(1+{\bar{\kappa}}^{4}\beta\right)\exp\left(-\tilde{c}_{2}\frac{m}{c_{\beta}{\bar{\kappa}}^{2}\beta\frac{n}{\lambda_{1}}\ln\frac{n}{\lambda_{1}}}\right)\ll 1\,,
\]
it is enough to assume,
\[
m =\Omega \prn{ \bar\kappa^2 \beta \frac{n}{\lambda_1} \ln \frac{n}{\lambda_1}\ln \prn{\bar\kappa^4 \beta}}\,,
\]
since $\beta=\Omega \prn{\frac{n}{\lambda_1}} \gg 1$, and $\bar\kappa\geq1$. This matches the assumptions of the theorem.

In particular, this implies $m\gg {\bar{\kappa}}^{2}\beta$, thus we can simplify
further:
\begin{align*}
& \frac{\mathrm{d}}{\mathrm{d}t}\mathbb{E}\mathcal{L} \left(\boldsymbol{\theta}_{t}\right) \\
& \lesssim\frac{1}{n}\Biggl(-\lambda_{1}s_{2}^{2}m\mathbb{E}\mathcal{L} \left(\boldsymbol{\theta}_{t}\right)\\
 & \qquad\ +\left[\beta^{-1}n\left(a_{2}+a_{1}+s_{2}s_{1}\sqrt{\frac{da_{1}}{a_{2}}}\sqrt{{\bar{\kappa}}^{2}\beta m}\right)+\beta^{-1/2}ns_{2}s_{1}^{2}d\sqrt{\frac{m}{a_{2}}{\bar{\kappa}}^{2}}\right]\sqrt{\mathbb{E}\mathcal{L} \left(\boldsymbol{\theta}_{t}\right)}\\
 & \qquad\ +\lambda_{1}s_{2}^{2}m\left(1+{\bar{\kappa}}^{4}\beta\right)\exp\left(-\tilde{c}_{2}\frac{m}{c_{\beta}{\bar{\kappa}}^{2}\beta\frac{n}{\lambda_{1}}\ln\frac{n}{\lambda_{1}}}\right)+\beta^{-1}ns_{2}^{2}s_{1}^{2}md\left(\frac{1}{a_{1}}+\frac{1}{a_{2}}\right)\Biggr)\,.
\end{align*}
By Young's inequality, we have for any $\varepsilon>0$,
\[
b\sqrt{\mathbb{E}\mathcal{L} \left(\boldsymbol{\theta}_{t}\right)}\leq\frac{b^{2}}{2\varepsilon}+\frac{\varepsilon}{2}\mathbb{E}\mathcal{L} \left(\boldsymbol{\theta}_{t}\right)\,.
\]
Choosing $\varepsilon=\lambda_{1}s_{2}^{2}m$, we have:
\begin{align*}
& \frac{\mathrm{d}}{\mathrm{d}t}\mathbb{E}\mathcal{L} \left(\boldsymbol{\theta}_{t}\right) \\
& \lesssim\frac{1}{n}\Biggl(-\frac{\lambda_{1}}{2}s_{2}^{2}m\mathbb{E}\mathcal{L} \left(\boldsymbol{\theta}_{t}\right)\\
 & \qquad\ +\frac{\left[\beta^{-1}n\left(a_{2}+a_{1}+s_{2}s_{1}\sqrt{\frac{da_{1}}{a_{2}}}\sqrt{{\bar{\kappa}}^{2}\beta m}\right)+\beta^{-1/2}ns_{2}s_{1}^{2}d\sqrt{\frac{m}{a_{2}}{\bar{\kappa}}^{2}}\right]^{2}}{2\lambda_{1}s_{2}^{2}m}\\
 & \qquad\ +\lambda_{1}s_{2}^{2}m\left(1+{\bar{\kappa}}^{4}\beta\right)\exp\left(-\tilde{c}_{2}\frac{m}{c_{\beta}{\bar{\kappa}}^{2}\beta\frac{n}{\lambda_{1}}\ln\frac{n}{\lambda_{1}}}\right)+\beta^{-1}ns_{2}^{2}s_{1}^{2}md\left(\frac{1}{a_{1}}+\frac{1}{a_{2}}\right)\Biggr)\\
 & \lesssim\frac{1}{n}\Biggl(-\lambda_{1}s_{2}^{2}m\mathbb{E}\mathcal{L} \left(\boldsymbol{\theta}_{t}\right)\\
 & \qquad\ +\frac{\beta^{-2}n^{2}\left(a_{2}^{2}+a_{1}^{2}+s_{2}^{2}s_{1}^{2}\frac{da_{1}}{a_{2}}{\bar{\kappa}}^{2}\beta m\right)+\beta^{-1}n^{2}s_{2}^{2}s_{1}^{4}d^{2}\frac{m}{a_{2}}{\bar{\kappa}}^{2}}{\lambda_{1}s_{2}^{2}m}\\
 & \qquad\ +\lambda_{1}s_{2}^{2}m\left(1+{\bar{\kappa}}^{4}\beta\right)\exp\left(-\tilde{c}_{2}\frac{m}{c_{\beta}{\bar{\kappa}}^{2}\beta\frac{n}{\lambda_{1}}\ln\frac{n}{\lambda_{1}}}\right)+\beta^{-1}ns_{2}^{2}s_{1}^{2}md\left(\frac{1}{a_{1}}+\frac{1}{a_{2}}\right)\Biggr)\,.
\end{align*}
Then, from \cref{claim:ode_inequality} we have
\[
\mathbb{E}\mathcal{L} \left(\boldsymbol{\theta}_{t}\right)\leq\left(\mathbb{E}\mathcal{L} \left(\boldsymbol{\theta}_{0}\right)-\mathcal{L}_\infty\right)e^{-cs_{2}^{2}\frac{\lambda_{1}m}{n}t}+\mathcal{L}_\infty\,,
\]
where the loss floor may be chosen, up to universal multiplicative constants, as
\begin{align*}
\mathcal{L}_\infty & \asymp \frac{1}{\frac{\lambda_{1}}{2}s_{2}^{2}m}\Biggl(\frac{\beta^{-2}n^{2}\left(a_{2}^{2}+a_{1}^{2}+s_{2}^{2}s_{1}^{2}\frac{da_{1}}{a_{2}}{\bar{\kappa}}^{2}\beta m\right)+\beta^{-1}n^{2}s_{2}^{2}s_{1}^{4}d^{2}\frac{m}{a_{2}}{\bar{\kappa}}^{2}}{\lambda_{1}s_{2}^{2}m}\\
 & \qquad\qquad\ +\lambda_{1}s_{2}^{2}m\left(1+{\bar{\kappa}}^{4}\beta\right)\exp\left(-\tilde{c}_{2}\frac{m}{c_{\beta}{\bar{\kappa}}^{2}\beta\frac{n}{\lambda_{1}}\ln\frac{n}{\lambda_{1}}}\right)+\beta^{-1}ns_{2}^{2}s_{1}^{2}md\left(\frac{1}{a_{1}}+\frac{1}{a_{2}}\right)\Biggr)\\
 & \asymp\beta^{-2}\frac{n^{2}}{\lambda_{1}^{2}}\left(\frac{a_{2}^{2}}{s_{2}^{4}m^{2}}+\frac{a_{1}^{2}}{s_{2}^{4}m^{2}}\right)+\beta^{-1}\frac{n^{2}}{\lambda_{1}^{2}}\frac{s_{1}^{2}d}{s_{2}^{2}a_{2}m}{\bar{\kappa}}^{2}\left(a_{1}+s_{1}^{2}d\right)\\
 & \qquad+\left(1+{\bar{\kappa}}^{4}\beta\right)\exp\left(-\tilde{c}_{2}\frac{m}{c_{\beta}{\bar{\kappa}}^{2}\beta\frac{n}{\lambda_{1}}\ln\frac{n}{\lambda_{1}}}\right)+\beta^{-1}\frac{n}{\lambda_{1}}s_{1}^{2}d\left(\frac{1}{a_{1}}+\frac{1}{a_{2}}\right)\,.
\end{align*}
\subsection{Variance preservation and output scaling}
In the proof of the lower bound for the minimal eigenvalue of the
NTK (\cref{thm:lower_bound_ntk_deep}), we demanded $\frac{a_{1}}{s_{1}^{2}}=d$, to achieve
the necessary variance-preservation. For the last layer, this demand does not
hold, and we use the scaling parameter. Then,
\begin{equation}
\begin{aligned}
\mathcal{L}_\infty& \asymp\left(1+\bar\kappa^{4}\beta\right)\exp\left(-\tilde{c}_{2}\frac{m}{c_{\beta}\bar\kappa^{2}\beta\frac{n}{\lambda_{1}}\ln\frac{n}{\lambda_{1}}}\right)+\beta^{-2}\frac{n^{2}}{\lambda_{1}^{2}}\left(\frac{1}{\kappa^{4}}+\frac{s_{1}^{4}d^{2}}{s_{2}^{4}m^{2}}\right) \\ 
& \qquad +\beta^{-1}\frac{n^{2}}{\lambda_{1}^{2}}\bar\kappa^{2} \kappa^2 \frac{s_{1}^{4}d^{2}}{s_{2}^{4}m^{2}}+\beta^{-1}\frac{n}{\lambda_{1}}\left(1+\kappa^{2}\frac{s_{1}^{2}d}{s_{2}^{2}m}\right)\,.
\end{aligned}
\end{equation}
Finally, since \(c_\beta=O(1)\), we absorb it into the universal constant in the exponent and obtain
\begin{equation}
\begin{aligned}
\mathcal{L}_\infty& \asymp\left(1+\bar\kappa^{4}\beta\right)\exp\left(-{c_{2}}\frac{m}{\bar\kappa^{2}\beta\frac{n}{\lambda_{1}}\ln\frac{n}{\lambda_{1}}}\right)+\beta^{-2}\frac{n^{2}}{\lambda_{1}^{2}}\left(\frac{1}{\kappa^{4}}+\frac{s_{1}^{4}d^{2}}{s_{2}^{4}m^{2}}\right) \\ 
& \qquad +\beta^{-1}\frac{n^{2}}{\lambda_{1}^{2}}\bar\kappa^{2} \kappa^2 \frac{s_{1}^{4}d^{2}}{s_{2}^{4}m^{2}}+\beta^{-1}\frac{n}{\lambda_{1}}\left(1+\kappa^{2}\frac{s_{1}^{2}d}{s_{2}^{2}m}\right)\,.\label{eq:L_infty_for_analysis}
\end{aligned}
\end{equation}
\end{proof}
\newpage

\section{Proof of global convergence of regularized CLD for deep networks}
\label{app:langevin_proof_deep}

\subsection{Setting}

\textbf{Network. }We consider the neural network defined in \cref{sec:setting_deep_nn}. We denote $\ensuremath{m\triangleq\min_{l\in\left[L\right]}m_{l}}$.

\textbf{Intermediate quantities.} Throughout the analysis, for every
$l\in\left[L\right]$, we define $\tilde{\boldsymbol{\alpha}}^{\left(l\right)}\left(\mathbf{x}\right)\in\mathbb{R}^{m_{l}}$
by
\begin{align*}
\tilde{\boldsymbol{\alpha}}^{\left(l\right)}\left(\mathbf{x}\right) & \triangleq s_{l}\mathbf{W}^{\left(l\right)}\boldsymbol{\alpha}^{\left(l-1\right)}\left(\mathbf{x}\right)
\end{align*}
and $\boldsymbol{\delta}^{\left(l\right)}\left(\mathbf{x}\right)\in\mathbb{R}^{m_{l}}$
by
\begin{align*}
\boldsymbol{\delta}^{\left(l\right)}\left(\mathbf{x}\right) & \triangleq\left(\frac{\partial f\left(\boldsymbol{\theta};\mathbf{x}\right)}{\partial\tilde{\boldsymbol{\alpha}}^{\left(l\right)}\left(\mathbf{x}\right)}\right)^{\top}.
\end{align*}
We can express $\boldsymbol{\delta}^{\left(l\right)}\left(\mathbf{x}\right)$
recursively by
\begin{align}
\boldsymbol{\delta}^{\left(L\right)}\left(\mathbf{x}\right) & =s_{L+1}\mathrm{diag}\left(\phi'\left(\tilde{\boldsymbol{\alpha}}^{\left(L\right)}\left(\mathbf{x}\right)\right)\right)\mathbf{v},\nonumber \\
\forall l\in\left[L-1\right],\ \boldsymbol{\delta}^{\left(l\right)}\left(\mathbf{x}\right) & =\left(\frac{\partial f\left(\boldsymbol{\theta};\mathbf{x}\right)}{\partial\tilde{\boldsymbol{\alpha}}^{\left(l+1\right)}\left(\mathbf{x}\right)}\frac{\partial\tilde{\boldsymbol{\alpha}}^{\left(l+1\right)}\left(\mathbf{x}\right)}{\partial\boldsymbol{\alpha}^{\left(l\right)}\left(\mathbf{x}\right)}\frac{\partial\boldsymbol{\alpha}^{\left(l\right)}\left(\mathbf{x}\right)}{\partial\tilde{\boldsymbol{\alpha}}^{\left(l\right)}\left(\mathbf{x}\right)}\right)^{\top}\label{eq:delta}\\
 & =\mathrm{diag}\left(\phi'\left(\tilde{\boldsymbol{\alpha}}^{\left(l\right)}\left(\mathbf{x}\right)\right)\right)s_{l+1}\left(\mathbf{W}^{\left(l+1\right)}\right)^{\top}\boldsymbol{\delta}^{\left(l+1\right)}\left(\mathbf{x}\right).\nonumber 
\end{align}
Note that for every $l\in\left[L\right]$, the matrix $\frac{\partial f\left(\boldsymbol{\theta};\mathbf{x}\right)}{\partial\mathbf{W}^{\left(l\right)}}\in\mathbb{R}^{m_{l}\times m_{l-1}}$
satisfies
\begin{align}
\frac{\partial f\left(\boldsymbol{\theta};\mathbf{x}\right)}{\partial\mathbf{W}^{\left(l\right)}} & =s_{l}\boldsymbol{\delta}^{\left(l\right)}\left(\mathbf{x}\right)\left(\boldsymbol{\alpha}^{\left(l-1\right)}\left(\mathbf{x}\right)\right)^{\top}.\label{eq:nabla_f_Wl}
\end{align}

\textbf{Matrix norms. }Throughout this section, unless stated otherwise,
$\left\Vert \cdot\right\Vert $ denotes the spectral norm when used
with a matrix. 

\textbf{Sums and products.} The analysis of the neural network's gradient
and Laplacian includes many sums and products over ranges of layers.
To make the math as compact as possible, we sometimes combine terms
in a way that makes a sum's (or product's) initial running index greater
than its final index. We define such occurrences to be equal to 0
for sums and to 1 for products. 

\textbf{Use of $\lesssim$ and $\asymp$.} Throughout the proof, inequalities
involving both descent and error terms are interpreted in the following
one-sided asymptotic sense: an estimate of the form
\begin{align*}
X\lesssim-A+B
\end{align*}

means that there exist universal constants $c,C>0$, independent of
the problem parameters, such that 
\begin{align*}
X\leq-cA+CB.
\end{align*}

Thus, positive error terms may be enlarged and descent coefficients
may be reduced by universal constant factors. These constants are
absorbed into the final constants $c_{1},c_{2}$ and into the implicit
constants in $\mathcal{L}_{\infty}$.

Similarly, when writing $\mathcal{L}_{\infty}\asymp F$, we mean that
the loss floor in the upper bound may be chosen as $\mathcal{L}_{\infty}=CF$
for a sufficiently large universal constant $C>0$. Thus, $\asymp$
records the parameter dependence of the upper-bound floor and is not
meant as a matching lower bound on the optimal possible floor.

For $\mathbf{a}\in\mathbb{R}^{p}$ and $\mathbf{b}\in\mathbb{R}^{q}$
we define $\frac{\partial\mathbf{a}}{\partial\mathbf{b}}\in\mathbb{R}^{p\times q}$
to be the matrix given by $\left(\frac{\partial\mathbf{a}}{\partial\mathbf{b}}\right)_{pq}=\frac{\partial a_{p}}{\partial b_{q}}$.
Whenever $p=1$ we treat this as a row vector and whenever $q=1$
we treat this as a column vector.

\subsection{Supporting propositions}

\begin{proposition}
\label{prop:param_norm_bounds} There exists a constant $c>0$ such
that for any $p\geq1$, the parameters at initialization satisfy
\begin{align*}
\left\Vert \left\Vert \mathbf{W}^{\left(l\right)}\right\Vert ^{2}\right\Vert _{L^{p}} & \leq c\frac{\nu_{0}^{2}}{a_{l}}\left(m_{l-1}+m_{l}+p\right),\quad\forall l\in\left[L\right];\\
\left\Vert \left\Vert \mathbf{v}\right\Vert ^{2}\right\Vert _{L^{p}} & \leq c\frac{1}{a_{L+1}}\left(p+m_{L}\right).
\end{align*}
\end{proposition}
\begin{proof}
Let $\mathbf{G}\in\mathbb{R}^{m\times n}$ be a matrix with i.i.d.
standard Gaussian entries. By \cref{claim:moment of power,claim:bounded moments of gaussian matrix}, there exists
$c'>0$ such that for any $l\in\left[L\right]$ and any $p\geq1$,
\begin{align*}
\left\Vert \left\Vert \mathbf{W}^{\left(l\right)}\right\Vert ^{2}\right\Vert _{L^{p}} & \leq\left\Vert \left\Vert \mathbf{W}^{\left(l\right)}\right\Vert \right\Vert _{L^{2p}}^{2}=\frac{\nu_{0}^{2}}{a_{l}}\left\Vert \left\Vert \mathbf{G}^{\left(l\right)}\right\Vert \right\Vert _{L^{2p}}^{2}\leq c'\frac{\nu_{0}^{2}}{a_{l}}\left(m_{l-1}+m_{l}+p\right).
\end{align*}
Additionally, denoting $\mathbf{z}^{\left(L+1\right)}\sim\mathcal{N}\left(\mathbf{0}_{m_{L}},\mathbf{I}_{m_{L}}\right)$,
by \cref{claim:moment of power,prop:bounded moments of gaussian vector}, there exists $c'>0$ such that 
\begin{align*}
\left\Vert \left\Vert \mathbf{v}\right\Vert ^{2}\right\Vert _{L^{p}}\leq\left\Vert \left\Vert \mathbf{v}\right\Vert \right\Vert _{L^{2p}}^{2} & =\frac{1}{a_{L+1}}\left\Vert \left\Vert \mathbf{z}^{\left(L+1\right)}\right\Vert \right\Vert _{L^{2p}}^{2}\\
 & \leq\frac{1}{a_{L+1}}\left(c\sqrt{2p}+\sqrt{m_{L}}\right)^{2}\leq4c'\frac{1}{a_{L+1}}\left(p+m_{L}\right).
\end{align*}
\end{proof}
\begin{proposition}
\label{prop:bound_delta}It holds that for all $l\in\left[L\right]$,
\begin{align*}
\left\Vert \boldsymbol{\delta}^{\left(l\right)}\left(\mathbf{x}\right)\right\Vert ^{2} & \leq\alpha_{\phi}^{2(L+1-l)}\left(\prod_{k=l+1}^{L}s_{k}^{2}\left\Vert \mathbf{W}^{\left(k\right)}\right\Vert ^{2}\right)s_{L+1}^{2}\left\Vert \mathbf{v}\right\Vert ^{2}.
\end{align*}
In addition, at initialization, there exists a constant $c>0$ such
that for all $p\geq1$, 
\begin{align*}
\left\Vert \left\Vert \boldsymbol{\delta}^{\left(l\right)}\left(\mathbf{x}\right)\right\Vert ^{2}\right\Vert _{L^{p}} & \leq\left(c\nu_{0}\alpha_{\phi}\right)^{2(L+1-l)}\left(\prod_{k=l+1}^{L}\frac{s_{k}^{2}}{a_{k}}\left(m_{k-1}+m_{k}+p\right)\right)\frac{s_{L+1}^{2}}{a_{L+1}}\left(p+m_{L}\right).
\end{align*}
As a direct consequence, at initialization, 
\begin{align*}
\left\Vert \left\Vert \boldsymbol{\delta}^{\left(l\right)}\left(\mathbf{x}\right)\right\Vert \right\Vert _{L^{p}} & \leq\left(c\nu_{0}\alpha_{\phi}\right)^{L+1-l}\left(\prod_{k=l+1}^{L}\sqrt{\frac{s_{k}^{2}}{a_{k}}\left(m_{k-1}+m_{k}+p\right)}\right)\sqrt{\frac{s_{L+1}^{2}}{a_{L+1}}\left(p+m_{L}\right)}.
\end{align*}
\end{proposition}
\begin{proof}
By \cref{eq:delta}, properties of the
Euclidean norm and the spectral norm, and \cref{ass:activation_bounded_gradients}, we have that
\begin{align*}
\left\Vert \boldsymbol{\delta}^{\left(L\right)}\left(\mathbf{x}\right)\right\Vert ^{2} & \leq s_{L+1}^{2}\left\Vert \mathrm{diag}\left(\phi'\left(\tilde{\boldsymbol{\alpha}}^{\left(L\right)}\left(\mathbf{x}\right)\right)\right)\right\Vert ^{2}\left\Vert \mathbf{v}\right\Vert ^{2}\leq\alpha_{\phi}^{2}s_{L+1}^{2}\left\Vert \mathbf{v}\right\Vert ^{2}.
\end{align*}
In addition, for all $l\in\left[L-1\right]$, 
\begin{align*}
\left\Vert \boldsymbol{\delta}^{\left(l\right)}\left(\mathbf{x}\right)\right\Vert ^{2} & \leq s_{l+1}^{2}\left\Vert \mathrm{diag}\left(\phi'\left(\tilde{\boldsymbol{\alpha}}^{\left(l\right)}\left(\mathbf{x}\right)\right)\right)\right\Vert ^{2}\left\Vert \mathbf{W}^{\left(l+1\right)}\right\Vert ^{2}\left\Vert \boldsymbol{\delta}^{\left(l+1\right)}\left(\mathbf{x}\right)\right\Vert ^{2}\\
 & \leq\alpha_{\phi}^{2}s_{l+1}^{2}\left\Vert \mathbf{W}^{\left(l+1\right)}\right\Vert ^{2}\left\Vert \boldsymbol{\delta}^{\left(l+1\right)}\left(\mathbf{x}\right)\right\Vert ^{2}.
\end{align*}
Continuing recursively, we get
\begin{align*}
\left\Vert \boldsymbol{\delta}^{\left(l\right)}\left(\mathbf{x}\right)\right\Vert ^{2} & \leq\alpha_{\phi}^{2(L+1-l)}\left(\prod_{k=l+1}^{L}s_{k}^{2}\left\Vert \mathbf{W}^{\left(k\right)}\right\Vert ^{2}\right)s_{L+1}^{2}\left\Vert \mathbf{v}\right\Vert ^{2}.
\end{align*}
Together with \cref{prop:param_norm_bounds},
we get
\begin{align*}
\left\Vert \left\Vert \boldsymbol{\delta}^{\left(l\right)}\left(\mathbf{x}\right)\right\Vert ^{2}\right\Vert _{L^{p}} & \leq\alpha_{\phi}^{2(L+1-l)}\left(\prod_{k=l+1}^{L}c_{3}\nu_{0}^{2}\frac{s_{k}^{2}}{a_{k}}\left(m_{k-1}+m_{k}+p\right)\right)c_{4}\frac{s_{L+1}^{2}}{a_{L+1}}\left(p+m_{L}\right).
\end{align*}
Therefore, there exists a constant $c_{5}>0$ such that 
\begin{align*}
\left\Vert \left\Vert \boldsymbol{\delta}^{\left(l\right)}\left(\mathbf{x}\right)\right\Vert ^{2}\right\Vert _{L^{p}} & \leq\left(c_{5}\nu_{0}^{2}\alpha_{\phi}^{2}\right)^{L+1-l}\left(\prod_{k=l+1}^{L}\frac{s_{k}^{2}}{a_{k}}\left(m_{k-1}+m_{k}+p\right)\right)\frac{s_{L+1}^{2}}{a_{L+1}}\left(p+m_{L}\right).
\end{align*}
Therefore,
\begin{align*}
\left\Vert \left\Vert \boldsymbol{\delta}^{\left(l\right)}\left(\mathbf{x}\right)\right\Vert \right\Vert _{L^{p}} & \overset{(i)}{\leq}\left\Vert \left\Vert \boldsymbol{\delta}^{\left(l\right)}\left(\mathbf{x}\right)\right\Vert \right\Vert _{L^{2p}}\\
 & \overset{(ii)}{=}\sqrt{\left\Vert \left\Vert \boldsymbol{\delta}^{\left(l\right)}\left(\mathbf{x}\right)\right\Vert ^{2}\right\Vert _{L^{p}}}\\
 & \leq\left(\sqrt{c_{5}}\nu_{0}\alpha_{\phi}\right)^{L+1-l}\left(\prod_{k=l+1}^{L}\sqrt{\frac{s_{k}^{2}}{a_{k}}\left(m_{k-1}+m_{k}+p\right)}\right)\sqrt{\frac{s_{L+1}^{2}}{a_{L+1}}\left(p+m_{L}\right)},
\end{align*}
where $(i)$ holds by the monotonicity of $L^{p}$ norms and $(ii)$
holds since
\begin{align*}
\sqrt{\left\Vert \left\Vert \boldsymbol{\delta}^{\left(l\right)}\left(\mathbf{x}\right)\right\Vert ^{2}\right\Vert _{L^{p}}} & =\sqrt{\left(\mathbb{E}\left[\left\Vert \boldsymbol{\delta}^{\left(l\right)}\left(\mathbf{x}\right)\right\Vert ^{2p}\right]\right)^{1/p}}=\left(\mathbb{E}\left[\left\Vert \boldsymbol{\delta}^{\left(l\right)}\left(\mathbf{x}\right)\right\Vert ^{2p}\right]\right)^{1/2p}=\left\Vert \left\Vert \boldsymbol{\delta}^{\left(l\right)}\left(\mathbf{x}\right)\right\Vert \right\Vert _{L^{2p}}.
\end{align*}
\end{proof}
\begin{proposition}
\label{prop:activation_norm_bound} For every $l\in\left[L\right]$,
there exists a constant $c>0$ such that for all $p\geq1$,
\begin{align*}
\left\Vert \left\Vert \boldsymbol{\alpha}^{\left(l\right)}\left(\mathbf{x}\right)\right\Vert ^{2}\right\Vert _{L^{p}} & \leq\left(c\alpha_{\phi}^{2}\nu_{0}^{2}\right)^{l}\left\Vert \mathbf{x}\right\Vert ^{2}\prod_{k=1}^{l}\frac{s_{k}^{2}}{a_{k}}\left(p+m_{k}\right).
\end{align*}
\end{proposition}
\begin{proof}
For any $k\in\left[L\right]$, let $\mathcal{F}_{k}\triangleq\sigma\left(\mathbf{W}^{\left(1\right)},\ldots,\mathbf{W}^{\left(k\right)}\right)$
be the $\sigma$-algebra generated by $\left(\mathbf{W}^{\left(1\right)},\ldots,\mathbf{W}^{\left(k\right)}\right)$
and let $\mathcal{F}_{0}$ be the trivial $\sigma$-algebra. Let $\left\{ \mathbf{z}^{\left(k\right)}\in\mathbb{R}^{m_{k}}\right\} _{k=1}^{L}$
be random vectors, independent of one another and of $\boldsymbol{\theta}$,
such that $\mathbf{z}^{\left(k\right)}\sim\mathcal{N}\left(\mathbf{0}_{m_{k}},\mathbf{I}_{m_{k}}\right)$
for all $k\in\left[L\right]$. 

Let $l\in\left[L\right]$. It holds that
\begin{align*}
\tilde{\boldsymbol{\alpha}}^{\left(l\right)}\left(\mathbf{x}\right)\mid\mathcal{F}_{l-1} & \sim\mathcal{N}\left(\mathbf{0}_{m_{l}},\nu_{0}^{2}\frac{s_{l}^{2}}{a_{l}}\left\Vert \boldsymbol{\alpha}^{\left(l-1\right)}\left(\mathbf{x}\right)\right\Vert ^{2}\mathbf{I}_{m_{l}}\right).
\end{align*}
Therefore, 
\begin{align*}
\tilde{\boldsymbol{\alpha}}^{\left(l\right)}\left(\mathbf{x}\right)\mid\mathcal{F}_{l-1} & \overset{d}{=}\nu_{0}\frac{s_{l}}{\sqrt{a_{l}}}\left\Vert \boldsymbol{\alpha}^{\left(l-1\right)}\left(\mathbf{x}\right)\right\Vert \mathbf{z}^{\left(l\right)},
\end{align*}
where $\overset{d}{=}$ means equal in distribution. Therefore 
\begin{align*}
\left\Vert \tilde{\boldsymbol{\alpha}}^{\left(l\right)}\left(\mathbf{x}\right)\right\Vert ^{2}\mid\mathcal{F}_{l-1} & \overset{d}{=}\nu_{0}^{2}\frac{s_{l}^{2}}{a_{l}}\left\Vert \mathbf{z}^{\left(l\right)}\right\Vert ^{2}\left\Vert \boldsymbol{\alpha}^{\left(l-1\right)}\left(\mathbf{x}\right)\right\Vert ^{2}.
\end{align*}
Together with \cref{ass:activation_bounded_gradients,ass:activation_at_origin,claim:holder_Lp_norm},
\begin{align*}
\left\Vert \left\Vert \boldsymbol{\alpha}^{\left(l\right)}\left(\mathbf{x}\right)\right\Vert ^{2}\right\Vert _{L^{p}\mid\mathcal{F}_{l-1}} & \leq\left\Vert \alpha_{\phi}^{2}\left\Vert \tilde{\boldsymbol{\alpha}}^{\left(l\right)}\left(\mathbf{x}\right)\right\Vert ^{2}\right\Vert _{L^{p}\mid\mathcal{F}_{l-1}}\\
 & \leq\alpha_{\phi}^{2}\nu_{0}^{2}\frac{s_{l}^{2}}{a_{l}}\left\Vert \left\Vert \mathbf{z}^{\left(l\right)}\right\Vert ^{2}\right\Vert _{L^{p}}\left\Vert \left\Vert \boldsymbol{\alpha}^{\left(l-1\right)}\left(\mathbf{x}\right)\right\Vert ^{2}\right\Vert _{L^{p}\mid\mathcal{F}_{l-1}}\\
 & =\alpha_{\phi}^{2}\nu_{0}^{2}\frac{s_{l}^{2}}{a_{l}}\left\Vert \left\Vert \mathbf{z}^{\left(l\right)}\right\Vert ^{2}\right\Vert _{L^{p}}\left\Vert \boldsymbol{\alpha}^{\left(l-1\right)}\left(\mathbf{x}\right)\right\Vert ^{2}.
\end{align*}
By \cref{claim:moments total expectation}, 
\begin{align*}
\left\Vert \left\Vert \boldsymbol{\alpha}^{\left(l\right)}\left(\mathbf{x}\right)\right\Vert ^{2}\right\Vert _{L^{p}} & \leq\left\Vert \left\Vert \left\Vert \boldsymbol{\alpha}^{\left(l\right)}\left(\mathbf{x}\right)\right\Vert ^{2}\right\Vert _{L^{p}\mid\mathcal{F}_{l-1}}\right\Vert _{L^{p}}\\
 & =\left\Vert \alpha_{\phi}^{2}\nu_{0}^{2}\frac{s_{l}^{2}}{a_{l}}\left\Vert \left\Vert \mathbf{z}^{\left(l\right)}\right\Vert ^{2}\right\Vert _{L^{p}}\left\Vert \boldsymbol{\alpha}^{\left(l-1\right)}\left(\mathbf{x}\right)\right\Vert ^{2}\right\Vert _{L^{p}}\\
 & =\alpha_{\phi}^{2}\nu_{0}^{2}\frac{s_{l}^{2}}{a_{l}}\left\Vert \left\Vert \mathbf{z}^{\left(l\right)}\right\Vert ^{2}\right\Vert _{L^{p}}\left\Vert \left\Vert \boldsymbol{\alpha}^{\left(l-1\right)}\left(\mathbf{x}\right)\right\Vert ^{2}\right\Vert _{L^{p}}.
\end{align*}
Continuing recursively, recalling that $\boldsymbol{\alpha}^{\left(0\right)}\left(\mathbf{x}\right)=\mathbf{x}$
and using \cref{prop:bounded moments of gaussian vector,claim:moments total expectation}, we get that there exists
a constant $c_{4}>0$ such that 
\begin{align*}
\left\Vert \left\Vert \boldsymbol{\alpha}^{\left(l\right)}\left(\mathbf{x}\right)\right\Vert ^{2}\right\Vert _{L^{p}} & \leq\left\Vert \mathbf{x}\right\Vert ^{2}\prod_{k=1}^{l}\alpha_{\phi}^{2}\nu_{0}^{2}\frac{s_{k}^{2}}{a_{k}}\left\Vert \left\Vert \mathbf{z}^{\left(k\right)}\right\Vert ^{2}\right\Vert _{L^{p}}\\
 & \leq\left\Vert \mathbf{x}\right\Vert ^{2}\prod_{k=1}^{l}\alpha_{\phi}^{2}\nu_{0}^{2}\frac{s_{k}^{2}}{a_{k}}\left(c_{4}\sqrt{2p}+\sqrt{m_{k}}\right)^{2}\\
 & \leq\left\Vert \mathbf{x}\right\Vert ^{2}\prod_{k=1}^{l}4c_{4}\alpha_{\phi}^{2}\nu_{0}^{2}\frac{s_{k}^{2}}{a_{k}}\left(p+m_{k}\right)\\
 & =\left(4c_{4}\alpha_{\phi}^{2}\nu_{0}^{2}\right)^{l}\left\Vert \mathbf{x}\right\Vert ^{2}\prod_{k=1}^{l}\frac{s_{k}^{2}}{a_{k}}\left(p+m_{k}\right).
\end{align*}
\end{proof}
\begin{proposition}
\label{prop:prop6}There exists a constant $c$ such that for any
$l\in\left[L\right]$ and any $p\geq1$, at initialization, 
\begin{align*}
 & \left\Vert \left\Vert \boldsymbol{\delta}^{\left(l\right)}\left(\mathbf{x}\right)\right\Vert ^{2}\left\Vert \boldsymbol{\alpha}^{\left(l-1\right)}\left(\mathbf{x}\right)\right\Vert ^{2}\right\Vert _{L^{p}}\\
 & \leq\left(c\alpha_{\phi}^{2}\nu_{0}^{2}\right)^{2L}\left\Vert \mathbf{x}\right\Vert ^{2}\left(\prod_{k=1}^{l-1}\frac{s_{k}^{2}}{a_{k}}\left(m_{k}+p\right)\right)\left(\prod_{k=l+1}^{L}\frac{s_{k}^{2}}{a_{k}}\left(m_{k-1}+m_{k}+p\right)\right)\frac{s_{L+1}^{2}}{a_{L+1}}\left(m_{L}+p\right).
\end{align*}
\end{proposition}
\begin{proof}
By \cref{prop:bound_delta}, we have 
\begin{align*}
\left\Vert \boldsymbol{\delta}^{\left(l\right)}\left(\mathbf{x}\right)\right\Vert ^{2}\left\Vert \boldsymbol{\alpha}^{\left(l-1\right)}\left(\mathbf{x}\right)\right\Vert ^{2} & \leq\alpha_{\phi}^{2(L+1-l)}\left(\prod_{k=l+1}^{L}s_{k}^{2}\left\Vert \mathbf{W}^{\left(k\right)}\right\Vert ^{2}\right)s_{L+1}^{2}\left\Vert \mathbf{v}\right\Vert ^{2}\left\Vert \boldsymbol{\alpha}^{\left(l-1\right)}\left(\mathbf{x}\right)\right\Vert ^{2}.
\end{align*}
Note that $\left\Vert \boldsymbol{\alpha}^{\left(l-1\right)}\left(\mathbf{x}\right)\right\Vert ^{2}$
is independent of the parameters of higher layers. Therefore, together
with properties of spectral and Euclidean norms, it holds that
\begin{align*}
 & \left\Vert \left\Vert \boldsymbol{\delta}^{\left(l\right)}\left(\mathbf{x}\right)\right\Vert ^{2}\left\Vert \boldsymbol{\alpha}^{\left(l-1\right)}\left(\mathbf{x}\right)\right\Vert ^{2}\right\Vert _{L^{p}}\\
 & \leq\alpha_{\phi}^{2(L+1-l)}\left(\prod_{k=l+1}^{L}s_{k}^{2}\left\Vert \left\Vert \mathbf{W}^{\left(k\right)}\right\Vert ^{2}\right\Vert _{L^{p}}\right)s_{L+1}^{2}\left\Vert \left\Vert \mathbf{v}\right\Vert ^{2}\right\Vert _{L^{p}}\left\Vert \left\Vert \boldsymbol{\alpha}^{\left(l-1\right)}\left(\mathbf{x}\right)\right\Vert ^{2}\right\Vert _{L^{p}}.
\end{align*}

Using \cref{prop:param_norm_bounds,prop:activation_norm_bound},
we get that there exists a constant $c>0$ such that
\begin{align*}
 & \left\Vert \left\Vert \boldsymbol{\delta}^{\left(l\right)}\left(\mathbf{x}\right)\right\Vert ^{2}\left\Vert \boldsymbol{\alpha}^{\left(l-1\right)}\left(\mathbf{x}\right)\right\Vert ^{2}\right\Vert _{L^{p}}\\
 & \leq\alpha_{\phi}^{2(L+1-l)}\left(\prod_{k=l+1}^{L}c\nu_{0}^{2}\frac{s_{k}^{2}}{a_{k}}\left(m_{k-1}+m_{k}+p\right)\right)\\
 & \quad\cdot c\frac{s_{L+1}^{2}}{a_{L+1}}\left(p+m_{L}\right)\left\Vert \mathbf{x}\right\Vert ^{2}\left(c\alpha_{\phi}^{2}\nu_{0}^{2}\right)^{l-1}\prod_{k=1}^{l-1}\frac{s_{k}^{2}}{a_{k}}\left(p+m_{k}\right)\\
 & \leq\left(c\alpha_{\phi}^{2}\nu_{0}^{2}\right)^{2L}\left\Vert \mathbf{x}\right\Vert ^{2}\frac{s_{L+1}^{2}}{a_{L+1}}\left(p+m_{L}\right)\left(\prod_{k=l+1}^{L}\frac{s_{k}^{2}}{a_{k}}\left(m_{k-1}+m_{k}+p\right)\right)\prod_{k=1}^{l-1}\frac{s_{k}^{2}}{a_{k}}\left(p+m_{k}\right)\\
 & \leq\left(c\alpha_{\phi}^{2}\nu_{0}^{2}\right)^{2L}\left\Vert \mathbf{x}\right\Vert ^{2}\left(\prod_{k=1}^{l-1}\frac{s_{k}^{2}}{a_{k}}\left(m_{k}+p\right)\right)\left(\prod_{k=l+1}^{L}\frac{s_{k}^{2}}{a_{k}}\left(m_{k-1}+m_{k}+p\right)\right)\frac{s_{L+1}^{2}}{a_{L+1}}\left(m_{L}+p\right).
\end{align*}
\end{proof}

\subsection{Introduction and goal}

Recalling \cref{eq:the_four_terms} from \cref{sec:convergence_proof_outline}, we have that 
\begin{equation}
\begin{aligned}\frac{\mathrm{d}}{\mathrm{d}t}\mathbb{E}\left[\mathcal{L}\left(\boldsymbol{\theta}_{t}\right)\right] & =\frac{1}{n}\biggl(-n^{-1}\mathbb{E}\left[\mathbf{r}\left(\boldsymbol{\theta}_{t}\right)^{\top}\mathbf{K}_{\mathrm{NTK}}\left(\boldsymbol{\theta}_{t}\right)\mathbf{r}\left(\boldsymbol{\theta}_{t}\right)\right]\\
 & \qquad\quad+\beta^{-1}\mathbb{E}\left[\mathrm{Tr}\left(\nabla f\left(\boldsymbol{\theta}_{t}\right)^{\top}\nabla f\left(\boldsymbol{\theta}_{t}\right)\right)\right]\\
 & \qquad\quad+\beta^{-1}\mathbb{E}\left[\sum_{i=1}^{n}r_{i}\left(\boldsymbol{\theta}_{t}\right)\Delta f_{i}\left(\boldsymbol{\theta}_{t}\right)\right]\\
 & \qquad\quad-\beta^{-1}\mathbb{E}\left[\mathbf{r}\left(\boldsymbol{\theta}_{t}\right)^{\top}\nabla f\left(\boldsymbol{\theta}_{t}\right)\boldsymbol{\Lambda}\boldsymbol{\theta}_{t}\right]\biggr)\,.
\end{aligned}
\label{eq:the_four_terms_app_deep}
\end{equation}

We wish to bound every term in the right-hand side of this equation
using \cref{cor:second_law_for_langevin}.
In order to use \cref{cor:second_law_for_langevin},
we must bound $\left(\mathbb{E}_{p_{0}}\left[e^{2\left(\alpha-1\right)\beta\mathcal{L}\left(\boldsymbol{\theta}_{0}\right)}\right]\right)^{\frac{1}{2\alpha}}$.
Under the bounded activation assumption (\cref{ass:activation_bounded}),
we show in \cref{prop:mgf_of_init_loss}
that the subgaussianity of the network outputs implies the following
bound on the exponential moment: 
\[
\mathbb{E}_{p_{0}}\left[e^{2\left(\alpha-1\right)\beta\mathcal{L}\left(\boldsymbol{\theta}_{0}\right)}\right]\leq2e^{2\left(\alpha-1\right)\beta R_{y}^{2}}\,,
\]
provided that $2\left(\alpha-1\right)\beta\leq L_{\psi}^{-1}$, where
$L_{\psi}\triangleq cM_{\phi}^{2}\frac{s_{L+1}^{2}m_{L}}{a_{L+1}}=cM_{\phi}^{2}\kappa^{2}$.
I.e., the moment is only guaranteed to exist for $\alpha\leq1+\frac{1}{2cM_{\phi}^{2}\kappa^{2}\beta}$.
However, we also require that $2e^{2\left(\alpha-1\right)\beta R_{y}^{2}}=O\left(1\right)$,
thus limiting ourselves only to $\alpha\leq1+\frac{1}{2cM_{\phi}^{2}\beta}$.
Taken together, by using $\alpha=1+\frac{1}{2cM_{\phi}^{2}\max\left(\kappa^{2},1\right)\beta}$,
we ensure $\left(\mathbb{E}_{p_{0}}\left[e^{2\left(\alpha-1\right)\beta\mathcal{L}\left(\boldsymbol{\theta}_{0}\right)}\right]\right)^{\frac{1}{2\alpha}}$
exists and is $O\left(1\right)$. Denoting $\bar{\kappa}\triangleq\max\left(\kappa,1\right)$,
we get that
\[
\alpha=1+\frac{1}{2cM_{\phi}^{2}\bar{\kappa}^{2}\beta}\,,
\]

and
\begin{align*}
\left(\mathbb{E}_{p_{0}}\left[e^{2\left(\alpha-1\right)\beta\mathcal{L}\left(\boldsymbol{\theta}_{0}\right)}\right]\right)^{\frac{1}{2\alpha}} & \leq\left(2e^{2\left(\alpha-1\right)\beta R_{y}^{2}}\right)^{\frac{1}{2\alpha}}\leq\sqrt{2}e^{\frac{\beta R_{y}^{2}}{2cM_{\phi}^{2}\bar{\kappa}^{2}\beta}}=\sqrt{2}e^{\frac{R_{y}^{2}}{2cM_{\phi}^{2}\bar{\kappa}^{2}}}=O\left(1\right)\,,
\end{align*}

since $\bar{\kappa}\geq1$. We denote from here on $\eta\triangleq\sqrt{2}e^{\frac{R_{y}^{2}}{2cM_{\phi}^{2}\bar{\kappa}^{2}}}=O\left(1\right)$.

Given this choice of $\alpha$, we have 
\begin{align*}
\frac{\alpha}{\alpha-1}=1+2cM_{\phi}^{2}\bar{\kappa}^{2}\beta=c_{\beta}\bar{\kappa}^{2}\beta\,,
\end{align*}
where we defined $c_{\beta}\triangleq2cM_{\phi}^{2}+\frac{1}{\bar{\kappa}^{2}\beta}$.
This presentation is more convenient when $\bar{\kappa}^{2}\beta\gg1$,
which is a demand we later show is indeed necessary, and then $c_{\beta}=O\left(1\right)$.
Under this, \cref{cor:second_law_for_langevin}
becomes for any \emph{non-negative} measurable function $\varphi:\mathbb{R}^{N_{p}}\to\mathbb{R}$,
\begin{equation}
\mathbb{E}_{p_{t}}\left[\varphi\left(\boldsymbol{\theta}_{t}\right)\right]\leq\eta\left\Vert \varphi\left(\boldsymbol{\theta}_{0}\right)\right\Vert _{L^{c_{\beta}\bar{\kappa}^{2}\beta}}.\label{eq:second_law_with_eta_deep}
\end{equation}

We proceed to bound each of the four terms in \cref{eq:the_four_terms} using the change-of-measure
inequality \cref{eq:second_law_with_eta_deep}.

\subsection{First term}

Observe that
\begin{align*}
-n^{-1}\mathbb{E}\left[\mathbf{r}\left(\boldsymbol{\theta}_{t}\right)^{\top}\mathbf{K}_{\mathrm{NTK}}\left(\boldsymbol{\theta}_{t}\right)\mathbf{r}\left(\boldsymbol{\theta}_{t}\right)\right] & \leq-n^{-1}\mathbb{E}\left[\left\Vert \mathbf{r}\left(\boldsymbol{\theta}_{t}\right)\right\Vert ^{2}\lambda_{\min}\left(\mathbf{K}_{\mathrm{NTK}}\left(\boldsymbol{\theta}_{t}\right)\right)\right]\\
 & =-n^{-1}\mathbb{E}\left[2n\mathcal{L}\left(\boldsymbol{\theta}_{t}\right)\lambda_{\min}\left(\mathbf{K}_{\mathrm{NTK}}\left(\boldsymbol{\theta}_{t}\right)\right)\right]\\
 & =-2\mathbb{E}\left[\mathcal{L}\left(\boldsymbol{\theta}_{t}\right)\lambda_{\min}\left(\mathbf{K}_{\mathrm{NTK}}\left(\boldsymbol{\theta}_{t}\right)\right)\right].
\end{align*}
From \cref{thm:lower_bound_ntk_deep} we have
that if $m=\Omega\left(\frac{n}{\lambda_{1}}\ln\frac{n}{\lambda_{1}}\ln n\right)$
then $\lambda_{\min}\left(\mathbf{K}_{\mathrm{NTK}}\left(\boldsymbol{\theta}_{0}\right)\right)\geq c_{0}\lambda_{1}s_{L+1}^{2}m$
with probability at least $1-\exp\left(-\tilde{c_{1}}m\left(\frac{n}{\lambda_{1}}\ln\frac{n}{\lambda_{1}}\right)^{-1}\right)$.
For the remainder of this section, denote $\lambda_{\star}=c_{0}\lambda_{1}s_{L+1}^{2}m$
and $\delta=\exp\left(-\tilde{c_{1}}m\left(\frac{n}{\lambda_{1}}\ln\frac{n}{\lambda_{1}}\right)^{-1}\right)$.

Observe that 
\begin{align*}
 & -n^{-1}\mathbb{E}\left[\mathbf{r}\left(\boldsymbol{\theta}_{t}\right)^{\top}\mathbf{K}_{\mathrm{NTK}}\left(\boldsymbol{\theta}_{t}\right)\mathbf{r}\left(\boldsymbol{\theta}_{t}\right)\right]\\
 & \leq-2\mathbb{E}\left[\mathcal{L}\left(\boldsymbol{\theta}_{t}\right)\lambda_{\min}\left(\mathbf{K}_{\mathrm{NTK}}\left(\boldsymbol{\theta}_{t}\right)\right)\right]\\
 & =-2\lambda_{\star}\mathbb{E}\left[\mathcal{L}\left(\boldsymbol{\theta}_{t}\right)\right]+2\mathbb{E}\left[\mathcal{L}\left(\boldsymbol{\theta}_{t}\right)\left(\lambda_{\star}-\lambda_{\min}\left(\mathbf{K}_{\mathrm{NTK}}\left(\boldsymbol{\theta}_{t}\right)\right)\right)\right]\\
 & \leq-2\lambda_{\star}\mathbb{E}\left[\mathcal{L}\left(\boldsymbol{\theta}_{t}\right)\right]+2\mathbb{E}\left[\mathcal{L}\left(\boldsymbol{\theta}_{t}\right)\max\left(0,\lambda_{\star}-\lambda_{\min}\left(\mathbf{K}_{\mathrm{NTK}}\left(\boldsymbol{\theta}_{t}\right)\right)\right)\right].
\end{align*}
We upper bound $\mathbb{E}\left[\mathcal{L}\left(\boldsymbol{\theta}_{t}\right)\max\left(0,\left(\lambda_{\star}-\lambda_{\min}\left(\mathbf{K}_{\mathrm{NTK}}\left(\boldsymbol{\theta}_{t}\right)\right)\right)\right)\right]$
using our second law change of measure theorem (\cref{eq:second_law_with_eta_deep}) and \cref{claim:holder_Lp_norm}:
\begin{align*}
 & \mathbb{E}\left[\mathcal{L}\left(\boldsymbol{\theta}_{t}\right)\max\left(0,\left(\lambda_{\star}-\lambda_{\min}\left(\mathbf{K}_{\mathrm{NTK}}\left(\boldsymbol{\theta}_{t}\right)\right)\right)\right)\right]\\
 & \leq\eta\left\Vert \mathcal{L}\left(\boldsymbol{\theta}_{0}\right)\max\left(0,\left(\lambda_{\star}-\lambda_{\min}\left(\mathbf{K}_{\mathrm{NTK}}\left(\boldsymbol{\theta}_{0}\right)\right)\right)\right)\right\Vert _{L^{c_{\beta}\bar{\kappa}^{2}\beta}}\\
 & \leq\eta\left\Vert \mathcal{L}\left(\boldsymbol{\theta}_{0}\right)\right\Vert _{L^{2c_{\beta}\bar{\kappa}^{2}\beta}}\left\Vert \max\left(0,\lambda_{\star}-\lambda_{\min}\left(\mathbf{K}_{\mathrm{NTK}}\left(\boldsymbol{\theta}_{0}\right)\right)\right)\right\Vert _{L^{2c_{\beta}\bar{\kappa}^{2}\beta}}.
\end{align*}
From \cref{prop:mgf_of_init_loss},
\begin{align*}
\left\Vert \mathcal{L}\left(\boldsymbol{\theta}_{0}\right)\right\Vert _{L^{2c_{\beta}\bar{\kappa}^{2}\beta}} & \leq\left\Vert R_{y}^{2}+\frac{1}{n}\sum_{i=1}^{n}f_{i}^{2}\left(\boldsymbol{\theta}_{0}\right)\right\Vert _{L^{2c_{\beta}\bar{\kappa}^{2}\beta}}\\
 & \overset{(i)}{\leq}\left\Vert R_{y}^{2}\right\Vert _{L^{2c_{\beta}\bar{\kappa}^{2}\beta}}+\left\Vert \frac{1}{n}\sum_{i=1}^{n}f_{i}^{2}\left(\boldsymbol{\theta}_{0}\right)\right\Vert _{L^{2c_{\beta}\bar{\kappa}^{2}\beta}}\\
 & \overset{(ii)}{\leq}R_{y}^{2}+C\cdot2c_{\beta}\bar{\kappa}^{2}\beta\left\Vert \frac{1}{n}\sum_{i=1}^{n}f_{i}^{2}\left(\boldsymbol{\theta}_{0}\right)\right\Vert _{\psi_{1}}\\
 & \leq R_{y}^{2}+C\cdot2c_{\beta}\bar{\kappa}^{4}\beta,
\end{align*}
where in $(i)$ we use Minkowski's inequality and in $(ii)$ we use
\citet[Proposition 2.8.1]{vershynin2018high}. For the
other term, by splitting according to the high-probability NTK event,
we get 
\begin{align*}
 & \left\Vert \max\left(0,\lambda_{\star}-\lambda_{\min}\left(\mathbf{K}_{\mathrm{NTK}}\left(\boldsymbol{\theta}_{0}\right)\right)\right)\right\Vert _{L^{2c_{\beta}\bar{\kappa}^{2}\beta}}\\
 & =\mathbb{E}\left[\left(\max\left(0,\lambda_{\star}-\lambda_{\min}\left(\mathbf{K}_{\mathrm{NTK}}\left(\boldsymbol{\theta}_{0}\right)\right)\right)\right)^{2c_{\beta}\bar{\kappa}^{2}\beta}\right]^{\nicefrac{1}{2c_{\beta}\bar{\kappa}^{2}\beta}}\\
 & \leq\left(0\cdot\mathbb{P}\left(\lambda_{\star}\leq\lambda_{\min}\left(\mathbf{K}_{\mathrm{NTK}}\left(\boldsymbol{\theta}_{0}\right)\right)\right)+\lambda_{\star}^{2c_{\beta}\bar{\kappa}^{2}\beta}\mathbb{P}\left(\lambda_{\star}>\lambda_{\min}\left(\mathbf{K}_{\mathrm{NTK}}\left(\boldsymbol{\theta}_{0}\right)\right)\right)\right)^{\nicefrac{1}{2c_{\beta}\bar{\kappa}^{2}\beta}}\\
 & \leq\lambda_{\star}\delta^{\nicefrac{1}{2c_{\beta}\bar{\kappa}^{2}\beta}}=c_{0}\lambda_{1}s_{L+1}^{2}m\exp\left(-\tilde{c_{1}}m\left(2c_{\beta}\bar{\kappa}^{2}\beta\frac{n}{\lambda_{1}}\ln\frac{n}{\lambda_{1}}\right)^{-1}\right),
\end{align*}
where in the first transition we used the fact that the empirical
NTK matrix is PSD, hence $\lambda_{\min}\left(\mathbf{K}_{\mathrm{NTK}}\left(\boldsymbol{\theta}_{0}\right)\right)\geq0$.
To conclude:
\begin{align*}
 & -n^{-1}\mathbb{E}\left[\mathbf{r}\left(\boldsymbol{\theta}_{t}\right)^{\top}\mathbf{K}_{\mathrm{NTK}}\left(\boldsymbol{\theta}_{t}\right)\mathbf{r}\left(\boldsymbol{\theta}_{t}\right)\right]\\
 & \leq-2c_{0}\lambda_{1}s_{L+1}^{2}m\mathbb{E}\left[\mathcal{L}\left(\boldsymbol{\theta}_{t}\right)\right]\\
 & \ +2\eta c_{0}\lambda_{1}s_{L+1}^{2}m\left(R_{y}^{2}+C\cdot2c_{\beta}\bar{\kappa}^{4}\beta\right)\exp\left(-\tilde{c_{1}}m\left(2c_{\beta}\bar{\kappa}^{2}\beta\frac{n}{\lambda_{1}}\ln\frac{n}{\lambda_{1}}\right)^{-1}\right)\\
 & \lesssim-\lambda_{1}s_{L+1}^{2}m\mathbb{E}\left[\mathcal{L}\left(\boldsymbol{\theta}_{t}\right)\right]+\eta\lambda_{1}s_{L+1}^{2}m\left(1+c_{\beta}\bar{\kappa}^{4}\beta\right)\exp\left(-\tilde{c_{1}}m\left(2c_{\beta}\bar{\kappa}^{2}\beta\frac{n}{\lambda_{1}}\ln\frac{n}{\lambda_{1}}\right)^{-1}\right).
\end{align*}

\subsection{Second term }

Observe that 
\begin{align*}
\beta^{-1}\mathbb{E}\left[\mathrm{Tr}\left(\nabla f\left(\boldsymbol{\theta}_{t}\right)^{\top}\nabla f\left(\boldsymbol{\theta}_{t}\right)\right)\right] & =\beta^{-1}\mathbb{E}\left[\sum_{j=1}^{N_{p}}\sum_{i=1}^{n}\left(\nabla f_{i}\left(\boldsymbol{\theta}_{t}\right)\right)_{j}^{2}\right]=\beta^{-1}\sum_{i=1}^{n}\mathbb{E}\left[\left\Vert \nabla f_{i}\left(\boldsymbol{\theta}_{t}\right)\right\Vert ^{2}\right].
\end{align*}
Let $i\in\left[n\right]$. We now analyze $\nabla f_{i}\left(\boldsymbol{\theta}\right)$
at initialization. We have
\begin{align*}
\left\Vert \nabla f_{i}\left(\boldsymbol{\theta}\right)\right\Vert ^{2} & =\left\Vert \frac{\partial f_{i}\left(\boldsymbol{\theta}\right)}{\partial\mathbf{v}}\right\Vert ^{2}+\sum_{l=1}^{L}\left\Vert \frac{\partial f_{i}\left(\boldsymbol{\theta}\right)}{\partial\mathbf{W}^{\left(l\right)}}\right\Vert _{\mathrm{F}}^{2}\\
 & \overset{(i)}{\leq}s_{L+1}^{2}\left\Vert \boldsymbol{\alpha}^{\left(L\right)}\left(\mathbf{x}_{i}\right)\right\Vert ^{2}+\sum_{l=1}^{L}s_{l}^{2}\left\Vert \boldsymbol{\delta}^{\left(l\right)}\left(\mathbf{x}_{i}\right)\left(\boldsymbol{\alpha}^{\left(l-1\right)}\left(\mathbf{x}_{i}\right)\right)^{\top}\right\Vert _{\mathrm{F}}^{2}\\
 & \overset{(ii)}{=}s_{L+1}^{2}\left\Vert \boldsymbol{\alpha}^{\left(L\right)}\left(\mathbf{x}_{i}\right)\right\Vert ^{2}+\sum_{l=1}^{L}s_{l}^{2}\left\Vert \boldsymbol{\delta}^{\left(l\right)}\left(\mathbf{x}_{i}\right)\right\Vert ^{2}\left\Vert \boldsymbol{\alpha}^{\left(l-1\right)}\left(\mathbf{x}_{i}\right)\right\Vert ^{2},
\end{align*}
where in $(i)$ we use \cref{eq:nabla_f_Wl},
and in $(ii)$ we use the identity $\left\Vert \mathbf{u}\mathbf{v}^{\top}\right\Vert _{\mathrm{F}}=\left\Vert \mathbf{u}\right\Vert \left\Vert \mathbf{v}\right\Vert $
which holds for any pair of column vectors.

By \cref{prop:activation_norm_bound},
there exists a constant $c_{4}>0$ such that for any $p\geq1$, at
initialization,
\begin{align*}
\left\Vert \left\Vert \boldsymbol{\alpha}^{\left(L\right)}\left(\mathbf{x}_{i}\right)\right\Vert ^{2}\right\Vert _{L^{p}} & \leq\left(c_{4}\alpha_{\phi}^{2}\nu_{0}^{2}\right)^{L}\left\Vert \mathbf{x}_{i}\right\Vert ^{2}\prod_{k=1}^{L}\frac{s_{k}^{2}}{a_{k}}\left(m_{k}+p\right).
\end{align*}

By \cref{prop:prop6}, for every $l\in\left[L\right]$
there exists a constant $c_{5}>0$ such that, at initialization,
\begin{align*}
 & \left\Vert \left\Vert \boldsymbol{\delta}^{\left(l\right)}\left(\mathbf{x}_{i}\right)\right\Vert ^{2}\left\Vert \boldsymbol{\alpha}^{\left(l-1\right)}\left(\mathbf{x}_{i}\right)\right\Vert ^{2}\right\Vert _{L^{p}}\\
 & \leq\left(c_{5}\alpha_{\phi}^{2}\nu_{0}^{2}\right)^{2L}\left\Vert \mathbf{x}_{i}\right\Vert ^{2}\left(\prod_{k=1}^{l-1}\frac{s_{k}^{2}}{a_{k}}\left(m_{k}+p\right)\right)\left(\prod_{k=l+1}^{L}\frac{s_{k}^{2}}{a_{k}}\left(m_{k-1}+m_{k}+p\right)\right)\frac{s_{L+1}^{2}}{a_{L+1}}\left(m_{L}+p\right).
\end{align*}
Thus, using Minkowski's inequality we get that at initialization,
\begin{align*}
 & \left\Vert \left\Vert \nabla f_{i}\left(\boldsymbol{\theta}\right)\right\Vert ^{2}\right\Vert _{L^{p}}\\
 & \leq s_{L+1}^{2}\left\Vert \left\Vert \boldsymbol{\alpha}^{\left(L\right)}\left(\mathbf{x}_{i}\right)\right\Vert ^{2}\right\Vert _{L^{p}}+\sum_{l=1}^{L}s_{l}^{2}\left\Vert \left\Vert \boldsymbol{\delta}^{\left(l\right)}\left(\mathbf{x}_{i}\right)\right\Vert ^{2}\left\Vert \boldsymbol{\alpha}^{\left(l-1\right)}\left(\mathbf{x}_{i}\right)\right\Vert ^{2}\right\Vert _{L^{p}}\\
 & \leq s_{L+1}^{2}\left(c_{4}\alpha_{\phi}^{2}\nu_{0}^{2}\right)^{L}\left\Vert \mathbf{x}_{i}\right\Vert ^{2}\prod_{k=1}^{L}\frac{s_{k}^{2}}{a_{k}}\left(m_{k}+p\right)\\
 & \ +\sum_{l=1}^{L}s_{l}^{2}\left(c_{5}\alpha_{\phi}^{2}\nu_{0}^{2}\right)^{2L}\left\Vert \mathbf{x}_{i}\right\Vert ^{2}\\
 & \quad\quad\quad\quad\cdot\left(\prod_{k=1}^{l-1}\frac{s_{k}^{2}}{a_{k}}\left(m_{k}+p\right)\right)\left(\prod_{k=l+1}^{L}\frac{s_{k}^{2}}{a_{k}}\left(m_{k-1}+m_{k}+p\right)\right)\frac{s_{L+1}^{2}}{a_{L+1}}\left(m_{L}+p\right).
\end{align*}
Denoting $C=\left(c_{4}+c_{5}\right)\alpha_{\phi}^{2}\nu_{0}^{2}$,
using \cref{ass:normalized_data} and noticing the
shared factor of $\left(m_{L}+p\right)$, we get 
\begin{align*}
\left\Vert \left\Vert \nabla f_{i}\left(\boldsymbol{\theta}_{0}\right)\right\Vert ^{2}\right\Vert _{L^{p}}\leq & c_{\phi,\sigma_{0}}C^{2L}ds_{L+1}^{2}\left(m_{L}+p\right)\\
\cdot & \left(\frac{s_{L}^{2}}{a_{L}}\prod_{k=1}^{L-1}\frac{s_{k}^{2}}{a_{k}}\left(m_{k}+p\right)\right.\\
 & \left.\quad+\frac{1}{a_{L+1}}\sum_{l=1}^{L}\left(\prod_{k=1}^{l-1}\frac{s_{k}^{2}}{a_{k}}\left(m_{k}+p\right)\right)s_{l}^{2}\left(\prod_{k=l+1}^{L}\frac{s_{k}^{2}}{a_{k}}\left(m_{k-1}+m_{k}+p\right)\right)\right).
\end{align*}
Using this together with \cref{eq:second_law_with_eta_deep}, we eventually
get
\begin{align*}
\cdots\leq & \beta^{-1}\sum_{i=1}^{n}\mathbb{E}\left[\left\Vert \nabla f_{i}\left(\boldsymbol{\theta}_{t}\right)\right\Vert ^{2}\right]\\
\leq & \beta^{-1}\sum_{i=1}^{n}\eta\left\Vert \left\Vert \nabla f_{i}\left(\boldsymbol{\theta}_{0}\right)\right\Vert ^{2}\right\Vert _{L^{c_{\beta}\bar{\kappa}^{2}\beta}}\\
\leq & c_{\phi,\sigma_{0}}\left(\left(c_{4}+c_{5}\right)\alpha_{\phi}^{2}\nu_{0}^{2}\right)^{2L}\eta\beta^{-1}nds_{L+1}^{2}\left(m_{L}+c_{\beta}\bar{\kappa}^{2}\beta\right)\\
 & \cdot\left(\frac{s_{L}^{2}}{a_{L}}\prod_{k=1}^{L-1}\frac{s_{k}^{2}}{a_{k}}\left(m_{k}+c_{\beta}\bar{\kappa}^{2}\beta\right)\right.\\
 & \left.\quad+\frac{1}{a_{L+1}}\sum_{l=1}^{L}\left(\prod_{k=1}^{l-1}\frac{s_{k}^{2}}{a_{k}}\left(m_{k}+c_{\beta}\bar{\kappa}^{2}\beta\right)\right)s_{l}^{2}\left(\prod_{k=l+1}^{L}\frac{s_{k}^{2}}{a_{k}}\left(m_{k-1}+m_{k}+c_{\beta}\bar{\kappa}^{2}\beta\right)\right)\right).
\end{align*}

\subsection{Third term}

\textbf{A remark on the second derivative of the activations.} In
this section, when we work toward bounding the $L^{p}$-norm of quantities
at initialization, we encounter expressions of the form $\phi''\left(\tilde{\boldsymbol{\alpha}}^{\left(l\right)}\left(\mathbf{x}\right)\right)$.
In \cref{ass:activation_bounded_gradients} we add a footnote
claiming the assumption can be relaxed to once continuously differentiable
activations with a bounded and $\beta_{\phi}$-Lipschitz gradient.
This calls for further justification here. By Rademacher's theorem,
whenever $\phi''$ exists we have $\left|\phi''\right|\leq\beta_{\phi}$
a.e. Thus, for bounding $L^{p}$-norms, we may replace $\phi''$ by
a measurable representative that agrees with the classical second
derivative wherever it exists and is defined as $0$ on the exceptional
null set. This representative satisfies $\left|\phi''\right|\leq\beta_{\phi}$
everywhere. We therefore use $\left|\phi''\right|\leq\beta_{\phi}$
freely throughout this section.

Observe that 
\begin{align*}
\beta^{-1}\mathbb{E}\left[\sum_{i=1}^{n}r_{i}\left(\boldsymbol{\theta}_{t}\right)\Delta f_{i}\left(\boldsymbol{\theta}_{t}\right)\right] & =\beta^{-1}\sum_{i=1}^{n}\mathbb{E}\left[r_{i}\left(\boldsymbol{\theta}_{t}\right)\Delta f_{i}\left(\boldsymbol{\theta}_{t}\right)\right]\\
 & \leq\beta^{-1}\sum_{i=1}^{n}\mathbb{E}\left[\left|r_{i}\left(\boldsymbol{\theta}_{t}\right)\Delta f_{i}\left(\boldsymbol{\theta}_{t}\right)\right|\right]\\
 & \leq\beta^{-1}\sum_{i=1}^{n}\sqrt{\mathbb{E}\left[\left|r_{i}\left(\boldsymbol{\theta}_{t}\right)\right|^{2}\right]\mathbb{E}\left[\left|\Delta f_{i}\left(\boldsymbol{\theta}_{t}\right)\right|^{2}\right]}\\
 & \leq\beta^{-1}\sum_{i=1}^{n}\sqrt{2\mathbb{E}\left[\ell_{i}\left(\boldsymbol{\theta}_{t}\right)\right]\mathbb{E}\left[\left|\Delta f_{i}\left(\boldsymbol{\theta}_{t}\right)\right|^{2}\right]},
\end{align*}
where $\ell_{i}\left(\boldsymbol{\theta}\right)=\frac{1}{2}\left|r_{i}\left(\boldsymbol{\theta}\right)\right|^{2}$.
If we find some deterministic $Q>0$ such that 
\begin{align*}
\mathbb{E}\left[\left|\Delta f_{i}\left(\boldsymbol{\theta}_{t}\right)\right|^{2}\right] & \leq\eta Q^{2},
\end{align*}
then we get 
\begin{align*}
\beta^{-1}\mathbb{E}\left[\sum_{i=1}^{n}r_{i}\left(\boldsymbol{\theta}_{t}\right)\Delta f_{i}\left(\boldsymbol{\theta}_{t}\right)\right] & \leq\beta^{-1}\sqrt{\eta}Q\sum_{i=1}^{n}\sqrt{2\mathbb{E}\left[\ell_{i}\left(\boldsymbol{\theta}_{t}\right)\right]}\\
 & =\beta^{-1}n\sqrt{\eta}Q\frac{1}{n}\sum_{i=1}^{n}\sqrt{2\mathbb{E}\left[\ell_{i}\left(\boldsymbol{\theta}_{t}\right)\right]}\\
 & \leq\beta^{-1}n\sqrt{\eta}Q\sqrt{\frac{2}{n}\sum_{i=1}^{n}\mathbb{E}\left[\ell_{i}\left(\boldsymbol{\theta}_{t}\right)\right]}=\beta^{-1}n\sqrt{\eta}\cdot2Q\sqrt{\mathbb{E}\left[\mathcal{L}\left(\boldsymbol{\theta}_{t}\right)\right]}.
\end{align*}

We move on to obtain such $Q$. The main task here is bounding the
squared Laplacian at initialization. Let $l\in\left[L\right]$. Recall
that 
\begin{align*}
\frac{\partial f_{i}\left(\boldsymbol{\theta}\right)}{\partial\mathbf{W}^{\left(l\right)}} & =s_{l}\boldsymbol{\delta}^{\left(l\right)}\left(\mathbf{x}_{i}\right)\left(\boldsymbol{\alpha}^{\left(l-1\right)}\left(\mathbf{x}_{i}\right)\right)^{\top}.
\end{align*}
Therefore, for any $j\in\left[m_{l}\right]$ and $k\in\left[m_{l-1}\right]$,
\begin{align*}
\frac{\partial f_{i}\left(\boldsymbol{\theta}\right)}{\partial W_{jk}^{\left(l\right)}} & =s_{l}\alpha_{k}^{\left(l-1\right)}\left(\mathbf{x}_{i}\right)\delta_{j}^{\left(l\right)}\left(\mathbf{x}_{i}\right).
\end{align*}
Observing that $\left(\boldsymbol{\alpha}^{\left(l-1\right)}\left(\mathbf{x}_{i}\right)\right)^{\top}$
is not a function of $\mathbf{W}^{\left(l\right)}$, we get
\begin{align*}
\frac{\partial^{2}f_{i}\left(\boldsymbol{\theta}\right)}{\partial\left(W_{jk}^{\left(l\right)}\right)^{2}} & =s_{l}\alpha_{k}^{\left(l-1\right)}\left(\mathbf{x}_{i}\right)\frac{\partial}{\partial W_{jk}^{\left(l\right)}}\delta_{j}^{\left(l\right)}\left(\mathbf{x}_{i}\right)\\
 & =s_{l}\alpha_{k}^{\left(l-1\right)}\left(\mathbf{x}_{i}\right)\frac{\partial\delta_{j}^{\left(l\right)}\left(\mathbf{x}_{i}\right)}{\partial\tilde{\alpha}_{j}^{\left(l\right)}\left(\mathbf{x}_{i}\right)}\frac{\partial\tilde{\alpha}_{j}^{\left(l\right)}\left(\mathbf{x}_{i}\right)}{\partial W_{jk}^{\left(l\right)}}=s_{l}^{2}\left(\alpha_{k}^{\left(l-1\right)}\left(\mathbf{x}_{i}\right)\right)^{2}\frac{\partial\delta_{j}^{\left(l\right)}\left(\mathbf{x}_{i}\right)}{\partial\tilde{\alpha}_{j}^{\left(l\right)}\left(\mathbf{x}_{i}\right)}.
\end{align*}
Therefore, the contribution of layer $l$ to the Laplacian is
\begin{align*}
\left|\Delta_{\mathbf{W}^{\left(l\right)}}f_{i}\left(\boldsymbol{\theta}\right)\right| & =\left|\sum_{j=1}^{m_{l}}\sum_{k=1}^{m_{l-1}}\frac{\partial^{2}f_{i}\left(\boldsymbol{\theta}\right)}{\partial\left(W_{jk}^{\left(l\right)}\right)^{2}}\right|\\
 & =\left|s_{l}^{2}\sum_{j=1}^{m_{l}}\frac{\partial\delta_{j}^{\left(l\right)}\left(\mathbf{x}_{i}\right)}{\partial\tilde{\alpha}_{j}^{\left(l\right)}\left(\mathbf{x}_{i}\right)}\sum_{k=1}^{m_{l-1}}\left(\alpha_{k}^{\left(l-1\right)}\left(\mathbf{x}_{i}\right)\right)^{2}\right|\\
 & =s_{l}^{2}\left\Vert \boldsymbol{\alpha}^{\left(l-1\right)}\left(\mathbf{x}_{i}\right)\right\Vert ^{2}\left|\mathrm{Tr}\left(\frac{\partial\boldsymbol{\delta}^{\left(l\right)}\left(\mathbf{x}_{i}\right)}{\partial\tilde{\boldsymbol{\alpha}}^{\left(l\right)}\left(\mathbf{x}_{i}\right)}\right)\right|\\
 & =s_{l}^{2}\left\Vert \boldsymbol{\alpha}^{\left(l-1\right)}\left(\mathbf{x}_{i}\right)\right\Vert ^{2}\left|\mathrm{Tr}\left(\frac{\partial^{2}f_{i}\left(\boldsymbol{\theta}\right)}{\partial\left(\tilde{\boldsymbol{\alpha}}^{\left(l\right)}\left(\mathbf{x}_{i}\right)\right)^{2}}\right)\right|.
\end{align*}
Let $p\geq1$. By H\"older's inequality it holds that 
\begin{align*}
\left\Vert \Delta_{\mathbf{W}^{\left(l\right)}}f_{i}\left(\boldsymbol{\theta}\right)\right\Vert _{L^{p}}\leq & s_{l}^{2}\left\Vert \left\Vert \boldsymbol{\alpha}^{\left(l-1\right)}\left(\mathbf{x}_{i}\right)\right\Vert ^{2}\left|\mathrm{Tr}\left(\frac{\partial^{2}f_{i}\left(\boldsymbol{\theta}\right)}{\partial\left(\tilde{\boldsymbol{\alpha}}^{\left(l\right)}\left(\mathbf{x}_{i}\right)\right)^{2}}\right)\right|\right\Vert _{L^{p}}\\
 & \leq s_{l}^{2}\left\Vert \left\Vert \boldsymbol{\alpha}^{\left(l-1\right)}\left(\mathbf{x}_{i}\right)\right\Vert ^{2}\right\Vert _{L^{2p}}\left\Vert \left|\mathrm{Tr}\left(\frac{\partial^{2}f_{i}\left(\boldsymbol{\theta}\right)}{\partial\left(\tilde{\boldsymbol{\alpha}}^{\left(l\right)}\left(\mathbf{x}_{i}\right)\right)^{2}}\right)\right|\right\Vert _{L^{2p}}.
\end{align*}
Summing over all weight matrices,
\begin{align*}
\left\Vert \Delta f_{i}\left(\boldsymbol{\theta}_{0}\right)\right\Vert _{L^{p}}\leq & \sum_{l=1}^{L}s_{l}^{2}\left\Vert \left\Vert \boldsymbol{\alpha}^{\left(l-1\right)}\left(\mathbf{x}_{i}\right)\right\Vert ^{2}\right\Vert _{L^{2p}}\left\Vert \left|\mathrm{Tr}\left(\frac{\partial^{2}f_{i}\left(\boldsymbol{\theta}\right)}{\partial\left(\tilde{\boldsymbol{\alpha}}^{\left(l\right)}\left(\mathbf{x}_{i}\right)\right)^{2}}\right)\right|\right\Vert _{L^{2p}}.
\end{align*}
By \cref{eq:second_law_with_eta_deep,claim:moment of power},
\begin{align*}
\mathbb{E}\left[\left|\Delta f_{i}\left(\boldsymbol{\theta}_{t}\right)\right|^{2}\right] & \leq\eta\left\Vert \left|\Delta f_{i}\left(\boldsymbol{\theta}_{0}\right)\right|^{2}\right\Vert _{L^{c_{\beta}\bar{\kappa}^{2}\beta}}\\
 & \leq\eta\left\Vert \left|\Delta f_{i}\left(\boldsymbol{\theta}_{0}\right)\right|\right\Vert _{L^{2c_{\beta}\bar{\kappa}^{2}\beta}}^{2}\\
 & \leq\eta\left(\sum_{l=1}^{L}s_{l}^{2}\left\Vert \left\Vert \boldsymbol{\alpha}^{\left(l-1\right)}\left(\mathbf{x}_{i}\right)\right\Vert ^{2}\right\Vert _{L^{4c_{\beta}\bar{\kappa}^{2}\beta}}\left\Vert \left|\mathrm{Tr}\left(\frac{\partial^{2}f_{i}\left(\boldsymbol{\theta}\right)}{\partial\left(\tilde{\boldsymbol{\alpha}}^{\left(l\right)}\left(\mathbf{x}_{i}\right)\right)^{2}}\right)\right|\right\Vert _{L^{4c_{\beta}\bar{\kappa}^{2}\beta}}\right)^{2}.
\end{align*}
Therefore, a valid choice for $Q$ is any deterministic upper bound
on 
\begin{align*}
 & \sum_{l=1}^{L}s_{l}^{2}\left\Vert \left\Vert \boldsymbol{\alpha}^{\left(l-1\right)}\left(\mathbf{x}_{i}\right)\right\Vert ^{2}\right\Vert _{L^{4c_{\beta}\bar{\kappa}^{2}\beta}}\left\Vert \left|\mathrm{Tr}\left(\frac{\partial^{2}f_{i}\left(\boldsymbol{\theta}\right)}{\partial\left(\tilde{\boldsymbol{\alpha}}^{\left(l\right)}\left(\mathbf{x}_{i}\right)\right)^{2}}\right)\right|\right\Vert _{L^{4c_{\beta}\bar{\kappa}^{2}\beta}}.
\end{align*}
We can bound the first $L^{4c_{\beta}\bar{\kappa}^{2}\beta}$ norm
using \cref{prop:activation_norm_bound}.
It remains to bound the second $L^{4c_{\beta}\bar{\kappa}^{2}\beta}$
norm.

Let us discuss the above trace in the context of any $\mathbf{x}$
with $\left\Vert \mathbf{x}\right\Vert ^{2}=c_{\phi,\sigma_{0}}d$.
The trace can be expressed as 
\begin{align*}
\mathrm{Tr}\left(\left(\frac{\partial^{2}f\left(\boldsymbol{\theta};\mathbf{x}\right)}{\partial\left(\tilde{\boldsymbol{\alpha}}^{\left(l\right)}\left(\mathbf{x}\right)\right)^{2}}\right)\right) & =\sum_{j=1}^{m_{l}}\frac{\partial^{2}f\left(\boldsymbol{\theta};\mathbf{x}\right)}{\partial\left(\tilde{\alpha}_{j}^{\left(l\right)}\left(\mathbf{x}\right)\right)^{2}}\\
 & =\sum_{j=1}^{m_{l}}s_{L+1}\mathbf{v}^{\top}\frac{\partial^{2}\boldsymbol{\alpha}^{\left(L\right)}\left(\mathbf{x}\right)}{\partial\left(\tilde{\alpha}_{j}^{\left(l\right)}\left(\mathbf{x}\right)\right)^{2}}\\
 & =s_{L+1}\sum_{j=1}^{m_{l}}\sum_{r=1}^{m_{L}}v_{r}\frac{\partial^{2}\alpha_{r}^{\left(L\right)}\left(\mathbf{x}\right)}{\partial\left(\tilde{\alpha}_{j}^{\left(l\right)}\left(\mathbf{x}\right)\right)^{2}}=s_{L+1}\sum_{r=1}^{m_{L}}v_{r}\sum_{j=1}^{m_{l}}\frac{\partial^{2}\alpha_{r}^{\left(L\right)}\left(\mathbf{x}\right)}{\partial\left(\tilde{\alpha}_{j}^{\left(l\right)}\left(\mathbf{x}\right)\right)^{2}}.
\end{align*}
Let us denote, for any $l\in\left[L\right]$, $r\in\left[m_{L}\right]$
and $\mathbf{M}\in\mathbb{R}^{m_{l}\times m_{l}}$,
\begin{align*}
\mathbf{H}_{r}^{\left(l\right)}\in\mathbb{R}^{m_{l}\times m_{l}}, & \mathbf{H}_{r}^{\left(l\right)}\triangleq\frac{\partial^{2}\alpha_{r}^{\left(L\right)}\left(\mathbf{x}\right)}{\partial\left(\tilde{\boldsymbol{\alpha}}^{\left(l\right)}\left(\mathbf{x}\right)\right)^{2}};\\
\mathbf{T}^{\left(l\right)}\left(\mathbf{M}\right)\in\mathbb{R}^{m_{L}}, & \mathbf{T}^{\left(l\right)}\left(\mathbf{M}\right)\triangleq\left(\mathrm{Tr}\left(\mathbf{H}_{1}^{\left(l\right)}\mathbf{M}\right),\ldots,\mathrm{Tr}\left(\mathbf{H}_{m_{L}}^{\left(l\right)}\mathbf{M}\right)\right)^{\top};\\
\mathbf{D}^{\left(l\right)}\in\mathbb{R}^{m_{l}\times m_{l}}, & \mathbf{D}^{\left(l\right)}=\mathrm{diag}\left(\phi'\left(\tilde{\boldsymbol{\alpha}}^{\left(l\right)}\left(\mathbf{x}\right)\right)\right)
\end{align*}
(these are functions of $\mathbf{x}$, but we omit this in favor of
readability). Then 
\begin{align*}
\mathrm{Tr}\left(\frac{\partial^{2}f\left(\boldsymbol{\theta};\mathbf{x}\right)}{\partial\left(\tilde{\boldsymbol{\alpha}}^{\left(l\right)}\left(\mathbf{x}\right)\right)^{2}}\right) & =s_{L+1}\sum_{r=1}^{m_{L}}v_{r}\sum_{j=1}^{m_{l}}\left(H_{r}^{\left(l\right)}\right)_{jj}\\
&=s_{L+1}\sum_{r=1}^{m_{L}}v_{r}\mathrm{Tr}\left(\mathbf{H}_{r}^{\left(l\right)}\right)=s_{L+1}\mathbf{v}^{\top}\mathbf{T}^{\left(l\right)}\left(\mathbf{I}_{m_{l}}\right).
\end{align*}
Let $\mathcal{F}\triangleq\sigma\left(\mathbf{W}^{\left(1\right)},\ldots,\mathbf{W}^{\left(L\right)}\right)$
be the $\sigma$-algebra generated by $\sigma\left(\mathbf{W}^{\left(1\right)},\ldots,\mathbf{W}^{\left(L\right)}\right)$.
Then
\begin{align*}
\mathrm{Tr}\left(\frac{\partial^{2}f\left(\boldsymbol{\theta};\mathbf{x}\right)}{\partial\left(\tilde{\boldsymbol{\alpha}}^{\left(l\right)}\left(\mathbf{x}\right)\right)^{2}}\right)\mid\mathcal{F} & \sim\mathcal{N}\left(0,\frac{s_{L+1}^{2}}{a_{L+1}}\left\Vert \mathbf{T}^{\left(l\right)}\left(\mathbf{I}_{m_{l}}\right)\right\Vert ^{2}\right).
\end{align*}
Therefore, for any $p\geq1$, there exists a constant $c>0$ such
that 
\begin{align*}
\left\Vert \mathrm{Tr}\left(\frac{\partial^{2}f\left(\boldsymbol{\theta};\mathbf{x}\right)}{\partial\left(\tilde{\boldsymbol{\alpha}}^{\left(l\right)}\left(\mathbf{x}\right)\right)^{2}}\right)\right\Vert _{L^{p}\mid\mathcal{F}} & \leq c\left\Vert \mathbf{T}^{\left(l\right)}\left(\mathbf{I}_{m_{l}}\right)\right\Vert \sqrt{\frac{s_{L+1}^{2}}{a_{L+1}}p},
\end{align*}
and together with \cref{claim:moments total expectation}, 
\begin{align*}
\left\Vert \mathrm{Tr}\left(\frac{\partial^{2}f\left(\boldsymbol{\theta};\mathbf{x}\right)}{\partial\left(\tilde{\boldsymbol{\alpha}}^{\left(l\right)}\left(\mathbf{x}\right)\right)^{2}}\right)\right\Vert _{L^{p}} & \leq c\left\Vert \left\Vert \mathbf{T}^{\left(l\right)}\left(\mathbf{I}_{m_{l}}\right)\right\Vert \right\Vert _{L^{p}}\sqrt{\frac{s_{L+1}^{2}}{a_{L+1}}p}.
\end{align*}
Let us consider $l=L$. It holds that 
\begin{align*}
\mathbf{H}_{r}^{\left(L\right)} & =\phi''\left(\tilde{\alpha}_{r}^{\left(L\right)}\left(\mathbf{x}\right)\right)\mathbf{e}_{r}\mathbf{e}_{r}^{\top},
\end{align*}
Thus, for every $\mathbf{M}$ with $\left\Vert \mathbf{M}\right\Vert \leq1$,
\begin{align*}
\left\Vert \mathbf{T}^{\left(L\right)}\left(\mathbf{M}\right)\right\Vert  & =\left\Vert \left(\mathrm{Tr}\left(\mathbf{H}_{1}^{\left(L\right)}\mathbf{M}\right),\ldots,\mathrm{Tr}\left(\mathbf{H}_{m_{L}}^{\left(L\right)}\mathbf{M}\right)\right)\right\Vert \\
 & =\left\Vert \left(\phi''\left(\tilde{\alpha}_{1}^{\left(L\right)}\left(\mathbf{x}\right)\right)M_{11},\ldots,\phi''\left(\tilde{\alpha}_{m_{L}}^{\left(L\right)}\left(\mathbf{x}\right)\right)M_{m_{L}m_{L}}\right)\right\Vert \leq\beta_{\phi}\sqrt{m_{L}}.
\end{align*}
Now consider $l\in\left[L-1\right]$. We aim to express $\mathbf{T}^{\left(l\right)}\left(\mathbf{M}\right)$
recursively. Each coordinate $r$ of $\mathbf{T}^{\left(l\right)}\left(\mathbf{M}\right)$
satisfies 
\begin{align*}
\mathrm{Tr}\left(\mathbf{H}_{r}^{\left(l\right)}\mathbf{M}\right)= & \sum_{j=1}^{m_{l}}\left(\mathbf{H}_{r}^{\left(l\right)}\mathbf{M}\right)_{jj}=\sum_{j=1}^{m_{l}}\left(\left(\mathbf{H}_{r}^{\left(l\right)}\right)_{j:}\left(\mathbf{M}\right)_{:j}\right).
\end{align*}
Using \citet[Theorem 2.1]{constantine1996multivariate} with layer
$l+1$ as the intermediate function, it holds that  
\begin{align*}
\left(\mathbf{H}_{r}^{\left(l\right)}\right)_{jk}= & \frac{\partial^{2}\alpha_{r}^{\left(L\right)}\left(\mathbf{x}\right)}{\partial\tilde{\alpha}_{j}^{\left(l\right)}\left(\mathbf{x}\right)\partial\tilde{\alpha}_{k}^{\left(l\right)}\left(\mathbf{x}\right)}\\
= & \frac{\partial\alpha_{r}^{\left(L\right)}\left(\mathbf{x}\right)}{\partial\tilde{\boldsymbol{\alpha}}^{\left(l+1\right)}\left(\mathbf{x}\right)}\frac{\partial^{2}\tilde{\boldsymbol{\alpha}}^{\left(l+1\right)}\left(\mathbf{x}\right)}{\partial\tilde{\alpha}_{j}^{\left(l\right)}\left(\mathbf{x}\right)\partial\tilde{\alpha}_{k}^{\left(l\right)}\left(\mathbf{x}\right)}\\
 & +\left(\frac{\partial\tilde{\boldsymbol{\alpha}}^{\left(l+1\right)}\left(\mathbf{x}\right)}{\partial\tilde{\alpha}_{j}^{\left(l\right)}\left(\mathbf{x}\right)}\right)^{\top}\frac{\partial^{2}\alpha_{r}^{\left(L\right)}\left(\mathbf{x}\right)}{\partial\left(\tilde{\boldsymbol{\alpha}}^{\left(l+1\right)}\left(\mathbf{x}\right)\right)^{2}}\frac{\partial\tilde{\boldsymbol{\alpha}}^{\left(l+1\right)}\left(\mathbf{x}\right)}{\partial\tilde{\alpha}_{k}^{\left(l\right)}\left(\mathbf{x}\right)}\\
= & \frac{\partial\alpha_{r}^{\left(L\right)}\left(\mathbf{x}\right)}{\partial\tilde{\boldsymbol{\alpha}}^{\left(l+1\right)}\left(\mathbf{x}\right)}s_{l+1}\mathbf{W}_{:j}^{\left(l+1\right)}\phi''\left(\tilde{\alpha}_{j}^{\left(l\right)}\left(\mathbf{x}\right)\right)\mathds{1}_{\left\{ j=k\right\} }\\
 & +s_{l+1}^{2}\left(\phi'\left(\tilde{\alpha}_{j}^{\left(l\right)}\left(\mathbf{x}\right)\right)\mathbf{W}_{:j}^{\left(l+1\right)}\right)^{\top}\mathbf{H}_{r}^{\left(l+1\right)}\left(\phi'\left(\tilde{\alpha}_{k}^{\left(l\right)}\left(\mathbf{x}\right)\right)\mathbf{W}_{:k}^{\left(l+1\right)}\right)\\
= & s_{l+1}\phi''\left(\tilde{\alpha}_{j}^{\left(l\right)}\left(\mathbf{x}\right)\right)\frac{\partial\alpha_{r}^{\left(L\right)}\left(\mathbf{x}\right)}{\partial\tilde{\boldsymbol{\alpha}}^{\left(l+1\right)}\left(\mathbf{x}\right)}\mathbf{W}_{:j}^{\left(l+1\right)}\mathds{1}_{\left\{ j=k\right\} }\\
 & +s_{l+1}^{2}\left[\mathbf{D}^{\left(l\right)}\left(\mathbf{W}^{\left(l+1\right)}\right)^{\top}\mathbf{H}_{r}^{\left(l+1\right)}\mathbf{W}^{\left(l+1\right)}\mathbf{D}^{\left(l\right)}\right]_{jk}.
\end{align*}
Thus, denoting $\mathrm{diag}\left(\mathbf{M}\right)\triangleq\left(M_{11},\ldots,M_{m_{l}m_{l}}\right)^{\top}$,
we get 
\begin{align*}
\mathrm{Tr}\left(\mathbf{H}_{r}^{\left(l\right)}\mathbf{M}\right)= & \sum_{j=1}^{m_{l}}s_{l+1}\phi''\left(\tilde{\alpha}_{j}^{\left(l\right)}\left(\mathbf{x}\right)\right)M_{jj}\frac{\partial\alpha_{r}^{\left(L\right)}\left(\mathbf{x}\right)}{\partial\tilde{\boldsymbol{\alpha}}^{\left(l+1\right)}\left(\mathbf{x}\right)}\mathbf{W}_{:j}^{\left(l+1\right)}\\
 & +\sum_{j=1}^{m_{l}}s_{l+1}^{2}\left[\mathbf{D}^{\left(l\right)}\left(\mathbf{W}^{\left(l+1\right)}\right)^{\top}\mathbf{H}_{r}^{\left(l+1\right)}\mathbf{W}^{\left(l+1\right)}\mathbf{D}^{\left(l\right)}\right]_{j:}\left(\mathbf{M}\right)_{:j}\\
= & s_{l+1}\sum_{j=1}^{m_{l}}\left(\phi''\left(\tilde{\boldsymbol{\alpha}}^{\left(l\right)}\left(\mathbf{x}\right)\right)\odot\mathrm{diag}\left(\mathbf{M}\right)\right)_{j}\left(\frac{\partial\alpha_{r}^{\left(L\right)}\left(\mathbf{x}\right)}{\partial\tilde{\boldsymbol{\alpha}}^{\left(l+1\right)}\left(\mathbf{x}\right)}\mathbf{W}^{\left(l+1\right)}\right)_{j}\\
 & +s_{l+1}^{2}\mathrm{Tr}\left(\mathbf{D}^{\left(l\right)}\left(\mathbf{W}^{\left(l+1\right)}\right)^{\top}\mathbf{H}_{r}^{\left(l+1\right)}\mathbf{W}^{\left(l+1\right)}\mathbf{D}^{\left(l\right)}\mathbf{M}\right)\\
= & s_{l+1}\frac{\partial\alpha_{r}^{\left(L\right)}\left(\mathbf{x}\right)}{\partial\tilde{\boldsymbol{\alpha}}^{\left(l+1\right)}\left(\mathbf{x}\right)}\mathbf{W}^{\left(l+1\right)}\left(\phi''\left(\tilde{\boldsymbol{\alpha}}^{\left(l\right)}\left(\mathbf{x}\right)\right)\odot\mathrm{diag}\left(\mathbf{M}\right)\right)\\
 & +s_{l+1}^{2}\mathrm{Tr}\left(\mathbf{H}_{r}^{\left(l+1\right)}\mathbf{W}^{\left(l+1\right)}\mathbf{D}^{\left(l\right)}\mathbf{M}\mathbf{D}^{\left(l\right)}\left(\mathbf{W}^{\left(l+1\right)}\right)^{\top}\right)\\
= & s_{l+1}\left[\frac{\partial\boldsymbol{\alpha}^{\left(L\right)}\left(\mathbf{x}\right)}{\partial\tilde{\boldsymbol{\alpha}}^{\left(l+1\right)}\left(\mathbf{x}\right)}\mathbf{W}^{\left(l+1\right)}\left(\phi''\left(\tilde{\boldsymbol{\alpha}}^{\left(l\right)}\left(\mathbf{x}\right)\right)\odot\mathrm{diag}\left(\mathbf{M}\right)\right)\right]_{r}\\
 & +s_{l+1}^{2}\mathrm{Tr}\left(\mathbf{H}_{r}^{\left(l+1\right)}\mathbf{W}^{\left(l+1\right)}\mathbf{D}^{\left(l\right)}\mathbf{M}\mathbf{D}^{\left(l\right)}\left(\mathbf{W}^{\left(l+1\right)}\right)^{\top}\right).
\end{align*}
Recalling the definition of $\mathbf{T}^{\left(l\right)}$, we get
\begin{align*}
\mathbf{T}^{\left(l\right)}\left(\mathbf{M}\right)= & s_{l+1}\frac{\partial\boldsymbol{\alpha}^{\left(L\right)}\left(\mathbf{x}\right)}{\partial\tilde{\boldsymbol{\alpha}}^{\left(l+1\right)}\left(\mathbf{x}\right)}\mathbf{W}^{\left(l+1\right)}\left(\phi''\left(\tilde{\boldsymbol{\alpha}}^{\left(l\right)}\left(\mathbf{x}\right)\right)\odot\mathrm{diag}\left(\mathbf{M}\right)\right)\\
 & +s_{l+1}^{2}\mathbf{T}^{\left(l+1\right)}\left(\mathbf{W}^{\left(l+1\right)}\mathbf{D}^{\left(l\right)}\mathbf{M}\mathbf{D}^{\left(l\right)}\left(\mathbf{W}^{\left(l+1\right)}\right)^{\top}\right).
\end{align*}
Hence, by the triangle inequality and other properties of the spectral
norm,
\begin{align*}
\left\Vert \mathbf{T}^{\left(l\right)}\left(\mathbf{M}\right)\right\Vert \leq & s_{l+1}\left\Vert \frac{\partial\boldsymbol{\alpha}^{\left(L\right)}\left(\mathbf{x}\right)}{\partial\tilde{\boldsymbol{\alpha}}^{\left(l+1\right)}\left(\mathbf{x}\right)}\right\Vert \left\Vert \mathbf{W}^{\left(l+1\right)}\right\Vert \left\Vert \phi''\left(\tilde{\boldsymbol{\alpha}}^{\left(l\right)}\left(\mathbf{x}\right)\right)\odot\mathrm{diag}\left(\mathbf{M}\right)\right\Vert \\
 & +s_{l+1}^{2}\left\Vert \mathbf{T}^{\left(l+1\right)}\left(\mathbf{W}^{\left(l+1\right)}\mathbf{D}^{\left(l\right)}\mathbf{M}\mathbf{D}^{\left(l\right)}\left(\mathbf{W}^{\left(l+1\right)}\right)^{\top}\right)\right\Vert .
\end{align*}
It holds that 
\begin{align*}
\left\Vert \frac{\partial\boldsymbol{\alpha}^{\left(L\right)}\left(\mathbf{x}\right)}{\partial\tilde{\boldsymbol{\alpha}}^{\left(l+1\right)}\left(\mathbf{x}\right)}\right\Vert  & =\left\Vert \frac{\partial\boldsymbol{\alpha}^{\left(L\right)}\left(\mathbf{x}\right)}{\partial\tilde{\boldsymbol{\alpha}}^{\left(L\right)}\left(\mathbf{x}\right)}\frac{\partial\tilde{\boldsymbol{\alpha}}^{\left(L\right)}\left(\mathbf{x}\right)}{\partial\boldsymbol{\alpha}^{\left(L-1\right)}\left(\mathbf{x}\right)}\cdots\frac{\partial\boldsymbol{\alpha}^{\left(l+2\right)}\left(\mathbf{x}\right)}{\partial\tilde{\boldsymbol{\alpha}}^{\left(l+2\right)}\left(\mathbf{x}\right)}\frac{\partial\tilde{\boldsymbol{\alpha}}^{\left(l+2\right)}\left(\mathbf{x}\right)}{\partial\boldsymbol{\alpha}^{\left(l+1\right)}\left(\mathbf{x}\right)}\frac{\partial\boldsymbol{\alpha}^{\left(l+1\right)}\left(\mathbf{x}\right)}{\partial\tilde{\boldsymbol{\alpha}}^{\left(l+1\right)}\left(\mathbf{x}\right)}\right\Vert \\
 & =\left\Vert \mathbf{D}^{\left(L\right)}s_{L}\mathbf{W}^{\left(L\right)}\cdots\mathbf{D}^{\left(l+2\right)}s_{l+2}\mathbf{W}^{\left(l+2\right)}\mathbf{D}^{\left(l+1\right)}\right\Vert \\
 & \leq\alpha_{\phi}^{L-l}\prod_{t=l+2}^{L}s_{t}\left\Vert \mathbf{W}^{\left(t\right)}\right\Vert .
\end{align*}
Let us consider $\mathbf{M}$ with $\left\Vert \mathbf{M}\right\Vert \leq1$.
Each entry on the diagonal of $\mathbf{M}$ has absolute value at
most 1, and thus 
\begin{align*}
\left\Vert \phi''\left(\tilde{\boldsymbol{\alpha}}^{\left(l\right)}\left(\mathbf{x}\right)\right)\odot\mathrm{diag}\left(\mathbf{M}\right)\right\Vert \leq & \left\Vert \phi''\left(\tilde{\boldsymbol{\alpha}}^{\left(l\right)}\left(\mathbf{x}\right)\right)\right\Vert \leq\beta_{\phi}\sqrt{m_{l}}.
\end{align*}
Additionally, denoting 
\begin{align*}
\mathbf{A}= & \mathbf{W}^{\left(l+1\right)}\mathbf{D}^{\left(l\right)}\mathbf{M}\mathbf{D}^{\left(l\right)}\left(\mathbf{W}^{\left(l+1\right)}\right)^{\top},
\end{align*}
we have that $\left\Vert \mathbf{T}^{\left(l+1\right)}\left(\mathbf{A}\right)\right\Vert $
is trivially bounded when $\mathbf{A}=0$, and when $\mathbf{A}\neq0$,
\begin{align*}
\left\Vert \mathbf{T}^{\left(l+1\right)}\left(\mathbf{A}\right)\right\Vert = & \left\Vert \mathbf{A}\right\Vert \left\Vert \mathbf{A}\right\Vert ^{-1}\left\Vert \mathbf{T}^{\left(l+1\right)}\left(\mathbf{A}\right)\right\Vert \\
= & \left\Vert \mathbf{A}\right\Vert \left\Vert \left\Vert \mathbf{A}\right\Vert ^{-1}\mathbf{T}^{\left(l+1\right)}\left(\mathbf{A}\right)\right\Vert \\
= & \left\Vert \mathbf{A}\right\Vert \left\Vert \mathbf{T}^{\left(l+1\right)}\left(\left\Vert \mathbf{A}\right\Vert ^{-1}\mathbf{A}\right)\right\Vert \\
\leq & \left\Vert \mathbf{A}\right\Vert \sup_{\left\Vert \mathbf{M}'\right\Vert \leq1}\left\Vert \mathbf{T}^{\left(l+1\right)}\left(\mathbf{M}'\right)\right\Vert \leq\alpha_{\phi}^{2}\left\Vert \mathbf{W}^{\left(l+1\right)}\right\Vert ^{2}\sup_{\left\Vert \mathbf{M}'\right\Vert \leq1}\left\Vert \mathbf{T}^{\left(l+1\right)}\left(\mathbf{M}'\right)\right\Vert .
\end{align*}
Since the above holds for every $\mathbf{M}$ with $\left\Vert \mathbf{M}\right\Vert \leq1$,
taken all together, we get 
\begin{align*}
\sup_{\left\Vert \mathbf{M}\right\Vert \leq1}\left\Vert \mathbf{T}^{\left(l\right)}\left(\mathbf{M}\right)\right\Vert \leq & s_{l+1}\left(\alpha_{\phi}^{L-l}\prod_{t=l+2}^{L}s_{t}\left\Vert \mathbf{W}^{\left(t\right)}\right\Vert \right)\left\Vert \mathbf{W}^{\left(l+1\right)}\right\Vert \beta_{\phi}\sqrt{m_{l}}\\
 & +s_{l+1}^{2}\alpha_{\phi}^{2}\left\Vert \mathbf{W}^{\left(l+1\right)}\right\Vert ^{2}\sup_{\left\Vert \mathbf{M}'\right\Vert \leq1}\left\Vert \mathbf{T}^{\left(l+1\right)}\left(\mathbf{M}'\right)\right\Vert \\
= & \beta_{\phi}\alpha_{\phi}^{L-l}\sqrt{m_{l}}\prod_{t=l+1}^{L}s_{t}\left\Vert \mathbf{W}^{\left(t\right)}\right\Vert \\
 & +s_{l+1}^{2}\alpha_{\phi}^{2}\left\Vert \mathbf{W}^{\left(l+1\right)}\right\Vert ^{2}\sup_{\left\Vert \mathbf{M}'\right\Vert \leq1}\left\Vert \mathbf{T}^{\left(l+1\right)}\left(\mathbf{M}'\right)\right\Vert .
\end{align*}
Unrolling this recurrence and absorbing powers of $\alpha_{\phi}$
and $\beta_{\phi}$ into a new constant $C>0$, we get 
\begin{align*}
\sup_{\left\Vert \mathbf{M}\right\Vert \leq1}\left\Vert \mathbf{T}^{\left(l\right)}\left(\mathbf{M}\right)\right\Vert \leq & C^{L}\sum_{q=l}^{L}\left(\prod_{t=l+1}^{q}s_{t}^{2}\left\Vert \mathbf{W}^{\left(t\right)}\right\Vert ^{2}\right)\sqrt{m_{q}}\left(\prod_{t=q+1}^{L}s_{t}\left\Vert \mathbf{W}^{\left(t\right)}\right\Vert \right).
\end{align*}
In particular, the above holds for $\mathbf{M}=\mathbf{I}_{m_{l}}$.
Thus, we finally get 
\begin{align*}
 & \left\Vert \left\Vert \mathbf{T}^{\left(l\right)}\left(\mathbf{I}_{m_{l}}\right)\right\Vert \right\Vert _{L^{p}} \\
 & \leq \left\Vert C^{L}\sum_{q=l}^{L}\left(\prod_{t=l+1}^{q}s_{t}^{2}\left\Vert \mathbf{W}^{\left(t\right)}\right\Vert ^{2}\right)\sqrt{m_{q}}\left(\prod_{t=q+1}^{L}s_{t}\left\Vert \mathbf{W}^{\left(t\right)}\right\Vert \right)\right\Vert _{L^{p}}\\
& \overset{(i)}{\leq} C^{L}\sum_{q=l}^{L}\left(\prod_{t=l+1}^{q}s_{t}^{2}\left\Vert \left\Vert \mathbf{W}^{\left(t\right)}\right\Vert ^{2}\right\Vert _{L^{p}}\right)\sqrt{m_{q}}\left(\prod_{t=q+1}^{L}s_{t}\left\Vert \left\Vert \mathbf{W}^{\left(t\right)}\right\Vert \right\Vert _{L^{p}}\right)\\
& \overset{(ii)}{\leq} C^{L}\sum_{q=l}^{L}\left(\prod_{t=l+1}^{q}c'\nu_{0}^{2}\frac{s_{t}^{2}}{a_{t}}\left(m_{t-1}+m_{t}+p\right)\right)\sqrt{m_{q}}\left(\prod_{t=q+1}^{L}\sqrt{c'\nu_{0}^{2}\frac{s_{t}^{2}}{a_{t}}\left(m_{t-1}+m_{t}+p\right)}\right),
\end{align*}
where in $(i)$ we use the triangle inequality and the independence
of the weight matrices, and in $(ii)$ we use \cref{prop:param_norm_bounds}, and $c'>0$ is some constant. Absorbing
constants once again into $C$ (w.l.o.g), we get 
\begin{align*}
& \left\Vert \left\Vert \mathbf{T}^{\left(l\right)}\left(\mathbf{I}_{m_{l}}\right)\right\Vert \right\Vert _{L^{p}} \\
& \quad \leq C^{2L}\sum_{q=l}^{L}\left(\prod_{t=l+1}^{q}\frac{s_{t}^{2}}{a_{t}}\left(m_{t-1}+m_{t}+p\right)\right)\sqrt{m_{q}}\left(\prod_{t=q+1}^{L}\sqrt{\frac{s_{t}^{2}}{a_{t}}\left(m_{t-1}+m_{t}+p\right)}\right).
\end{align*}
Therefore, 
\begin{align*}
& \left\Vert \mathrm{Tr}\left(\frac{\partial^{2}f\left(\boldsymbol{\theta};\mathbf{x}\right)}{\partial\left(\tilde{\boldsymbol{\alpha}}^{\left(l\right)}\left(\mathbf{x}\right)\right)^{2}}\right)\right\Vert _{L^{p}} \\
& \quad \leq cC^{2L}\sum_{q=l}^{L}\left(\prod_{t=l+1}^{q}\frac{s_{t}^{2}}{a_{t}}\left(m_{t-1}+m_{t}+p\right)\right)\sqrt{m_{q}}\left(\prod_{t=q+1}^{L}\sqrt{\frac{s_{t}^{2}}{a_{t}}\left(m_{t-1}+m_{t}+p\right)}\right)\sqrt{\frac{s_{L+1}^{2}}{a_{L+1}}p}.
\end{align*}
Recall from \cref{prop:activation_norm_bound}
that there exists a constant $c''>0$ such that 
\begin{align*}
\left\Vert \left\Vert \boldsymbol{\alpha}^{\left(l-1\right)}\left(\mathbf{x}\right)\right\Vert ^{2}\right\Vert _{L^{p}} & \leq\left(c''\right)^{l-1}\left\Vert \mathbf{x}\right\Vert ^{2}\prod_{t=1}^{l-1}\frac{s_{t}^{2}}{a_{t}}\left(m_{t}+p\right).
\end{align*}
W.l.o.g $c''=c=C$. Together, we get 
\begin{align*}
 & \sum_{l=1}^{L}s_{l}^{2}\left\Vert \left\Vert \boldsymbol{\alpha}^{\left(l-1\right)}\left(\mathbf{x}_{i}\right)\right\Vert ^{2}\right\Vert _{L^{4c_{\beta}\bar{\kappa}^{2}\beta}}\left\Vert \left|\mathrm{Tr}\left(\frac{\partial^{2}f_{i}\left(\boldsymbol{\theta}\right)}{\partial\left(\tilde{\boldsymbol{\alpha}}^{\left(l\right)}\left(\mathbf{x}_{i}\right)\right)^{2}}\right)\right|\right\Vert _{L^{4c_{\beta}\bar{\kappa}^{2}\beta}}\\
 & \leq\sum_{l=1}^{L}s_{l}^{2}c^{l-1}\left\Vert \mathbf{x}\right\Vert ^{2}\left(\prod_{t=1}^{l-1}\frac{s_{t}^{2}}{a_{t}}\left(m_{t}+4c_{\beta}\bar{\kappa}^{2}\beta\right)\right)\\
 & \quad\cdot C^{2L+1}\sum_{q=l}^{L}\left(\prod_{t=l+1}^{q}\frac{s_{t}^{2}}{a_{t}}\left(m_{t-1}+m_{t}+4c_{\beta}\bar{\kappa}^{2}\beta\right)\right)\sqrt{m_{q}}\\
 & \qquad\qquad\cdot\left(\prod_{t=q+1}^{L}\sqrt{\frac{s_{t}^{2}}{a_{t}}\left(m_{t-1}+m_{t}+4c_{\beta}\bar{\kappa}^{2}\beta\right)}\right)\sqrt{\frac{s_{L+1}^{2}}{a_{L+1}}4c_{\beta}\bar{\kappa}^{2}\beta}\\
 & \lesssim C^{3L}d\sum_{l=1}^{L}\sum_{q=l}^{L}\left(\prod_{t=1}^{l-1}\frac{s_{t}^{2}}{a_{t}}\left(m_{t-1}+m_{t}+4c_{\beta}\bar{\kappa}^{2}\beta\right)\right)s_{l}^{2}\left(\prod_{t=l+1}^{q}\frac{s_{t}^{2}}{a_{t}}\left(m_{t-1}+m_{t}+4c_{\beta}\bar{\kappa}^{2}\beta\right)\right)\\
 & \qquad\qquad\qquad\cdot\sqrt{m_{q}}\left(\prod_{t=q+1}^{L}\sqrt{\frac{s_{t}^{2}}{a_{t}}\left(m_{t-1}+m_{t}+4c_{\beta}\bar{\kappa}^{2}\beta\right)}\right)\sqrt{\frac{s_{L+1}^{2}}{a_{L+1}}c_{\beta}\bar{\kappa}^{2}\beta},
\end{align*}
where in the last step we use \cref{ass:normalized_data}. This
is the deterministic upper bound we set out to find, which means it
is a valid upper bound on $Q$. Therefore, 
\begin{align*}
 & \beta^{-1}\mathbb{E}\left[\sum_{i=1}^{n}r_{i}\left(\boldsymbol{\theta}_{t}\right)\Delta f_{i}\left(\boldsymbol{\theta}_{t}\right)\right]\\
 & \lesssim C^{3L}\sqrt{\eta}\beta^{-1}nd\\
 & \cdot\left(\sum_{l=1}^{L}\sum_{q=l}^{L}\left(\prod_{t=1}^{l-1}\frac{s_{t}^{2}}{a_{t}}\left(m_{t-1}+m_{t}+4c_{\beta}\bar{\kappa}^{2}\beta\right)\right)s_{l}^{2}\left(\prod_{t=l+1}^{q}\frac{s_{t}^{2}}{a_{t}}\left(m_{t-1}+m_{t}+4c_{\beta}\bar{\kappa}^{2}\beta\right)\right)\right.\\
 & \left.\qquad\qquad\cdot\sqrt{m_{q}}\left(\prod_{t=q+1}^{L}\sqrt{\frac{s_{t}^{2}}{a_{t}}\left(m_{t-1}+m_{t}+4c_{\beta}\bar{\kappa}^{2}\beta\right)}\right)\sqrt{\frac{s_{L+1}^{2}}{a_{L+1}}c_{\beta}\bar{\kappa}^{2}\beta}\right)\sqrt{\mathbb{E}\left[\mathcal{L}\left(\boldsymbol{\theta}_{t}\right)\right]}.
\end{align*}

\subsection{Fourth term}

\label{subsec:ito-term-4-deep}Let $i\in\left[n\right]$. It holds
that 
\begin{align*}
\nabla f_{i}\left(\boldsymbol{\theta}\right)^{\top}\boldsymbol{\Lambda}\boldsymbol{\theta} & =\sum_{l=1}^{L}\frac{\partial f_{i}\left(\boldsymbol{\theta}\right)}{\partial\mathrm{vec}\left(\mathbf{W}^{\left(l\right)}\right)}\left(\frac{a_{l}}{\nu_{0}^{2}}\mathbf{I}_{m_{l-1}m_{l}}\right)\mathrm{vec}\left(\mathbf{W}^{\left(l\right)}\right)+\frac{\partial f_{i}\left(\boldsymbol{\theta}\right)}{\partial\mathbf{v}}\left(a_{L+1}\mathbf{I}_{m_{L}}\right)\mathbf{v}\\
 & =\frac{1}{\nu_{0}^{2}}\sum_{l=1}^{L}a_{l}\frac{\partial f_{i}\left(\boldsymbol{\theta}\right)}{\partial\mathrm{vec}\left(\mathbf{W}^{\left(l\right)}\right)}\mathrm{vec}\left(\mathbf{W}^{\left(l\right)}\right)+a_{L+1}\frac{\partial f_{i}}{\partial\mathbf{v}}\mathbf{v}.
\end{align*}
For the first sum, observe that 
\begin{align*}
\frac{\partial\tilde{\boldsymbol{\alpha}}^{\left(l\right)}\left(\mathbf{x}_{i}\right)}{\partial\mathrm{vec}\left(\mathbf{W}^{\left(l\right)}\right)}\mathrm{vec}\left(\mathbf{W}^{\left(l\right)}\right) & =\frac{\partial\tilde{\boldsymbol{\alpha}}^{\left(l\right)}\left(\mathbf{x}_{i}\right)}{\partial\mathbf{W}^{\left(l\right)}}\frac{\partial\mathbf{W}^{\left(l\right)}}{\partial\mathrm{vec}\left(\mathbf{W}^{\left(l\right)}\right)}\mathrm{vec}\left(\mathbf{W}^{\left(l\right)}\right)\\
 & =\frac{\partial\tilde{\boldsymbol{\alpha}}^{\left(l\right)}\left(\mathbf{x}_{i}\right)}{\partial\mathbf{W}^{\left(l\right)}}\mathbf{W}^{\left(l\right)}\\
 & =\tilde{\boldsymbol{\alpha}}^{\left(l\right)}\left(\mathbf{x}_{i}\right).
\end{align*}
Using the above and denoting 
\begin{align*}
\mathbf{h}^{\left(l\right)}\left(\mathbf{x}\right) & \triangleq\frac{\partial\boldsymbol{\alpha}^{\left(L\right)}\left(\mathbf{x}\right)}{\partial\tilde{\boldsymbol{\alpha}}^{\left(l\right)}\left(\mathbf{x}\right)}\tilde{\boldsymbol{\alpha}}^{\left(l\right)}\left(\mathbf{x}\right),
\end{align*}
we get
\begin{align*}
\frac{\partial f_{i}\left(\boldsymbol{\theta}\right)}{\partial\mathrm{vec}\left(\mathbf{W}^{\left(l\right)}\right)}\mathrm{vec}\left(\mathbf{W}^{\left(l\right)}\right) & =\frac{\partial f_{i}\left(\boldsymbol{\theta}\right)}{\partial\boldsymbol{\alpha}^{\left(L\right)}\left(\mathbf{x}_{i}\right)}\frac{\partial\boldsymbol{\alpha}^{\left(L\right)}\left(\mathbf{x}_{i}\right)}{\partial\tilde{\boldsymbol{\alpha}}^{\left(l\right)}\left(\mathbf{x}_{i}\right)}\frac{\partial\tilde{\boldsymbol{\alpha}}^{\left(l\right)}\left(\mathbf{x}_{i}\right)}{\partial\mathrm{vec}\left(\mathbf{W}^{\left(l\right)}\right)}\mathrm{vec}\left(\mathbf{W}^{\left(l\right)}\right)\\
 & =s_{L+1}\mathbf{v}^{\top}\mathbf{h}^{\left(l\right)}\left(\mathbf{x}_{i}\right).
\end{align*}
For the final term in $\nabla f_{i}\left(\boldsymbol{\theta}\right)^{\top}\boldsymbol{\Lambda}\boldsymbol{\theta}$,
it holds that
\begin{align*}
a_{L+1}\frac{\partial f_{i}\left(\boldsymbol{\theta}\right)}{\partial\mathbf{v}}\mathbf{v} & =a_{L+1}s_{L+1}\mathbf{v}^{\top}\boldsymbol{\alpha}^{\left(L\right)}\left(\mathbf{x}_{i}\right)\\
 & =a_{L+1}y_{i}+a_{L+1}\left(s_{L+1}\mathbf{v}^{\top}\boldsymbol{\alpha}^{\left(L\right)}\left(\mathbf{x}_{i}\right)-y_{i}\right)=a_{L+1}y_{i}+a_{L+1}r_{i}\left(\boldsymbol{\theta}\right).
\end{align*}
Thus, we get
\begin{align*}
\nabla f_{i}\left(\boldsymbol{\theta}\right)^{\top}\boldsymbol{\Lambda}\boldsymbol{\theta} & =\frac{1}{\nu_{0}^{2}}s_{L+1}\mathbf{v}^{\top}\left(\sum_{l=1}^{L}a_{l}\mathbf{h}^{\left(l\right)}\left(\mathbf{x}_{i}\right)\right)+a_{L+1}y_{i}+a_{L+1}r_{i}\left(\boldsymbol{\theta}\right).
\end{align*}
Therefore, 
\begin{align*}
r_{i}\left(\boldsymbol{\theta}\right)\nabla f_{i}\left(\boldsymbol{\theta}\right)^{\top}\boldsymbol{\Lambda}\boldsymbol{\theta} & =r_{i}\left(\boldsymbol{\theta}\right)\left(\frac{1}{\nu_{0}^{2}}s_{L+1}\mathbf{v}^{\top}\left(\sum_{l=1}^{L}a_{l}\mathbf{h}^{\left(l\right)}\left(\mathbf{x}_{i}\right)\right)+a_{L+1}y_{i}\right)+2a_{L+1}\frac{1}{2}\left|r_{i}\left(\boldsymbol{\theta}\right)\right|^{2}.
\end{align*}
Using this, we get
\begin{align*}
 & -\beta^{-1}\mathbb{E}\left[\mathbf{r}\left(\boldsymbol{\theta}_{t}\right)^{\top}\nabla f\left(\boldsymbol{\theta}_{t}\right)\boldsymbol{\Lambda}\boldsymbol{\theta}_{t}\right]\\
 & =-\beta^{-1}\sum_{i=1}^{n}\mathbb{E}\left[r_{i}\left(\boldsymbol{\theta}_{t}\right)\nabla f_{i}\left(\boldsymbol{\theta}_{t}\right)^{\top}\boldsymbol{\Lambda}\boldsymbol{\theta}_{t}\right]\\
 & =-\beta^{-1}\sum_{i=1}^{n}\mathbb{E}\left[r_{i}\left(\boldsymbol{\theta}_{t}\right)\left(\frac{1}{\nu_{0}^{2}}s_{L+1}\mathbf{v}^{\top}\left(\sum_{l=1}^{L}a_{l}\mathbf{h}^{\left(l\right)}\left(\mathbf{x}_{i}\right)\right)+a_{L+1}y_{i}\right)\right]\\
 & \quad-\beta^{-1}\sum_{i=1}^{n}\mathbb{E}\left[2a_{L+1}\frac{1}{2}\left|r_{i}\left(\boldsymbol{\theta}_{t}\right)\right|^{2}\right]\\
 & \leq\beta^{-1}\sum_{i=1}^{n}\mathbb{E}\left[\left|r_{i}\left(\boldsymbol{\theta}_{t}\right)\left(\frac{1}{\nu_{0}^{2}}s_{L+1}\mathbf{v}^{\top}\left(\sum_{l=1}^{L}a_{l}\mathbf{h}^{\left(l\right)}\left(\mathbf{x}_{i}\right)\right)+a_{L+1}y_{i}\right)\right|\right]\\
 & \quad-\beta^{-1}\sum_{i=1}^{n}\mathbb{E}\left[2a_{L+1}\frac{1}{2}\left|r_{i}\left(\boldsymbol{\theta}_{t}\right)\right|^{2}\right]\\
 & =\beta^{-1}\sum_{i=1}^{n}\mathbb{E}\left[\left|r_{i}\left(\boldsymbol{\theta}_{t}\right)\left(\frac{1}{\nu_{0}^{2}}s_{L+1}\mathbf{v}^{\top}\left(\sum_{l=1}^{L}a_{l}\mathbf{h}^{\left(l\right)}\left(\mathbf{x}_{i}\right)\right)+a_{L+1}y_{i}\right)\right|\right]\\
 & \quad-2\beta^{-1}na_{L+1}\mathbb{E}\left[\mathcal{L}\left(\boldsymbol{\theta}_{t}\right)\right].
\end{align*}
Focusing on the first term, we have that 
\begin{align*}
 & \sum_{i=1}^{n}\mathbb{E}\left[\left|r_{i}\left(\boldsymbol{\theta}_{t}\right)\left(\frac{1}{\nu_{0}^{2}}s_{L+1}\mathbf{v}^{\top}\left(\sum_{l=1}^{L}a_{l}\mathbf{h}^{\left(l\right)}\left(\mathbf{x}_{i}\right)\right)+a_{L+1}y_{i}\right)\right|\right]\\
 & \overset{(i)}{\leq}\sum_{i=1}^{n}\sqrt{\mathbb{E}\left[\left|r_{i}\left(\boldsymbol{\theta}_{t}\right)\right|^{2}\right]}\sqrt{\mathbb{E}\left[\left(\frac{1}{\nu_{0}^{2}}s_{L+1}\mathbf{v}^{\top}\left(\sum_{l=1}^{L}a_{l}\mathbf{h}^{\left(l\right)}\left(\mathbf{x}_{i}\right)\right)+a_{L+1}y_{i}\right)^{2}\right]}\\
 & \overset{(ii)}{\leq}\sum_{i=1}^{n}\sqrt{\mathbb{E}\left[\left|r_{i}\left(\boldsymbol{\theta}_{t}\right)\right|^{2}\right]}\sqrt{2\left(a_{L+1}^{2}y_{i}^{2}+\frac{1}{\nu_{0}^{4}}s_{L+1}^{2}\mathbb{E}\left[\left(\mathbf{v}^{\top}\left(\sum_{l=1}^{L}a_{l}\mathbf{h}^{\left(l\right)}\left(\mathbf{x}_{i}\right)\right)\right)^{2}\right]\right)}\\
 & \overset{(iii)}{\leq}\sum_{i=1}^{n}\sqrt{\mathbb{E}\left[\left|r_{i}\left(\boldsymbol{\theta}_{t}\right)\right|^{2}\right]}\left(2a_{L+1}\left|y_{i}\right|+2\frac{1}{\nu_{0}^{2}}s_{L+1}\sqrt{\mathbb{E}\left[\left(\mathbf{v}^{\top}\left(\sum_{l=1}^{L}a_{l}\mathbf{h}^{\left(l\right)}\left(\mathbf{x}_{i}\right)\right)\right)^{2}\right]}\right),
\end{align*}
where $(i)$ uses Cauchy--Schwarz, $(ii)$ uses $\left(a+b\right)^{2}\leq2a^{2}+2b^{2}$
and $(iii)$ uses $\sqrt{a+b}\leq\sqrt{a}+\sqrt{b}$. If we find some
expression $Q$, independent of $i$, such that 
\begin{align*}
\mathbb{E}\left[\left(\mathbf{v}^{\top}\left(\sum_{l=1}^{L}a_{l}\mathbf{h}^{\left(l\right)}\left(\mathbf{x}_{i}\right)\right)\right)^{2}\right] & \leq Q^{2},
\end{align*}
then we get 
\begin{align*}
 & \sum_{i=1}^{n}\mathbb{E}\left[\left|r_{i}\left(\boldsymbol{\theta}_{t}\right)\left(\frac{1}{\nu_{0}^{2}}s_{L+1}\mathbf{v}^{\top}\left(\sum_{l=1}^{L}a_{l}\mathbf{h}^{\left(l\right)}\left(\mathbf{x}_{i}\right)\right)+a_{L+1}y_{i}\right)\right|\right]\\
 & \leq\sum_{i=1}^{n}\sqrt{\mathbb{E}\left[\left|r_{i}\left(\boldsymbol{\theta}_{t}\right)\right|^{2}\right]}\left(2a_{L+1}\left|y_{i}\right|+2\frac{1}{\nu_{0}^{2}}s_{L+1}Q\right)\\
 & \leq\left(2a_{L+1}R_{y}+2\frac{1}{\nu_{0}^{2}}s_{L+1}Q\right)n\frac{1}{n}\sum_{i=1}^{n}\sqrt{\mathbb{E}\left[\left|r_{i}\left(\boldsymbol{\theta}_{t}\right)\right|^{2}\right]}\\
\left[\mathrm{Jensen}\right] & \leq\left(2a_{L+1}R_{y}+2\frac{1}{\nu_{0}^{2}}s_{L+1}Q\right)n\sqrt{\frac{1}{n}\sum_{i=1}^{n}\mathbb{E}\left[\left|r_{i}\left(\boldsymbol{\theta}_{t}\right)\right|^{2}\right]}\\
 & \leq4n\sqrt{\mathbb{E}\left[\mathcal{L}\left(\boldsymbol{\theta}_{t}\right)\right]}\left(a_{L+1}R_{y}+\frac{1}{\nu_{0}^{2}}s_{L+1}Q\right),
\end{align*}
which yields 
\begin{align*}
-\beta^{-1}\mathbb{E}\left[\mathbf{r}\left(\boldsymbol{\theta}_{t}\right)^{\top}\nabla f\left(\boldsymbol{\theta}_{t}\right)\boldsymbol{\Lambda}\boldsymbol{\theta}_{t}\right] & \leq4\beta^{-1}n\sqrt{\mathbb{E}\left[\mathcal{L}\left(\boldsymbol{\theta}_{t}\right)\right]}\left(a_{L+1}R_{y}+\frac{1}{\nu_{0}^{2}}s_{L+1}Q\right)\\
 & -2\beta^{-1}na_{L+1}\mathbb{E}\left[\mathcal{L}\left(\boldsymbol{\theta}_{t}\right)\right].
\end{align*}
Let us find such $Q$. Let $i\in\left[n\right]$. Let $\mathcal{F}\triangleq\sigma\left(\mathbf{W}^{\left(1\right)},\ldots,\mathbf{W}^{\left(L\right)}\right)$
be the $\sigma$-algebra generated by $\sigma\left(\mathbf{W}^{\left(1\right)},\ldots,\mathbf{W}^{\left(L\right)}\right)$.
It holds that
\begin{align*}
\mathbf{v}^{\top}\left(\sum_{l=1}^{L}a_{l}\mathbf{h}^{\left(l\right)}\left(\mathbf{x}_{i}\right)\right)\mid\mathcal{F} & \sim\mathcal{N}\left(0,\frac{1}{a_{L+1}}\left\Vert \sum_{l=1}^{L}a_{l}\mathbf{h}^{\left(l\right)}\left(\mathbf{x}_{i}\right)\right\Vert ^{2}\right).
\end{align*}
Therefore,
\begin{align*}
\left\Vert \mathbf{v}^{\top}\left(\sum_{l=1}^{L}a_{l}\mathbf{h}^{\left(l\right)}\left(\mathbf{x}_{i}\right)\right)\right\Vert _{L^{p}\mid\mathcal{F}} & \leq\frac{1}{\sqrt{a_{L+1}}}\left\Vert \sum_{l=1}^{L}a_{l}\mathbf{h}^{\left(l\right)}\left(\mathbf{x}_{i}\right)\right\Vert \sqrt{p}.
\end{align*}
Using \cref{claim:moments total expectation},
applying the triangle inequality and applying Minkowski's inequality,
we get
\begin{align*}
\left\Vert \mathbf{v}^{\top}\left(\sum_{l=1}^{L}a_{l}\mathbf{h}^{\left(l\right)}\left(\mathbf{x}_{i}\right)\right)\right\Vert _{L^{p}} & \leq\frac{1}{\sqrt{a_{L+1}}}\sqrt{p}\left\Vert \left\Vert \sum_{l=1}^{L}a_{l}\mathbf{h}^{\left(l\right)}\left(\mathbf{x}_{i}\right)\right\Vert \right\Vert _{L^{p}}\\
 & \leq\frac{1}{\sqrt{a_{L+1}}}\sqrt{p}\left\Vert \sum_{l=1}^{L}a_{l}\left\Vert \mathbf{h}^{\left(l\right)}\left(\mathbf{x}_{i}\right)\right\Vert \right\Vert _{L^{p}}\\
 & \leq\frac{1}{\sqrt{a_{L+1}}}\sqrt{p}\sum_{l=1}^{L}a_{l}\left\Vert \left\Vert \mathbf{h}^{\left(l\right)}\left(\mathbf{x}_{i}\right)\right\Vert \right\Vert _{L^{p}}.
\end{align*}
Observe that for every $l\in\left[L\right]$,
\begin{align*}
\mathbf{h}^{\left(l\right)}\left(\mathbf{x}_{i}\right) & =\frac{\partial\boldsymbol{\alpha}^{\left(L\right)}\left(\mathbf{x}_{i}\right)}{\partial\tilde{\boldsymbol{\alpha}}^{\left(l\right)}\left(\mathbf{x}_{i}\right)}\tilde{\boldsymbol{\alpha}}^{\left(l\right)}\left(\mathbf{x}_{i}\right)\\
 & =\left(\prod_{t=l+1}^{L}\left(\frac{\partial\boldsymbol{\alpha}^{\left(t\right)}\left(\mathbf{x}_{i}\right)}{\partial\boldsymbol{\alpha}^{\left(t-1\right)}\left(\mathbf{x}_{i}\right)}\right)^{\top}\right)^{\top}\frac{\partial\boldsymbol{\alpha}^{\left(l\right)}\left(\mathbf{x}_{i}\right)}{\partial\tilde{\boldsymbol{\alpha}}^{\left(l\right)}\left(\mathbf{x}_{i}\right)}\tilde{\boldsymbol{\alpha}}^{\left(l\right)}\left(\mathbf{x}_{i}\right)\\
 & =\left(\prod_{t=l+1}^{L}s_{t}\left(\mathbf{W}^{\left(t\right)}\right)^{\top}\mathrm{diag}\left(\phi'\left(\tilde{\boldsymbol{\alpha}}^{\left(t\right)}\left(\mathbf{x}_{i}\right)\right)\right)\right)^{\top}\mathrm{diag}\left(\phi'\left(\tilde{\boldsymbol{\alpha}}^{\left(l\right)}\left(\mathbf{x}_{i}\right)\right)\right)\tilde{\boldsymbol{\alpha}}^{\left(l\right)}\left(\mathbf{x}_{i}\right).
\end{align*}
Therefore, 
\begin{align*}
\left\Vert \mathbf{h}^{\left(l\right)}\left(\mathbf{x}_{i}\right)\right\Vert  & \leq\alpha_{\phi}^{L-l+1}\left(\prod_{t=l+1}^{L}s_{t}\left\Vert \mathbf{W}^{\left(t\right)}\right\Vert \right)\left\Vert \tilde{\boldsymbol{\alpha}}^{\left(l\right)}\left(\mathbf{x}_{i}\right)\right\Vert \\
 & \leq\alpha_{\phi}^{L-l+1}\left(\prod_{t=l}^{L}s_{t}\left\Vert \mathbf{W}^{\left(t\right)}\right\Vert \right)\left\Vert \boldsymbol{\alpha}^{\left(l-1\right)}\left(\mathbf{x}_{i}\right)\right\Vert .
\end{align*}
Since $\left\Vert \boldsymbol{\alpha}^{\left(l-1\right)}\left(\mathbf{x}_{i}\right)\right\Vert $
is independent of weights in upper layers, it holds that
\begin{align*}
\left\Vert \left\Vert \mathbf{h}^{\left(l\right)}\left(\mathbf{x}_{i}\right)\right\Vert \right\Vert _{L^{p}} & \leq\alpha_{\phi}^{L-l+1}\left(\prod_{t=l}^{L}s_{t}\left\Vert \left\Vert \mathbf{W}^{\left(t\right)}\right\Vert \right\Vert _{L^{p}}\right)\left\Vert \left\Vert \boldsymbol{\alpha}^{\left(l-1\right)}\left(\mathbf{x}_{i}\right)\right\Vert \right\Vert _{L^{p}}.
\end{align*}
Using \cref{prop:param_norm_bounds,prop:activation_norm_bound},
we get that there exists a constant $c>0$ such that, at initialization,
\begin{align*}
\left\Vert \left\Vert \mathbf{h}^{\left(l\right)}\left(\mathbf{x}_{i}\right)\right\Vert \right\Vert _{L^{p}} & \leq\sqrt{\left(c\alpha_{\phi}^{2}\nu_{0}^{2}\right)^{L}\left\Vert \mathbf{x}_{i}\right\Vert ^{2}\left(\prod_{t=1}^{l-1}\frac{s_{t}^{2}}{a_{t}}\left(m_{t}+p\right)\right)\left(\prod_{t=l}^{L}\frac{s_{t}^{2}}{a_{t}}\left(m_{t-1}+m_{t}+p\right)\right)}.
\end{align*}
Thus, we get that
\begin{align*}
 & \left\Vert \mathbf{v}^{\top}\left(\sum_{l=1}^{L}a_{l}\mathbf{h}^{\left(l\right)}\left(\mathbf{x}_{i}\right)\right)\right\Vert _{L^{p}}\\
 & \leq\frac{1}{\sqrt{a_{L+1}}}\sqrt{p}\sum_{l=1}^{L}a_{l}\sqrt{\left(c\alpha_{\phi}^{2}\nu_{0}^{2}\right)^{L}\left\Vert \mathbf{x}_{i}\right\Vert ^{2}\left(\prod_{t=1}^{l-1}\frac{s_{t}^{2}}{a_{t}}\left(m_{t}+p\right)\right)\left(\prod_{t=l}^{L}\frac{s_{t}^{2}}{a_{t}}\left(m_{t-1}+m_{t}+p\right)\right)}\\
 & =\sum_{l=1}^{L}\sqrt{\left(c\alpha_{\phi}^{2}\nu_{0}^{2}\right)^{L}\left\Vert \mathbf{x}_{i}\right\Vert ^{2}a_{l}^{2}\frac{1}{a_{L+1}}p\left(\prod_{t=1}^{l-1}\frac{s_{t}^{2}}{a_{t}}\left(m_{t}+p\right)\right)\left(\prod_{t=l}^{L}\frac{s_{t}^{2}}{a_{t}}\left(m_{t-1}+m_{t}+p\right)\right)}\\
 & =L\frac{1}{L}\sum_{l=1}^{L}\sqrt{\left(c\alpha_{\phi}^{2}\nu_{0}^{2}\right)^{L}\left\Vert \mathbf{x}_{i}\right\Vert ^{2}a_{l}^{2}\frac{1}{a_{L+1}}p\left(\prod_{t=1}^{l-1}\frac{s_{t}^{2}}{a_{t}}\left(m_{t}+p\right)\right)\left(\prod_{t=l}^{L}\frac{s_{t}^{2}}{a_{t}}\left(m_{t-1}+m_{t}+p\right)\right)}\\
 & \overset{(*)}{\leq}L\sqrt{\frac{1}{L}\sum_{l=1}^{L}\left(c\alpha_{\phi}^{2}\nu_{0}^{2}\right)^{L}\left\Vert \mathbf{x}_{i}\right\Vert ^{2}a_{l}^{2}\frac{1}{a_{L+1}}p\left(\prod_{t=1}^{l-1}\frac{s_{t}^{2}}{a_{t}}\left(m_{t}+p\right)\right)\left(\prod_{t=l}^{L}\frac{s_{t}^{2}}{a_{t}}\left(m_{t-1}+m_{t}+p\right)\right)}\\
 & =L^{1/2}\sqrt{\left(c\alpha_{\phi}^{2}\nu_{0}^{2}\right)^{L}\left\Vert \mathbf{x}_{i}\right\Vert ^{2}\frac{1}{a_{L+1}}p\sum_{l=1}^{L}a_{l}^{2}\left(\prod_{t=1}^{l-1}\frac{s_{t}^{2}}{a_{t}}\left(m_{t}+p\right)\right)\left(\prod_{t=l}^{L}\frac{s_{t}^{2}}{a_{t}}\left(m_{t-1}+m_{t}+p\right)\right)},
\end{align*}
where in $(*)$ we use Jensen's inequality. Using \cref{claim:moment of power}, we get that at initialization, 
\begin{align*}
 & \left\Vert \left(\mathbf{v}^{\top}\left(\sum_{l=1}^{L}a_{l}\mathbf{h}^{\left(l\right)}\left(\mathbf{x}_{i}\right)\right)\right)^{2}\right\Vert _{L^{p}}\\
 & \leq\left\Vert \mathbf{v}^{\top}\left(\sum_{l=1}^{L}a_{l}\mathbf{h}^{\left(l\right)}\left(\mathbf{x}_{i}\right)\right)\right\Vert _{L^{2p}}^{2}\\
 & \leq L\left(c\alpha_{\phi}^{2}\nu_{0}^{2}\right)^{L}\left\Vert \mathbf{x}_{i}\right\Vert ^{2}\frac{1}{a_{L+1}}2p\sum_{l=1}^{L}a_{l}^{2}\left(\prod_{t=1}^{l-1}\frac{s_{t}^{2}}{a_{t}}\left(m_{t}+2p\right)\right)\left(\prod_{t=l}^{L}\frac{s_{t}^{2}}{a_{t}}\left(m_{t-1}+m_{t}+2p\right)\right)\\
 & \leq c_{\phi,\sigma_{0}}L\left(c\alpha_{\phi}^{2}\nu_{0}^{2}\right)^{L}d\frac{1}{a_{L+1}}2p\sum_{l=1}^{L}a_{l}^{2}\left(\prod_{t=1}^{l-1}\frac{s_{t}^{2}}{a_{t}}\left(m_{t}+2p\right)\right)\left(\prod_{t=l}^{L}\frac{s_{t}^{2}}{a_{t}}\left(m_{t-1}+m_{t}+2p\right)\right).
\end{align*}
Together with \cref{eq:second_law_with_eta_deep}, we get
\begin{align*}
 & \mathbb{E}_{p_{t}}\left[\left(\mathbf{v}^{\top}\left(\sum_{l=1}^{L}a_{l}\mathbf{h}^{\left(l\right)}\left(\mathbf{x}_{i}\right)\right)\right)^{2}\right]\\
 & \leq\eta c_{\phi,\sigma_{0}}L\left(c\alpha_{\phi}^{2}\nu_{0}^{2}\right)^{L}d\frac{1}{a_{L+1}}2c_{\beta}\bar{\kappa}^{2}\beta\\
 & \quad\cdot\sum_{l=1}^{L}a_{l}^{2}\left(\prod_{t=1}^{l-1}\frac{s_{t}^{2}}{a_{t}}\left(m_{t}+2c_{\beta}\bar{\kappa}^{2}\beta\right)\right)\left(\prod_{t=l}^{L}\frac{s_{t}^{2}}{a_{t}}\left(m_{t-1}+m_{t}+2c_{\beta}\bar{\kappa}^{2}\beta\right)\right).
\end{align*}
Thus, the right-hand side is a valid choice for $Q^{2}$. Back to
the main inequality, we get,
\begin{align*}
 & -\beta^{-1}\mathbb{E}\left[\mathbf{r}\left(\boldsymbol{\theta}_{t}\right)^{\top}\nabla f\left(\boldsymbol{\theta}_{t}\right)\boldsymbol{\Lambda}\boldsymbol{\theta}_{t}\right]\\
 & \lesssim\beta^{-1}n\sqrt{\mathbb{E}\left[\mathcal{L}\left(\boldsymbol{\theta}_{t}\right)\right]}\\
 & \quad\cdot\Bigg(a_{L+1}R_{y}+\frac{1}{\nu_{0}^{2}}s_{L+1}\sqrt{\eta L\left(c\alpha_{\phi}^{2}\nu_{0}^{2}\right)^{L}d\frac{1}{a_{L+1}}2c_{\beta}\bar{\kappa}^{2}\beta}\\
 & \quad\quad\quad\cdot\sqrt{\sum_{l=1}^{L}a_{l}^{2}\left(\prod_{t=1}^{l-1}\frac{s_{t}^{2}}{a_{t}}\left(m_{t}+2c_{\beta}\bar{\kappa}^{2}\beta\right)\right)\left(\prod_{t=l}^{L}\frac{s_{t}^{2}}{a_{t}}\left(m_{t-1}+m_{t}+2c_{\beta}\bar{\kappa}^{2}\beta\right)\right)}\Bigg)\\
 & \quad-2\beta^{-1}na_{L+1}\mathbb{E}\left[\mathcal{L}\left(\boldsymbol{\theta}_{t}\right)\right].
\end{align*}

\subsection{Bounding the entire right-hand side}

Recall that 
\begin{align*}
\frac{\mathrm{d}}{\mathrm{d}t}\mathbb{E}\left[\mathcal{L}\left(\boldsymbol{\theta}_{t}\right)\right]=\frac{1}{n}\biggl( & \underbrace{-n^{-1}\mathbb{E}\left[\mathbf{r}\left(\boldsymbol{\theta}_{t}\right)^{\top}\mathbf{K}_{\mathrm{NTK}}\left(\boldsymbol{\theta}_{t}\right)\mathbf{r}\left(\boldsymbol{\theta}_{t}\right)\right]}_{T_{1}}\\
 & +\underbrace{\beta^{-1}\mathbb{E}\left[\mathrm{Tr}\left(\nabla f\left(\boldsymbol{\theta}_{t}\right)^{\top}\nabla f\left(\boldsymbol{\theta}_{t}\right)\right)\right]}_{T_{2}}\\
 & +\underbrace{\beta^{-1}\mathbb{E}\left[\sum_{i=1}^{n}r_{i}\left(\boldsymbol{\theta}_{t}\right)\Delta f_{i}\left(\boldsymbol{\theta}_{t}\right)\right]}_{T_{3}}\\
 & \underbrace{-\beta^{-1}\mathbb{E}\left[\mathbf{r}\left(\boldsymbol{\theta}_{t}\right)^{\top}\nabla f\left(\boldsymbol{\theta}_{t}\right)\boldsymbol{\Lambda}\boldsymbol{\theta}_{t}\right]\biggr)}_{T_{4}}.
\end{align*}
We have shown that there exists a constant $C>0$ (enlarging $C$
if necessary, the same constant may be used in all four bounds) such
that
\begin{align*}
T_{1}\lesssim & -\lambda_{1}s_{L+1}^{2}m\mathbb{E}\left[\mathcal{L}\left(\boldsymbol{\theta}_{t}\right)\right]+\eta\lambda_{1}s_{L+1}^{2}m\left(1+c_{\beta}\bar{\kappa}^{4}\beta\right)\exp\left(-\tilde{c_{1}}m\left(c_{\beta}\bar{\kappa}^{2}\beta\frac{n}{\lambda_{1}}\ln\frac{n}{\lambda_{1}}\right)^{-1}\right);\\
T_{2}\lesssim & C^{2L}\eta\beta^{-1}nds_{L+1}^{2}\left(m_{L}+c_{\beta}\bar{\kappa}^{2}\beta\right)\\
 & \cdot\left(\frac{s_{L}^{2}}{a_{L}}\prod_{k=1}^{L-1}\frac{s_{k}^{2}}{a_{k}}\left(m_{k}+c_{\beta}\bar{\kappa}^{2}\beta\right)\right.\\
 & \left.\quad+\frac{1}{a_{L+1}}\sum_{l=1}^{L}\left(\prod_{k=1}^{l-1}\frac{s_{k}^{2}}{a_{k}}\left(m_{k}+c_{\beta}\bar{\kappa}^{2}\beta\right)\right)s_{l}^{2}\left(\prod_{k=l+1}^{L}\frac{s_{k}^{2}}{a_{k}}\left(m_{k-1}+m_{k}+c_{\beta}\bar{\kappa}^{2}\beta\right)\right)\right);\\
T_{3}\lesssim & C^{3L}\sqrt{\eta}\beta^{-1}nd\\
 & \cdot\sum_{l=1}^{L}\sum_{q=l}^{L}\left(\prod_{t=0}^{l-2}\frac{s_{t+1}^{2}}{a_{t+1}}\left(m_{t}+m_{t+1}+4c_{\beta}\bar{\kappa}^{2}\beta\right)\right)s_{l}^{2}\\
 & \quad\quad\quad\quad\cdot\left(\prod_{t=l}^{q-1}\frac{s_{t+1}^{2}}{a_{t+1}}\left(m_{t}+m_{t+1}+4c_{\beta}\bar{\kappa}^{2}\beta\right)\right)\\
 & \quad\quad\quad\quad\cdot\sqrt{m_{q}}\left(\prod_{t=q}^{L-1}\sqrt{\frac{s_{t+1}^{2}}{a_{t+1}}\left(m_{t}+m_{t+1}+4c_{\beta}\bar{\kappa}^{2}\beta\right)}\right)\\
 & \quad\quad\quad\quad\cdot\sqrt{\frac{s_{L+1}^{2}}{a_{L+1}}c_{\beta}\bar{\kappa}^{2}\beta}\sqrt{\mathbb{E}\left[\mathcal{L}\left(\boldsymbol{\theta}_{t}\right)\right]};\\
T_{4} & \lesssim\beta^{-1}n\sqrt{\mathbb{E}\left[\mathcal{L}\left(\boldsymbol{\theta}_{t}\right)\right]}\\
 & \quad\cdot\Bigg(a_{L+1}+s_{L+1}\sqrt{\eta LC^{L}d\frac{1}{a_{L+1}}2c_{\beta}\bar{\kappa}^{2}\beta}\\
 & \quad\quad\quad\cdot\sqrt{\sum_{l=1}^{L}a_{l}^{2}\left(\prod_{t=1}^{l-1}\frac{s_{t}^{2}}{a_{t}}\left(m_{t}+2c_{\beta}\bar{\kappa}^{2}\beta\right)\right)\left(\prod_{t=l}^{L}\frac{s_{t}^{2}}{a_{t}}\left(m_{t-1}+m_{t}+2c_{\beta}\bar{\kappa}^{2}\beta\right)\right)}\Bigg)\\
 & \quad-\beta^{-1}na_{L+1}\mathbb{E}\left[\mathcal{L}\left(\boldsymbol{\theta}_{t}\right)\right].
\end{align*}
From now on we shall assume $m_{1}=\cdots=m_{L}=m$, $d\leq m$ and
$m>4c_{\beta}\bar{\kappa}^{2}\beta$ (the third will be required later
anyway). We shall also treat the depth $L$ as a fixed constant, thereby
making $C^{L}$ a constant. Lastly, we apply the substitution $\kappa=\sqrt{\frac{s_{L+1}^{2}m}{a_{L+1}}}$
wherever applicable. This allows us to simplify the bounds above to
\begin{align*}
T_{1}\lesssim & -\lambda_{1}s_{L+1}^{2}m\mathbb{E}\left[\mathcal{L}\left(\boldsymbol{\theta}_{t}\right)\right]+\eta\lambda_{1}s_{L+1}^{2}m\left(1+c_{\beta}\bar{\kappa}^{4}\beta\right)\exp\left(-\tilde{c_{1}}m\left(c_{\beta}\bar{\kappa}^{2}\beta\frac{n}{\lambda_{1}}\ln\frac{n}{\lambda_{1}}\right)^{-1}\right);\\
T_{2}\lesssim & \eta\beta^{-1}nds_{L+1}^{2}m\left(\frac{s_{L}^{2}}{a_{L}}\prod_{k=1}^{L-1}\frac{s_{k}^{2}}{a_{k}}m+\frac{1}{a_{L+1}}\sum_{l=1}^{L}\left(\prod_{k=1}^{l-1}\frac{s_{k}^{2}}{a_{k}}m\right)s_{l}^{2}\left(\prod_{k=l+1}^{L}\frac{s_{k}^{2}}{a_{k}}m\right)\right);\\
T_{3}\lesssim & \sqrt{\eta c_{\beta}\kappa^{2}\bar{\kappa}^{2}}\beta^{-1/2}nd\\
 & \cdot\sum_{l=1}^{L}\left(\prod_{t=0}^{l-2}\frac{s_{t+1}^{2}}{a_{t+1}}m\right)s_{l}^{2}\sum_{q=l}^{L}\left(\prod_{t=l}^{q-1}\frac{s_{t+1}^{2}}{a_{t+1}}m\right)\left(\prod_{t=q}^{L-1}\sqrt{\frac{s_{t+1}^{2}}{a_{t+1}}m}\right)\sqrt{\mathbb{E}\left[\mathcal{L}\left(\boldsymbol{\theta}_{t}\right)\right]};\\
T_{4} & \lesssim\beta^{-1}n\left(a_{L+1}+\sqrt{\eta c_{\beta}d\kappa^{2}\bar{\kappa}^{2}\frac{\beta}{m}}\sqrt{\left(\prod_{t=1}^{L}\frac{s_{t}^{2}}{a_{t}}m\right)\sum_{l=1}^{L}a_{l}^{2}}\right)\sqrt{\mathbb{E}\left[\mathcal{L}\left(\boldsymbol{\theta}_{t}\right)\right]}\\
 & -\beta^{-1}na_{L+1}\mathbb{E}\left[\mathcal{L}\left(\boldsymbol{\theta}_{t}\right)\right].
\end{align*}
Substituting these bounds, we get 
\begin{align}
 & \frac{\mathrm{d}}{\mathrm{d}t}\mathbb{E}\left[\mathcal{L}\left(\boldsymbol{\theta}_{t}\right)\right]\nonumber \\
 & \lesssim\frac{1}{n}\biggl(-\lambda_{1}s_{L+1}^{2}m\mathbb{E}\left[\mathcal{L}\left(\boldsymbol{\theta}_{t}\right)\right]\nonumber \\
 & \quad+\eta\lambda_{1}s_{L+1}^{2}m\left(1+c_{\beta}\bar{\kappa}^{4}\beta\right)\exp\left(-\tilde{c_{1}}m\left(c_{\beta}\bar{\kappa}^{2}\beta\frac{n}{\lambda_{1}}\ln\frac{n}{\lambda_{1}}\right)^{-1}\right)\nonumber \\
 & \quad+\eta\beta^{-1}nds_{L+1}^{2}m\left(\frac{s_{L}^{2}}{a_{L}}\prod_{k=1}^{L-1}\frac{s_{k}^{2}}{a_{k}}m+\frac{1}{a_{L+1}}\sum_{l=1}^{L}\left(\prod_{k=1}^{l-1}\frac{s_{k}^{2}}{a_{k}}m\right)s_{l}^{2}\left(\prod_{k=l+1}^{L}\frac{s_{k}^{2}}{a_{k}}m\right)\right)\nonumber \\
 & \quad+\sqrt{\eta c_{\beta}\kappa^{2}\bar{\kappa}^{2}}\beta^{-\frac{1}{2}}nd\sum_{l=1}^{L}\left(\prod_{t=0}^{l-2}\frac{s_{t+1}^{2}}{a_{t+1}}m\right)s_{l}^{2}\sum_{q=l}^{L}\left(\prod_{t=l}^{q-1}\frac{s_{t+1}^{2}}{a_{t+1}}m\right)\left(\prod_{t=q}^{L-1}\sqrt{\frac{s_{t+1}^{2}}{a_{t+1}}m}\right)\sqrt{\mathbb{E}\left[\mathcal{L}\left(\boldsymbol{\theta}_{t}\right)\right]}\label{eq:4term}\\
 & \quad+\beta^{-1}n\left(a_{L+1}+\sqrt{\eta c_{\beta}d\kappa^{2}\bar{\kappa}^{2}\frac{\beta}{m}}\sqrt{\left(\prod_{t=1}^{L}\frac{s_{t}^{2}}{a_{t}}m\right)\sum_{l=1}^{L}a_{l}^{2}}\right)\sqrt{\mathbb{E}\left[\mathcal{L}\left(\boldsymbol{\theta}_{t}\right)\right]}\nonumber \\
 & \quad-\beta^{-1}na_{L+1}\mathbb{E}\left[\mathcal{L}\left(\boldsymbol{\theta}_{t}\right)\right]\Bigg).\nonumber 
\end{align}
From now on we shall also assume $\frac{a_{l}}{s_{l}^{2}}=m$ for
$2\leq l\leq L$. We shall deal with sums over layers carefully, as
the first summand sometimes does not incur a factor of $\frac{s_{1}^{2}}{a_{1}}$
where the rest of the terms do (as a result of some products being
empty, and thus defined as 1). Specifically, the fourth term \cref{eq:4term} satisfies
\begin{align*}
 & \sqrt{\eta c_{\beta}\kappa^{2}\bar{\kappa}^{2}}\beta^{-1/2}nd\sum_{l=1}^{L}\underbrace{\left(\prod_{t=0}^{l-2}\frac{s_{t+1}^{2}}{a_{t+1}}m\right)}_{=1\textrm{ for }l=1}s_{l}^{2}\sum_{q=l}^{L}\left(\prod_{t=l}^{q-1}\frac{s_{t+1}^{2}}{a_{t+1}}m\right)\left(\prod_{t=q}^{L-1}\sqrt{\frac{s_{t+1}^{2}}{a_{t+1}}m}\right)\sqrt{\mathbb{E}\left[\mathcal{L}\left(\boldsymbol{\theta}_{t}\right)\right]}\\
 & =\sqrt{\eta c_{\beta}\kappa^{2}\bar{\kappa}^{2}}\beta^{-1/2}nd\left(s_{1}^{2}+\frac{s_{1}^{2}}{a_{1}}m\sum_{l=2}^{L}s_{l}^{2}\right)\sqrt{\mathbb{E}\left[\mathcal{L}\left(\boldsymbol{\theta}_{t}\right)\right]}.
\end{align*}
This allows us to simplify to
\begin{align*}
\lesssim\frac{1}{n}\biggl( & \overbrace{-\lambda_{1}s_{L+1}^{2}m\mathbb{E}\left[\mathcal{L}\left(\boldsymbol{\theta}_{t}\right)\right]}^{N_{1}}\\
 & +\eta\lambda_{1}s_{L+1}^{2}m\left(1+c_{\beta}\bar{\kappa}^{4}\beta\right)\exp\left(-\tilde{c_{1}}m\left(c_{\beta}\bar{\kappa}^{2}\beta\frac{n}{\lambda_{1}}\ln\frac{n}{\lambda_{1}}\right)^{-1}\right)\\
 & +\eta\beta^{-1}nds_{L+1}^{2}m\left(\frac{s_{1}^{2}}{a_{1}}+\frac{1}{a_{L+1}}s_{1}^{2}+\frac{s_{1}^{2}}{a_{1}}m\frac{1}{a_{L+1}}\sum_{l=2}^{L}s_{l}^{2}\right)\\
 & +\sqrt{\eta c_{\beta}\kappa^{2}\bar{\kappa}^{2}}\beta^{-1/2}nd\left(s_{1}^{2}+\frac{s_{1}^{2}}{a_{1}}m\sum_{l=2}^{L}s_{l}^{2}\right)\sqrt{\mathbb{E}\left[\mathcal{L}\left(\boldsymbol{\theta}_{t}\right)\right]}\\
 & +\beta^{-1}n\left(a_{L+1}+\sqrt{\eta c_{\beta}d\kappa^{2}\bar{\kappa}^{2}\frac{\beta}{m}}\sqrt{\frac{s_{1}^{2}}{a_{1}}m\sum_{l=1}^{L}a_{l}^{2}}\right)\sqrt{\mathbb{E}\left[\mathcal{L}\left(\boldsymbol{\theta}_{t}\right)\right]}\\
 & \underbrace{-\beta^{-1}na_{L+1}\mathbb{E}\left[\mathcal{L}\left(\boldsymbol{\theta}_{t}\right)\right]}_{N_{2}}\Bigg).
\end{align*}
The negative regularization contribution $N_{2}$ is not useful for
closing the small-loss differential inequality. Observe that the term
above $N_{2}$ includes $\beta^{-1}na_{L+1}\sqrt{\mathbb{E}\left[\mathcal{L}(\boldsymbol{\theta}_{t})\right]}$.
For $u=\mathbb{E}\left[\mathcal{L}(\boldsymbol{\theta}_{t})\right]\leq1$,
the sum
\begin{align*}
-\beta^{-1}na_{L+1}u+\beta^{-1}na_{L+1}\sqrt{u}
\end{align*}

is nonnegative. We therefore close the inequality using only the NTK
descent term $N_{1}$, while treating all $\sqrt{u}$-terms as error
terms to be absorbed by Young's inequality. Discarding $N_{2}$ gives
\begin{align*}
\cdots\lesssim\frac{1}{n}\Bigg[ & -\lambda_{1}s_{L+1}^{2}m\mathbb{E}\left[\mathcal{L}\left(\boldsymbol{\theta}_{t}\right)\right]\\
 & +\beta^{-1}n\left(\sqrt{\eta c_{\beta}\kappa^{2}\bar{\kappa}^{2}}\beta^{1/2}d\left(s_{1}^{2}+\frac{s_{1}^{2}}{a_{1}}m\sum_{l=2}^{L}s_{l}^{2}\right)+a_{L+1}+\sqrt{\eta c_{\beta}\kappa^{2}\bar{\kappa}^{2}\beta\frac{s_{1}^{2}}{a_{1}}d\sum_{l=1}^{L}a_{l}^{2}}\right)\\
 & \quad\cdot\sqrt{\mathbb{E}\left[\mathcal{L}\left(\boldsymbol{\theta}_{t}\right)\right]}\\
 & +\eta\beta^{-1}nds_{L+1}^{2}m\left(\frac{s_{1}^{2}}{a_{1}}+\frac{1}{a_{L+1}}s_{1}^{2}+\frac{s_{1}^{2}}{a_{1}}m\frac{1}{a_{L+1}}\sum_{l=2}^{L}s_{l}^{2}\right)\\
 & +\lambda_{1}\eta s_{L+1}^{2}m\left(1+c_{\beta}\bar{\kappa}^{4}\beta\right)\exp\left(-\tilde{c_{1}}m\left(c_{\beta}\bar{\kappa}^{2}\beta\frac{n}{\lambda_{1}}\ln\frac{n}{\lambda_{1}}\right)^{-1}\right)\Bigg].
\end{align*}
Using Young's inequality $ab\leq\frac{a^{2}}{2\varepsilon}+\frac{\varepsilon}{2}b^{2}$
with $\varepsilon=\lambda_{1}s_{L+1}^{2}m$, $b=\sqrt{\mathbb{E}\left[\mathcal{L}\left(\boldsymbol{\theta}_{t}\right)\right]}$
and $a$ being the coefficient of $\sqrt{\mathbb{E}\left[\mathcal{L}\left(\boldsymbol{\theta}_{t}\right)\right]}$,
we get
\begin{align*}
\cdots\lesssim\frac{1}{n}\Bigg[ & -\frac{\lambda_{1}s_{L+1}^{2}m}{2}\mathbb{E}\left[\mathcal{L}\left(\boldsymbol{\theta}_{t}\right)\right]\\
 & +\frac{\beta^{-2}n^{2}}{2\lambda_{1}s_{L+1}^{2}m}\left(\sqrt{\eta c_{\beta}\kappa^{2}\bar{\kappa}^{2}}\beta^{1/2}d\left(s_{1}^{2}+\frac{s_{1}^{2}}{a_{1}}m\sum_{l=2}^{L}s_{l}^{2}\right)\right.\\
 & \qquad\qquad\qquad\qquad\qquad\qquad\left.+a_{L+1}+\sqrt{\eta c_{\beta}\kappa^{2}\bar{\kappa}^{2}\beta\frac{s_{1}^{2}}{a_{1}}d\sum_{l=1}^{L}a_{l}^{2}}\right)^{2}\\
 & +\eta\beta^{-1}nds_{L+1}^{2}m\left(\frac{s_{1}^{2}}{a_{1}}+\frac{1}{a_{L+1}}s_{1}^{2}+\frac{s_{1}^{2}}{a_{1}}m\frac{1}{a_{L+1}}\sum_{l=2}^{L}s_{l}^{2}\right)\\
 & +\lambda_{1}\eta s_{L+1}^{2}m\left(1+c_{\beta}\bar{\kappa}^{4}\beta\right)\exp\left(-\tilde{c_{1}}m\left(c_{\beta}\bar{\kappa}^{2}\beta\frac{n}{\lambda_{1}}\ln\frac{n}{\lambda_{1}}\right)^{-1}\right)\Bigg]\\
\lesssim\frac{1}{n}\Bigg[ & -\lambda_{1}s_{L+1}^{2}m\mathbb{E}\left[\mathcal{L}\left(\boldsymbol{\theta}_{t}\right)\right]\\
 & +\frac{\beta^{-2}n^{2}}{2\lambda_{1}s_{L+1}^{2}m}\left(c_{\beta}\kappa^{2}\bar{\kappa}^{2}\beta d^{2}\left(s_{1}^{4}+\left(\frac{s_{1}^{2}}{a_{1}}m\right)^{2}\left(\sum_{l=2}^{L}s_{l}^{2}\right)^{2}\right)\right.\\
 & \qquad\qquad\qquad\qquad\qquad\qquad\left.+a_{L+1}^{2}+c_{\beta}\kappa^{2}\bar{\kappa}^{2}\beta\frac{s_{1}^{2}}{a_{1}}d\sum_{l=1}^{L}a_{l}^{2}\right)\\
 & +\beta^{-1}ns_{L+1}^{2}m\left(\frac{s_{1}^{2}}{a_{1}}d+\frac{1}{a_{L+1}}s_{1}^{2}d+\frac{s_{1}^{2}}{a_{1}}dm\frac{1}{a_{L+1}}\sum_{l=2}^{L}s_{l}^{2}\right)\\
 & +\lambda_{1}s_{L+1}^{2}m\left(1+c_{\beta}\bar{\kappa}^{4}\beta\right)\exp\left(-\tilde{c_{1}}m\left(c_{\beta}\bar{\kappa}^{2}\beta\frac{n}{\lambda_{1}}\ln\frac{n}{\lambda_{1}}\right)^{-1}\right)\Bigg].
\end{align*}
Then, from \cref{claim:ode_inequality} we get
\begin{align*}
\mathbb{E}\left[\mathcal{L}\left(\boldsymbol{\theta}_{t}\right)\right] & \leq\left(\mathbb{E}\left[\mathcal{L}\left(\boldsymbol{\theta}_{0}\right)\right]-\mathcal{L}_{\infty}\right)\exp\left(-\frac{1}{2}\lambda_{1}s_{L+1}^{2}\frac{m}{n}t\right)+\mathcal{L}_{\infty},
\end{align*}
where
\begin{align*}
\mathcal{L}_{\infty}\asymp\frac{1}{2\lambda_{1}s_{L+1}^{2}m}\Bigg[ & \frac{\beta^{-2}n^{2}}{2\lambda_{1}s_{L+1}^{2}m}\left(c_{\beta}\kappa^{2}\bar{\kappa}^{2}\beta d^{2}\left(s_{1}^{4}+\left(\frac{s_{1}^{2}}{a_{1}}m\right)^{2}\left(\sum_{l=2}^{L}s_{l}^{2}\right)^{2}\right)\right.\\
 & \qquad\qquad\qquad\qquad\qquad\qquad\left.+a_{L+1}^{2}+c_{\beta}\kappa^{2}\bar{\kappa}^{2}\beta\frac{s_{1}^{2}}{a_{1}}d\sum_{l=1}^{L}a_{l}^{2}\right)\\
 & +\beta^{-1}ns_{L+1}^{2}m\left(\frac{s_{1}^{2}}{a_{1}}d+\frac{1}{a_{L+1}}s_{1}^{2}d+\frac{s_{1}^{2}}{a_{1}}dm\frac{1}{a_{L+1}}\sum_{l=2}^{L}s_{l}^{2}\right)\\
 & +\lambda_{1}s_{L+1}^{2}m\left(1+c_{\beta}\bar{\kappa}^{4}\beta\right)\exp\left(-\tilde{c_{1}}m\left(c_{\beta}\bar{\kappa}^{2}\beta\frac{n}{\lambda_{1}}\ln\frac{n}{\lambda_{1}}\right)^{-1}\right)\Bigg].
\end{align*}
Opening the parenthesis, we get
\begin{align*}
\mathcal{L}_{\infty}\asymp & \left(\frac{\beta^{-1}n}{\lambda_{1}s_{L+1}^{2}m}\right)^{2}\left(c_{\beta}\kappa^{2}\bar{\kappa}^{2}\beta d^{2}\left(s_{1}^{4}+\left(\frac{s_{1}^{2}}{a_{1}}m\right)^{2}\left(\sum_{l=2}^{L}s_{l}^{2}\right)^{2}\right)\right.\\
 & \qquad\qquad\qquad\qquad\qquad\qquad\left.+a_{L+1}^{2}+c_{\beta}\kappa^{2}\bar{\kappa}^{2}\beta\frac{s_{1}^{2}}{a_{1}}d\sum_{l=1}^{L}a_{l}^{2}\right)\\
 & +\frac{\beta^{-1}n}{\lambda_{1}}\left(\frac{s_{1}^{2}}{a_{1}}d+\frac{1}{a_{L+1}}s_{1}^{2}d+\frac{s_{1}^{2}}{a_{1}}dm\frac{1}{a_{L+1}}\sum_{l=2}^{L}s_{l}^{2}\right)\\
 & +\left(1+c_{\beta}\bar{\kappa}^{4}\beta\right)\exp\left(-\tilde{c_{1}}m\left(c_{\beta}\bar{\kappa}^{2}\beta\frac{n}{\lambda_{1}}\ln\frac{n}{\lambda_{1}}\right)^{-1}\right).
\end{align*}
Splitting the first term into two, we get 
\begin{align*}
\mathcal{L}_{\infty}\asymp & \left(\frac{\beta^{-1}na_{L+1}}{\lambda_{1}s_{L+1}^{2}m}\right)^{2}\\
 & +\left(\frac{\beta^{-1/2}n}{\lambda_{1}s_{L+1}^{2}m}\right)^{2}c_{\beta}\kappa^{2}\bar{\kappa}^{2}\left(d^{2}\left(s_{1}^{4}+\left(\frac{s_{1}^{2}}{a_{1}}m\right)^{2}\left(\sum_{l=2}^{L}s_{l}^{2}\right)^{2}\right)+\frac{s_{1}^{2}}{a_{1}}d\sum_{l=1}^{L}a_{l}^{2}\right)\\
 & +\frac{\beta^{-1}n}{\lambda_{1}}\left(\frac{s_{1}^{2}}{a_{1}}d+\frac{1}{a_{L+1}}s_{1}^{2}d+\frac{s_{1}^{2}}{a_{1}}dm\frac{1}{a_{L+1}}\sum_{l=2}^{L}s_{l}^{2}\right)\\
 & +\left(1+c_{\beta}\bar{\kappa}^{4}\beta\right)\exp\left(-\tilde{c_{1}}m\left(c_{\beta}\bar{\kappa}^{2}\beta\frac{n}{\lambda_{1}}\ln\frac{n}{\lambda_{1}}\right)^{-1}\right).
\end{align*}

\subsubsection{An equivalent formulation of the loss floor}

Let us rewrite the loss floor by using more of our assumptions, substituting
$a_{1},...,a_{L+1}$ with equivalent expressions, and expressing things
in terms of $\kappa$ whenever possible. The first term is equal to
$\left(\frac{\beta^{-1}n}{\lambda_{1}\kappa^{2}}\right)^{2}$. The
second term satisfies
\begin{align*}
 & \left(\frac{\beta^{-1/2}n}{\lambda_{1}s_{L+1}^{2}m}\right)^{2}c_{\beta}\kappa^{2}\bar{\kappa}^{2}\left(d^{2}\left(s_{1}^{4}+\left(\frac{s_{1}^{2}}{a_{1}}m\right)^{2}\left(\sum_{l=2}^{L}s_{l}^{2}\right)^{2}\right)+\frac{s_{1}^{2}}{a_{1}}d\sum_{l=1}^{L}a_{l}^{2}\right)\\
 & =\left(\frac{\beta^{-1/2}n}{\lambda_{1}}\right)^{2}c_{\beta}\kappa^{2}\bar{\kappa}^{2}\left(\frac{d^{2}}{s_{L+1}^{4}m^{2}}\left(s_{1}^{4}+\left(\frac{s_{1}^{2}}{a_{1}}m\right)^{2}\left(\sum_{l=2}^{L}s_{l}^{2}\right)^{2}\right)+\frac{s_{1}^{2}}{a_{1}}\frac{d}{s_{L+1}^{4}m^{2}}\sum_{l=1}^{L}a_{l}^{2}\right)\\
 & \overset{(i)}{=}\left(\frac{\beta^{-1/2}n}{\lambda_{1}}\right)^{2}c_{\beta}\kappa^{2}\bar{\kappa}^{2}\left(\frac{d^{2}}{s_{L+1}^{4}m^{2}}\left(s_{1}^{4}+\left(\frac{m}{d}\right)^{2}\left(\sum_{l=2}^{L}s_{l}^{2}\right)^{2}\right)+\frac{1}{s_{L+1}^{4}m^{2}}\sum_{l=1}^{L}a_{l}^{2}\right)\\
 & =\left(\frac{\beta^{-1/2}n}{\lambda_{1}}\right)^{2}c_{\beta}\kappa^{2}\bar{\kappa}^{2}\left(\frac{s_{1}^{4}}{s_{L+1}^{4}}\left(\frac{d}{m}\right)^{2}+\frac{1}{s_{L+1}^{4}}\left(\sum_{l=2}^{L}s_{l}^{2}\right)^{2}+\frac{1}{s_{L+1}^{4}m^{2}}\sum_{l=1}^{L}a_{l}^{2}\right)\\
 & \overset{(ii)}{=}\left(\kappa\bar{\kappa}\frac{\beta^{-1/2}n}{\lambda_{1}}\right)^{2}c_{\beta}\left(\frac{s_{1}^{4}}{s_{L+1}^{4}}\left(\frac{d}{m}\right)^{2}+\frac{1}{s_{L+1}^{4}}\left(\sum_{l=2}^{L}s_{l}^{2}\right)^{2}+\frac{s_{1}^{4}}{s_{L+1}^{4}}\left(\frac{d}{m}\right)^{2}+\frac{1}{s_{L+1}^{4}}\sum_{l=2}^{L}s_{l}^{4}\right)\\
 & \asymp\left(\kappa\bar{\kappa}\frac{\beta^{-1/2}n}{\lambda_{1}}\right)^{2}c_{\beta}\left(\frac{s_{1}^{4}}{s_{L+1}^{4}}\left(\frac{d}{m}\right)^{2}+\frac{1}{s_{L+1}^{4}}\left(\sum_{l=2}^{L}s_{l}^{2}\right)^{2}+\frac{1}{s_{L+1}^{4}}\sum_{l=2}^{L}s_{l}^{4}\right),
\end{align*}
where in $(i)$ we use $\frac{a_{1}}{s_{1}^{2}}=m_{0}=d$ and in $(ii)$
we use $a_{l}=s_{l}^{2}m_{l-1}$ for all $l\in\left[L\right]$. The
third term satisfies 
\begin{align*}
 & \frac{\beta^{-1}n}{\lambda_{1}}\left(\frac{s_{1}^{2}}{a_{1}}d+\frac{1}{a_{L+1}}s_{1}^{2}d+\frac{s_{1}^{2}}{a_{1}}dm\frac{1}{a_{L+1}}\sum_{l=2}^{L}s_{l}^{2}\right)\\
 & \overset{(i)}{=}\frac{\beta^{-1}n}{\lambda_{1}}\left(1+\frac{1}{a_{L+1}}s_{1}^{2}d+m\frac{1}{a_{L+1}}\sum_{l=2}^{L}s_{l}^{2}\right)\overset{(ii)}{=}\frac{\beta^{-1}n}{\lambda_{1}}\left(1+\kappa^{2}\frac{s_{1}^{2}}{s_{L+1}^{2}}\frac{d}{m}+\frac{\kappa^{2}}{s_{L+1}^{2}}\sum_{l=2}^{L}s_{l}^{2}\right),
\end{align*}
where $(i)$ is as in the previous equation and $(ii)$ uses $a_{L+1}=\frac{s_{L+1}^{2}m}{\kappa^{2}}$.
Overall, we get 
\begin{align*}
\mathcal{L}_{\infty}\asymp & \left(\frac{\beta^{-1}n}{\lambda_{1}\kappa^{2}}\right)^{2}\\
 & +\left(\kappa\bar{\kappa}\frac{\beta^{-1/2}n}{\lambda_{1}}\right)^{2}c_{\beta}\left(\frac{s_{1}^{4}}{s_{L+1}^{4}}\left(\frac{d}{m}\right)^{2}+\frac{1}{s_{L+1}^{4}}\left(\sum_{l=2}^{L}s_{l}^{2}\right)^{2}+\frac{1}{s_{L+1}^{4}}\sum_{l=2}^{L}s_{l}^{4}\right)\\
 & +\frac{\beta^{-1}n}{\lambda_{1}}\left(1+\kappa^{2}\frac{s_{1}^{2}}{s_{L+1}^{2}}\frac{d}{m}+\frac{\kappa^{2}}{s_{L+1}^{2}}\sum_{l=2}^{L}s_{l}^{2}\right)\\
 & +\left(1+c_{\beta}\bar{\kappa}^{4}\beta\right)\exp\left(-\tilde{c_{1}}m\left(c_{\beta}\bar{\kappa}^{2}\beta\frac{n}{\lambda_{1}}\ln\frac{n}{\lambda_{1}}\right)^{-1}\right).
\end{align*}
Finally, we observe that in order to achieve $\mathcal{L}_{\infty}<1$,
the third term here requires $\beta>\frac{n}{\lambda_{1}}$, establishing
that $\bar{\kappa}^{2}\beta\geq\beta\gg1$, and then $c_{\beta}=2cM_{\phi}^{2}+\frac{1}{\bar{\kappa}^{2}\beta}=\Theta\left(1\right)$,
allowing us to absorb this constant and yield 
\begin{align*}
\mathcal{L}_{\infty}\asymp & \left(\frac{\beta^{-1}n}{\lambda_{1}\kappa^{2}}\right)^{2}\\
 & +\left(\kappa\bar{\kappa}\frac{\beta^{-1/2}n}{\lambda_{1}}\right)^{2}\left(\frac{s_{1}^{4}}{s_{L+1}^{4}}\left(\frac{d}{m}\right)^{2}+\frac{1}{s_{L+1}^{4}}\left(\sum_{l=2}^{L}s_{l}^{2}\right)^{2}+\frac{1}{s_{L+1}^{4}}\sum_{l=2}^{L}s_{l}^{4}\right)\\
 & +\frac{\beta^{-1}n}{\lambda_{1}}\left(1+\kappa^{2}\frac{s_{1}^{2}}{s_{L+1}^{2}}\frac{d}{m}+\frac{\kappa^{2}}{s_{L+1}^{2}}\sum_{l=2}^{L}s_{l}^{2}\right)\\
 & +\left(1+\bar{\kappa}^{4}\beta\right)\exp\left(-\tilde{c_{1}}m\left(\bar{\kappa}^{2}\beta\frac{n}{\lambda_{1}}\ln\frac{n}{\lambda_{1}}\right)^{-1}\right).
\end{align*}
Rearranging this, we get
\begin{align*}
\mathcal{L}_{\infty}\asymp & \left(1+\bar{\kappa}^{4}\beta\right)\exp\left(-\tilde{c_{1}}m\left(\bar{\kappa}^{2}\beta\frac{n}{\lambda_{1}}\ln\frac{n}{\lambda_{1}}\right)^{-1}\right)\\
 & +\frac{\beta^{-1}n}{\lambda_{1}}\left(1+\kappa^{2}\frac{s_{1}^{2}}{s_{L+1}^{2}}\frac{d}{m}+\frac{\kappa^{2}}{s_{L+1}^{2}}\sum_{l=2}^{L}s_{l}^{2}\right)\\
 & +\left(\frac{\beta^{-1/2}n}{\lambda_{1}}\right)^{2}\kappa^{2}\bar{\kappa}^{2}\left(\frac{s_{1}^{4}}{s_{L+1}^{4}}\left(\frac{d}{m}\right)^{2}+\frac{1}{s_{L+1}^{4}}\left(\sum_{l=2}^{L}s_{l}^{2}\right)^{2}+\frac{1}{s_{L+1}^{4}}\sum_{l=2}^{L}s_{l}^{4}\right)\\
 & +\left(\frac{\beta^{-1}n}{\lambda_{1}}\right)^{2}\left(\frac{1}{\kappa^{4}}\right).
\end{align*}

\subsection{Standard parameterization}

Under standard parameterization ($\forall l\in\left[L+1\right]:s_{l}=1$,
$a_{1}=d$, $\forall l>1:a_{l}=m$, $\kappa=1$), we have 
\begin{align*}
\mathcal{L}_{\infty}\asymp & \left(\frac{\beta^{-1}n}{\lambda_{1}}\right)^{2}+\left(\frac{\beta^{-1/2}n}{\lambda_{1}}\right)^{2}\left(\left(\frac{d}{m}\right)^{2}+L^{2}+L\right)\\
 & +\frac{\beta^{-1}n}{\lambda_{1}}\left(1+\frac{d}{m}+L\right)+\left(1+\beta\right)\exp\left(-\tilde{c_{1}}m\left(\beta\frac{n}{\lambda_{1}}\ln\frac{n}{\lambda_{1}}\right)^{-1}\right)\\
\overset{(i)}{\asymp} & \left(\frac{\beta^{-1}n}{\lambda_{1}}\right)^{2}+\left(\frac{\beta^{-1/2}n}{\lambda_{1}}\right)^{2}+\frac{\beta^{-1}n}{\lambda_{1}}+\left(1+\beta\right)\exp\left(-\tilde{c_{1}}m\left(\beta\frac{n}{\lambda_{1}}\ln\frac{n}{\lambda_{1}}\right)^{-1}\right)\\
\overset{(ii)}{\asymp} & \left(\beta^{-1}n\right)^{2}+\left(\beta^{-1/2}n\right)^{2}+\beta^{-1}n+\left(1+\beta\right)\exp\left(-\tilde{c_{1}}\frac{m}{\beta n\ln n}\right),
\end{align*}
where in $(i)$ we use $d\leq m$ and in $(ii)$ we treat $\lambda_{1}$
as $\Theta\left(1\right)$. Demands to achieve $\mathcal{L}_{\infty}\leq\mathcal{L}_{g}$
for some $\mathcal{L}_{g}\leq1$ are then 
\begin{enumerate}
\item $m=\Omega\left(\beta n\ln n\ln\left(m\slash\mathcal{L}_{g}\right)\right)$.
\item $\beta=\Omega\left(n\slash\mathcal{L}_{g}\right)$.
\item $\beta=\Omega\left(n^{2}\slash\mathcal{L}_{g}\right)$.
\end{enumerate}
This can be achieved with $m=\tilde{\Omega}\left(n^{3}\slash\mathcal{L}_{g}\right)$
and $\beta=\tilde{\Omega}\left(n^{2}\slash\mathcal{L}_{g}\right)$.

\textbf{Comparing the results to the case of shallow networks.} For
shallow networks, the analysis establishes convergence to a small
loss floor under milder conditions. The stronger requirements in the
deep-network case come from sums over layers $l=2,\ldots,L$: for
shallow networks they are 0, yet for deep networks they result in
new, dominant terms.

\subsection{NTK parameterization}

Under NTK parameterization ($s_{1}=d^{-1/2}$, $\forall l>1:s_{l}=m^{-1/2}$,
$\forall l\in\left[L+1\right]:a_{l}=1$, $\kappa=1$), we have 
\begin{align*}
\mathcal{L}_{\infty}\asymp & \left(\frac{\beta^{-1}n}{\lambda_{1}}\right)^{2}\\
 & +\left(\frac{\beta^{-1/2}n}{\lambda_{1}}\right)^{2}\left(\frac{m^{2}}{d^{2}}\left(\frac{d}{m}\right)^{2}+m^{2}\left(\frac{L}{m}\right)^{2}+m^{2}\frac{L}{m^{2}}\right)\\
 & +\frac{\beta^{-1}n}{\lambda_{1}}\left(1+\frac{m}{d}\frac{d}{m}+m\frac{L}{m}\right)\\
 & +\left(1+\beta\right)\exp\left(-\tilde{c_{1}}m\left(\beta\frac{n}{\lambda_{1}}\ln\frac{n}{\lambda_{1}}\right)^{-1}\right)\\
\asymp & \left(\frac{\beta^{-1}n}{\lambda_{1}}\right)^{2}+\left(\frac{\beta^{-1/2}n}{\lambda_{1}}\right)^{2}\left(1+L^{2}+L\right)+\frac{\beta^{-1}n}{\lambda_{1}}\left(1+1+L\right)\\
 & +\left(1+\beta\right)\exp\left(-\tilde{c_{1}}m\left(\beta\frac{n}{\lambda_{1}}\ln\frac{n}{\lambda_{1}}\right)^{-1}\right)\\
\asymp & \left(\beta^{-1}n\right)^{2}+\left(\beta^{-1/2}n\right)^{2}+\beta^{-1}n+\left(1+\beta\right)\exp\left(-\tilde{c_{1}}\frac{m}{\beta n\ln n}\right)
\end{align*}
where in the last step we treat $\lambda_{1},L$ as $\Theta\left(1\right)$.
Demands to achieve $\mathcal{L}_{\infty}\leq\mathcal{L}_{g}$ for
some $\mathcal{L}_{g}\leq1$ are then the same as under standard parameterization.

\newpage

\section{Convergence with gradient flow from stopped CLD}
\label{app:convergence_with_gf_from_stopped_cld}
This appendix proves \cref{prop:convergence_to_zero_loss_shallow_stopped_cld},
the two-stage scheduling guarantee: regularized CLD is run until a deterministic
time $\tau$, at which both the noise and the regularization are switched off,
and gradient flow from that point on drives the training loss of a shallow
network to zero.
\begin{recall}[\cref{prop:convergence_to_zero_loss_shallow_stopped_cld}][Convergence to zero loss for shallow networks]
For shallow networks, \ie $L=1$, suppose that CLD is stopped at a deterministic
time $\tau=\tilde{\Omega}\left(\frac{n}{s_{2}^{2}\lambda_{1}m}\right)$, and that
training then proceeds by gradient flow, as in \cref{eq:gradient_flow}. Let
$K>2$. Then, the network converges to zero loss with probability at least
\mbox{$
1-\eta\exp\left(-c_{4}\frac{m}{c_{\beta}\bar\kappa^2\beta\frac{n}{\lambda_{1}}\ln\left(\frac{n}{\lambda_{1}}\right)}\right)-\frac{2}{K}
$}, provided that parameters are chosen such that
\[
\mathcal{L}_{\infty}^{\prn{L=1}}\lesssim\frac{s_{2}^{2}m^{2}}{Kn^{2}d^{2}s_{1}^{4}\left(\frac{1}{ds_{1}^{2}}+\frac{\kappa^{2}}{ms_{2}^{2}}\right)\left(\frac{\kappa^{2}s_{1}^{2}d}{s_{2}^{2}}+\sqrt{nm}\right)}\,,
\]
as in \cref{eq:L_infty_for_song_convergence}.
\end{recall}
The recall above is not verbatim. The original statement of
\cref{prop:convergence_to_zero_loss_shallow_stopped_cld} closes with a worked
example, which instantiates this demand under the standard parameterization and
concludes that, with $d=o\prn{\sqrt{n}}$ and $K=\Theta\prn{\sqrt{n}/d}$, the
loss vanishes with probability $1-o\prn{1}$ at $\beta=\tilde{\Omega}\prn{n}$ and
$m=\tilde{\Omega}\prn{n^{2}}$. That example is a consequence of the proof rather
than a part of it, and so is stated here without it; it is derived in
\cref{sec:song_result_under_parameterizations}, immediately after the proof.

The proof combines a deterministic ingredient with a probabilistic one. The
deterministic ingredient, developed in \cref{sec:gf_zero_loss_criterion}, is a
criterion under which gradient flow started at a point
$\boldsymbol{\theta}_{0}$ converges to zero loss: the Jacobian there must be a
near-isometry, the network must be smooth on a ball around that point, and the
loss there must be small relative to those two quantities. The probabilistic
ingredient is our change-of-measure inequality
(\cref{thm:second_law_change_of_measure}), which transfers high-probability
properties of the initialization to \emph{every} time along the CLD trajectory,
together with the expected-loss bound of
\cref{thm:upper_bound_on_expected_loss}. \Cref{sec:stopped_cld_init_properties}
collects the initialization estimates that this transfer acts upon, and
\cref{sec:stopped_cld_zero_loss_proof} carries out the transfer, verifies the
criterion at $\vtheta_{\mathrm{CLD}}\left(\tau\right)$, and closes the proof.
\Cref{sec:song_result_under_parameterizations} then works out the resulting
demand on the loss floor under the standard and NTK parameterizations.
\subsection{Gradient flow convergence by proximity to good initialization}
\label{sec:gf_zero_loss_criterion}
Our analysis here closely follows that of \citet{song2021subquadratic}, with adaptation
to our notations and settings. Importantly, we directly analyze without
the bounded second layer assumption (which is part of their Assumption
2), similar to the (somewhat informal) treatment in their Appendix E.

We begin by citing some useful definitions and Lemmas from \citet{song2021subquadratic}, adapted to our notations. Importantly, we restrict the smoothness
definition to a specific convex set.
\begin{definition}[Smoothness]
A function $f:\mathbb{R}^{N_{p}}\to\mathbb{R}^{n}$ is $\beta_{f}$-smooth
in a convex set $\Theta\subseteq\mathbb{R}^{N_{p}}$ if for some $\beta_{f}\geq0$
and for all $\boldsymbol{\theta}_{1},\boldsymbol{\theta}_{2}\in\Theta$,
\[
\left\Vert \nabla f\left(\boldsymbol{\theta}_{1}\right)-\nabla f\left(\boldsymbol{\theta}_{2}\right)\right\Vert _{2}\leq\beta_{f}\left\Vert \boldsymbol{\theta}_{1}-\boldsymbol{\theta}_{2}\right\Vert \,,
\]
where $\nabla f\left(\boldsymbol{\theta}\right)\in\mathbb{R}^{n\times N_{p}}$
is the Jacobian of $f$, evaluated at $\boldsymbol{\theta}$.
\end{definition}
\begin{definition}[Near-isometry]
A matrix $\mathbf{A}$ is $\left(\mu,\nu\right)$-near-isometry if
there exist $0<\mu\leq\nu$ such that
\[
\mu\leq\sigma_{\min}\left(\mathbf{A}\right)\leq\sigma_{\max}\left(\mathbf{A}\right)\leq\nu\,,
\]
where $\sigma_{\min}\left(\mathbf{A}\right),\,\sigma_{\max}\left(\mathbf{A}\right)$
are the minimum and maximum singular values of $\mathbf{A}$, respectively.
\end{definition}
We now proceed to cite a slightly altered version of their Lemma 1,
assuming smoothness only in the relevant ball. The proof is identical
to theirs, and we do not repeat it here.
\begin{lemma}
\label{lem:near-isometry-maintained}Given some $\boldsymbol{\theta}_{0}\in\mathbb{R}^{N_{p}}$
and $f:\mathbb{R}^{N_{p}}\to\mathbb{R}^{n}$ which is $\beta_{f}$-smooth
in $\Theta=\left\{ \boldsymbol{\theta}:\,\left\Vert \boldsymbol{\theta}-\boldsymbol{\theta}_{0}\right\Vert \leq\frac{\mu}{2\beta_{f}}\right\} $,
and assuming $\nabla f\left(\boldsymbol{\theta}_{0}\right)$ is $\left(\mu,\nu\right)$-near-isometry,
then for all $\boldsymbol{\theta}\in\Theta$, $\nabla f\left(\boldsymbol{\theta}\right)$
is $\left(\frac{1}{2}\mu,\frac{3}{2}\nu\right)$-near-isometry.
\end{lemma}
\begin{definition}[Polyak Łojasiewicz (PL) condition]
A non-negative function $\mathcal{F}:\mathbb{R}^{n}\to\mathbb{R}$
satisfies the PL condition with $\alpha_{\mathcal{F}}>0$ if for all
$\mathbf{z}\in\mathbb{R}^{n}$,
\[
\mathcal{F}\left(\mathbf{z}\right)\leq\frac{\left\Vert \nabla\mathcal{F}\left(\mathbf{z}\right)\right\Vert ^{2}}{2\alpha_{\mathcal{F}}}\,.
\]
\end{definition}
We now proceed by analyzing gradient flow of $\boldsymbol{\theta}\in\mathbb{R}^{N_{p}}$
on the loss function $\mathcal{L}:\mathbb{R}^{N_{p}}\to\mathbb{R}$
defined as $\mathcal{L}\left(\boldsymbol{\theta}\right)\triangleq\mathcal{F}\left(f\left(\boldsymbol{\theta}\right)\right)$,
with $f:\mathbb{R}^{N_{p}}\to\mathbb{R}^{n}$, $\mathcal{F}:\mathbb{R}^{n}\to\mathbb{R}$:
\begin{equation}
\frac{\mathrm{d}}{\mathrm{d}t}\boldsymbol{\theta}\left(t\right)=-\nabla\mathcal{L}\left(\boldsymbol{\theta}\left(t\right)\right)\,.\label{eq:gradient_flow_appendix}
\end{equation}
We cite \citet[Lemma 2]{song2021subquadratic}, adapted to our notations. The
proof remains identical, hence we do not repeat it.
\begin{lemma}
\label{lem:trajectory-length-bound}Let $\mathcal{L}\left(\boldsymbol{\theta}\right)\triangleq\mathcal{F}\left(f\left(\boldsymbol{\theta}\right)\right)$
such that $\mathcal{F}\geq0$ satisfies the PL condition with $\alpha_{\mathcal{F}}>0$,
$f$ is $\beta_{f}$-smooth in $\Theta=\left\{ \boldsymbol{\theta}:\,\left\Vert \boldsymbol{\theta}-\boldsymbol{\theta}_{0}\right\Vert \leq\frac{\mu}{2\beta_{f}}\right\} $,
and $\nabla f\left(\boldsymbol{\theta}_{0}\right)$ is $\left(\mu,\nu\right)$-near-isometry.
Given the parameters $\boldsymbol{\theta}\left(t\right)$ evolve by
gradient flow (\cref{eq:gradient_flow_appendix}) initialized by $\boldsymbol{\theta}\left(0\right)=\boldsymbol{\theta}_{0}$,
we denote $\tau_{\mathrm{b}}=\inf\left\{ t>0:\,\left\Vert \boldsymbol{\theta}\left(t\right)-\boldsymbol{\theta}_{0}\right\Vert >\frac{\mu}{2\beta_{f}}\right\} $
. Then, for all $t<\tau_{\mathrm{b}}$, the length of the trajectory
followed by gradient flow is upper bounded by
\[
c\frac{\nu\sqrt{\mathcal{L}\left(\boldsymbol{\theta}_{0}\right)}}{\mu^{2}\sqrt{\alpha_{\mathcal{F}}}}\,,
\]
where $c>0$ is a global constant.
\end{lemma}
We now extend \citet[Theorem 1]{song2021subquadratic}, which proves convergence
to a global minimum, by stating the convergence rate.
\begin{theorem}
\label{thm:song-convergence}Let $\mathcal{L}\left(\boldsymbol{\theta}\right)\triangleq\mathcal{F}\left(f\left(\boldsymbol{\theta}\right)\right)$
such that $\mathcal{F}\geq0$ satisfies the PL condition with $\alpha_{\mathcal{F}}>0$,
$f$ is $\beta_{f}$-smooth in $\Theta=\left\{ \boldsymbol{\theta}:\,\left\Vert \boldsymbol{\theta}-\boldsymbol{\theta}_{0}\right\Vert \leq\frac{\mu}{2\beta_{f}}\right\} $,
$\nabla f\left(\boldsymbol{\theta}_{0}\right)$ is $\left(\mu,\nu\right)$-near-isometry,
and $\mathcal{L}\left(\boldsymbol{\theta}_{0}\right)\leq\frac{1}{16c^{2}}\frac{\alpha_{\mathcal{F}}\mu^{6}}{\beta_{f}^{2}\nu^{2}}$,
where $c$ is the constant from \cref{lem:trajectory-length-bound}.
Then, gradient flow (\cref{eq:gradient_flow_appendix}) initialized by
$\boldsymbol{\theta}\left(0\right)=\boldsymbol{\theta}_{0}$ converges
to zero loss such that
\[
\mathcal{L}\left(\boldsymbol{\theta}\left(t\right)\right)\leq\mathcal{L}\left(\boldsymbol{\theta}_{0}\right)e^{-\frac{1}{2}\alpha_{\mathcal{F}}\mu^{2}t}\,.
\]
\end{theorem}
\begin{proof}
We show that under the stated conditions, $\boldsymbol{\theta}\left(t\right)$
remains in $\Theta$ indefinitely, i.e., $\tau_{\mathrm{b}}=\infty$.
Suppose $\tau_{\mathrm{b}}<\infty$, then the previous Lemma bounds
the trajectory length up to $\tau_{\mathrm{b}}$. We show that under
the assumptions here, the trajectory could not have exited the ball,
by bounding the trajectory length by the ball's radius:
\[
c\frac{\nu\sqrt{\mathcal{L}\left(\boldsymbol{\theta}_{0}\right)}}{\mu^{2}\sqrt{\alpha_{\mathcal{F}}}}\leq c\frac{\nu\sqrt{\frac{1}{16c^{2}}\frac{\alpha_{\mathcal{F}}\mu^{6}}{\beta_{f}^{2}\nu^{2}}}}{\mu^{2}\sqrt{\alpha_{\mathcal{F}}}}=\frac{\mu}{4\beta_{f}}<\frac{\mu}{2\beta_{f}},
\]
where we employed the loss bound at $\boldsymbol{\theta}_{0}$. Therefore,
$\tau_{\mathrm{b}}=\infty$.

Then, by \cref{lem:near-isometry-maintained}, we have that for
all $t\geq0$, $\nabla f\left(\boldsymbol{\theta}\left(t\right)\right)$
is $\left(\frac{1}{2}\mu,\frac{3}{2}\nu\right)$-near-isometry. Hence,
via chain rule,
\begin{align*}
\frac{\mathrm{d}}{\mathrm{d}t}\mathcal{L}\left(\boldsymbol{\theta}\left(t\right)\right) & =\nabla\mathcal{L}\left(\boldsymbol{\theta}\left(t\right)\right)^{\top}\frac{\mathrm{d}}{\mathrm{d}t}\boldsymbol{\theta}\left(t\right)=-\left\Vert \nabla\mathcal{L}\left(\boldsymbol{\theta}\left(t\right)\right)\right\Vert ^{2}\\
 & =-\left\Vert \nabla\mathcal{F}\left(f\left(\boldsymbol{\theta}\left(t\right)\right)\right)^{\top}\nabla f\left(\boldsymbol{\theta}\left(t\right)\right)\right\Vert ^{2}\\
 & \leq-\left\Vert \nabla\mathcal{F}\left(f\left(\boldsymbol{\theta}\left(t\right)\right)\right)\right\Vert ^{2}\sigma_{\min}\left(\nabla f\left(\boldsymbol{\theta}\left(t\right)\right)\right)^{2}\\
 & \leq-2\alpha_{\mathcal{F}}\mathcal{F}\left(f\left(\boldsymbol{\theta}\left(t\right)\right)\right)\left(\frac{1}{2}\mu\right)^{2}=-\frac{1}{2}\alpha_{\mathcal{F}}\mu^{2}\mathcal{L}\left(\boldsymbol{\theta}\left(t\right)\right)\,,
\end{align*}
where in the last inequality transition we used the PL condition satisfied
by $\mathcal{F}$ and that $\nabla f\left(\boldsymbol{\theta}\left(t\right)\right)$
is $\left(\frac{1}{2}\mu,\frac{3}{2}\nu\right)$-near-isometry. Then,
via Grönwall's inequality, we get
\[
\mathcal{L}\left(\boldsymbol{\theta}\left(t\right)\right)\leq\mathcal{L}\left(\boldsymbol{\theta}_{0}\right)e^{-\frac{1}{2}\alpha_{\mathcal{F}}\mu^{2}t}\,.
\]
\end{proof}
\Cref{thm:song-convergence} is entirely deterministic, and asks for three
things at the point from which gradient flow is started: a near-isometric
Jacobian, smoothness of $f$ on the surrounding ball of radius
$\frac{\mu}{2\beta_{f}}$, and a loss below
$\frac{1}{16c^{2}}\frac{\alpha_{\mathcal{F}}\mu^{6}}{\beta_{f}^{2}\nu^{2}}$.
The rest of this appendix establishes these three requirements at
$\vtheta_{\mathrm{CLD}}\left(\tau\right)$.
\subsection{The shallow network at initialization}
\label{sec:stopped_cld_init_properties}
We follow the shallow neural network setting, employing the assumptions and initialization
presented in \cref{sec:setting_deep_nn} with $L=1$, such that 
\[
f\left(\boldsymbol{\theta}\right)=\left(s_{2}\mathbf{v}^{\top}\phi\left(s_{1}\mathbf{W}\mathbf{X}^{\top}\right)\right)^{\top}\,,
\]
where $\boldsymbol{\theta}=\left\{ \mathbf{W},\mathbf{v}\right\} $,
and the MSE loss 
\[
\tilde{\mathcal{L}}\left(\mathbf{z}\right)\triangleq\frac{1}{2n}\left\Vert \mathbf{z}-\mathbf{y}\right\Vert ^{2}\,,
\]
such that
\[
\mathcal{L}\left(\boldsymbol{\theta}\right)=\tilde{\mathcal{L}}\left(f\left(\boldsymbol{\theta}\right)\right)\,.
\]
\Cref{thm:song-convergence} is thus applied with $\mathcal{F}=\tilde{\mathcal{L}}$,
and we now collect the estimates its three requirements call for. The PL
constant of the MSE loss (\cref{claim:mse-loss-pl-condition}) is deterministic
and holds everywhere. The near-isometry constants
(\cref{prop:mu_at_initialization,prop:nu_at_initialization}) and the smoothness
constant
(\cref{prop:hessian_global_bound,prop:network_beta_smooth_in_ball,prop:max_v_at_initialization})
are instead controlled at initialization with high probability---precisely the
form on which the change-of-measure argument of
\cref{sec:stopped_cld_zero_loss_proof} acts.
\begin{claim}
\label{claim:mse-loss-pl-condition}The MSE loss $\tilde{\mathcal{L}}:\mathbb{R}^{n}\to\mathbb{R}$,
defined as $\tilde{\mathcal{L}}\left(\mathbf{z}\right)\triangleq\frac{1}{2n}\left\Vert \mathbf{z}-\mathbf{y}\right\Vert ^{2}$
for some predetermined $\mathbf{y}\in\mathbb{R}^{n}$, satisfies the
PL condition with $\alpha_{\mathcal{L}}=n^{-1}$.
\end{claim}
\begin{proof}
This is shown by simple algebra:
\begin{align*}
\nabla\tilde{\mathcal{L}}\left(\mathbf{z}\right) & =\frac{1}{n}\left(\mathbf{z}-\mathbf{y}\right)\,,\\
\Longrightarrow\left\Vert \nabla\tilde{\mathcal{L}}\left(\mathbf{z}\right)\right\Vert ^{2} & =\frac{1}{n^{2}}\left\Vert \mathbf{z}-\mathbf{y}\right\Vert ^{2}=\frac{2}{n}\frac{1}{2n}\left\Vert \mathbf{z}-\mathbf{y}\right\Vert ^{2}=\frac{2}{n}\tilde{\mathcal{L}}\left(\mathbf{z}\right)\\
\Longrightarrow\tilde{\mathcal{L}}\left(\mathbf{z}\right) & =\frac{\left\Vert \nabla\tilde{\mathcal{L}}\left(\mathbf{z}\right)\right\Vert ^{2}}{2n^{-1}}\,.
\end{align*}
\end{proof}
The two near-isometry constants are read off from the spectrum of the Jacobian
at initialization: the NTK minimum eigenvalue bound of
\cref{thm:lower_bound_ntk_deep} gives $\mu$, and a trace bound gives $\nu$.
\begin{proposition}
\label{prop:mu_at_initialization}At initialization, for $m=\Omega\left(\frac{n}{\lambda_{1}}\ln\left(\frac{n}{\lambda_{1}}\right)\ln n\right)$,
we have with probability at least $1-\exp\left(-c_{1}\frac{m}{\frac{n}{\lambda_{1}}\ln\left(\frac{n}{\lambda_{1}}\right)}\right)$,
\begin{align*}
\mu & \triangleq\sigma_{\min}\left(\nabla f\left(\boldsymbol{\theta}_{0}\right)\right)\geq\tilde{c}_{1}\sqrt{\lambda_{1}s_{2}^{2}m}\,.
\end{align*}
\end{proposition}
\begin{proof}
From \cref{thm:lower_bound_ntk_deep}, we have for $m=\Omega\left(\frac{n}{\lambda_{1}}\ln\left(\frac{n}{\lambda_{1}}\right)\ln n\right)$
with probability at least $1-\exp\left(-c_{1}\frac{m}{\frac{n}{\lambda_{1}}\ln\left(\frac{n}{\lambda_{1}}\right)}\right)$,
\[
\lambda_{\min}\left(\mathbf{K}_{\mathrm{NTK}}\left(\boldsymbol{\theta}_{0}\right)\right)\geq c_{0}\lambda_{1}s_{2}^{2}m\,.
\]
Since $\mathbf{K}_{\mathrm{NTK}}$ is defined as
\[
\mathbf{K}_{\mathrm{NTK}}\left(\boldsymbol{\theta}\right)=\nabla f\left(\boldsymbol{\theta}\right)\nabla f\left(\boldsymbol{\theta}\right)^{\top}
\]
for all $\boldsymbol{\theta}\in\mathbb{R}^{N_{p}}$,
\[
\mu=\sigma_{\min}\left(\nabla f\left(\boldsymbol{\theta}_{0}\right)\right)=\sqrt{\lambda_{\min}\left(\mathbf{K}_{\mathrm{NTK}}\left(\boldsymbol{\theta}_{0}\right)\right)}\geq\sqrt{c_{0}\lambda_{1}s_{2}^{2}m}=\tilde{c}_{1}\sqrt{\lambda_{1}s_{2}^{2}m}\,.
\]
\end{proof}
\begin{proposition}
\label{prop:nu_at_initialization}At initialization, we have that
with probability at least $1-4n\exp\left(-\tilde{c}_{2}m\right)$,
\[
\nu\triangleq\sigma_{\max}\left(\nabla f\left(\boldsymbol{\theta}_{0}\right)\right)\leq\sqrt{4\alpha_{\phi}^{2}c_{\phi,\sigma_{0}}s_{2}^{2}s_{1}^{2}md\left(\frac{\nu_{0}^{2}}{a_{1}}+\frac{1}{a_{2}}\right)n}\,.
\]
With $m=\Omega\left(\frac{n}{\lambda_{1}}\ln\left(\frac{n}{\lambda_{1}}\right)\ln n\right)$,
there exists a suitable constant $c_{2}$ such that the above holds
with probability at least $1-\exp\left(-c_{2}m\right)$.
\end{proposition}
\begin{proof}
Fix $i\in\left[n\right]$. Then,
\begin{align*}
\left\Vert \nabla f_{i}\left(\boldsymbol{\theta}_{0}\right)\right\Vert ^{2} & =\left\Vert s_{2}\phi\left(s_{1}\mathbf{W}\mathbf{x}_{i}\right)\right\Vert ^{2}+\sum_{j=1}^{m}\left\Vert s_{2}v_{j}\phi'\left(s_{1}\mathbf{w}_{j}^{\top}\mathbf{x}_{i}\right)s_{1}\mathbf{x}_{i}\right\Vert ^{2}\\
\left[\phi\left(\cdot\right)\leq\alpha_{\phi}\left|\cdot\right|,\,\left|\phi'\right|\le\alpha_{\phi}\right] & \leq\alpha_{\phi}^{2}s_{2}^{2}s_{1}^{2}\left(\left\Vert \mathbf{W}\mathbf{x}_{i}\right\Vert ^{2}+\sum_{j=1}^{m}v_{j}^{2}\left\Vert \mathbf{x}_{i}\right\Vert ^{2}\right)\,,
\end{align*}
and the RHS is distributed as 
\[
\alpha_{\phi}^{2}s_{2}^{2}s_{1}^{2}c_{\phi,\sigma_{0}}d\left(\frac{\nu_{0}^{2}}{a_{1}}\left\Vert \mathbf{z}^{\left(1\right)}\right\Vert ^{2}+\frac{1}{a_{2}}\left\Vert \mathbf{z}^{\left(2\right)}\right\Vert ^{2}\right)\,,
\]
where $\mathbf{z}^{\left(1\right)},\mathbf{z}^{\left(2\right)}\sim\mathcal{N}\left(\mathbf{0}_{m},\mathbf{I}_{m}\right)$
independently. From \citet[Theorem 3.1.1]{vershynin2018high}, we have
\[
\left\Vert \left\Vert \mathbf{z}\right\Vert _{2}-\sqrt{m}\right\Vert _{\psi_{2}}\leq C\,,
\]
thus, for any $\epsilon>0$,
\[
\mathbb{P}\left(\left|\left\Vert \mathbf{z}\right\Vert _{2}-\sqrt{m}\right|\geq\epsilon\right)\leq2\exp\left(-\frac{c}{C^{2}}\epsilon^{2}\right)=2\exp\left(-\tilde{c}\epsilon^{2}\right)\,.
\]
Observe that
\[
\mathbb{P}\left(\left|\left\Vert \mathbf{z}\right\Vert _{2}-\sqrt{m}\right|\geq\epsilon\right)\geq\mathbb{P}\left(\left\Vert \mathbf{z}\right\Vert _{2}\geq\sqrt{m}+\epsilon\right)=\mathbb{P}\left(\left\Vert \mathbf{z}\right\Vert ^{2}\geq\left(\sqrt{m}+\epsilon\right)^{2}\right)\,.
\]
Then, via union bound, with probability of at least $1-4\exp\left(-\tilde{c}\epsilon^{2}\right),$
\[
\left\Vert \nabla f_{i}\left(\boldsymbol{\theta}_{0}\right)\right\Vert ^{2}\leq\alpha_{\phi}^{2}s_{2}^{2}s_{1}^{2}c_{\phi,\sigma_{0}}d\left(\frac{\nu_{0}^{2}}{a_{1}}+\frac{1}{a_{2}}\right)\left(\sqrt{m}+\epsilon\right)^{2}\,.
\]
Choosing $\epsilon=\sqrt{m}$, we get that with probability of at
least $1-4\exp\left(-\tilde{c}m\right),$for any $i\in\left[n\right]$,
\[
\left\Vert \nabla f_{i}\left(\boldsymbol{\theta}_{0}\right)\right\Vert ^{2}\leq4\alpha_{\phi}^{2}c_{\phi,\sigma_{0}}s_{2}^{2}s_{1}^{2}md\left(\frac{\nu_{0}^{2}}{a_{1}}+\frac{1}{a_{2}}\right)\,.
\]
Via union bound, we get that with probability at least $1-4n\exp\left(-\tilde{c}m\right)$,
\begin{align*}
\nu & =\sigma_{\max}\left(\nabla f\left(\boldsymbol{\theta}_{0}\right)\right)=\sqrt{\lambda_{\max}\left(\nabla f\left(\boldsymbol{\theta}_{0}\right)^{\top}\nabla f\left(\boldsymbol{\theta}_{0}\right)\right)}\\
 & \leq\sqrt{\mathrm{Tr}\left(\nabla f\left(\boldsymbol{\theta}_{0}\right)^{\top}\nabla f\left(\boldsymbol{\theta}_{0}\right)\right)}=\sqrt{\sum_{j=1}^{N_{p}}\sum_{i=1}^{n}\left(\left(\left(\nabla f_{i}\left(\boldsymbol{\theta}_{0}\right)\right)_{j}\right)^{2}\right)}\\
 & =\sqrt{\sum_{i=1}^{n}\left\Vert \nabla f_{i}\left(\boldsymbol{\theta}_{0}\right)\right\Vert ^{2}}\leq\sqrt{4\alpha_{\phi}^{2}c_{\phi,\sigma_{0}}s_{2}^{2}s_{1}^{2}md\left(\frac{\nu_{0}^{2}}{a_{1}}+\frac{1}{a_{2}}\right)n}\,.
\end{align*}
For $m=\Omega\left(\frac{n}{\lambda_{1}}\ln\left(\frac{n}{\lambda_{1}}\right)\ln n\right)\gg\ln n$,
we can lower bound the probability by $1-\exp\left(-c_{2}m\right)$
for a suitable constant $c_{2}$.
\end{proof}
The smoothness constant is established in three steps. We first bound the
Hessian of each network output at \emph{every} parameter vector, in terms of
the largest second-layer weight (\cref{prop:hessian_global_bound}). We then
convert that bound into smoothness on a ball whose radius is itself set by the
smoothness constant; since the constant grows linearly with the radius, the two
are pinned down together by a quadratic inequality in the radius
(\cref{prop:network_beta_smooth_in_ball}). Finally, we
bound the largest second-layer weight at initialization
(\cref{prop:max_v_at_initialization}), the only one of the three steps that is
probabilistic.
\begin{proposition}
\label{prop:hessian_global_bound}For all $\boldsymbol{\theta}\in\mathbb{R}^{N_{p}},$
and all $i\in\left[n\right]$,
\[
\left\Vert \nabla^{2}f_{i}\left(\boldsymbol{\theta}\right)\right\Vert _{2}\leq s_{2}s_{1}\sqrt{c_{\phi,\sigma_{0}}d}\left(\alpha_{\phi}+\beta_{\phi}\sqrt{c_{\phi,\sigma_{0}}}s_{1}\sqrt{d}\max_{j\in\left[m\right]}\left|v_{j}\right|\right)\,.
\]
\end{proposition}
\begin{proof}
Observe the shallow network definition:
\[
f_{i}\left(\boldsymbol{\theta}\right)=s_{2}\mathbf{v}^{\top}\phi\left(s_{1}\mathbf{W}\mathbf{x}_{i}\right)=\sum_{j=1}^{m}s_{2}v_{j}\phi\left(s_{1}\mathbf{w}_{j}^{\top}\mathbf{x}_{i}\right)\,.
\]
The network is additively separable over neurons, and each summand
depends only on $\left\{ v_{j},\mathbf{w}_{j}\right\} $. Hence, the
Hessian is block-diagonal across neurons. Observing some neuron $j\in\left[m\right]$:
\begin{align*}
f_{i}^{\left(j\right)}\left(v_{j},\mathbf{w}_{j}\right) & =s_{2}v_{j}\phi\left(s_{1}\mathbf{w}_{j}^{\top}\mathbf{x}_{i}\right)\,,\\
\frac{\partial f_{i}^{\left(j\right)}\left(v_{j},\mathbf{w}_{j}\right)}{\partial v_{j}} & =s_{2}\phi\left(s_{1}\mathbf{w}_{j}^{\top}\mathbf{x}_{i}\right)\,,\\
\frac{\partial^{2}f_{i}^{\left(j\right)}\left(v_{j},\mathbf{w}_{j}\right)}{\partial v_{j}^{2}} & =0\,,\\
\frac{\partial^{2}f_{i}^{\left(j\right)}\left(v_{j},\mathbf{w}_{j}\right)}{\partial v_{j}\partial\mathbf{w}_{j}} & =s_{2}\phi'\left(s_{1}\mathbf{w}_{j}^{\top}\mathbf{x}_{i}\right)s_{1}\mathbf{x}_{i}\,,\\
\frac{\partial f_{i}^{\left(j\right)}\left(v_{j},\mathbf{w}_{j}\right)}{\partial\mathbf{w}_{j}} & =s_{2}v_{j}\phi'\left(s_{1}\mathbf{w}_{j}^{\top}\mathbf{x}_{i}\right)s_{1}\mathbf{x}_{i}\,,\\
\frac{\partial^{2}f_{i}^{\left(j\right)}\left(v_{j},\mathbf{w}_{j}\right)}{\partial\mathbf{w}_{j}^{2}} & =s_{2}s_{1}^{2}v_{j}\phi''\left(s_{1}\mathbf{w}_{j}^{\top}\mathbf{x}_{i}\right)\mathbf{x}_{i}\mathbf{x}_{i}^{\top}\,.
\end{align*}
The hessian block per neuron is then
\[
\nabla^{2}f_{i}^{\left(j\right)}=\begin{bmatrix}\mathbf{A} & \mathbf{b}\\
\mathbf{b}^{\top} & 0
\end{bmatrix}\,,
\]
where 
\begin{align*}
\mathbf{A} & =s_{2}s_{1}^{2}v_{j}\phi''\left(s_{1}\mathbf{w}_{j}^{\top}\mathbf{x}_{i}\right)\mathbf{x}_{i}\mathbf{x}_{i}^{\top}\,,\\
\mathbf{b} & =s_{2}s_{1}\phi'\left(s_{1}\mathbf{w}_{j}^{\top}\mathbf{x}_{i}\right)\mathbf{x}_{i}\,.
\end{align*}
Hence, 
\begin{align*}
\left\Vert \nabla^{2}f_{i}^{\left(j\right)}\right\Vert _{2} & \leq\left\Vert \begin{bmatrix}\mathbf{A} & \mathbf{0}\\
\mathbf{0}^{\top} & 0
\end{bmatrix}\right\Vert _{2}+\left\Vert \begin{bmatrix}\mathbf{0} & \mathbf{b}\\
\mathbf{b}^{\top} & 0
\end{bmatrix}\right\Vert _{2}=\left\Vert \mathbf{A}\right\Vert _{2}+\left\Vert \mathbf{b}\right\Vert _{2}\\
 & =s_{2}s_{1}^{2}\left|v_{j}\right|\left|\phi''\left(s_{1}\mathbf{w}_{j}^{\top}\mathbf{x}_{i}\right)\right|\left\Vert \mathbf{x}_{i}\mathbf{x}_{i}^{\top}\right\Vert _{2}+s_{2}s_{1}\left|\phi'\left(s_{1}\mathbf{w}_{j}^{\top}\mathbf{x}_{i}\right)\right|\left\Vert \mathbf{x}_{i}\right\Vert _{2}\\
 & =s_{2}s_{1}\left\Vert \mathbf{x}_{i}\right\Vert \left(s_{1}\left|v_{j}\right|\left|\phi''\left(s_{1}\mathbf{w}_{j}^{\top}\mathbf{x}_{i}\right)\right|\left\Vert \mathbf{x}_{i}\right\Vert +\left|\phi'\left(s_{1}\mathbf{w}_{j}^{\top}\mathbf{x}_{i}\right)\right|\right)\\
 & \leq s_{2}s_{1}\sqrt{c_{\phi,\sigma_{0}}d}\left(\beta_{\phi}s_{1}\sqrt{c_{\phi,\sigma_{0}}d}\left|v_{j}\right|+\alpha_{\phi}\right)\,.
\end{align*}
Now, from the block structure of the entire Hessian,
\[
\left\Vert \nabla^{2}f_{i}\left(\boldsymbol{\theta}\right)\right\Vert _{2}=\max_{j\in\left[m\right]}\left\Vert \nabla^{2}f_{i}^{\left(j\right)}\right\Vert _{2}\leq s_{2}s_{1}\sqrt{c_{\phi,\sigma_{0}}d}\left(\alpha_{\phi}+\beta_{\phi}\sqrt{c_{\phi,\sigma_{0}}}s_{1}\sqrt{d}\max_{j\in\left[m\right]}\left|v_{j}\right|\right)\,.
\]
\end{proof}
\begin{proposition}
\label{prop:network_beta_smooth_in_ball}Given some $\boldsymbol{\theta}_{b}$
which satisfies $\max_{j\in\left[m\right]}\left|v_{j}\right|\leq\xi$
for some $\xi>0$, then for any $\mu>0$ we have that $f$ is $\beta_{f}$-smooth
in $\Theta=\left\{ \boldsymbol{\theta}:\,\left\Vert \boldsymbol{\theta}-\boldsymbol{\theta}_{b}\right\Vert \leq\frac{\mu}{2\beta_{f}}\right\} $,
with
\[
\beta_{f}=\sqrt{n}s_{2}s_{1}\sqrt{c_{\phi,\sigma_{0}}d}\left(\alpha_{\phi}+\beta_{\phi}\sqrt{c_{\phi,\sigma_{0}}}s_{1}\sqrt{d}\xi\right)+\sqrt{\frac{1}{2}\sqrt{n}s_{2}s_{1}^{2}c_{\phi,\sigma_{0}}d\beta_{\phi}\mu}\,.
\]
\end{proposition}
\begin{proof}
In a ball with radius $\rho>0$ around $\boldsymbol{\theta}_{b}$,
i.e., for $\boldsymbol{\theta}$ such that $\left\Vert \boldsymbol{\theta}-\boldsymbol{\theta}_{b}\right\Vert \leq\rho$,
we have $\max_{j\in\left[m\right]}\left|v_{j}\right|\leq\xi+\rho$.
Then, in the ball, we have for all $i\in\left[n\right]$,
\[
\left\Vert \nabla^{2}f_{i}\left(\boldsymbol{\theta}\right)\right\Vert _{2}\leq s_{2}s_{1}\sqrt{c_{\phi,\sigma_{0}}d}\left(\alpha_{\phi}+\beta_{\phi}\sqrt{c_{\phi,\sigma_{0}}}s_{1}\sqrt{d}\left(\xi+\rho\right)\right)\,.
\]
Then, for every $\boldsymbol{\theta}_{1},\boldsymbol{\theta}_{2}$
in the ball, we have that
\begin{align*}
\left\Vert \nabla f\left(\boldsymbol{\theta}_{1}\right)-\nabla f\left(\boldsymbol{\theta}_{2}\right)\right\Vert _{2} & \leq\left\Vert \nabla f\left(\boldsymbol{\theta}_{1}\right)-\nabla f\left(\boldsymbol{\theta}_{2}\right)\right\Vert _{\mathrm{F}}\\
 & =\sqrt{\sum_{i\in\left[n\right]}\left\Vert \nabla f_{i}\left(\boldsymbol{\theta}_{1}\right)-\nabla f_{i}\left(\boldsymbol{\theta}_{2}\right)\right\Vert ^{2}}\\
 & \leq\sqrt{n}s_{2}s_{1}\sqrt{c_{\phi,\sigma_{0}}d}\left(\alpha_{\phi}+\beta_{\phi}\sqrt{c_{\phi,\sigma_{0}}}s_{1}\sqrt{d}\left(\xi+\rho\right)\right)\left\Vert \boldsymbol{\theta}_{1}-\boldsymbol{\theta}_{2}\right\Vert \,,
\end{align*}
where we used that, on the ball $\left\{ \boldsymbol{\theta}:\,\left\Vert \boldsymbol{\theta}-\boldsymbol{\theta}_{b}\right\Vert \leq\rho\right\} $,
the bound on $\left\Vert \nabla^{2}f_{i}\left(\boldsymbol{\theta}\right)\right\Vert _{2}$
implies that $\nabla f_{i}$ is Lipschitz with the same constant.
Denoting 
\begin{align*}
\beta_{0} & =\sqrt{n}s_{2}s_{1}\sqrt{c_{\phi,\sigma_{0}}d}\left(\alpha_{\phi}+\beta_{\phi}\sqrt{c_{\phi,\sigma_{0}}}s_{1}\sqrt{d}\xi\right)\,,\\
\gamma & =\sqrt{n}s_{2}s_{1}\sqrt{c_{\phi,\sigma_{0}}d}\beta_{\phi}\sqrt{c_{\phi,\sigma_{0}}}s_{1}\sqrt{d}\,,
\end{align*}
we observe that $f$ is $\beta\left(\rho\right)$-smooth in the ball,
with
\[
\beta\left(\rho\right)=\beta_{0}+\gamma\rho\,.
\]
We now demand that the ball's radius is no larger than $\frac{\mu}{2\beta_{f}},$i.e.,
\begin{align*}
\rho & \leq\frac{\mu}{2\left(\beta_{0}+\gamma\rho\right)}\\
\Longrightarrow & 2\gamma\rho^{2}+2\beta_{0}\rho-\mu\leq0\\
\Longrightarrow & \rho\leq\frac{\sqrt{\beta_{0}^{2}+2\gamma\mu}-\beta_{0}}{2\gamma}\,.
\end{align*}
In the ball with radius $\rho=\frac{\mu}{2\beta\left(\rho\right)}=\frac{\sqrt{\beta_{0}^{2}+2\gamma\mu}-\beta_{0}}{2\gamma}$,
we have that $f$ is $\beta\left(\rho\right)$-smooth with
\begin{align*}
\beta\left(\rho\right) & =\beta_{0}+\gamma\rho=\beta_{0}+\gamma\frac{\sqrt{\beta_{0}^{2}+2\gamma\mu}-\beta_{0}}{2\gamma}\\
 & \leq\beta_{0}+\frac{1}{2}\left(\beta_{0}+\sqrt{2\gamma\mu}-\beta_{0}\right)=\beta_{0}+\frac{1}{\sqrt{2}}\sqrt{\gamma\mu}\,.
\end{align*}
Denoting $\beta_{f}=\beta_{0}+\frac{1}{\sqrt{2}}\sqrt{\gamma\mu}\geq\beta\left(\rho\right)$
we immediately get that $f$ is also $\beta_{f}$-smooth in the ball
with radius $\rho=\frac{\mu}{2\beta\left(\rho\right)}$, and hence
that $f$ is $\beta_{f}$-smooth in the ball with the smaller radius
$\frac{\mu}{2\beta_{f}}$, concluding the proof. 
\end{proof}
\begin{proposition}
\label{prop:max_v_at_initialization}At initialization, we have that
with probability at least $1-\exp\left(-\frac{m}{n}\right)$,
\[
\max_{j\in\left[m\right]}\left|v_{j}\right|<\sqrt{\frac{2}{a_{2}}\left(\ln\left(2m\right)+\frac{m}{n}\right)}\,.
\]
With $m=\Omega\left(\frac{n}{\lambda_{1}}\ln\left(\frac{n}{\lambda_{1}}\right)\ln n\right)$
and for large enough $n$ (satisfied by $n\gg1$), we have that with
probability at least $1-\exp\left(-\frac{m}{n}\right)$,
\[
\max_{j\in\left[m\right]}\left|v_{j}\right|<\sqrt{\frac{4m}{a_{2}n}}\,.
\]
\end{proposition}
\begin{proof}
At initialization, $v_{j}\sim\mathcal{N}\left(0,\frac{1}{a_{2}}\right),\,j\in\left[m\right]$
i.i.d. Hence, for some $\epsilon>0$, via union bound and the Gaussian
tail property \citet[Exercise 2.6]{vershynin2018high},
\[
\mathbb{P}\left(\max_{j\in\left[m\right]}\left|v_{j}\right|\geq\epsilon\right)\leq\sum_{j\in\left[m\right]}\mathbb{P}\left(\left|v_{j}\right|\geq\epsilon\right)\leq2m\exp\left(-a_{2}\frac{\epsilon^{2}}{2}\right)\,.
\]
Demanding this probability to be smaller than some $\delta>0$, we
get
\begin{align*}
2m\exp\left(-a_{2}\frac{\epsilon^{2}}{2}\right) & \leq\delta\\
\epsilon^{2} & \geq\frac{2}{a_{2}}\ln\left(\frac{2m}{\delta}\right)\,.
\end{align*}
Choosing $\delta=e^{-\nicefrac{m}{n}}$, we get that w.p. at least
$1-\exp\left(-\frac{m}{n}\right)$, 
\[
\max_{j\in\left[m\right]}\left|v_{j}\right|<\sqrt{\frac{2}{a_{2}}\left(\ln\left(2m\right)+\frac{m}{n}\right)}\,.
\]
We now assume $m=\Omega\left(\frac{n}{\lambda_{1}}\ln\left(\frac{n}{\lambda_{1}}\right)\ln n\right)=\Omega\left(n\left(\ln n\right)^{2}\right)$,
where the last equality is due to $\lambda_{1}=O\left(1\right)$.
Since for all $x>e$, $\frac{\ln x}{x}$ is monotonically decreasing,
we observe that for large enough $n$,
\begin{align*}
\ln\left(2m\right) & =m\frac{\ln\left(2m\right)}{m}\lesssim m\frac{\ln\left(2n\left(\ln n\right)^{2}\right)}{n\left(\ln n\right)^{2}}\\
 & <m\frac{\ln\left(n^{2}\right)}{n\left(\ln n\right)^{2}}=m\frac{2\ln\left(n\right)}{n\left(\ln n\right)^{2}}=\frac{m}{n}\frac{2}{\ln n}<\frac{m}{n}\,.
\end{align*}
 Hence, 
\[
\max_{j\in\left[m\right]}\left|v_{j}\right|<\sqrt{\frac{2}{a_{2}}\left(\ln\left(2m\right)+\frac{m}{n}\right)}<\sqrt{\frac{4m}{a_{2}n}}\,.
\]
\end{proof}
\subsection{Proof of the zero-loss guarantee by applying the second law change-of-measure inequality}
\label{sec:stopped_cld_zero_loss_proof}
Our change-of-measure method allows us to transfer initialization
properties to arbitrary time properties with high probability, as
well as upper bounding the expected loss at an arbitrary time by a
finite value, by our main theorem \cref{thm:upper_bound_on_expected_loss}. We will employ
this method to show that with finite, lower bounded probability,
necessary conditions for the application of \cref{thm:song-convergence}
hold when stopping CLD at a large enough deterministic time, allowing
convergence to zero loss via gradient flow from this point.
\begin{proof}[Proof of \cref{prop:convergence_to_zero_loss_shallow_stopped_cld}]
\emph{Step 1: transferring the initialization properties to time $t$.}
For an event $B$ whose probability at initialization is at most $\delta_{0}$,
we can apply the second law change-of measure inequality (\cref{eq:second_law_with_eta}) on the indicator
function $\mathds{1}_{B}$ to achieve that
\[
p_{t}\left(B\right)\leq\eta\delta_{0}^{\nicefrac{1}{c_{\beta}\bar\kappa^2\beta}}
\]
for any $t\geq0$, when evolving by regularized CLD as in \cref{eq:regularized_langevin}. Applying this to the complementary events of our
high probability bounds, we get that at time $t$, we have: 
\begin{enumerate}
\item By applying the second law with \cref{prop:mu_at_initialization},
we get that with probability at least $1-\eta\exp\left(-c_{1}\frac{m}{c_{\beta}\bar\kappa^2\beta\frac{n}{\lambda_{1}}\ln\left(\frac{n}{\lambda_{1}}\right)}\right)$,
$\mu\geq C_{1}\sqrt{\lambda_{1}s_{2}^{2}m}$, where $C_{1}$ is a
global constant.
\item By applying the second law with \cref{prop:nu_at_initialization},
we get that with probability at least $1-\eta\exp\left(-c_{2}\frac{m}{c_{\beta}\bar\kappa^2\beta}\right)$,
$\nu\leq C_{2}\sqrt{\frac{s_{2}^{2}s_{1}^{2}md}{a_{2}a_{1}}\left(a_{1}+a_{2}\right)n}$,
where $C_{2}$ is a global constant.
\item By applying the second law with \cref{prop:max_v_at_initialization},
we get that with probability at least $1-\eta\exp\left(-\frac{m}{c_{\beta}\bar\kappa^2\beta n}\right)$,
$\max_{j\in\left[m\right]}\left|v_{j}\right|\leq\sqrt{\frac{4m}{a_{2}n}}$.
Then, by applying \cref{prop:network_beta_smooth_in_ball}, we
get that $f$ is $\beta_{f}$-smooth in $\Theta=\left\{ \boldsymbol{\theta}:\,\left\Vert \boldsymbol{\theta}-\boldsymbol{\theta}_{\mathrm{CLD}}\left(t\right)\right\Vert \leq\frac{C_{1}\sqrt{\lambda_{1}s_{2}^{2}m}}{2\beta_{f}}\right\} $,
with 
\begin{align*}
\beta_{f} & \leq C_{3}\sqrt{n}s_{2}s_{1}\sqrt{d}\left(1+s_{1}\sqrt{d}\sqrt{\frac{m}{a_{2}n}}\right)+\sqrt{\sqrt{n}s_{2}s_{1}^{2}d\sqrt{\lambda_{1}s_{2}^{2}m}}\\
 & =C_{3}s_{2}s_{1}\sqrt{d}\left(\sqrt{n}+s_{1}\sqrt{\frac{dm}{a_{2}}}+\left(\lambda_{1}nm\right)^{1/4}\right)\,,
\end{align*}
where $C_{3}$ is a global constant.
\end{enumerate}
Via union bound, all of these will hold together with probability
at least
\begin{align*}
 & 1-\eta\left(\exp\left(-c_{1}\frac{m}{c_{\beta}\bar\kappa^2\beta\frac{n}{\lambda_{1}}\ln\left(\frac{n}{\lambda_{1}}\right)}\right)+\exp\left(-c_{2}\frac{m}{c_{\beta}\bar\kappa^2\beta}\right)+\exp\left(-\frac{m}{c_{\beta}\bar\kappa^2\beta n}\right)\right)\\
 & \geq1-\eta\exp\left(-c_{4}\frac{m}{c_{\beta}\bar\kappa^2\beta\frac{n}{\lambda_{1}}\ln\left(\frac{n}{\lambda_{1}}\right)}\right)\,,
\end{align*}
for a suitable constant $c_{4}$. 
This probability can be lower bounded under the condition $m=\Omega\left(\bar\kappa^2\beta\frac{n}{\lambda_{1}}\ln\left(\frac{n}{\lambda_{1}}\right)\right)$, also necessary for our expected loss bound \cref{thm:upper_bound_on_expected_loss}.

\emph{Step 2: the loss at the stopping time.}
Our convergence theorem (\cref{thm:upper_bound_on_expected_loss}) bounds the loss's expected value
(over the randomness of the initialization and the noise) at time
$t$ by 
\[
\mathbb{E}\mathcal{L}\left(\boldsymbol{\theta}_{\mathrm{CLD}}\left(t\right)\right)\leq\left(\mathbb{E}\mathcal{L}\left(\boldsymbol{\theta}_{0}\right)-\mathcal{L}_{\infty}\right)e^{-cs_{2}^{2}\frac{\lambda_{1}m}{n}t}+\mathcal{L}_{\infty}\,.
\]
Throughout this appendix, $\mathcal{L}_{\infty}$ abbreviates the
shallow-network loss floor $\mathcal{L}_{\infty}^{\prn{L=1}}$ of
\cref{eq:L_infty_raw}.
Then, for any $t\geq0$, $\mathcal{L}_{g}>0$ we have via Markov's
inequality that:
\begin{align*}
\mathbb{P}\left(\mathcal{L}\left(\boldsymbol{\theta}_{\mathrm{CLD}}\left(t\right)\right)\geq\mathcal{L}_{g}\right) & \leq\frac{\mathbb{E}\mathcal{L}\left(\boldsymbol{\theta}_{\mathrm{CLD}}\left(t\right)\right)}{\mathcal{L}_{g}}\\
 & \leq\frac{\left(\mathbb{E}\mathcal{L}\left(\boldsymbol{\theta}_{0}\right)-\mathcal{L}_{\infty}\right)e^{-cs_{2}^{2}\frac{\lambda_{1}m}{n}t}+\mathcal{L}_{\infty}}{\mathcal{L}_{g}}\,.
\end{align*}
For $t=\tau$ large enough (up to logarithmic terms, it is enough
to demand $\tau=\tilde{\Omega}\left(\frac{n}{s_{2}^{2}\lambda_{1}m}\right)$),
\[
\left(\mathbb{E}\mathcal{L}\left(\boldsymbol{\theta}_{0}\right)-\mathcal{L}_{\infty}\right)e^{-cs_{2}^{2}\frac{\lambda_{1}m}{n}\tau}\leq\mathcal{L}_{\infty}\,,
\]
then
\[
\mathbb{E}\mathcal{L}\left(\boldsymbol{\theta}_{\mathrm{CLD}}\left(\tau\right)\right)\leq2\mathcal{L}_{\infty}\,.
\]
By Markov's inequality, we have that for any $K>2$, 
\[
\mathbb{P}\prn{\mathcal{L}\left(\boldsymbol{\theta}_{\mathrm{CLD}}\left(\tau\right)\right) > K \mathcal{L}_\infty} \leq \frac{\mathbb{E}\mathcal{L}\left(\boldsymbol{\theta}_{\mathrm{CLD}}\left(\tau\right)\right)}{K \mathcal{L}_\infty}\leq\frac{2}{K}\,.
\]

\emph{Step 3: verifying the criterion at $\boldsymbol{\theta}_{\mathrm{CLD}}\left(\tau\right)$.}
Via union bound, we will have at time $\tau$ that with probability
at least \[
1-\eta\exp\left(-c_{4}\frac{m}{c_{\beta}\bar\kappa^2\beta\frac{n}{\lambda_{1}}\ln\left(\frac{n}{\lambda_{1}}\right)}\right)-\frac{2}{K}
\]
over the randomness
of the initialization and the noise, that $f$ is $\beta_{f}$-smooth
in $\Theta=\left\{ \boldsymbol{\theta}:\,\left\Vert \boldsymbol{\theta}-\boldsymbol{\theta}_{\mathrm{CLD}}\left(\tau\right)\right\Vert \leq\frac{\mu}{2\beta_{f}}\right\} $
and $\nabla f\left(\boldsymbol{\theta}_{\mathrm{CLD}}\left(\tau\right)\right)$
is $\left(\mu,\nu\right)$-near-isometry, and the loss is bounded,
with 
\begin{align*}
\mathcal{L}\left(\boldsymbol{\theta}_{\mathrm{CLD}}\left(\tau\right)\right) & <K \mathcal{L}_{\infty}\,,\\
\mu & =C_{1}\sqrt{\lambda_{1}s_{2}^{2}m}\,,\\
\nu & =C_{2}\sqrt{\frac{s_{2}^{2}s_{1}^{2}md}{a_{2}a_{1}}\left(a_{1}+a_{2}\right)n}\,,\\
\beta_{f} & =C_{3}s_{2}s_{1}\sqrt{d}\left(\sqrt{n}+s_{1}\sqrt{\frac{dm}{a_{2}}}+\left(\lambda_{1}nm\right)^{1/4}\right)\,.
\end{align*}
Together with $\tilde{\mathcal{L}}$ satisfying the PL condition with
$\alpha_{\mathcal{L}}=n^{-1}$ (\cref{claim:mse-loss-pl-condition}),
all that remains in order to apply \cref{thm:song-convergence}
with gradient flow initialized at $\boldsymbol{\theta}_{\mathrm{CLD}}\left(\tau\right)$,
is to show that
\[
K\mathcal{L}_{\infty}\leq\mathcal{L}_{\mathrm{GF}}=\hat{c}\frac{\alpha_{\mathcal{L}}\mu^{6}}{\beta_{f}^{2}\nu^{2}},
\]
where
\begin{align*}
\mathcal{L}_{\mathrm{GF}} & =\hat{c}\frac{\alpha_{\mathcal{L}}\mu^{6}}{\beta_{f}^{2}\nu^{2}}\\
 & =\hat{c}\frac{n^{-1}C_{1}^{6}\lambda_{1}^{3}s_{2}^{6}m^{3}}{C_{3}^{2}s_{2}^{2}s_{1}^{2}d\left(\sqrt{n}+s_{1}\sqrt{\frac{dm}{a_{2}}}+\left(\lambda_{1}nm\right)^{1/4}\right)^{2}C_{2}^{2}\frac{s_{2}^{2}s_{1}^{2}md}{a_{2}a_{1}}\left(a_{1}+a_{2}\right)n}\\
\left[\lambda_{1}=\Theta\left(1\right)\right] & \asymp\frac{s_{2}^{2}m^{2}}{n^{2}d^{2}s_{1}^{4}\left(\frac{1}{a_{1}}+\frac{1}{a_{2}}\right)\left(n+\frac{s_{1}^{2}dm}{a_{2}}+\sqrt{nm}\right)}\\
 & \asymp\frac{s_{2}^{2}m^{2}}{n^{2}d^{2}s_{1}^{4}\left(\frac{1}{a_{1}}+\frac{1}{a_{2}}\right)\left(\frac{s_{1}^{2}dm}{a_{2}}+\sqrt{nm}\right)}\,,
\end{align*}
where in the last transition we assumed $m\gg n$. Substituting $a_{1}=ds_{1}^{2},\,a_{2}=\frac{1}{\kappa^{2}}ms_{2}^{2}$,
we get that
\[
\mathcal{L}_{\mathrm{GF}}\asymp\frac{s_{2}^{2}m^{2}}{n^{2}d^{2}s_{1}^{4}\left(\frac{1}{ds_{1}^{2}}+\frac{\kappa^{2}}{ms_{2}^{2}}\right)\left(\frac{\kappa^{2}s_{1}^{2}d}{s_{2}^{2}}+\sqrt{nm}\right)}.
\]
Hence, it is enough to demand that 
\begin{equation}
\mathcal{L}_{\infty}\lesssim\frac{s_{2}^{2}m^{2}}{K n^{2}d^{2}s_{1}^{4}\left(\frac{1}{ds_{1}^{2}}+\frac{\kappa^{2}}{ms_{2}^{2}}\right)\left(\frac{\kappa^{2}s_{1}^{2}d}{s_{2}^{2}}+\sqrt{nm}\right)}
\label{eq:stopped_cld_loss_floor_demand}
\end{equation}
in order to achieve convergence via gradient flow from time $t=\tau$
with probability at least 
\(
1-\eta\exp\left(-c_{4}\frac{m}{c_{\beta}\bar\kappa^2\beta\frac{n}{\lambda_{1}}\ln\left(\frac{n}{\lambda_{1}}\right)}\right)-\frac{2}{K}
\)%
, which is the claim.
\end{proof}
\subsection{Examples of known parameterizations}
\label{sec:song_result_under_parameterizations}
We conclude by making the demand \cref{eq:stopped_cld_loss_floor_demand} on the
loss floor explicit under the two common parameterizations, and reading off
from \cref{tab:param_comparison} the width and inverse temperature that meet it.
\subsubsection{Standard parameterization}
Under standard parameterization, $s_{1}=s_{2}=1$, $\kappa=1$, and
the demand becomes
\[
\mathcal{L}_{\infty}\lesssim\frac{m^{2}}{Kn^{2}d^{2}\left(\frac{1}{d}+\frac{1}{m}\right)\left(d+\sqrt{nm}\right)}\,.
\]
Under $d\ll n^{3/2}$, $n^{2}\lesssim m$, as we assume in \cref{sec:convergence_result_analysis}, the above simplifies to
\begin{align*}
\mathcal{L}_{\infty} & \lesssim\frac{m^{3/2}}{Kn^{5/2}d}\,.
\end{align*}
From \cref{tab:param_comparison}, we have that this is achieved when demanding
\begin{align*}
\beta &= \Omega\prn{ \frac{K n^{7/2} d}{m^{3/2}}}\,,\quad m=\tilde\Omega \prn{\frac{K n^{9/2} d}{m^{3/2}}}\,.
\end{align*}
To achieve this, it is enough to demand:
\begin{align*}
m & =\tilde\Omega \prn{{K^{2/5} n^{9/5} d^{2/5}}}\,,\quad
 \beta  =\Omega \prn{{K^{2/5} n^{4/5} d^{2/5}}}\,.
\end{align*}
Since $\beta = \Omega \prn{n}$, the above will hold for $K=O \prn{\frac{\sqrt{n}}{d}}$, yielding an asymptotically small failure probability when $d \ll \sqrt{n}$, which is common.
This is the example stated in \cref{prop:convergence_to_zero_loss_shallow_stopped_cld}: at $K=\Theta \prn{\frac{\sqrt{n}}{d}}$ the two demands read $m=\tilde{\Omega}\prn{n^{2}}$ and $\beta=\Omega\prn{n}$.

We further note that by allowing $K=\Theta \prn{1}$, 
The RHS of this demand is $\gtrsim1$ when
\[
m\gtrsim n^{5/3}d^{2/3}\,.
\]
This result matches the result of \citet{song2021subquadratic}, up to $d$ contributions, which are also affected by
the data scaling---in \citet{song2021subquadratic} the data is scaled such that $\forall i\in\left[n\right],\,\left\Vert \mathbf{x}_{i}\right\Vert =1$
(see their Assumption 2), where in our analysis, $\left\Vert \mathbf{x}_{i}\right\Vert \asymp\sqrt{d}$.
\subsubsection{NTK parameterization}
Under NTK parameterization, $s_{1}=\frac{1}{\sqrt{d}},\,s_{2}=\frac{1}{\sqrt{m}}$,
$\kappa=1$, and the demand becomes
\[
\mathcal{L}_{\infty}\lesssim\frac{m}{Kn^{2}\left(1+1\right)\left(m+\sqrt{nm}\right)}\,,
\]
which under $m\gg n$ simplifies to
\[
\mathcal{L}_{\infty}\lesssim\frac{1}{Kn^{2}}\,.
\]
As elaborated in \cref{tab:param_comparison}, this is achieved when $\beta=\Omega\left(Kn^{4}\right)$
and $m=\tilde{\Omega}\left(\beta n\right)=\tilde{\Omega}\left(Kn^{5}\right)$.

\newpage

\section{Experiments}
\label{app:97_app_experiments}

\subsection{Experimental details}
The experiments are performed according to the following implementation. We sample $n=10$ datapoints from the digits dataset of scikit-learn \citep{optical_recognition_of_handwritten_digits_80, scikit-learn}, center the inputs to have $0$ mean and scale them to have Euclidean norm $\sqrt{d}$ with $d = 64$. We train a neural network defined as in \Cref{eq:network_definition} with scales $s_1 = s_2 = 1$, activation $\phi = \tanh$ and output dimension 10 (the only deviation from the setup of \Cref{eq:network_definition} is the output dimension). The parameter initialization matches \Cref{eq:Wv_initialization} with $a_1 = d$ and $a_2 = m$. We train the neural network according to the Euler–Maruyama discretization of the SDE in \Cref{eq:regularized_langevin} with step size $\eta = 10^{-5}$. We set the hidden layer width as $m=160,000$ to ensure  $m = \max(d^2, n \beta_{\max} \log n \log \beta_{\max})$, where $\beta_{\max}$ is the largest $\beta$ value we use, so that the width assumptions of our theorem will hold up to constants. We repeat this procedure over 5 seeds for $\beta \in \{10n, 50n, 100n, \infty\}$, where $\beta=\infty$ is equivalent to noiseless dynamics. In each graph we plot the mean across seeds with a shaded region showing the standard deviation. The code is available at \url{https://github.com/og-research/cld-optimization-experiments}.

The experiments were run on a dedicated server using 5 Nvidia RTX A6000 GPUs, each with 48GB of RAM. The run utilized 30GB of each GPU's RAM and took 20 hours. The amount of required resources is due to the tracking of the neural network Jacobian. Its computation relies on the function \texttt{torch.autograd.functional.jacobian}, and storing its output takes a substantial amount of memory, as the number of parameters is $n \times m \times (d+1) \approx 10^8$.

\begin{figure*}[hb]
    \centering
    \begin{subfigure}{0.40\textwidth}
        \captionsetup{labelfont=normalfont,textfont=normalfont,justification=centering}
        \centering
        \includegraphics[width=\textwidth]{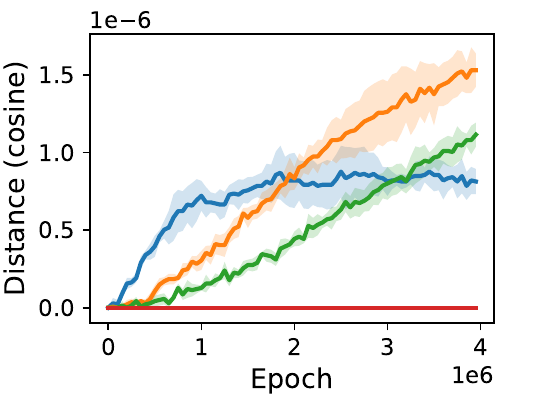}
        \caption{Deviation from linearized dynamics in cosine distance.}
        \label{fig:expr1_nn_to_lin_dist_co}
    \end{subfigure}
    \begin{subfigure}{0.40\textwidth}
        \captionsetup{labelfont=normalfont,textfont=normalfont,justification=centering}
        \centering
        \includegraphics[width=\textwidth]{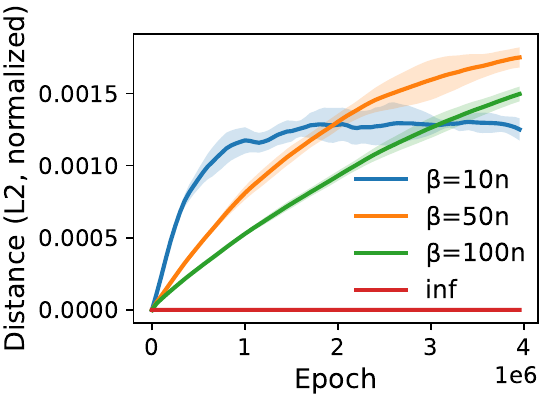}
        \caption{Deviation from linearized dynamics in $L_2$ distance.}
        \label{fig:expr1_nn_to_lin_dist_l2}
    \end{subfigure}
    \begin{subfigure}{0.40\textwidth}
        \captionsetup{labelfont=normalfont,textfont=normalfont,justification=centering}
        \centering
        \includegraphics[width=\textwidth]{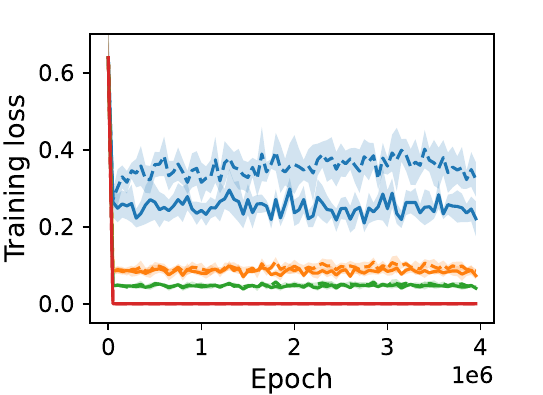}
        \caption{Training loss (NN in solid, linearized model in dashed).}
        \label{fig:expr1_train_loss_with_lin}
    \end{subfigure}
    
    \caption{The deviation from the linearized dynamics as a function of number of updates for $\beta \in \{10n, 50n, 100n, \infty\}$, measured as the parameter distance between the neural network and its linearized model given by the first-order approximation at initialization. In \cref{fig:expr1_nn_to_lin_dist_co} we see the cosine distance and in \cref{fig:expr1_nn_to_lin_dist_l2} we see the $L_2$ distance divided by the parameter norm at initialization. In \cref{fig:expr1_train_loss_with_lin} we see the training loss of the neural network (solid) and the linearized model (dashed).}
    \label{fig:figure2}
\end{figure*}

\begin{figure*}[ht]
    \centering
    \begin{subfigure}[t]{0.40\textwidth}
        \captionsetup{labelfont=normalfont,textfont=normalfont,justification=centering}
        \centering
        \includegraphics[width=\linewidth]{figures/expr1_jacobian_dist_co.pdf}
        \caption{Jacobian drift in cosine distance.}
        \label{fig:expr1_jacobian_dist_co_2}
    \end{subfigure}
    \begin{subfigure}[t]{0.40\textwidth}
        \captionsetup{labelfont=normalfont,textfont=normalfont,justification=centering}
        \centering
        \includegraphics[width=\linewidth]{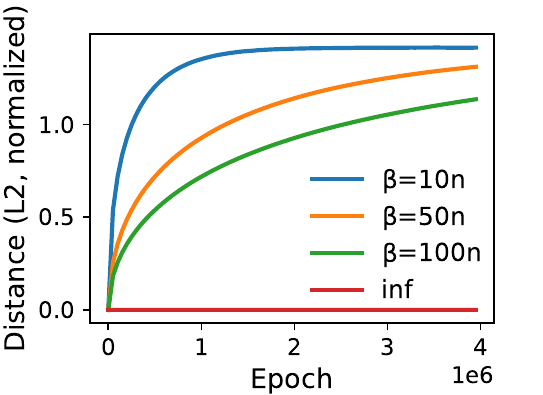}
        \caption{Jacobian drift in Frobenius norm.}
        \label{fig:expr1_jacobian_dist_l2}
    \end{subfigure}

    \caption{The distance of the Jacobian from its value at initialization as a function of number of updates for $\beta \in \{10n, 50n, 100n, \infty\}$. In \cref{fig:expr1_jacobian_dist_co_2} the distance is the cosine distance; In \cref{fig:expr1_jacobian_dist_l2} the distance is the Frobenius norm divided by the norm of the Jacobian at initialization. }
    \label{fig:figure3}
\end{figure*}

\subsection{Additional results}
\label{sec:additional_experiments}
In \Cref{fig:expr1_nn_to_lin_dist_co,fig:expr1_nn_to_lin_dist_l2} we measure the neural network's \textit{deviation from its linearized dynamics}: we construct the neural network's linearization around its initialization, defined by 
\mbox{$h(\boldsymbol{\theta}) = f(\boldsymbol{\theta}_0) + \nabla_{\boldsymbol{\theta}} f(\boldsymbol{\theta}_0) (\boldsymbol{\theta} - \boldsymbol{\theta}_0)$}.
We train both $f$ and $h$ with the same noise samples and measure the distance between their parameter vectors. 
In all cases of noisy training, the cosine distance is of order $10^{-6}$, showing that the dynamics flow in similar directions in parameter space. Interestingly, the distance is not monotone in $\beta$. In \Cref{fig:expr1_train_loss_with_lin} we observe that despite the similarity in parameter space, the loss of the neural network improves over that of the linearized model for sufficiently large noise. This suggests that the dynamics differ in the space of \textit{predictors}.

In \Cref{fig:figure3} we again consider the \textit{Jacobian drift} throughout training. \Cref{fig:expr1_jacobian_dist_co_2} is identical to \Cref{fig:expr1_jacobian_dist_co} and presented again for completion; \Cref{fig:expr1_jacobian_dist_l2} plots the distance in Frobenius norm, divided by the norm of the Jacobian at initialization.

\begin{figure*}[hb]
    \centering
    \begin{subfigure}[t]{0.40\textwidth}
        \captionsetup{labelfont=normalfont,textfont=normalfont,justification=centering}
        \centering
        \includegraphics[width=\linewidth]{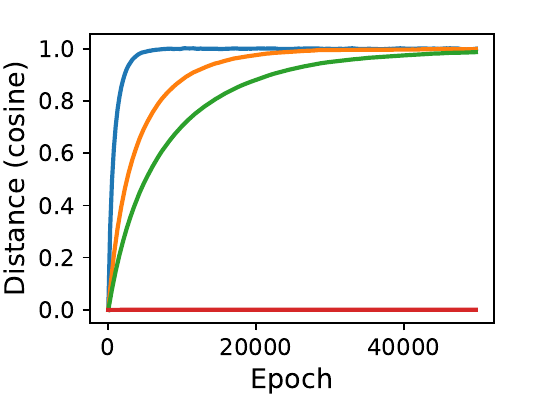}
        \caption{Jacobian drift in cosine distance.}
        \label{fig:expr1_jacobian_dist_co_deep}
    \end{subfigure}
    \begin{subfigure}[t]{0.40\textwidth}
        \captionsetup{labelfont=normalfont,textfont=normalfont,justification=centering}
        \centering
        \includegraphics[width=\linewidth]{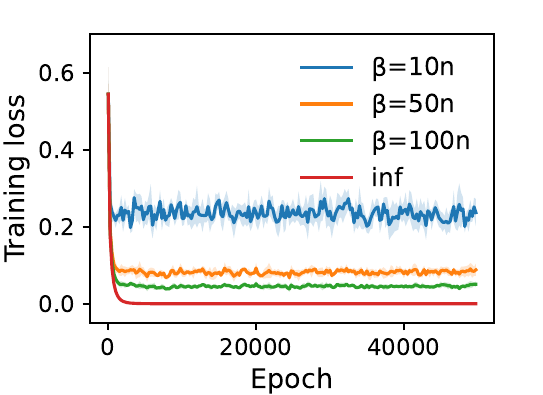}
        \caption{Training loss.}
        \label{fig:expr1_train_loss_deep}
    \end{subfigure}
    \begin{subfigure}[t]{0.50\textwidth}
        \captionsetup{labelfont=normalfont,textfont=normalfont,justification=centering}
        \centering
        \includegraphics[width=\linewidth]{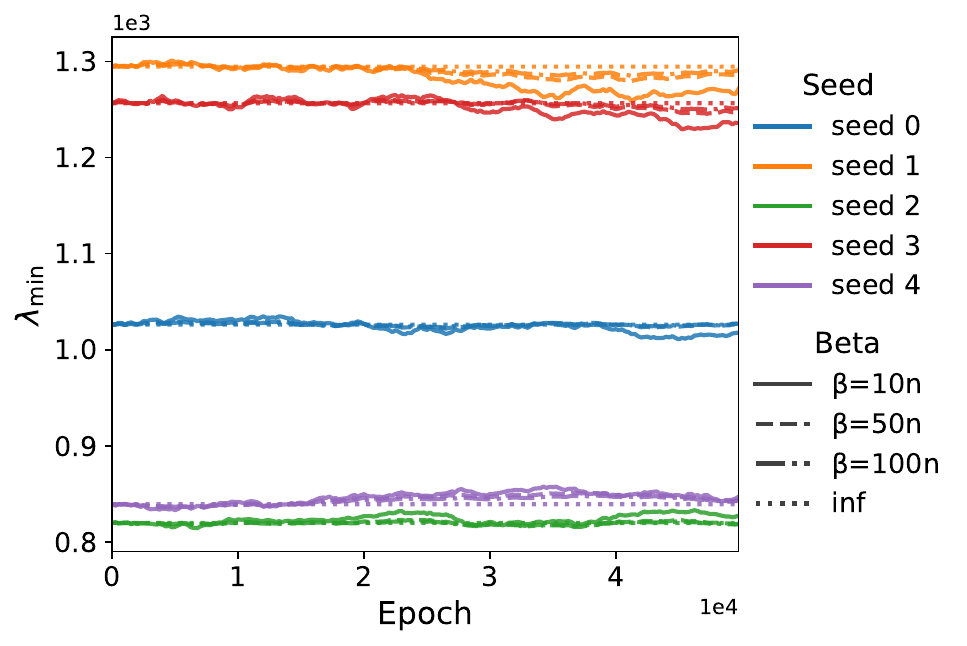}
        \caption{Minimal eigenvalue of feature Gram matrix.}
        \label{fig:feat_gram_lambda_all_betas_seeds_deep}
    \end{subfigure}

    \caption{Deep network training statistics as a function of the number of updates for $\beta \in \{10n, 50n, 100n, \infty\}$. 
    In \cref{fig:expr1_jacobian_dist_co_deep} we plot the (cosine) Jacobian drift throughout training. In \cref{fig:expr1_train_loss_deep} we plot the training loss. In \cref{fig:feat_gram_lambda_all_betas_seeds_deep} we show the minimal eigenvalue of the feature Gram matrix, plotting each ($\beta$, seed) pair separately. For $\beta=\infty$ the seed controls only the dataset.
    }
    \label{fig:deep}
\end{figure*}

In \Cref{fig:deep} we consider a \textit{deep} neural network defined as in \Cref{eq:network_definition} with $L=3$ and with scales $s_l = 1,\,l \in [L+1]$. The parameter initialization matches \Cref{eq:Wv_initialization} with $a_l = m_{l-1},\,l \in [L+1]$. Due to limited resources, we only match the hidden width with our theory up to log factors: we set the hidden layer width as $m=10,000$ to ensure $m = \max(d^2, n \beta_{\max})$, where $\beta_{\max}$ is the largest $\beta$ value we use. The rest of the implementation details are as in the previous experiments. Compared to the shallow network, here the Jacobian cosine distance increases at a significantly higher rate. This may be attributed to either the change in depth, change in width or both. The rest matches the behavior observed in the shallow network training.

\clearpage

\end{document}